\documentclass{article} %
\usepackage{kohakulab_arxiv,times}
\ClearShipoutPicture

\usepackage[utf8]{inputenc}
\usepackage[T1]{fontenc}
\usepackage{booktabs}
\usepackage{amsfonts}
\usepackage{amsmath}
\usepackage{amssymb}
\usepackage{amsthm}

\theoremstyle{plain}
\newtheorem{theorem}{Theorem}[section]
\newtheorem{proposition}{Proposition}[section]
\newtheorem{corollary}{Corollary}[section]
\newtheorem{lemma}{Lemma}[section]
\theoremstyle{definition}
\newtheorem{assumption}{Assumption}[section]

\usepackage{nicefrac}
\usepackage{microtype}
\usepackage{xcolor}
\usepackage{colortbl}
\usepackage{graphicx}
\usepackage{float}
\usepackage{placeins}
\usepackage{adjustbox}
\usepackage{enumitem}
\usepackage{wrapfig}
\usepackage{multirow}
\usepackage{titletoc}

\usepackage{etoolbox}
\newlength{\tablecaptionskip}
\AtBeginEnvironment{table}{\setlength{\abovecaptionskip}{\tablecaptionskip}\setlength{\belowcaptionskip}{5pt plus 1pt minus 1pt}}
\newcommand{\bodyfloats}{%
  \renewcommand{\topfraction}{0.5}%
  \renewcommand{\bottomfraction}{0.7}%
  \renewcommand{\floatpagefraction}{0.8}%
  \setcounter{topnumber}{3}%
  \setcounter{bottomnumber}{1}%
  \setcounter{totalnumber}{4}%
  \setlength{\tablecaptionskip}{3pt plus 1pt minus 1pt}%
  \setlength{\abovecaptionskip}{5pt plus 1pt minus 1pt}%
  \setlength{\textfloatsep}{12pt plus 2pt minus 2pt}%
  \setlength{\floatsep}{12pt plus 2pt minus 2pt}}
\makeatletter
\newcommand{\appendixfloats}{%
  \renewcommand{\topfraction}{0.95}%
  \renewcommand{\bottomfraction}{0.95}%
  \renewcommand{\floatpagefraction}{0.9}%
  \setcounter{topnumber}{4}%
  \setcounter{bottomnumber}{2}%
  \setcounter{totalnumber}{6}%
  \setlength{\tablecaptionskip}{4pt plus 1pt minus 1pt}%
  \setlength{\abovecaptionskip}{5pt plus 1pt minus 1pt}%
  \setlength{\textfloatsep}{12pt plus 2pt minus 2pt}%
  \setlength{\floatsep}{5pt plus 1pt minus 1pt}%
  \setlength{\@fptop}{0pt}\setlength{\@fpsep}{\floatsep}\setlength{\@fpbot}{0pt plus 1fil}}
\makeatother
\bodyfloats

\definecolor{priorcolor}{RGB}{215,232,250}   %
\definecolor{regcolor}{RGB}{225,243,220}     %
\definecolor{ourscolor}{RGB}{255,236,200}    %
\definecolor{ownprior}{RGB}{170,200,240}     %
\definecolor{ownours}{RGB}{255,210,140}      %

\makeatletter
\newcommand{\figcaption}{\def\@captype{figure}\caption}
\newcommand{\tabcaption}{\def\@captype{table}\caption}
\makeatother

\definecolor{kbgreen}{RGB}{0,128,64}
\providecommand{\cyl}[1]{#1}
\providecommand{\kb}[1]{#1}

\newcommand{\method}{Trinity}
\newcommand{\Vrel}{V_{\mathrm{rel}}}
\newcommand{\hpwl}{\mathrm{HPWL}}
\newcommand{\nfe}{\mathrm{NFE}}

\usepackage{hyperref}
\makeatletter
\edef\kl@ptcsum{\pdfmdfivesum file {\jobname.ptc}}
\AtBeginDocument{\immediate\write\@auxout{\gdef\string\kl@ptcprev{\kl@ptcsum}}}
\AtEndDocument{\ifdefined\kl@ptcprev\ifnum\pdfstrcmp{\kl@ptcprev}{\kl@ptcsum}=0 \else
  \@latex@warning@no@line{Label(s) may have changed. Rerun to get cross-references right}\fi\fi}
\makeatother
\newcommand{\amk}[1]{\ensuremath{\kern0.3pt\vcenter{\hbox{\small$\scriptstyle#1$}}}}
\newcommand{\amkk}[2]{\ensuremath{\kern0.3pt\vcenter{\small\offinterlineskip\halign{$\scriptstyle##$\hfil\cr#1\cr\noalign{\kern0.8pt}#2\cr}}}}
\usepackage{cleveref}
\crefname{figure}{Fig.}{Figs.}
\Crefname{figure}{Fig.}{Figs.}
\crefname{table}{Table}{Tables}
\Crefname{table}{Table}{Tables}
\crefname{section}{Section}{Sections}
\Crefname{section}{Section}{Sections}
\crefname{subsection}{Section}{Sections}
\Crefname{subsection}{Section}{Sections}
\crefname{appendix}{Appendix}{Appendices}
\Crefname{appendix}{Appendix}{Appendices}
\usepackage{url}

\title{\hyphenpenalty=10000 \exhyphenpenalty=10000 \method: One Differentiable Physics for Training, Refining and Scoring Generative Floorplanners \vspace{-0.5em}}

\author{\begin{minipage}[t]{\dimexpr\textwidth-2\tabcolsep\relax}\vspace{0pt}\centering\normalfont\setlength{\baselineskip}{1.25\baselineskip}
\mbox{\textbf{Shih-Ying Yeh}\amkk{\heartsuit\spadesuit}{\ast\dagger}}\hskip 0.55em plus 0.3em\relax 
\mbox{\textbf{Tzu-Sian Wang}\amkk{\heartsuit\spadesuit}{\ast}}\hskip 0.55em plus 0.3em\relax 
\mbox{\textbf{Xuehai Wang}\amk{\clubsuit\triangle}}\hskip 0.55em plus 0.3em\relax 
\mbox{\textbf{Jia-Hua Lee}\amk{\bigstar}}\hskip 0.55em plus 0.3em\relax 
\mbox{\textbf{Daniel Z.~Kaplan}\amk{\diamondsuit}}\hskip 0.55em plus 0.3em\relax 
\mbox{\textbf{Ming-Qi Xu}\amk{\heartsuit}}\hskip 0.55em plus 0.3em\relax 
\mbox{\textbf{Wuqian Tang}\amk{\heartsuit}}\hskip 0.55em plus 0.3em\relax 
\mbox{\textbf{Chun-Yao Wang}\amk{\heartsuit}}\hskip 0.55em plus 0.3em\relax 
\mbox{\textbf{Shang-Hong Lai}\amk{\heartsuit}}\hskip 0.55em plus 0.3em\relax 
\mbox{\textbf{Chun-Yi Lee}\amk{\bigstar}}\\[\dimexpr0.6em+1.5pt\relax]
{\small \mbox{\amk{\heartsuit}National Tsing Hua University}\hskip 0.55em plus 0.3em\relax  \mbox{\amk{\spadesuit}Kohaku Lab}\hskip 0.55em plus 0.3em\relax  \mbox{\amk{\clubsuit}Karolinska Institutet}\hskip 0.55em plus 0.3em\relax  \mbox{\amk{\triangle}Stockholm University}\hskip 0.55em plus 0.3em\relax  \mbox{\amk{\diamondsuit}realiz.ai}\hskip 0.55em plus 0.3em\relax  \mbox{\amk{\bigstar}National Taiwan University}\\[0.4em]
\mbox{\amk{\ast}Equal contribution}\hskip 0.55em plus 0.3em\relax  \mbox{\amk{\dagger}Corresponding author: \texttt{kohaku@kblueleaf.net}}\\[-1.5pt]\mbox{Project page: \url{https://kohaku-lab.github.io/Trinity/}}}
\end{minipage}}

\iclrfinalcopy %
\begin{document}

\maketitle

\vspace{-2em}
\begin{abstract}
\vspace{-1em}
Floorplanning arranges the blocks of a chip and decides their shapes under objectives that press blocks together, short wirelength and a small outline, and constraints that hold them apart, non-overlap, clusters, MIB shapes and boundary blocks. Recent diffusion placers train on reference layouts alone and leave this coupled system to guidance, post-hoc loops and a legalizer, reporting only the endpoint, which hides what the generator contributes. We re-implement four of them under one recipe on FloorSet, score raw, refined and legalized layouts on one scale, and propose \method, a flow-matching floorplanner whose six differentiable functions for the constraints and objectives are its training loss term, the energy of a closed-form refiner after sampling and the base of a soft cost for every stage. The network thus learns the correction prior placers apply in their samplers, and sampling needs no guidance. Stage by stage, the training term lowers a plain transformer's raw soft cost by 26\% and matters most at short budgets, the shared refiner decides more of the final cost than the generator and matches a ported placer's loop in 16 to 660 times fewer steps, \method's refined soft cost is 36\% below the best ported pipeline, the soft cost ranks settings as the contest's hard cost does, and on the FloorSet val set the pipeline reaches a mean hard cost of 1.014 in 1.63 s per case.\vspace{-0.7em}
\end{abstract}

\begin{figure}[H]
  \centering
  \includegraphics[width=\linewidth]{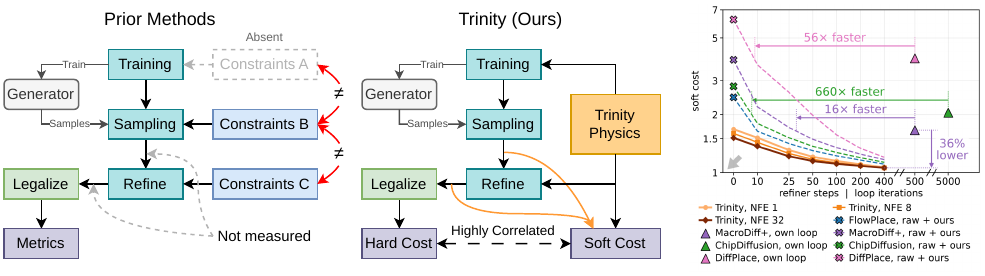}
  \vspace{-2em}
  \caption{\cyl{The network beside each pipeline is the generator that training updates and sampling evaluates. \textbf{Left:} In prior diffusion placers the constraints are absent from training (dashed), differ in form between stages ($\neq$), and the layouts in between are not measured. \textbf{Middle:} \method\ carries one set of constraint and objective functions, its differentiable physics, through training, refinement, and scoring. \textbf{Right:} Soft cost against refinement budget, each curve starting at the raw sample with zero refinement steps. \emph{Faster}: the step at which each prior generator under our refiner reaches the soft cost of its own full loop (triangle). \emph{Lower}: the refined soft cost of the best prior pipeline against ours.}}
  \label{fig:teaser}
   \vspace{-1em}
\end{figure}
\vspace{-0.95em}
\section{Introduction}
\label{sec:intro}
\vspace{-0.47em}

\cyl{Floorplanning is an early stage of the physical design flow in Electronic Design Automation (EDA). Given a netlist and each block's target area, floorplanning arranges the position and shape of every block on the chip, and the floorplan bounds what each later stage can achieve. A floorplanner pursues two objectives, \textit{short wirelength between connected blocks} and \textit{a small outline area}, and must reach them under hard constraints: a floorplan is feasible only if every block matches its target area, fixed-shape blocks retain their shapes, pre-placed blocks remain in place, and no two blocks overlap. Three soft constraints follow in practice: blocks of a cluster should stay adjacent, multi-instance blocks (MIB) should share one shape, and boundary blocks should touch the die boundary~\citep{iccad2026problemC}. The objectives and constraints pull in opposite directions, and the strongest floorplans sit where the hard constraints bind. Learning-based methods promise to amortize this per-instance optimization into a trained model. Floorplanning moreover suits a generative treatment: an instance admits many strong layouts rather than one optimum, and a generative model drafts complete layouts from this distribution~\citep{DBLP:conf/icml/LeeNEDAW25, trung2026diffplaceconditionaldiffusionframework, yoon2026physicsguidedgeometricdiffusionmacro, xie2026flowplaceflowmatchingchip}. However, such a model learns the distribution of reference layouts rather than the geometry of a single one: a network trained with a squared loss estimates a conditional mean of the reference layouts given its input, and a mean of feasible layouts is not necessarily feasible, since the feasible set is nonconvex. An infeasible draft leaves the constraints to corrections applied after sampling. Such corrections act on a landscape whose couplings belong to floorplanning itself rather than any model: as blocks compete for the same die area, a step that resolves an overlap can scatter a cluster or displace a pre-placed block, and a push toward feasibility may sacrifice the wirelength and outline area that the objectives demand. A generative floorplanner must therefore treat the constraints and objectives as one coupled system, and where this system enters the pipeline determines what the model learns.}

\cyl{Existing methods differ in where they place this coupled system, and each choice carries its own cost. Classic floorplanners place it inside the search procedure: simulated annealing~\citep{DBLP:conf/dac/WongL86} explores topological representations such as the sequence pair~\citep{DBLP:journals/tcad/MurataFNK96} and the B$^*$-tree~\citep{DBLP:conf/dac/ChangCWW00}, in which non-overlap holds by construction, and analytical placement~\citep{DBLP:journals/todaes/LuCCSHTC15, DBLP:journals/tcad/LinJGLDRKP21} minimizes a differentiable relaxation of overlap jointly with wirelength. Since every move faces the constraints and the objectives together, these solvers remain on the boundary, at the price of a search that begins anew on every netlist. Learning-based methods lift this per-instance cost, and the constraints begin to recede from training. Reinforcement learning stays nearest to the classic position, with the constraints written into its actions and the objectives into its reward~\citep{DBLP:journals/nature/MirhoseiniGYJSW21, DBLP:conf/nips/LaiM022, DBLP:conf/icml/LaiLTWH023}, although the two do not meet in one function that the policy learns, and the agent is retrained or fine-tuned for every new netlist. Regression abandons the constraints and trains on reference layouts alone~\citep{DBLP:journals/todaes/LiuJLDZWYZS23, DBLP:conf/kdd/AminiZPZHL22}, which leaves the model exposed to the conditional-mean difficulty above.}

\cyl{The recent diffusion placers~\citep{DBLP:conf/icml/LeeNEDAW25, trung2026diffplaceconditionaldiffusionframework, yoon2026physicsguidedgeometricdiffusionmacro, xie2026flowplaceflowmatchingchip} push this coupled system farthest from training: their losses only reproduce the reference layouts, and every correction arrives after the model. The first of these corrections, a hand-tuned guidance term, intervenes during sampling. The finished sample then enters post-hoc optimization loops, and the remaining violations are left to the legalizer. \Cref{fig:teaser} (left) draws this arrangement. These placers come from macro placement, a task where every shape is fixed: the shape freedom of floorplanning, and its soft constraints, stay outside their scope. The network's contribution is not measured stage by stage: their papers examine only the endpoint of the pipeline, and each stage expresses the constraints in its own form, which leaves the intermediate layouts without a common measure. At that endpoint, a weak generator behind a strong correction and a strong generator behind a weak one appear identical. The constraints thus migrate from the core of the classic search to the tail of the correction chain. Among the generative placers, the constraints have appeared everywhere except the training loss. Their place in that loss, and its effect on every later stage, nevertheless remain unexplored.}

\cyl{Given these observations, the remedy is to move the coupled system into training itself. The constraints can enter the training loss directly, since every objective and constraint of floorplanning measures a geometric quantity of the layout and admits a differentiable expression in the positions and shapes of the blocks. The overlap between two blocks, for instance, shrinks continuously as they separate and reaches zero exactly when the pair satisfies the non-overlap constraint. A weighted sum of these expressions defines an energy, one function that records the constraints and objectives together with the trade-off between them. The energy therefore assigns every layout a graded value rather than a binary indicator of feasibility. Analytical placement already minimizes an energy of this form that couples overlap with wirelength~\citep{DBLP:journals/tcad/KleinhansSJA91, DBLP:journals/todaes/LuCCSHTC15, DBLP:journals/tcad/LinJGLDRKP21}. A generative floorplanner can instead place such an energy in its training loss, where each gradient step responds to the reference layouts and the coupled system together. After training, the same expressions determine what a later refinement minimizes and provide a common measure for the layout that each stage produces (\Cref{fig:teaser}, middle). On this measure, the intermediate layouts reveal how much the generator and each later stage contribute, a distinction that the endpoint alone conceals.}

\cyl{To expose this distinction and place the coupled system in training rather than leave it to corrections after sampling, we propose \textbf{\method}, a flow-matching floorplanner whose six differentiable functions for the constraints and objectives are incorporated into its training loss, serve as the energy of its refiner, and form the basis of a soft cost that scores the layout of every stage. Trained with these functions, the network learns the correction that prior placers apply inside their samplers, and sampling runs without guidance (\Cref{sec:method}). Measured stage by stage, the pipeline reveals distinct roles for training and refinement: the training term contributes most when the refinement budget is short, and the refiner accounts for more of the final cost than the generator does. Added to the training of a plain transformer, the six functions alone lower the soft cost of its raw samples by 26\% (\Cref{sec:exp-generator}). With our refiner shared across the generators of four diffusion placers ported to our training recipe, their final costs converge despite a wide spread in raw costs, and on a ported placer's samples, the refiner reaches the soft cost of that placer's own loop in 16 to 660 times fewer iterations (\Cref{fig:teaser}, right). A comparison of generators therefore calls for a shared refiner. Against the ported pipelines with their own post-hoc loops, the full pipeline of \method\ reaches a refined soft cost 36\% below the best of them (\Cref{sec:exp-refine}). After legalization, \method\ remains ahead of these pipelines under the hard cost of the official contest~\citep{iccad2026problemC}, and on its validation set it reaches a mean hard cost of 1.014 at 1.63 seconds per case on one GPU with parallel legalization (\Cref{sec:exp-hard}).}

\vspace{-0.47em}
\textbf{Contributions.}
\cyl{\textbf{(1)} We locate the limit of generative placers in the position of their physics rather than the strength of their networks: the coupled objectives and constraints of floorplanning enter every stage except the training loss. \method\ instead trains on the coupled system directly, with one energy built from six differentiable functions of the layout, and feasibility becomes a property the network learns rather than one that corrections restore. \textbf{(2)} We establish the theory behind this choice: a network trained on reference layouts alone estimates a conditional mean that the nonconvex feasible set does not necessarily contain, whereas the energy term in training approximates constraint guidance inside the sampling loop at no sampling cost. The same analysis selects the clean prediction as the head, the one output an energy can score, and experiment confirms each prediction. \textbf{(3)} We replace the hand-tuned correction loops of prior placers with a refiner that descends the training energy itself, every gradient in closed form. One refiner serves every generator: it carries \method\ and each ported placer below the final cost their own loops reach, and supplies the common ground for comparing generators. \textbf{(4)} We extend the scale of the contest beyond feasibility: our soft cost keeps the objective factor of the hard cost and admits overlap and soft-constraint severity as smooth penalties, and every stage's layout shares one scale. Since the soft cost agrees with the hard cost on which settings perform better, the sampling and refinement budget can be chosen before the legalizer runs. \textbf{(5)} We lift the fixed-shape assumption that diffusion placers inherited from macro placement: \method\ treats positions and shapes alike and covers the full constraint family of floorplanning, from soft and pre-placed blocks to clusters, boundary blocks, and shared multi-instance shapes, and on the contest's validation set the full pipeline surpasses the classical annealer at a fraction of its runtime.}

\vspace{-0.95em}
\section{Related Work}\label{sec:related-work}
\vspace{-0.47em}
\cyl{Classical optimization and direct learned placement construct layouts through per-instance search, sequential decisions, or joint coordinate prediction: simulated annealing explores sequence pairs and B$^*$-trees~\citep{DBLP:conf/dac/WongL86, DBLP:journals/tcad/MurataFNK96, DBLP:conf/dac/ChangCWW00, DBLP:conf/dac/LinC01, DBLP:journals/tvlsi/AdyaM03}, analytical placement couples a continuous relaxation of overlap with legalization~\citep{DBLP:journals/tcad/KleinhansSJA91, DBLP:journals/tcad/ChenJHCC08, DBLP:journals/todaes/LuCCSHTC15, DBLP:journals/tcad/LinJGLDRKP21}, sequential construction commits one macro per decision under reinforcement learning~\citep{DBLP:journals/nature/MirhoseiniGYJSW21, DBLP:conf/nips/ChengY21, DBLP:conf/nips/LaiM022, DBLP:conf/icml/LaiLTWH023, DBLP:conf/icml/Geng0LXTYH0024, DBLP:conf/iclr/GengWLXTKYH025, ICLR2026_cbc1ad20, DBLP:conf/nips/Shi0L023}, and joint prediction regresses every block position from the netlist graph~\citep{DBLP:journals/todaes/LiuJLDZWYZS23, DBLP:conf/kdd/AminiZPZHL22}. Across the learned lines, block shapes stay fixed at placement time, whereas FloorSet requires the placer to settle a shape for every soft block~\citep{DBLP:conf/iccad/MallappaM0PM24}. Generative macro placement keeps the fixed-shape setting and learns layout distributions through noise or velocity prediction, with geometric corrections during or after sampling: velocity-based FlowPlace~\citep{xie2026flowplaceflowmatchingchip} applies overlap-free grid projection at each step, while the noise-based MacroDiff+~\citep{yoon2026physicsguidedgeometricdiffusionmacro}, ChipDiffusion~\citep{DBLP:conf/icml/LeeNEDAW25}, and DiffPlace~\citep{trung2026diffplaceconditionaldiffusionframework} steer sampling with overlap and wirelength gradient guidance before post-hoc loops and legalization. On the objective side, linear transformations between data and noise connect flow matching~\citep{DBLP:conf/iclr/LipmanCBNL23, DBLP:conf/iclr/LiuG023}, interpolants~\citep{DBLP:conf/eccv/MaGABVX24}, and noise prediction~\citep{DBLP:conf/nips/HoJA20, DBLP:conf/icml/NicholD21, DBLP:conf/iclr/SalimansH22}, LACE~\citep{DBLP:conf/iclr/ChenZZC24} and physics-informed diffusion~\citep{DBLP:conf/iclr/BastekSK25} supervise the clean prediction with constraint losses, and Diffusion Transformers~\citep{DBLP:conf/iccv/PeeblesX23} supply the backbone, conditioned on the netlist through Graphormer-style attention bias~\citep{DBLP:conf/nips/YingCLZKHSL21}, GraphGPS-style message passing~\citep{DBLP:conf/nips/RampasekGDLWB22}, and Laplacian positional encoding~\citep{DBLP:journals/corr/abs-2012-09699, DBLP:conf/nips/KreuzerBHLT21, DBLP:conf/iclr/DwivediL0BB22}. Following JiT~\citep{Li_2026_CVPR}, \method\ predicts the clean layout and trains on the velocity it implies. For variable-shape floorplanning, it combines netlist-conditioned clean prediction in a DiT backbone with constraint supervision through one set of geometric functions for training, refinement, and scoring, and it scores raw, refined, and legalized layouts separately. }\kb{\Cref{sec:appendix-related} }\cyl{surveys these lines in full.}

\vspace{-0.95em}
\section{Methodology}
\label{sec:method}
\vspace{-0.95em}
\subsection{Problem Formulation}
\label{sec:method-problem}
\vspace{-0.47em}

\cyl{The floorplanning problem takes as input $n$ blocks with target areas $a_i$, a weighted netlist $A_{ij}$ and the constraints of Section~\ref{sec:intro}. A layout assigns each block a box $(x_i, y_i, w_i, h_i)$ and is feasible when no two boxes overlap, every area is within 1\% of its target and fixed-shape and pre-placed blocks remain at their targets, while the three soft constraints are counted and not enforced. Every layout is scored, at every stage, by the relative gaps $\Delta_{\hpwl}$ and $\Delta_{\mathrm{area}}$ of its wirelength and outline area to the reference layout, by its overlap ratio and by one metric per constraint (Appendix~\ref{sec:appendix-scoring}), and two multiplicative aggregates of the same quantities rank every layout on one scale, as follows:}
\begin{equation}
\mathrm{Cost} = \bigl(1 + \tfrac{1}{2}(\Delta^{+}_{\hpwl} + \Delta^{+}_{\mathrm{area}})\bigr) e^{2\Vrel},
\quad
\mathrm{SoftCost} = \bigl(1 + \tfrac{1}{2}(\Delta^{+}_{\hpwl} + \Delta^{+}_{\mathrm{area}})\bigr) e^{2\,\mathrm{sev}_{\mathrm{rel}} + 10\,\mathrm{overlap}},
\label{eq:cost}
\end{equation}

\cyl{where $\Delta^{+}$ is the positive part of a gap. The \emph{hard cost} is defined as the cost of the ICCAD 2026 CAD Contest (Problem C)~\citep{iccad2026problemC}, unchanged, on feasible layouts only, with $\Vrel$ the fraction of soft constraints violated. The \emph{soft cost} keeps the objective factor, replaces the violated fraction $\Vrel$ by the mean severity $\mathrm{sev}_{\mathrm{rel}}$ of the soft constraints and penalizes overlap smoothly where the hard cost admits none, and is therefore defined on any layout, raw samples included.}

\vspace{-0.95em}
\subsection{Theoretical Analysis}
\label{sec:method-theory}
\vspace{-0.47em}

\cyl{The analysis grounds two design decisions before the framework (Section~\ref{sec:method-arch}): the latent that absorbs the area rule, and the training term that carries the energy prior placers defer to their samplers. The latent $z_i=(c_{x,i}/s,\, c_{y,i}/s,\, \rho_i)$, with the scale $s=\sqrt{\sum_i a_i}$, carries the normalized center and the log aspect ratio of block $i$, and the box decodes at the target area in closed form, as follows:\vspace{-0.2em}}
\begin{equation}
w_i = \sqrt{a_i}\,e^{\rho_i/2}, \qquad h_i = \sqrt{a_i}\,e^{-\rho_i/2}.
\label{eq:decode}
\vspace{-0.47em}
\end{equation}

\cyl{The decode is a bijection onto boxes of area $a_i$ and linear interpolation of $\rho$ keeps the area, and the area rule is the coordinate system rather than a constraint.} A head trained with a squared loss emits a conditional mean, and in the latent the mean decodes at area exactly $a_i$ while in width-height coordinates it decodes above $a_i$ unless the shape is deterministic (Proposition~\ref{prop:averaging}). 
\cyl{The feasible set $\mathcal{F}$ of the other constraints is not convex, so the posterior mean $\bar x_0=\mathbb{E}[x_0\mid x_t]$ is in general infeasible even for an exact denoiser, its squared distance from $\mathcal{F}$ at most $\operatorname{tr}\Sigma_{0|t}$ for feasible references, with $\Sigma_{0|t}=\operatorname{Cov}(x_0\mid x_t)$ (Proposition~\ref{prop:nonconvex}).} 
\cyl{A correction is therefore needed, and only its place is open.}

\cyl{The generator trains with rectified flow~\citep{DBLP:conf/iclr/LiuG023}, $x_t=(1-t)\,x_0+t\,x_1$, $x_1\sim\mathcal{N}(0,I)$, and predicts the clean layout $\hat x_0$. $U=U_{\mathrm{C}}+U_{\mathrm{O}}\ge 0$ denotes a smooth energy of the decoded layout, $U_{\mathrm{C}}$ its constraint terms, zero on feasible layouts, and $U_{\mathrm{O}}$ the objectives. Prior placers take gradient steps on $U$ in the sampler, whereas \method\ adds $\lambda\,U(\hat x_0)$ to the training loss (Eq.~\eqref{eq:aux}, Appendix~\ref{sec:appendix-theory}).}

\begin{proposition}\label{prop:body-prox}
If the network can represent any function of $(x,t)$, the minimizer of Eq.~\eqref{eq:aux} is $\hat x_0^{\star}=\operatorname{prox}_{\gamma_t U}(\bar x_0)$ with $\gamma_t=\lambda t^{2}/2$.
\end{proposition}

\begin{samepage}
\begin{proposition}\label{prop:body-loop}
Let a sampler apply $k$ gradient steps of size $\eta$ on $U$ to $\bar x_0$ at every step. 
\cyl{(i) The exact guidance admits the closed form $\nabla\log\mathbb{E}[e^{-U(x_0)}\mid x_t]=\frac{1-t}{t^{2}}\bigl(\tilde x_0-\bar x_0\bigr)$, with $\tilde x_0$ the posterior mean reweighted by $e^{-U}$; the loop's step $-\nabla U(\bar x_0)$ replaces this difference of means with a clean-space gradient, a substitution whose error persists even as $\Sigma_{0|t}$ vanishes. (ii) At a reference layout $x^{*}$, $\nabla U_{\mathrm{C}}(x^{*})=0$, so the $k$ gradient steps displace $x^{*}$ by $k\eta\lVert\nabla U_{\mathrm{O}}(x^{*})\rVert$ to first order in the total step, with $U_{\mathrm{O}}$ the wirelength and area part of $U$, independent of $t$ and nonzero whenever the objectives do not vanish at $x^{*}$, while the proximal displacement $\gamma_t\lVert\nabla U(\hat x_0^{\star})\rVert$ vanishes quadratically with $t$.}
\end{proposition}
\end{samepage}

\cyl{Training with $U$ thus turns the network into a proximal step on $U$ at no cost in sampling, with a step size that decays as $t^{2}$ and vanishes at $t=0$. $U_{\mathrm{C}}$ is zero on the data (up to the references' own residual violations, Figure~\ref{fig:theory}c): the constraint term moves no feasible reference, and $U_{\mathrm{O}}$ shifts each prediction toward the objectives by $\gamma_t\lVert\nabla U_{\mathrm{O}}\rVert$ to first order.} 
\cyl{A loop inside the sampler instead substitutes clean-space gradients for the exact guidance at every $t>0$ and moves even a reference layout}, and Section~\ref{sec:exp-refine} measures the step and the displacement. 
\cyl{The training step and the refiner step are one operation on $U$ at different noise levels $t$, and $U$ also measures how far a layout is from feasibility. One set of six functions therefore serves training, refinement and evaluation.}

\vspace{0em}
\subsection{Framework Overview}
\label{sec:method-physics}
\label{sec:method-arch}
\label{sec:method-loss}
\vspace{-0.47em}

\cyl{The framework shares a single set of physics functions across all three components: the generator trains on them, the refiner descends them, and the soft cost derived from them scores every stage on one scale (Figure~\ref{fig:arch}).} 
\cyl{\method\ expresses each constraint and objective once as a function $\mathcal{E}_k$ of the decoded boxes, positions and shapes alike and dimensionless in the layout scale $s$: a constraint function is zero when its rule holds, the intersection area over every pair of blocks, the gap to contact along a spanning tree of each cluster, the deviation of each MIB member from the median aspect of its group, and the distance of every boundary-coded side to its edge, while the wirelength and the outline area carry the objectives (Appendix~\ref{sec:appendix-refiner-derivatives}).} The weighted sum $U=\sum_k\omega_k\,\mathcal{E}_k$ is the physics term of the generator's loss, the energy the refiner descends and the base of the soft cost of Section~\ref{sec:method-problem} (Figure~\ref{fig:teaser}, middle). 
\cyl{It depends only on the constraint annotations and the netlist and never on the reference layout, and thus supervises the network with the rules rather than the reference answer.}

\begin{figure}[t]
  \centering
  \vspace{-1em}
  \includegraphics[width=0.96\linewidth]{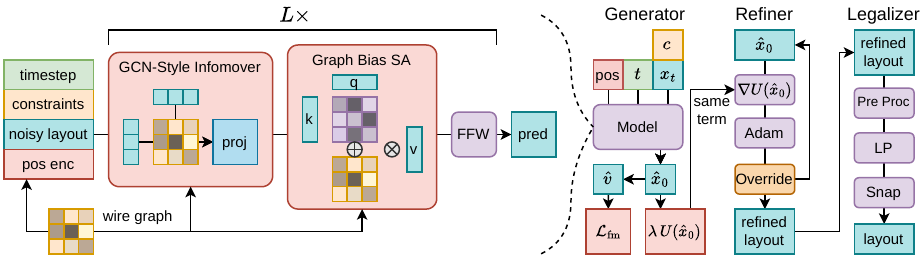}
  \vspace{-0.3em}\caption{\cyl{\textbf{Architecture} (left): $L$ blocks of the information mover on the wire graph, graph-biased attention and a feed-forward layer. \textbf{Pipeline} (right): each draw passes the generator, the refiner and the legalizer, top to bottom in each column. \textbf{Key components:} one physics term $U$ enters the loss on $\hat x_0$ and serves as the refiner's energy, and the legalizer is one linear program with a contact snap. $c$ constraints, $t$ time step, pos spectral coordinates, $x_t$ noisy layout, $\hat v$ its implied velocity, $\mathcal{L}_{\mathrm{fm}}$ the flow-matching loss, FFW the feed-forward layer, proj the projection, Pre Proc the relation extraction.}}
  \label{fig:arch}
  \vspace{-1em}
\end{figure}

\textbf{Generator.} \cyl{The generator drafts complete layouts from noise, conditioned on the netlist, and carries the physics term in its training loss.}
\cyl{A DiT \citep{DBLP:conf/iccv/PeeblesX23} with one token per block and the time step as a token \citep{Bao_2023_CVPR}.} The netlist enters as an additive attention bias \citep{DBLP:conf/nips/YingCLZKHSL21}, through the information mover, a gated one-hop graph convolution on the residual stream \citep{DBLP:conf/iclr/KipfW17, DBLP:conf/nips/RampasekGDLWB22}, and as two Laplacian eigenvectors in the input coordinates \citep{DBLP:journals/corr/abs-2012-09699, DBLP:conf/nips/KreuzerBHLT21} (Appendix~\ref{sec:appendix-architecture}). Following JiT \citep{Li_2026_CVPR}, the network predicts the clean layout $\hat x_0$ and trains on its implied velocity, 
\cyl{with $t$ logit-normal \citep{DBLP:conf/icml/EsserKBEMSLLSBP24}, $m$ the padding mask and the weights $\lambda$, $\omega_k$ of Section~\ref{sec:exp-setup}, as follows:\vspace{-0.2em}}
\begin{equation}
\mathcal{L} = \mathbb{E}_{t,\,x_0,\,x_1}\bigl[\mathcal{L}_{\mathrm{fm}} + \lambda\,U(\hat x_0)\bigr],
\qquad
\mathcal{L}_{\mathrm{fm}} = \bigl\|\bigl(\tfrac{x_t - \hat x_0}{t} - (x_1 - x_0)\bigr)\odot m\bigr\|_2^2.
\label{eq:aux}
\end{equation}
The clean head keeps the physics term inexpensive, since the loss needs a decoded layout and $\hat x_0$ is one, whereas a velocity head reaches the same minimizer through a gradient scaled by $-t$ and a target that still contains the noise (Proposition~\ref{prop:head}), which Section~\ref{sec:exp-generator} measures. Sampling follows plain Euler in uniform steps from $t=1$ to $0$ over $\nfe$ evaluations with no guidance term, and after every step the known answers of the instance enter the sample, a fixed-shape block keeping its $\rho_i$, a pre-placed block all three coordinates and the blocks of a MIB group the mean of their predicted $\rho_i$.

\textbf{Refiner.} \cyl{The refiner descends the training energy explicitly once sampling ends, at negligible cost.}
\cyl{From each sample at $t=0$, the refiner takes $K$ steps of Adam~\citep{DBLP:conf/iclr/KingmaB15} on $U$ under a heavier overlap weight than training, every gradient evaluated in closed form and the same sampling overrides written back after every step. Refinement thus resumes the descent of $U$ where the proximal step of Proposition~\ref{prop:body-prox} has vanished at $t=0$ (schedule, sweep and running cost in Appendix~\ref{sec:appendix-refiner}).}

\textbf{Legalizer.} \cyl{The legalizer converts each refined layout into a hard-feasible one while penalizing the displacement it introduces.}
\cyl{A sequence-pair relation, under which non-overlap holds by construction, is extracted for every pair of blocks from the refined layout, and one linear program positions the blocks at their decoded shapes under those relations, with the pre-placed blocks pinned, cluster members abutting along a spanning tree of the cluster and coded sides on the outline, and minimizes the chip outline plus the displacement from the refined positions~\citep{DBLP:journals/tcad/MurataFNK96, DBLP:conf/ispd/MurataK98}. An infeasible program first relaxes the soft rules, flips the conflicting relations and re-solves the program within a bounded number of rounds, and a final contact snap that moves a block only where no overlap appears completes each draw at a hard-feasible final layout (Appendix~\ref{sec:appendix-legalizer}).}

\section{Experimental Results}
\label{sec:experiments}
\vspace{-0.47em}

\cyl{The experiments evaluate every stage, since endpoint quality alone cannot identify the generator's contribution. Raw, refined, and legalized layouts therefore share the one scale of Section~\ref{sec:method-problem}, which tests whether pre-legalization quality persists in feasible layouts and separates generator quality from later optimization. Appendices~\ref{sec:appendix-generator} to~\ref{sec:appendix-seeds} hold the full tables, Appendix~\ref{sec:appendix-figures} every data view.}

\vspace{-0.95em}
\subsection{Experimental Setup}
\label{sec:exp-setup}
\vspace{-0.47em}
\textbf{Benchmarks/Datasets.}
\cyl{We train on MIB-consistent FloorSet-Lite~\citep{DBLP:conf/iccad/MallappaM0PM24}, whose 996,000 synthetic floorplans span 21 to 120 blocks per instance. Unless stated otherwise, every reported number on our held-out validation set of 12,000 cases averages four draws per case and then across all cases. We evaluate separately on the official FloorSet val set of 100 cases for the official comparison in Table~\ref{tab:official} and run one full GSRC evaluation as an external check in Table~\ref{tab:gsrc}.}

\vspace{-0.47em}
\textbf{Baselines.}
\cyl{We re-implement ChipDiffusion~\citep{DBLP:conf/icml/LeeNEDAW25}, DiffPlace~\citep{trung2026diffplaceconditionaldiffusionframework}, MacroDiff+~\citep{yoon2026physicsguidedgeometricdiffusionmacro}, and FlowPlace~\citep{xie2026flowplaceflowmatchingchip} as comparisons within the learned generative line. Under one training recipe, the ported placers retain their networks, objectives, and post-hoc loops while adopting our latent representation, conditioning, Euler sampler, and overrides, as specified in }\kb{\Cref{sec:appendix-ports}}\cyl{. The generator comparison excludes sampler-time guidance}\kb{, which \Cref{sec:appendix-guided} evaluates on its own, }\cyl{and grid projection, whereas Section~\ref{sec:exp-refine} includes the post-hoc loops under their released settings. To examine conditional-mean prediction, Reg-$z$ and Reg-$xywh$ predict latents and raw boxes on our backbone, with details in Appendix~\ref{sec:appendix-hard-lines}. The classical comparison uses the FloorSet team's PARSAC annealer \citep{mostafa2024parsacfasthumanqualityfloorplanning} only on the FloorSet val set, since its published budget of $10^7$ steps per worker is prohibitive at the 12,000-case scale, with details in Appendix~\ref{sec:appendix-classical}.}

\vspace{-0.47em}
\textbf{Metrics.}
\cyl{We evaluate every layout by the quantities in Section~\ref{sec:method-problem} for cross-stage comparison, with soft cost derived from the shared differentiable physics and hard cost restricted to feasible layouts after legalization. We also report FLD and MMD for the generator comparison as reference-distribution distances, with scoring details in Appendix~\ref{sec:appendix-scoring}. We report first-draw hard costs apart from best-of-$N$ hard costs selected by hard cost among the $N$ draws, with full tables in Appendices~\ref{sec:appendix-theory-measured} to~\ref{sec:appendix-size}.}

\vspace{-0.47em}
\textbf{Implementation Details.}
\cyl{Our three recurring comparisons distinguish the plain control, a plain transformer with a clean-sample head and no physics, from full \method\ with two graph components, six functions, and $\lambda{=}0.01$, and from \method\ at $\lambda{=}0$, the \method\ configuration without physics. All 66 models share a training recipe of 200,000 steps with batch size 256 and AdamW at an initial learning rate of $2\times10^{-4}$ with cosine decay. This shared recipe uses exponential moving average (EMA) weights and augmentation by rotation, flip, shift, constraint dropout, and wire dropout at ten GPU-hours per model on one RTX 5090, the machine behind every reported timing. The legalizer runs on 16 CPU workers in parallel, and every training $\omega_k$ equals one, with further training details in Appendix~\ref{sec:appendix-train}. We sample with $\nfe\in\{1,2,4,8,16,32\}$ and cap refinement at 400 Adam steps under a tenfold overlap weight, which keeps the sampling and refinement budgets distinct.}

\begin{table}[t]\centering\vspace{-1em}
\caption{\cyl{Raw samples: overlap ratio, wirelength (WL) and area gaps in percent, and soft cost. (a) The three heads on the plain transformer at $\nfe{=}32$ and $\nfe{=}1$. (b) The graph mechanisms added one at a time, without ($\lambda{=}0$) and with ($\lambda{=}0.01$) the six functions, and the last row is \method.}}
\label{tab:generator}
\scriptsize\setlength{\tabcolsep}{3pt}
\makebox[\linewidth]{\hfill%
\rotatebox[origin=c]{90}{(a) Objective}\hspace{2pt}%
\begin{tabular}{lrrrrr}
\toprule
Head & $\nfe$ & Overlap & WL & Area & Soft cost \\
\midrule
\rowcolor{ourscolor} $x_0$ (ours) & 32 & 0.026 & $-0.2$ & $+1.1$ & \textbf{2.11} \\
$v$ & 32 & 0.037 & $-0.5$ & $+1.2$ & 2.60 \\
$\epsilon$ & 32 & 0.036 & $-0.1$ & $+1.5$ & 2.39 \\
\midrule
\rowcolor{ourscolor} $x_0$ (ours) & 1 & 0.038 & $-1.5$ & $+0.9$ & \textbf{2.49} \\
$v$ & 1 & 0.250 & $+18$ & $+61$ & 77.2 \\
$\epsilon$ & 1 & 0.042 & $+591$ & $+2{,}030$ & 152 \\
\bottomrule
\end{tabular}\hfill%
\rotatebox[origin=c]{90}{(b) Graph mechanisms}\hspace{2pt}%
\begin{tabular}{lrrrrr}
\toprule
Backbone & $\lambda$ & Overlap & WL & Area & Soft cost \\
\midrule
Plain transformer & 0 & 0.026 & $-0.21$ & $+1.04$ & 2.09 \\
+ spectral encoding & 0 & 0.024 & $+0.12$ & $+1.91$ & 2.13 \\
+ information mover & 0 & 0.022 & $+0.22$ & $+1.81$ & \textbf{2.02} \\
\midrule
Plain transformer & 0.01 & 0.023 & $-0.15$ & $+0.20$ & 1.56 \\
+ spectral encoding & 0.01 & 0.023 & $-0.06$ & $+0.31$ & 1.59 \\
\rowcolor{ourscolor} + information mover & 0.01 & 0.023 & $-0.07$ & $+0.22$ & \textbf{1.51} \\
\bottomrule
\end{tabular}\hfill}
\vspace{-1.5em}
\end{table}

\vspace{-0.95em}
\subsection{Raw Generator Quality before Refinement}
\label{sec:exp-generator}
\label{sec:exp-prior}
\vspace{-0.47em}

\cyl{\Cref{tab:generator} compares the generator's design choices, while \Cref{tab:terms} reports the contribution of physics training to raw-sample quality.
The clean-sample head exposes a layout for direct energy evaluation, which ties its practical role to the theoretical choice of head.
The information mover complements this choice through improved constraint objectives.
Joint physics training improves every constraint metric without sacrificing the backbone's wirelength or area.
The resulting configuration combines the only head with a valid single-step layout and all six physics functions for the best trade-off.
These gains persist }\kb{down to the parameter count of the ported networks}\cyl{, where depth contributes more than width (\Cref{fig:size,tab:nfe,sec:appendix-size}).
The improvement also exceeds the variation across seeds in \Cref{sec:appendix-seeds,sec:appendix-generator}.
This stability thus supports the contribution of the shared differentiable physics at the raw stage, where its role in training is assessed ahead of any refinement.}

\begin{figure}[t]
\vspace{-1em}
\begin{minipage}[t]{0.40\linewidth}\centering
\scriptsize\setlength{\tabcolsep}{1.8pt}
\raisebox{\depth}[\totalheight][0pt]{\begin{tabular}{lrrrr}
\toprule
 & \multicolumn{4}{c}{Weight $\lambda$} \\
\cmidrule(lr){2-5}
Metric & none & 0.01 & 0.05 & 0.1 \\
\midrule
Overlap (\%) & 2.60 & 2.42 & 1.95 & \textbf{1.80} \\
Cluster gap (\smash{$10^{-3}$}) & 2.65 & 2.29 & 2.16 & \textbf{2.04} \\
MIB variance (\smash{$10^{-6}$}) & 8.75 & 3.51 & \textbf{0.69} & 1.80 \\
Boundary dist. (\smash{$10^{-3}$}) & 3.23 & 0.99 & \textbf{0.74} & 0.88 \\
WL gap (\%) & $-$0.21 & $-$0.20 & $-$0.41 & \textbf{$-$0.69} \\
Area gap (\%) & 1.04 & 0.63 & $-$0.13 & \textbf{$-$0.93} \\
\midrule
Soft cost & 2.09 & \textbf{1.56} & \textbf{1.56} & 1.63 \\
\bottomrule
\end{tabular}}
\tabcaption{\cyl{One physics function in the loss at a time, each one scored on its own metric, and all six together scored by the soft cost, on the plain control's raw samples at $\nfe{=}32$; bold marks the best weight in each metric row.}}
\label{tab:terms}
\end{minipage}\hfill
\begin{minipage}[t]{0.575\linewidth}\centering
\begin{minipage}[b]{0.535\linewidth}\centering
\scriptsize\setlength{\tabcolsep}{1.5pt}
\raisebox{\depth}[\totalheight][0pt]{\begin{tabular}{lrrrr}
\toprule
Generator & Params & WL & Area & Soft \\
\midrule
FlowPlace & 6.3M & $-0.2$ & $+0.9$ & 2.46 \\
ChipDiffusion & 6.3M & $-0.1$ & $+1.2$ & 2.80 \\
MacroDiff+ & 5.0M & $+0.2$ & $+2.4$ & 3.85 \\
DiffPlace & 23.9M & $-0.3$ & $+1.2$ & 6.23 \\
\midrule
\method\ 256/12 & 10.6M & $+0.1$ & $+0.3$ & 1.70 \\
\rowcolor{ourscolor} \method\ 768/12 & 92.9M & $-0.1$ & $+0.2$ & 1.51 \\
\midrule
\method\ 192/12 & 5.9M & $+0.1$ & $+0.3$ & 1.83 \\
\method\ 512/16 & 56.0M & $-0.2$ & $-0.2$ & 1.53 \\
\bottomrule
\end{tabular}%
}
\end{minipage}\hfill
\begin{minipage}[b]{0.452\linewidth}\centering
\includegraphics[width=\linewidth]{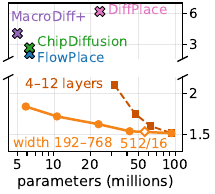}
\end{minipage}
\figcaption{\cyl{Left: parameters, wirelength and area gaps in percent and raw soft cost at $\nfe{=}32$ for the four prior placers and four \method\ rungs. Right: raw soft cost against parameters for the nine \method\ sizes and the four prior placers. The dashed rung varies depth (4 to 12 layers), the solid one width from 192 to 768.}}
\label{fig:size}
\end{minipage}
\vspace{-1em}
\end{figure}

\cyl{\Cref{tab:prior} compares raw layouts across generators, while \Cref{tab:nfe} reports whether the quality advantage persists under restricted sampling budgets.
Regressors average symmetric optima (\Cref{prop:nonconvex}), whereas the log-aspect latent preserves block areas (\Cref{prop:averaging}, \Cref{sec:appendix-hard-lines}).
Against the four ported placers, \method{} produces substantially better raw layouts and remains effective at the smallest evaluated sampling budgets, where the competing generators degrade severely or disperse their layouts.
PARSAC on the FloorSet val set illustrates the trade-off of annealing, which avoids overlap at the expense of wirelength and area.
Within the diffusion family, the models preserve the reference objectives, whereas the physics term distinguishes \method{} on overlap and soft constraints.
}\kb{Their own sampler-time guidance does not substitute for the physics: it lowers the term it targets, raises the terms it does not model, and leaves every guided sample above \method's free one (\Cref{tab:nfe,tab:refine}, \Cref{sec:appendix-guided}).
}\cyl{This separation identifies a training contribution before refinement and sets the raw starting point for assessing its short-budget benefit apart from the refiner's control of the endpoint.}

\begin{table}[tb]\centering
\caption{\cyl{Every method through its own pipeline at $\nfe{=}32$: the raw sample, the refined layout (own loop, released setting) and the legalized first draw; overlap ratio, WL and area gaps in percent, refined soft cost, $\Vrel$, hard cost as first / best of 4, and seconds per draw. FlowPlace and PARSAC run no loop. $^\ddagger$PARSAC on the FloorSet val set, one worker / best of 4, seconds per worker.}}
\label{tab:prior}
\scriptsize\setlength{\tabcolsep}{3.5pt}
\begin{tabular}{lrrrrrrrrrrrr}
\toprule
 & \multicolumn{3}{c}{Raw sample} & \multicolumn{4}{c}{Refined} & \multicolumn{4}{c}{Legalized} & \\
\cmidrule(lr){2-4}\cmidrule(lr){5-8}\cmidrule(lr){9-12}
Method & Overlap & WL & Area & Overlap & WL & Area & Soft & WL & Area & $\Vrel$ & Hard cost & s/draw \\
\midrule
\rowcolor{ourscolor} \textbf{\method} & 0.023 & $-0.1$ & $+0.2$ & 0.002 & $+0.3$ & $+1.1$ & \textbf{1.06} & $+9.3$ & $+10.5$ & 0.045 & \textbf{1.234} / \textbf{1.089} & \textbf{0.28} \\
\rowcolor{priorcolor} FlowPlace & 0.044 & $-0.2$ & $+0.9$ & -- & -- & -- & -- & $+18.5$ & $+21.6$ & 0.114 & 1.567 / 1.343 & 0.36 \\
\rowcolor{priorcolor} ChipDiffusion & 0.051 & $-0.1$ & $+1.2$ & 0.007 & $+5.3$ & $+11.7$ & 2.05 & $+15.5$ & $+18.2$ & 0.114 & 1.521 / 1.300 & 0.41 \\
\rowcolor{priorcolor} MacroDiff+ & 0.064 & $+0.2$ & $+2.4$ & 0.003 & $+7.2$ & $+12.7$ & 1.66 & $+15.9$ & $+18.3$ & 0.122 & 1.546 / 1.308 & 0.39 \\
\rowcolor{priorcolor} DiffPlace & 0.113 & $-0.3$ & $+1.2$ & 0.002 & $+6.7$ & $+29.9$ & 3.91 & $+13.9$ & $+19.1$ & 0.121 & 1.537 / 1.313 & 0.36 \\
PARSAC$^\ddagger$ & 0.000 & $+36.1$ & $+6.2$ & -- & -- & -- & -- & $+47.2$ & $+22.2$ & 0.091 & 1.685 / 1.348 & 1,120 \\
\bottomrule
\end{tabular}
\vspace{-1em}
\end{table}

\begin{figure}[t]
  \centering
  \vspace{-2em}
  \includegraphics[width=\linewidth]{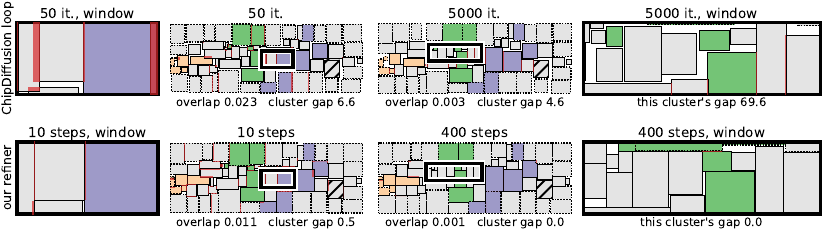}
  \caption{\cyl{One raw draw of a 67-block case under ChipDiffusion's loop (top) and our refiner (bottom), early and late, the boxed region magnified beside each layout: left where the loop's early state keeps the most overlap, right the cluster it opens most. Clusters coloured, pre-placed hatched, boundary-coded dashed, overlap red. Below each layout, the overlap ratio and the cluster gap in $10^{-3}$.}}
  \label{fig:motivation}
    \vspace{-0em}
\end{figure}
\begin{table}[t]
\vspace{-1em}
\begin{minipage}[b]{0.56\linewidth}\centering
\caption{\cyl{Raw soft cost at four sampling budgets $\nfe$, with FLD and MMD ($10^{-3}$) at $\nfe$ 32 / 1; the rungs 384 / 12 and 256 / 12 give width / depth, and the guided rows sample under each placer's own guidance.}}
\label{tab:nfe}
\scriptsize\setlength{\tabcolsep}{2.5pt}
\begin{tabular}{lrrrrrr}
\toprule
Generator & 1 & 4 & 8 & 32 & FLD & MMD \\
\midrule
\rowcolor{ourscolor} \textbf{\method} & \textbf{1.68} & \textbf{1.63} & \textbf{1.60} & \textbf{1.51} & \textbf{0.042} / \textbf{0.162} & \textbf{2.48} / \textbf{9.63} \\
\method\ at $\lambda{=}0$ & 2.43 & 2.25 & 2.16 & 2.02 & 0.112 / 0.426 & 21.7 / 35.1 \\
\method, 384 / 12 & 1.76 & 1.72 & 1.69 & 1.61 & 0.044 / 0.156 & 3.71 / 12.7 \\
\method, 256 / 12 & 1.81 & 1.78 & 1.76 & 1.70 & 0.055 / 0.157 & 5.00 / 10.7 \\
\rowcolor{priorcolor} FlowPlace & 11.5 & 5.09 & 3.82 & 2.46 & 0.227 / 3.77 & 24.4 / 350 \\
\rowcolor{priorcolor} ChipDiffusion & 152 & 11.9 & 3.83 & 2.80 & 0.273 / 102 & 31.2 / 1,223 \\
\rowcolor{priorcolor} MacroDiff+ & 150 & 71.7 & 9.89 & 3.85 & 0.624 / 102 & 81.5 / 1,216 \\
\rowcolor{priorcolor} DiffPlace & 150 & 99.5 & 14.6 & 6.23 & 2.86 / 102 & 247 / 1,230 \\
\addlinespace[2pt]
\rowcolor{priorcolor} ChipDiffusion, guided & 88.3 & 5.96 & 3.97 & 3.78 & 3.10 / 101 & 512 / 1,171 \\
\rowcolor{priorcolor} MacroDiff+, guided & 54.1 & 10.8 & 3.20 & 2.04 & 4.36 / 92.1 & 615 / 1,077 \\
\rowcolor{priorcolor} DiffPlace, guided & 158 & 228 & 418 & 132 & 19.7 / 102 & 594 / 1,231 \\
\bottomrule
\end{tabular}
\end{minipage}\hfill
\begin{minipage}[b]{0.42\linewidth}\centering
\caption{\cyl{Soft cost along our shared refiner at 0 (raw), 25, 100 and 400 steps; the guided rows start from guided samples. Bold marks the best soft cost per column.}}
\label{tab:refine}
\scriptsize\setlength{\tabcolsep}{4pt}
\begin{tabular}{lrrrr}
\toprule
Generator, sampling budget & raw & 25 & 100 & 400 \\
\midrule
\rowcolor{ourscolor} \textbf{\method}, $\nfe$ 1 & \textbf{1.68} & \textbf{1.31} & \textbf{1.15} & \textbf{1.07} \\
\rowcolor{ourscolor} \textbf{\method}, $\nfe$ 8 & \textbf{1.60} & \textbf{1.26} & \textbf{1.12} & \textbf{1.06} \\
\rowcolor{ourscolor} \textbf{\method}, $\nfe$ 32 & \textbf{1.51} & \textbf{1.22} & \textbf{1.11} & \textbf{1.06} \\
\method\ at $\lambda{=}0$, $\nfe$ 1 & 2.43 & 1.35 & 1.16 & 1.07 \\
\rowcolor{priorcolor} FlowPlace, $\nfe$ 32 & 2.46 & 1.42 & 1.22 & 1.11 \\
\rowcolor{priorcolor} ChipDiffusion, $\nfe$ 32 & 2.80 & 1.52 & 1.27 & 1.13 \\
\rowcolor{priorcolor} MacroDiff+, $\nfe$ 32 & 3.85 & 1.71 & 1.35 & 1.17 \\
\rowcolor{priorcolor} DiffPlace, $\nfe$ 32 & 6.23 & 2.61 & 1.57 & 1.20 \\
\addlinespace[2pt]
\rowcolor{priorcolor} ChipDiffusion, guided & 3.78 & 1.83 & 1.29 & 1.15 \\
\rowcolor{priorcolor} MacroDiff+, guided & 2.04 & 1.59 & 1.32 & 1.22 \\
\rowcolor{priorcolor} DiffPlace, guided & 132 & 30.1 & 5.55 & 1.87 \\
\bottomrule
\end{tabular}
\end{minipage}
\vspace{-0.6em}
\end{table}

\begin{table}[t!]
\vspace{-0.5em}
\begin{minipage}[b]{0.352\linewidth}\centering
\caption{\cyl{Generator $\times$ loop first-draw hard cost at $\nfe{=}32$, columns short for the rows; none: no loop, darker cells a placer's own loop, bold best per column.}}
\label{tab:genref}
\scriptsize\setlength{\tabcolsep}{1.6pt}
\begin{tabular}{lrrrrr}
\toprule
Gen $\backslash$ loop & none & CD & DP & MD+ & ours \\
\midrule
\rowcolor{ourscolor} \textbf{\method} & \textbf{1.513} & \textbf{1.456} & \textbf{1.442} & \textbf{1.442} & \cellcolor{ownours}\textbf{1.234} \\
\rowcolor{priorcolor} FlowPlace & \cellcolor{ownprior}1.567 & 1.495 & 1.454 & 1.471 & 1.296 \\
\rowcolor{priorcolor} ChipDiffusion & 1.609 & \cellcolor{ownprior}1.521 & 1.483 & 1.492 & 1.322 \\
\rowcolor{priorcolor} DiffPlace & 1.731 & 1.587 & \cellcolor{ownprior}1.537 & 1.539 & 1.355 \\
\rowcolor{priorcolor} MacroDiff+ & 1.672 & 1.571 & 1.536 & \cellcolor{ownprior}1.546 & 1.352 \\
\bottomrule
\end{tabular}
\end{minipage}\hfill
\begin{minipage}[b]{0.312\linewidth}\centering
\caption{\cyl{FloorSet val set: the best of $N$ draws per budget pair, with the mean hard cost, sec per case, and WL and area gaps in \%.}}
\label{tab:official}
\scriptsize\setlength{\tabcolsep}{1.5pt}
\begin{tabular}{rrrrrrr}
\toprule
$\nfe$ & Steps & $N$ & Mean & sec & WL & Area \\
\midrule
32 & 400 & 48 & 1.015 & 1.86 & $-0.3$ & $-0.1$ \\
8 & 400 & 48 & \textbf{1.014} & 1.63 & $-0.4$ & $+0.0$ \\
8 & 400 & 16 & 1.019 & 1.14 & $-0.2$ & $+0.4$ \\
4 & 100 & 16 & 1.043 & 0.63 & $+0.3$ & $+1.7$ \\
1 & 50 & 16 & 1.115 & \textbf{0.48} & $+2.7$ & $+5.5$ \\
\bottomrule
\end{tabular}
\end{minipage}\hfill
\begin{minipage}[b]{0.308\linewidth}\centering
\caption{\cyl{GSRC HPWL in mm: bold best, underline second; $^{*}$our reproduction, $^{**}$their published numbers, $^{\dagger}$outline above GSRC's 110\% area bound.}}
\label{tab:gsrc}
\scriptsize\setlength{\tabcolsep}{1.2pt}
\begin{tabular}{lrrr}
\toprule
Method & n100 & n200 & n300 \\
\midrule
\rowcolor{ourscolor} \method & 294.7 & \underline{554.9} & 681.3$^{\dagger}$ \\
\method\ at $\lambda{=}0$ & \underline{293.9} & 560.1$^{\dagger}$ & 679.8$^{\dagger}$ \\
PARSAC, Bo16$^{*}$ & 303.5 & 576.0 & 709.3 \\
PARSAC, Bo112$^{**}$ & \textbf{288} & \textbf{526} & \textbf{627} \\
GoodFloorplan$^{**}$ & 319 & 575 & \underline{705} \\
\bottomrule
\end{tabular}
\end{minipage}
\end{table}

\begin{figure}[t!]
  \centering
  \vspace{-1.5em}
  \includegraphics[width=\linewidth]{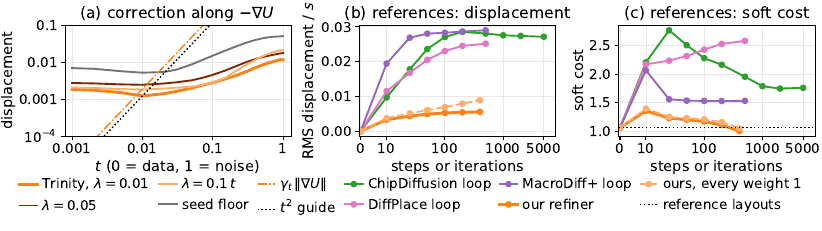}
  \vspace{-2.5em}
  \caption{\cyl{\textbf{(a)} The gap between \method\ and its $\lambda{=}0$ backbone along $-\nabla U$ at three physics weights, with the seed floor (grey), the proximal prediction $\gamma_t\|\nabla U\|$ and a $t^{2}$ guide. \textbf{(b, c)} The 12,000 reference layouts through each placer's loop, with the original references dotted in (c).}}
  \label{fig:theory}
  \vspace{-0.8em}
\end{figure}

\vspace{-0.95em}
\subsection{Refinement under the Training Energy}
\label{sec:exp-refine}
\vspace{-0.47em}

\begin{figure}[t!]
  \centering
  \vspace{-2em}
  \includegraphics[width=\linewidth]{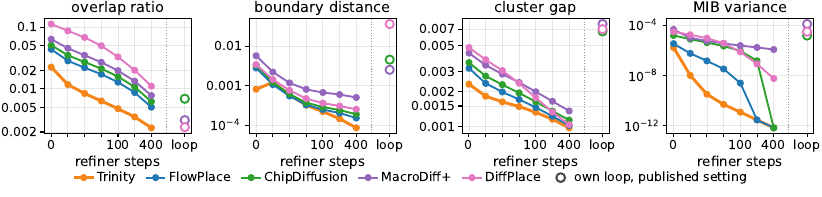}
  \vspace{-1.5em}
  \caption{\cyl{The four constraint quantities, overlap, cluster gap, MIB variance and boundary distance, along our refiner at $\nfe{=}32$, hollow markers each ported loop at its released count.}}
  \label{fig:constraints}
  \vspace{-1em}
\end{figure}

\begin{figure}[t!]
  \centering
  \includegraphics[width=\linewidth]{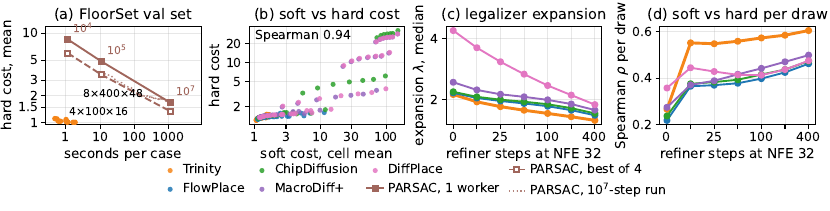}
  \vspace{-2em}
  \caption{\cyl{\textbf{(a)} FloorSet val set, every \method\ setting of $\nfe$, steps and $N$ against PARSAC at $10^{4}$--$10^{7}$ steps per worker. \textbf{(b)} One point per setting of $\nfe$ and steps. \textbf{(c, d)} Along our refiner, $\nfe{=}32$: the legalizer's median expansion factor and the per-draw soft-hard correlation.}}
  \label{fig:results}
  \vspace{-1em}
\end{figure}

\begin{table}[t!]
\begin{minipage}[b]{0.43\linewidth}\centering
\caption{\cyl{Each loop against ours: Iterations pairs the loop's released count with our steps to the same soft cost, and $\mu$s per step and the time advantage read at 32 / 1,000 layouts per batch on the RTX 5090.}}
\label{tab:matched_time}
\scriptsize\setlength{\tabcolsep}{3pt}
\begin{tabular}{lrrr}
\toprule
Loop & Iterations & $\mu$s / step & Less time \\
\midrule
\rowcolor{ourscolor} our refiner & -- & 50.8 / 2.6 & -- \\
DiffPlace & 500 / 9 & 60.7 / 3.1 & 67$\times$ / 67$\times$ \\
MacroDiff+ & 500 / 31 & 292.5 / 20.7 & 92$\times$ / 129$\times$ \\
ChipDiffusion & 5,000 / 8 & 69.6 / 5.3 & 904$\times$ / 1,359$\times$ \\
\bottomrule
\end{tabular}
\end{minipage}\hfill
\begin{minipage}[b]{0.555\linewidth}\centering
\caption{\cyl{Soft against legalized hard cost per generator line: Spearman $\rho$ and Kendall $\tau$, across the cell means of the $\nfe{\times}$steps grid and over every draw; CD, DP and MD+ as in Table~\ref{tab:genref}, and the all-eight column adds \method\ at $\lambda{=}0$ and the two regressors of Section~\ref{sec:exp-setup}.}}
\label{tab:corr_lines}
\scriptsize\setlength{\tabcolsep}{4pt}
\begin{tabular}{lrrrrrr}
\toprule
Correlation & \method & FlowPlace & CD & DP & MD+ & all eight \\
\midrule
across cells, $\rho$ & 0.938 & 0.923 & 0.989 & 0.975 & 0.985 & 0.947 \\
across cells, $\tau$ & 0.754 & 0.738 & 0.919 & 0.898 & 0.919 & 0.807 \\
over draws, $\rho$ & 0.472 & 0.369 & 0.805 & 0.858 & 0.861 & 0.768 \\
over draws, $\tau$ & 0.326 & 0.252 & 0.619 & 0.665 & 0.671 & 0.578 \\
\bottomrule
\end{tabular}

\end{minipage}
\vspace{-0.8em}
\end{table}

\cyl{\Cref{tab:refine} compares refined soft cost across sampling and refinement budgets, with the convergence traced in \Cref{fig:teaser}.
After a few hundred refinement steps, samples from every sampling budget reach nearly the same soft cost, which supports a low-budget draw followed by refinement, while the training physics keeps a substantial short-budget advantage over the $\lambda{=}0$ backbone (\Cref{sec:appendix-refine}), the training contribution that an endpoint comparison alone would conceal.}
\cyl{\Cref{fig:constraints} traces constraint improvement under the shared refiner against the ported placers' released-loop outcomes, with wirelength and area preserved in \Cref{tab:prior}.
In \Cref{fig:motivation}, ChipDiffusion's loop trades overlap for a wider cluster gap, whereas our refiner removes overlap without opening the cluster, which supports one set of six physics quantities across training, refinement and scoring.}
\cyl{On the references, \Cref{fig:theory}~(b)--(c) separates corrective refinement from motion that degrades the objectives: our refiner moves movable blocks by 0.0055 of the scale and lowers soft cost in 85\% of cases, whereas the released loops move them by 0.025 to 0.029 and raise the cost in every case from 1.065 to 1.54 to 2.58 through wirelength and area, in line with the small-step account of \Cref{prop:body-loop}(ii) (\Cref{sec:appendix-theory-measured}).
On a ported generator's samples our refiner matches the own-loop soft cost in a fraction of the steps and the time, and 400 of its steps carry each generator below its own loop's endpoint, with or without its guidance (\Cref{tab:prior,tab:refine}, \Cref{sec:appendix-guided}).
\Cref{tab:matched_time} compares the cost of matching each ported placer's own-loop soft cost with iteration counts and per-step time, not wall-clock measurements.
The training term thus governs raw layouts and short-budget quality, and the refiner the converged endpoint.}

\vspace{-0.95em}
\subsection{Legalized Layouts and the Official Hard Cost}
\label{sec:exp-hard}
\vspace{-0.47em}

\cyl{\Cref{tab:prior} and \Cref{tab:genref} compare legalized outcomes for the ported pipelines and for generators under a shared refiner, respectively. The advantage of \method{} over the ported placers persists in the official hard cost on feasible layouts, while the compression of generator differences under the shared refiner shows why endpoint comparisons understate generator quality. The narrowed differences indicate that refinement determines much of the final hard cost, (\Cref{sec:appendix-hard}). This dependence on refinement motivates comparing generators under a shared refiner within the stage-by-stage evaluation.}
\cyl{\Cref{fig:results}~(b)--(d) traces the relation between soft cost and legalized hard cost that supports this stage-by-stage evaluation. Agreement between the two costs supports the soft cost as a proxy for comparing settings, with caution for individual draws early in refinement. Across the evaluated \method{} settings, mean soft and hard costs correlate strongly across the generator lines (\Cref{tab:corr_lines}). At the individual-draw level, the initially weaker correlation improves as refinement produces samples that are easier to legalize. The soft cost therefore evaluates generators and refiners without a legalizer run for every experiment.}

\cyl{\Cref{tab:official} and \Cref{tab:gsrc} report the practical performance and transfer limits of the resulting legalized layouts. On FloorSet, the hard cost of \method{} approaches the reference cost at a small fraction of the computation PARSAC spends, consistent with \Cref{fig:results}~(a). Within the GSRC training range, the variant without the physics term matches the full model, as expected in distribution. Beyond that range, the six terms separate the variants and carry the model to larger block counts. At still larger sizes, no current checkpoint fits the outline (\Cref{sec:appendix-gsrc-checkpoints}). The remaining limit therefore lies in the training range rather than the physics itself, which supports transfer beyond that range in the evaluated cases.}

\vspace{-0.95em}
\section{Conclusion}
\label{sec:conclusion}
\vspace{-0.47em}
\cyl{Generative floorplanners placed the coupled objectives and constraints of floorplanning everywhere except where the model learns, and endpoint scores hid what each stage contributed. \method\ moved the coupled system into training: one energy of six differentiable functions served as a term of its loss, as the energy its refiner descends in closed form, and as the base of a soft cost that scored every stage, so sampling ran without guidance. Measured on that one scale, training and refinement played distinct roles. The training term improved every raw constraint at once and mattered most at short refinement budgets, while the refiner decided most of the final cost and carried every ported generator below its own loop in far fewer steps. Generators were therefore best compared under a shared refiner, and the soft cost, which ranked settings as the hard cost does, supported that comparison before legalization. With positions and shapes treated alike, \method\ covers the full constraint family of floorplanning and stays ahead of the ported pipelines after legalization.}

\subsection*{AI use statement}
Generative AI tools were used for writing assistance, drafting passages and
correcting grammar, and for visualizing results, where they wrote the plotting
code that renders the figures from the result files. The authors reviewed every
AI-assisted passage and figure and take responsibility for the content.

\subsection*{Reproducibility statement}
Every trained model in this paper, the ported prior placers included, uses the
frozen recipe of Section~\ref{sec:exp-setup} and Appendix~\ref{sec:appendix-train}.
The six constraint quantities, the refiner and the legalizer are specified in
Section~\ref{sec:method} and Appendices~\ref{sec:appendix-refiner} and
\ref{sec:appendix-legalizer}, and every scoring rule in
Appendix~\ref{sec:appendix-scoring}. FloorSet and the ICCAD 2026 evaluator are
public. The supplementary material is the code needed to reproduce the study,
the solver-agnostic domain library and the generator library, with the
configuration file of every trained arm, every refiner sweep and every legalizer
run, and the scripts that print every table from what a run writes
(Appendix~\ref{sec:appendix-data}).

\bibliography{refs}
\bibliographystyle{iclr2027_conference}

\newpage %
\appendix %
\appendixfloats
\makeatletter
\def\section{\@startsection{section}{1}{\z@}{-1.2ex plus -0.3ex minus -.2ex}{0.7ex plus 0.2ex}{\large\sc\raggedright}}
\def\subsection{\@startsection{subsection}{2}{\z@}{-1.0ex plus -0.3ex minus -.2ex}{0.4ex plus .1ex}{\normalsize\sc\raggedright}}
\def\paragraph{\@startsection{paragraph}{4}{\z@}{0.8ex plus 0.3ex minus .2ex}{-1em}{\normalsize\bf}}
\makeatother
\setlength{\parskip}{3.5pt plus 1pt minus 0.5pt}
\noindent\rule{\textwidth}{1pt}
\begin{center}
{\Large Appendix}
\end{center}
\noindent\rule{\textwidth}{1pt}

\section*{Table of Contents}

\definecolor{sectionblue}{RGB}{65, 105, 225}
\titlecontents{section}
  [1.6em]
  {\addvspace{0.15pc}\color{sectionblue}}
  {\contentslabel[\thecontentslabel]{1.75em}}
  {\hspace*{-1.5em}}
  {\hfill\contentspage}
  []

\titlecontents{subsection}
  [4em]
  {\addvspace{0pc}\color{sectionblue}}
  {\contentslabel[\thecontentslabel]{2.5em}}
  {\hspace*{-2.5em}}
  {\titlerule*[1pc]{.}\contentspage}
  []

\startcontents[sections]
{\footnotesize
\printcontents[sections]{l}{1}{\setcounter{tocdepth}{2}}}

\newpage

\section{Background and Related Work}
\label{sec:appendix-background}
\suppressfloats[t]

\subsection{The Physical Design Flow}
\label{sec:appendix-flow}

A chip goes from a netlist, the gates and macros of the design and the nets that connect them, to a geometric layout through a fixed sequence of stages: partitioning, floorplanning, placement, clock tree synthesis, routing and timing closure \citep{Kahng_2011}. Floorplanning is the first geometric stage. The design is partitioned into a few tens to a few hundred blocks, hard macros such as memories and IP whose shape is fixed and soft blocks of synthesized logic whose area is an estimate and whose aspect ratio is free within limits \citep{DBLP:conf/dac/Otten82, DBLP:journals/iandc/Stockmeyer83}, and floorplanning decides the position of every block and the shape of every soft one. Placement then arranges the standard cells inside the blocks into rows of sites, tens of thousands to millions of cells of one height and fixed shape \citep{DBLP:journals/pieee/MarkovHK15}, and routing realizes every net as wires on the metal layers under capacity and spacing rules \citep{Kahng_2011}. Figure~\ref{fig:background_problem} (top) sketches the three. They differ in what they move and in what binds them: floorplanning moves tens to a few hundred blocks and the shapes of the soft ones under non-overlap, fixed shapes, pre-placed blocks and the design constraints of Section~\ref{sec:intro}, placement moves $10^{5}$ to $10^{7}$ cells under row and site legality and a density target, and routing moves $10^{5}$ to $10^{7}$ nets over layers, tracks and vias under capacity and design rules, while the objectives widen along the flow from wirelength and outline area to congestion, timing, vias and overflow. They share one proxy objective, the half-perimeter wirelength (HPWL), the sum over nets of the half perimeter of the smallest box around the net's pins, because the wires do not exist until routing; congestion and timing enter later \citep{DBLP:conf/dac/ViswanathanASLW12}. The floorplan bounds all of it, since the relative position of the blocks fixes the length of the nets between them, the regions the router can use and the area of the die. Mixed-size placement handles macros and cells together with the macro shapes fixed \citep{DBLP:journals/tcad/AdyaYMVPM04, DBLP:journals/tcad/ChenJHCC08, DBLP:journals/tcad/ChenYCHL08}, and macro placement, the task of the reinforcement learning line \citep{DBLP:journals/nature/MirhoseiniGYJSW21}, is floorplanning with every shape fixed.

\begin{figure}[tbp]\centering
\includegraphics[width=\linewidth]{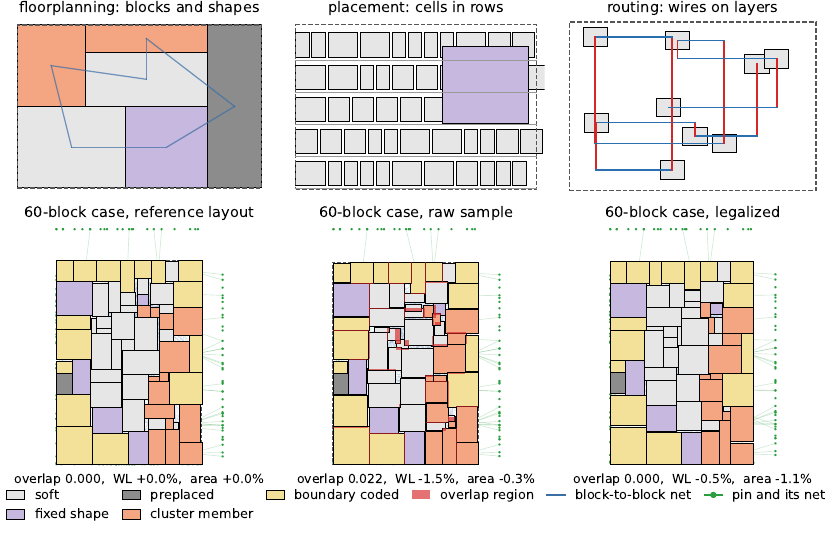}
\caption{Top: the three stages, drawn schematically: floorplanning arranges blocks and shapes the soft ones, placement fills rows with fixed-height cells around the macros, routing draws every net as wires on metal layers. Bottom: a 60-block validation case of FloorSet: the reference layout with its annotations, nets and outline pins, the raw sample of \method\ at $\nfe{=}8$ with overlap in red, and the legalized sample, each with its overlap ratio and wirelength and area gaps.}
\label{fig:background_problem}
\end{figure}

\subsection{What an Instance Contains}
\label{sec:appendix-instance}

A floorplanning instance is a netlist reduced to the block level. FloorSet gives for every case the target area of each block, the weighted nets between blocks, the pins on the die outline with their nets, the five kinds of annotation, and a reference layout \citep{DBLP:conf/iccad/MallappaM0PM24}; Figure~\ref{fig:background_problem} (bottom left) draws one case. Over the 100 validation cases the instances span 21 to 120 blocks, a median of 70, with 29 to 7,056 nets between blocks and 50 to 466 pins on the outline. The annotations are not rare: a median case has 9\% fixed-shape blocks, 4\% pre-placed blocks, 31\% of its blocks in one of three or four clusters, 8\% in a multi-instance group and 35\% of its blocks coded to a boundary, so most blocks carry a rule beyond non-overlap. Each stands for something a design team knows before the floorplan exists: a fixed shape is a hard macro delivered as a finished layout, a pre-placed block an interface or analog part whose position the package or an earlier floorplan settled, a cluster a functional unit or power domain whose blocks must stay close, a multi-instance group a block instantiated several times whose one internal layout is reused, and a boundary code a block that must reach the pins at the die edge. The reference layouts are tight, with 1 to 5\% dead space, and Table~\ref{tab:theory_reference} scores their residual violations. The hard cost of Eq.~\eqref{eq:cost} follows a public floorplanning evaluation on this data; Appendix~\ref{sec:appendix-scoring} states every rule as implemented.

\subsection{What Makes It Hard}
\label{sec:appendix-hardness}

Four things, and they compound. Packing: placing rectangles without overlap into the smallest outline is NP-hard even with fixed shapes and no nets, and the sequence pair alone encodes $(n!)^{2}$ arrangements of $n$ blocks \citep{DBLP:journals/tcad/MurataFNK96, Kahng_2011}. Shapes: a soft block adds a continuous choice to every combinatorial one, which has an exact solution on a slicing tree \citep{DBLP:journals/iandc/Stockmeyer83} and none on a general topology. Coupling: wirelength pulls connected blocks together, non-overlap holds them apart, a cluster asks for adjacency, a boundary code for the edge, a fixed outline for a total that fits \citep{DBLP:journals/tvlsi/AdyaM03}, all on the same coordinates, so a layout that satisfies the rules is in general not near one that scores well. The bottom row of Figure~\ref{fig:background_problem} shows the size of the effect: the raw sample overlaps 2.2\% of its area with a wirelength 1.5\% below the reference, and legalizing it removes every overlap by moving blocks until the wirelength is 0.5\% and the area 1.1\% below. Evaluation: feasibility is binary, so a layout with one thin overlap scores as infeasible although it is nearly done, and both an instance-level search and any comparison of intermediate layouts need a graded measure of the distance to feasible; annealers carry one in their cost, analytical placers in their density penalty, and Section~\ref{sec:method-problem} defines one for this paper. And the search repeats per instance, minutes to hours of annealing per floorplan, the motivation for parallel annealers \citep{mostafa2024parsacfasthumanqualityfloorplanning} and learned placers alike.

\subsection{How the Methods Developed}
\label{sec:appendix-related}

\paragraph{The constraints inside the search.}
In the 1980s floorplanning and cell placement were both done by simulated annealing \citep{Kirkpatrick_1983, DBLP:conf/dac/SechenS86}, floorplanning over slicing structures encoded as Polish expressions \citep{DBLP:conf/dac/Otten82, DBLP:conf/dac/WongL86}. What survived is the search in an encoding whose decode is overlap-free, so every move is judged on wirelength and area alone. What did not is the slicing restriction, which cannot express some tight packings. The 1990s produced non-slicing encodings with the same property, the sequence pair \citep{DBLP:journals/tcad/MurataFNK96}, the bounded-sliceline grid \citep{DBLP:conf/iccad/NakatakeFMK96}, the O-tree \citep{DBLP:conf/dac/GuoCY99}, the B$^{*}$-tree \citep{DBLP:conf/dac/ChangCWW00}, the transitive closure graph \citep{DBLP:conf/dac/LinC01} and the corner block list \citep{DBLP:conf/iccad/HongHCGDCG00}, and extended the sequence pair to hard, soft and pre-placed modules through a linear program over positions \citep{DBLP:conf/ispd/MurataK98}, the ancestor of the program in Appendix~\ref{sec:appendix-legalizer}. The 2000s made the outline a constraint to fit \citep{DBLP:conf/iccd/AdyaM01, DBLP:journals/tvlsi/AdyaM03} and put macros and cells into one problem \citep{DBLP:journals/tcad/AdyaYMVPM04, DBLP:journals/tcad/ChenYCHL08}, and PARSAC runs this line in parallel on the constraint family of FloorSet \citep{mostafa2024parsacfasthumanqualityfloorplanning}. Placement meanwhile left annealing, which does not scale past $10^{5}$ cells, for the analytical route \citep{DBLP:conf/dac/SiglDJ91}: minimize wirelength plus a differentiable penalty for overlap or density over continuous coordinates, then legalize. Quadratic placers \citep{Hall_1970, DBLP:journals/tcad/KleinhansSJA91, DBLP:journals/tcad/ViswanathanC05, DBLP:journals/tcad/SpindlerSJ08} gave way to nonlinear ones \citep{DBLP:journals/tcad/ChenJHCC08} and to the electrostatic density model \citep{DBLP:journals/todaes/LuCCSHTC15, DBLP:journals/tcad/ChengKKW19}, with legalization a stage of its own \citep{DBLP:conf/ispd/SpindlerSJ08}. What survived is the split into a graded relaxation and a legalizer, and DREAMPlace ran the relaxation in a deep learning framework on a GPU \citep{DBLP:journals/tcad/LinJGLDRKP21}, the bridge to learned placers. Surveys of the two eras are \citet{DBLP:journals/pieee/MarkovHK15} and \citet{Huang_2021}. Both branches meet the constraints and the objectives at every move, at the price of a search per netlist. \method\ keeps the graded energy of the analytical branch and moves it from that search into the training loss.

\paragraph{The constraints outside training.}
The 2020s brought learning and a debate about evidence. A graph network policy that places one macro per step, trained on a set of designs and fine-tuned on each new one \citep{DBLP:journals/nature/MirhoseiniGYJSW21}, was re-run by an independent group against annealing and commercial tools, which questioned its margins \citep{DBLP:conf/ispd/ChengKKWW23, DBLP:journals/cacm/Markov24}. The replies improved the representation and the search, with joint placement and routing \citep{DBLP:conf/nips/ChengY21}, visual states \citep{DBLP:conf/nips/LaiM022}, offline decision transformers \citep{DBLP:conf/icml/LaiLTWH023}, tree and evolutionary search \citep{DBLP:conf/icml/Geng0LXTYH0024, DBLP:conf/icml/DengLZG25}, cross-stage rewards \citep{DBLP:conf/iclr/GengWLXTKYH025}, expert knowledge in the policy \citep{ICLR2026_cbc1ad20} and black-box optimization with a wire mask in place of a policy \citep{DBLP:conf/nips/Shi0L023}, and they produced the shared data the debate had lacked: netlists with placement and routing labels \citep{https://doi.org/10.48550/arxiv.2208.01040}, floorplans with design constraints \citep{DBLP:conf/iccad/MallappaM0PM24}, and benchmarks that score a placer by what the finished chip achieves or by a black-box protocol \citep{https://doi.org/10.48550/arxiv.2407.15026, xue2026bboplacebenchbenchmarkingblackboxoptimization}. Floorplanning with soft blocks entered through a graph convolutional policy \citep{Xu_2022}, a corner-block-list representation learned across designs \citep{DBLP:conf/kdd/AminiZPZHL22} and direct regression of every position from the netlist graph \citep{DBLP:journals/todaes/LiuJLDZWYZS23}. Across the line the shapes stay fixed and the constraints stay outside the function the model learns, and Appendix~\ref{sec:appendix-hard-lines} measures what the regressors' conditional mean costs on our own backbone.

\paragraph{Generative placement and its corrections.}
The generative turn came from layout generation in graphics and architecture, where diffusion models draft documents and building floor plans under alignment and adjacency rules \citep{DBLP:conf/cvpr/InoueKSOY23, DBLP:conf/iccv/0004GSL023, DBLP:conf/cvpr/ShabaniHF23}, and from molecules and combinatorial problems \citep{DBLP:conf/icml/HoogeboomSVW22, DBLP:conf/nips/SunY23}, together with the toolkit those models use to satisfy rules at sampling time: gradient guidance \citep{DBLP:conf/nips/DhariwalN21, https://doi.org/10.48550/arxiv.2207.12598}, posterior sampling under a measurement \citep{DBLP:conf/iclr/ChungKMKY23} and constraint losses on the clean prediction \citep{DBLP:conf/iclr/ChenZZC24, DBLP:conf/iclr/BastekSK25}. The diffusion placers apply this toolkit to macro placement. ChipDiffusion \citep{DBLP:conf/icml/LeeNEDAW25}, DiffPlace \citep{trung2026diffplaceconditionaldiffusionframework} and MacroDiff+ \citep{yoon2026physicsguidedgeometricdiffusionmacro} predict the noise and steer every denoising step with an overlap or wirelength gradient in the forms Table~\ref{tab:guided_ports} lists, a heterogeneous graph variant follows the same design \citep{DBLP:conf/dac/ZhengGPWZZ25}, and FlowPlace \citep{xie2026flowplaceflowmatchingchip} predicts the velocity and projects each step onto a grid without overlap. All four train with a plain denoising loss, so the constraint terms appear first in the sampler and then in another form in a post-hoc loop, ten gradient steps in each of 1,000 denoising steps, up to 700 optimizer steps per denoising step, or 5,000 steps after sampling, before a legalizer closes what remains. They keep the learned proposal of the sequential line and the graded penalty of the analytical line but apply the penalty after the model, where the earlier eras had it inside the optimization, and their papers score the legalized endpoint at one sampling budget. Appendix~\ref{sec:appendix-guided} reproduces the three in-sampler corrections on our recipe, Section~\ref{sec:exp-refine} the loops, and Appendix~\ref{sec:appendix-hard} shows how much of the difference between generators the legalizer hides. \method\ moves the penalty back into training and samples without guidance.

\paragraph{Objectives and conditioning.}
Noise prediction \citep{DBLP:conf/nips/HoJA20, DBLP:conf/icml/NicholD21, DBLP:conf/iclr/SalimansH22}, flow matching and rectified flow \citep{DBLP:conf/iclr/LipmanCBNL23, DBLP:conf/iclr/LiuG023} and stochastic interpolants \citep{DBLP:conf/eccv/MaGABVX24} are linear transformations between data and noise, and a design-space view separates the prediction target, the time distribution and the weighting \citep{DBLP:conf/nips/KarrasAAL22}, the interface of Appendix~\ref{sec:appendix-framing}. Following JiT \citep{Li_2026_CVPR}, \method\ predicts the clean layout and trains on the velocity it implies, the one output an energy can score, and LACE \citep{DBLP:conf/iclr/ChenZZC24} and physics-informed diffusion \citep{DBLP:conf/iclr/BastekSK25} supervise the same clean prediction with constraint losses in other domains. The backbone is a Diffusion Transformer \citep{DBLP:conf/iccv/PeeblesX23} with the time step as a token \citep{Bao_2023_CVPR}, and the netlist enters it through an attention bias from the net incidence in the manner of Graphormer \citep{DBLP:conf/nips/YingCLZKHSL21}, a message-passing residual beside attention in the manner of GraphGPS \citep{DBLP:conf/nips/RampasekGDLWB22}, and a spectral positional encoding from the Laplacian eigenvectors \citep{DBLP:journals/neco/BelkinN03, DBLP:journals/corr/abs-2012-09699, DBLP:conf/nips/KreuzerBHLT21, DBLP:conf/iclr/DwivediL0BB22, DBLP:journals/jmlr/DwivediJL0BB23}.

\subsection{The Benchmark Suites}
\label{sec:appendix-suites}

Table~\ref{tab:background_suites} lists the suites the field measures itself on. The floorplanning suites are small and old: MCNC \citep{DBLP:conf/dac/Kozminski91} has 9 to 49 blocks per instance and GSRC extends it to 300 soft blocks, the scale of Table~\ref{tab:gsrc}. The placement suites are large: the mixed-size suite \citep{DBLP:journals/tcad/AdyaYMVPM04} and the ISPD contests of 2005 and 2006 \citep{DBLP:conf/ispd/NamAVWY05, DBLP:conf/ispd/Nam06} fixed the form of an industrial placement benchmark, millions of cells, fixed macros, a fixed outline and a density target, and the DAC 2012 contest added routability \citep{DBLP:conf/dac/ViswanathanASLW12}. Learning needs many instances rather than a few large ones, which CircuitNet supplies for placement and routing prediction \citep{https://doi.org/10.48550/arxiv.2208.01040} and FloorSet for floorplanning, 996,000 synthetic training instances with the design constraints the older suites lack \citep{DBLP:conf/iccad/MallappaM0PM24}.

\begin{table}[htbp]\centering\footnotesize
\caption{The benchmark suites of floorplanning and placement, each cited in the text. Sizes are per instance; ``constraints'' means design constraints beyond non-overlap, fixed shapes and the outline.}
\label{tab:background_suites}
\setlength{\tabcolsep}{4pt}
\begin{tabular}{llllll}
\toprule
Suite & Year & Task & Size per instance & Soft shapes & Constraints \\
\midrule
MCNC & 1991 & floorplanning & 9 to 49 blocks & yes & no \\
GSRC & 2000s & floorplanning & 10 to 300 blocks & yes & no \\
IBM mixed-size & 2004 & mixed-size placement & 10k to 200k cells & no & no \\
ISPD 2005 and 2006 & 2005 & placement & 0.2M to 2.2M cells & no & density \\
DAC 2012 & 2012 & placement & 0.1M to 1M cells & no & routability \\
CircuitNet & 2022 & placement and routing & 10k to 100k cells & no & no \\
FloorSet & 2024 & floorplanning & 21 to 120 blocks & yes & five kinds \\
\bottomrule
\end{tabular}
\end{table}

\section{Training and Model Details}
\label{sec:appendix-hparams}
\setlength{\tabcolsep}{3.5pt}

\subsection{Objective Framings}
\label{sec:appendix-framing}

Every objective compared in the paper is one per-$t$ linear transform of the pair $(x_0, x_1)$ of data and Gaussian noise,
\begin{equation}
x_t = a(t)\,x_0 + b(t)\,x_1, \qquad \mathrm{target}(t) = c(t)\,x_0 + d(t)\,x_1 .
\label{eq:framing}
\end{equation}
DDPM with a cosine or linear schedule and an $\epsilon$ target \citep{DBLP:conf/nips/HoJA20, DBLP:conf/icml/NicholD21}, the $v$ target \citep{DBLP:conf/iclr/SalimansH22} and rectified flow with target $x_1 - x_0$ \citep{DBLP:conf/iclr/LiuG023, DBLP:conf/iclr/LipmanCBNL23} differ only in $(a, b, c, d)$. Given a prediction of the target and $x_t$, the $2\times2$ system recovers $(x_0, x_1)$ whenever $ad - bc \neq 0$, and the probability-flow velocity is $\dot a\,x_0 + \dot b\,x_1$. A network may therefore emit $\epsilon$, $x_0$, $v$ or the target itself and be trained against any target by a change of variables inside the loss. The paper's models emit $x_0$ and are trained against the rectified-flow velocity with logit-normal $t$. The objective arms of Table~\ref{tab:generator} are the velocity head under the same rectified flow with uniform $t$ (the FlowPlace objective) and the $\epsilon$ head under the cosine DDPM schedule with uniform $t$ (the ChipDiffusion objective), both on the plain backbone. The ported DiffPlace and MacroDiff+ keep their linear-$\beta$ $\epsilon$ objectives with $T = 1{,}000$ and $T = 200$.

\subsection{The Frozen Recipe}
\label{sec:appendix-train}

Table~\ref{tab:train_hparams} is the recipe of every trained model in the paper, the ported prior arts, the regressors and the seed replicates included. An arm changes exactly one thing against its control: the head, the wire-dropout mixture, the term weights, or a backbone component. The two rows marked Trinity are the backbone components that the composed arms and Trinity add to the plain backbone; every soft number of the paper is the mean over the draws of a case, then over the cases.

\begin{table}[p]\centering
\caption{The frozen training recipe of every trained model; the two rows marked Trinity are the components the composed arms add.}
\label{tab:train_hparams}
\footnotesize
\setlength{\tabcolsep}{4.5pt}
\renewcommand{\arraystretch}{0.94}
\begin{tabular}{l>{\raggedright\arraybackslash}p{9.3cm}}
\toprule
\multicolumn{2}{l}{\textbf{Data}} \\
Training set & FloorSet-Lite, MIB-consistent variant (Appendix~\ref{sec:appendix-data}) \\
Training layouts & 996,000 of the 1,008,000, with 21 to 120 blocks each \\
Our 12,000 val set & 12,000: 100 per block count plus 2,000 random rows \\
Test set & the FloorSet val set of 100 cases, used once \\
\midrule
\multicolumn{2}{l}{\textbf{Latent and conditioning}} \\
Latent per block & $(u, v, \rho)$: center over $s$ and $\log(w/h)$ at the exact target area \\
Conditioning columns & 19 per block: area, kind, anchor value and mask, group ids, boundary code, pin pull \\
Time conditioning & one prepended token; sequences padded to 127 blocks plus the token \\
\midrule
\multicolumn{2}{l}{\textbf{Backbone}} \\
Layers $L$ & 12 \\
Width $d$ & 768 \\
Attention heads & 12, head dimension 64 \\
MLP & SwiGLU, ratio 4 \\
Normalization & RMSNorm before every sublayer, QK-normalization \\
Netlist attention bias & $\log(1+A)$, max-normalized per instance, one learned gate \\
Spectral PE (Trinity) & Fiedler pair of the normalized netlist Laplacian, $\log(1+A)$ weights \\
Information mover (Trinity) & $x \leftarrow x + g\,W\hat A x$, $\hat A = D^{-1}A$, gate zero-initialized, every layer \\
Parameters & 85.8M plain backbone, 92.9M with the information mover \\
\midrule
\multicolumn{2}{l}{\textbf{Objective}} \\
Framing & rectified flow, $x_t = (1-t)\,x_0 + t\,x_1$, $\hat x_0$ head trained against the velocity \\
Time sampling & logit-normal, mean 0, std 1 \\
Overrides & fixed shapes and preplaced boxes inpainted at every step \\
\kb{MIB groups with an anchor} & \kb{every soft member's aspect overridden with the fixed-shape or preplaced member's, no mean taken} \\
MIB groups without an anchor & set to the group-mean aspect of the prediction \\
\midrule
\multicolumn{2}{l}{\textbf{Auxiliary terms}} \\
Terms & the six functions of Appendix~\ref{sec:appendix-refiner-derivatives} on the decoded boxes \\
Trinity weights & all six at 0.01, constant in $t$ \\
Ladder weights & 0.01, 0.05, 0.1, and $w \cdot t$ with $t{=}1$ at pure noise \\
Fixed and preplaced channels & detached in every term \\
\midrule
\multicolumn{2}{l}{\textbf{Augmentation}} \\
Rotation & multiples of $90^\circ$, uniform \\
Flip & probability 0.5 \\
Shift & Gaussian, std $0.05\,s$ \\
Constraint dropout & 1\% per boundary code, cluster, MIB group, fixed and preplaced block \\
Wire dropout (adopted recipe) & per instance a rate from 50\% none, 40\% Beta(1,3), 10\% U(0.9,1) \\
Wire dropout effect & nets and pins dropped at that rate, reference HPWL recomputed \\
\midrule
\multicolumn{2}{l}{\textbf{Optimization}} \\
Optimizer & AdamW \\
Peak learning rate & $2\times10^{-4}$ \\
Betas & (0.9, 0.98) \\
Weight decay & $10^{-3}$ \\
Gradient clip & 1.0 \\
Schedule & cosine decay to 1\% of the peak, 1\% warmup \\
Effective batch & 256 instances: micro-batch 128, two accumulation steps \\
Steps & 200,000 \\
Precision & bf16 mixed \\
EMA & 0.9995, used for all sampling \\
\midrule
\multicolumn{2}{l}{\textbf{Sampling}} \\
Integrator & Euler on the probability-flow ODE, no projection beyond the overrides \\
\kb{Time grid} & \kb{uniform from $t=1$ to $t=10^{-3}$ in $\nfe$ steps; the sample is the network's $\hat x_0$ at the last time} \\
$\nfe$ & 1, 2, 4, 8, 16, 32 \\
Draws per case & 4; 16 for the best-of-$N$ caches \\
\midrule
\multicolumn{2}{l}{\textbf{Compute}} \\
Hardware & one NVIDIA RTX 5090 per run \\
Wall time & 7.6 to 12.3 hours per run (Appendix~\ref{sec:appendix-compute}) \\
\bottomrule
\end{tabular}
\end{table}

\paragraph{Wire dropout.}
Three mixtures were trained on the plain base: a mild one (70\% of instances untouched, 30\% at a rate from Beta(1,3)), the adopted mixture (50\% untouched, 40\% Beta(1,3), 10\% U(0.9,1)) and a heavy one (30\% untouched, 50\% Beta(1,2), 20\% U(0.8,1)). On our 12,000 val set the adopted mixture is within $\pm1\%$ of the base on every soft metric except the immutable-block deviation ($-20\%$) and the MIB variance ($-16\%$), has the lowest soft cost of the set (2.092 against 2.106) and is the closest arm to the reference distribution on all four distances with $p = 0.005$. The mild and the heavy mixture move the wirelength gap toward zero and the area gap up by 30 to 40\%. The adopted mixture is the recipe of every model from the term ladders on.

\subsection{Architecture Specification}
\label{sec:appendix-architecture}
The backbone is a plain transformer with one token per block: 12 layers, width 768, 12 heads, RMSNorm, SwiGLU and QK-normalization. The input projection embeds the noisy latent $z_{t,i}$ with 19 conditioning columns per block (area, kind, anchor value and mask, group identifiers, boundary code, pin pull). The time step enters as one prepended token in the manner of U-ViT \citep{Bao_2023_CVPR}, so the block has no adaptive normalization and the same block serves the deterministic regressors of Section~\ref{sec:exp-setup}. Three components put the netlist into the network (Figure~\ref{fig:arch}). The input projection is one linear layer from the concatenated 3 latent, 19 conditioning and 2 spectral coordinates, 24 in total, to the width of 768.
\textbf{Attention bias.} The max-normalized $\log(1+A_{ij})$ is added to every attention logit through one learned gate \citep{DBLP:conf/nips/YingCLZKHSL21}; padded keys receive a large negative bias, and the bias rides the mask argument of the fused attention kernel.
\textbf{Information mover.} Every layer adds one graph-convolution hop behind a zero-initialized gate, $x \leftarrow x + g\,W\hat A x$ with $\hat A = D^{-1}A$ the row-normalized netlist adjacency, so each block receives the mean of its neighbours \citep{DBLP:conf/iclr/KipfW17, DBLP:conf/nips/RampasekGDLWB22}. The residual supplies one-hop netlist information in parallel with attention, the zero gate makes the layer a plain transformer layer at initialization, and it adds 7.1M parameters for 92.9M in total.
\textbf{Spectral positional encoding.} Two input columns hold the second and third eigenvectors of the normalized netlist Laplacian, the quadratic-placement drawing of the netlist \citep{Hall_1970, DBLP:journals/tcad/KleinhansSJA91, DBLP:conf/cocoon/Koren03}, in the same $(c_x/s, c_y/s)$ frame the model generates. Sign and axis are fixed by unit norm and non-negative skewness, because a reflected drawing is a different layout.

\subsection{The Ported Prior Arts}
\label{sec:appendix-ports}

Each port keeps the published network and objective, trains under Table~\ref{tab:train_hparams} on the same data, consumes the same 3-D latent and the same 19 conditioning columns as our models, and samples with the same Euler integrator and the same overrides. FlowPlace and ChipDiffusion share the AttGNN network, three blocks of two GATv2 layers, an MLP, global attention and an MLP at hidden width 256 with FiLM time conditioning, 6.3M parameters; FlowPlace trains rectified flow with a velocity head and uniform $t$ and publishes no loop, ChipDiffusion trains a cosine DDPM with an $\epsilon$ head and uniform $t$ and follows sampling with a 5,000-iteration gradient legalizer on legality and HPWL potentials, in its standard, scheduled and dual-ascent Adam variants. DiffPlace keeps VectorGNNV2Global, eight blocks of two vector message-passing layers with FiLM, a global supernode and hidden width 256, 23.9M parameters with the orientation head dropped, trains a linear-$\beta$ DDPM with an $\epsilon$ head at $T{=}1{,}000$, and follows sampling with a 500-iteration gradient refinement. MacroDiff+ keeps its heterogeneous GATv2 cell and net network of five layers at width 64 with a 1-D token U-Net of widths $32\cdot[1,2,4,8]$, 5.0M parameters, trains a linear-$\beta$ DDPM at $T{=}200$ whose $\epsilon$ is composed as $\alpha\,(\partial\hpwl/\partial x)^{\top}\,\mathrm{net} + \beta\,\mathrm{cell}$ from a net head and a cell head, and follows sampling with a 500-iteration Adam loop on HPWL and overlap. The post-hoc loops (Section~\ref{sec:exp-refine}) are ported with autograd at their published optimizers, schedules and iteration counts, act on positions only, and run on a canvas equal to the raw bounding box plus a 5\% margin, which their die-boundary terms need to separate blocks. A canvas tuned per method and shape freedom are not granted: the released energies have no shape variable, and shape freedom after sampling is the refiner's (Appendix~\ref{sec:appendix-refiner}). The corrections three of the placers apply inside their samplers are ported separately, on the same canvas, for Appendix~\ref{sec:appendix-guided} (Table~\ref{tab:guided_ports}).

\subsection{The Regressors}
\label{sec:appendix-regression}

Reg-$z$ is the plain backbone driven with a constant input and $t = 0$, so its output depends on the conditioning only, trained by masked regression of the reference latent under the frozen recipe. Reg-$xywh$ changes the head to $(c_x/s, c_y/s, w/s, h/s)$, expresses the overrides in that space and adds an area penalty because $wh = a$ no longer holds by construction. Each returns one layout per case.

\subsection{Compute}
\label{sec:appendix-compute}

Table~\ref{tab:compute} accounts for every job of the study. Each job ran on one NVIDIA RTX 5090, with the CPU workers named where a stage uses them; the whole campaign, training and evaluation, ran on one four-GPU box. Training is the bulk: 68 runs under the frozen recipe (objective 3, recipe 3, architecture 2, single terms 24, all six 6, composed grid 8, prior arts 4, regressors 2, seed replicates 8, size rungs 8), 7.6 to 12.3 hours each, for 710 GPU-hours, with DiffPlace at 29.4 hours because its network needs a micro-batch of 64, MacroDiff+ at 14.7 hours, the width rungs at 9.0, 8.8 and 8.5 hours, the depth rungs at 10.5, 8.6 and 7.6 hours and the two late rungs at 5.9 and 7.4 hours; the per-run wall times, read from each run's tracker summary, are in the result files of the supplementary material. Earlier attempts of the same runs that crashed or were superseded cost about 61 GPU-hours and enter no result. Generation and scoring take about a quarter of an hour per run: 12,000 cases at six $\nfe$ with four draws, the 15 soft columns, the 16-draw \method\ caches and one rescoring pass. The guided sampling of Appendix~\ref{sec:appendix-guided} is five generations, one per placer and one for DiffPlace's sign-flipped arm, with MacroDiff+ at 9.2 of the 11 GPU-hours because of its 500 Adam iterations per step, scored on 16 CPU workers; their refiner sweeps and distribution metrics are counted in those rows. A refiner sweep takes 2 to 8 minutes, 1 to 4 of them compiling. Evaluation is dominated by the legalizer track, 5 full $\nfe \times$ steps grids, 3 single-shard lines, 15 prior-art loop runs and 3 best-of-16 ladders on 24 to 32 CPU workers: 10.6 million legalizations at 240 to 300 ms each, during which the GPU only refines and then waits, so the track costs about 1,400 CPU-core-hours and about three active GPU-hours although the GPU is held for about 57 hours. The six size-rung grids of Appendix~\ref{sec:appendix-size} add 2,016,000 legalizations each on 24 workers, about 1,340 CPU-core-hours with the GPUs held for about 56 hours. The theory measurements of Appendix~\ref{sec:appendix-theory-measured} are four scripts and two GSRC runs. The classical track of Appendix~\ref{sec:appendix-classical} is CPU time on single cores: PARSAC on the FloorSet val set with 16 workers at $10^7$ steps for about 500 core-hours, its GSRC runs about 60, the sequence-pair annealer about 12, the continuations about 14, and the pipeline's own passes about 1.5 GPU-hours. The runtime study against the block count took 0.7 hours on an idle GPU.

\begin{table}[htbp]\centering\footnotesize
\caption{Compute of the study. One RTX 5090 per job, the CPU workers of each stage as given in the text. "Held" counts the hours the GPU is reserved while the CPU workers legalize.}
\label{tab:compute}
\setlength{\tabcolsep}{5pt}
\begin{tabular}{lrlr}
\toprule
Stage & Jobs & GPU-hours & CPU-core-hours \\
\midrule
\multicolumn{4}{l}{\textit{Training}} \\
Trained runs of the study & 68 & 710 & -- \\
Earlier attempts, not used & -- & $\approx 61$ & -- \\
\midrule
\multicolumn{4}{l}{\textit{Evaluation}} \\
Generation and scoring & 68 & $\approx 22$ & -- \\
Guided sampling of three placers & 5 & $\approx 11$ & $\approx 5$ \\
Distribution metrics & 10 & $\approx 9$ & $\approx 69$ \\
Refiner sweeps & 131 & $\approx 3$ & -- \\
Hard-cost grids and ladders & 26 & $\approx 3$ active ($\approx 57$ held) & $\approx 1{,}400$ \\
Size-rung grids & 6 & $< 2$ active ($\approx 56$ held) & $\approx 1{,}340$ \\
Theory measurements & 6 & $\approx 0.3$ & -- \\
Classical solvers & -- & $\approx 1.5$ & $\approx 590$ \\
FloorSet val set runs & 2 & $\approx 0.6$ & $\approx 10$ \\
Runtime study & 1 & 0.7 & -- \\
\midrule
Evaluation total & & $\approx 53$ active ($\approx 167$ held) & $\approx 3{,}410$ \\
Training and evaluation & & $\approx 760$ active & $\approx 3{,}410$ \\
\bottomrule
\end{tabular}
\end{table}

\section{Scoring Details}
\label{sec:appendix-scoring}

\paragraph{Hard feasibility and the hard cost.}
A layout is feasible when no pair of boxes overlaps by more than $10^{-6}$ on both axes, every soft block keeps $wh$ within 1\% of its target area, every fixed-shape block keeps its $(w, h)$ and every preplaced block keeps its $(x, y, w, h)$. Let $c_p$ be the number of edge-connected components among the members of cluster $p$, where two boxes are connected when they share an edge segment of positive length, $s_q$ the number of distinct shapes in MIB group $q$ after rounding to $10^{-4}$, and $B$ the set of boundary-coded blocks with codes 1 left, 2 right, 4 top, 8 bottom and corners as sums. Then
\begin{align}
V_{\mathrm{group}} &= \sum_p (c_p - 1), \qquad V_{\mathrm{mib}} = \sum_q (s_q - 1), \nonumber \\
V_{\mathrm{bnd}} &= \sum_{i\in B} \mathbb{1}\bigl[\text{a required side of } i \text{ is off its outline edge by more than } 10^{-6}\bigr],
\end{align}
and $\Vrel = (V_{\mathrm{group}} + V_{\mathrm{mib}} + V_{\mathrm{bnd}}) / N_{\mathrm{soft}}$ with $N_{\mathrm{soft}}$ the number of soft-constrained slots. $\hpwl$ is the net-weighted sum of centroid-to-centroid Manhattan distances over block-to-block and pin-to-block nets, the area is that of the layout bounding box, and both gaps are the layout value minus the reference value over the reference value, clamped at zero. The hard-cost total over a case set is $\sum_c e^{n_c/12}\,\mathrm{Cost}_c / \sum_c e^{n_c/12}$; the paper reports the per-case mean and median and gives the totals in Appendix~\ref{sec:appendix-hard}.

\paragraph{The soft metric vector.}
Every layout, raw, refined or legalized, is scored with the same columns, every length in units of $s = \sqrt{\sum_i a_i}$:
\begin{itemize}[leftmargin=*,nosep]
\item \emph{overlap ratio}: $\sum_{i<j}$ intersection area over $\sum_i$ block area;
\item \emph{compactness}: bounding-box area over $\sum_i$ block area, at least 1;
\item \emph{normalized wirelength} $\mathrm{wl}$: $\hpwl / (s \cdot \sum \text{net weights})$;
\item \emph{hpwl gap}, \emph{area gap}: the two relative gaps to the reference, signed;
\item \emph{boundary distance}: mean over boundary-coded blocks of the distance from each required side to its outline edge;
\item \emph{group gap}: mean over same-cluster pairs of the axis-wise gap to contact;
\item \emph{MIB variance}: mean within-group variance of $(\log w, \log h)$;
\item \emph{fixed distance}: mean over fixed and preplaced blocks of the shape deviation over $\sqrt{a}$ plus, for preplaced blocks, the center distance;
\item \emph{severities} $\mathrm{sev}_{\mathrm{group}}$, $\mathrm{sev}_{\mathrm{mib}}$, $\mathrm{sev}_{\mathrm{bnd}}$: per soft slot $\tanh(d/\tau)$ with $d$ the slot's distance to satisfaction and $\tau = 0.01, 0.05, 0.01$, summed over slots;
\item $\mathrm{sev}_{\mathrm{rel}}$: $(\mathrm{sev}_{\mathrm{group}} + \mathrm{sev}_{\mathrm{mib}} + \mathrm{sev}_{\mathrm{bnd}}) / N_{\mathrm{soft}}$, the severity analogue of $\Vrel$;
\item \emph{soft cost}: the second aggregate of Eq.~\eqref{eq:cost}, computed per case before averaging.
\end{itemize}
In every table the boundary distance, the group gap and the fixed distance are given in units of $10^{-3}$, the MIB variance in units of $10^{-6}$ and MMD in units of $10^{-3}$; the other columns are unscaled. The severities grade how far a slot is off and never how many slots are off, so the soft vector is continuous in the layout, and a jitter of 0.2\% of $s$ on a reference layout moves every soft column monotonically while the hard cost jumps to its violated value. Where a model's mean soft cost overflows (a few cases with gaps in the tens, only for the regressors and for the one-step samples of the $\epsilon$-head models) the per-case median and 90th percentile are given.

\subsection{Distribution Metrics}
\label{sec:appendix-dist}

The per-layout metrics above score one layout against its own constraints and its own reference. The distribution metrics ask a different question: whether the set of layouts a model draws looks like the set of reference layouts. Each layout is mapped to a fixed-length, permutation-invariant descriptor that is a plain function of its geometry, with no learned encoder, and two standard distances are taken between the descriptor cloud of a model (all draws of one run at one $\nfe$, 48,000 rows) and the descriptor cloud of the 12,000 reference layouts.

\paragraph{Density descriptor.}
Let $W$ be a rasterization window and $G$ a grid size. In the \emph{bbox} frame $W$ is the layout's own bounding box, so the descriptor sees the arrangement of the blocks and factors out the size of the outline; in the \emph{fixed} frame $W$ is the square of side $2s$ centered on the bounding-box center, so the descriptor also sees how far the layout spreads in absolute units. $W$ is cut into $G\times G$ cells $g$ of equal size, and the coverage of cell $g$ by block $i$ is the fraction of the cell the box $B_i$ occupies,
\begin{equation}
c_i(g) = \frac{|B_i \cap g|}{|g|} \in [0, 1].
\end{equation}
Four fields are stacked: all blocks, soft blocks, fixed and preplaced blocks, and the net-weighted center density,
\begin{align}
f_{\mathrm{all}}(g) &= \sum_i c_i(g), \qquad
f_{\mathrm{soft}}(g) = \sum_{i\ \mathrm{soft}} c_i(g), \qquad
f_{\mathrm{lock}}(g) = \sum_{i\ \mathrm{fixed,\ preplaced}} c_i(g), \nonumber \\
f_{\mathrm{net}}(g) &= \frac{\sum_{i:\, \mathrm{center}(i)\in g} \omega_i}{\sum_i \omega_i},
\end{align}
where $\omega_i$ is the total weight of the nets incident on block $i$ (block-to-block weights counted at both endpoints, pin weights at the block). The coverage fields sum above $1$ in a cell where blocks overlap, so the descriptor sees overlap as well as arrangement. The descriptor $\phi_G(\text{layout})\in\mathbb{R}^{4G^2}$ is the concatenation of the four fields: 256 dimensions at $G = 8$ and 576 at $G = 12$. FLD8b, FLD8f, FLD12b and FLD12f name the Fr\'echet distance below on $\phi_8$ and $\phi_{12}$ in the two frames.

\paragraph{Position descriptor.}
The block centers $c_i$ are scaled by $s$ and shifted by the bounding-box center. The descriptor collects the biased covariance of the center cloud ($\Sigma_{xx}, \Sigma_{yy}, \Sigma_{xy}$), the per-axis skewness and excess kurtosis of the standardized coordinates, the radius of gyration $\sqrt{\mathrm{mean}_i\,\|c_i\|^2}$, a 16-bin histogram of the nearest-neighbor center distance on $[0, 0.5]$ and a 16-bin histogram of the pairwise center distance on $[0, 2]$, both normalized to sum to one: $8 + 16 + 16 = 40$ dimensions. FPD names the Fr\'echet distance on it.

\paragraph{Fr\'echet distance.}
For a model cloud $X$ and the reference cloud $R$ with means $\mu_X, \mu_R$ and sample covariances $\Sigma_X, \Sigma_R$ (each with $10^{-6} I$ added for the matrix square root),
\begin{equation}
d_{\mathrm{F}}^2(X, R) = \|\mu_X - \mu_R\|_2^2 + \operatorname{Tr}\!\bigl(\Sigma_X + \Sigma_R - 2\,(\Sigma_X \Sigma_R)^{1/2}\bigr),
\label{eq:frechet}
\end{equation}
the squared Wasserstein-2 distance between the Gaussian fits of the two clouds, the form used by FID \citep{DBLP:conf/nips/HeuselRUNH17}. The mean term and the covariance term are stored separately in the result files of the supplementary material. The distance is computed on the full clouds, uses the first two moments only, and is not standardized, so cells with a large natural variance weigh more.

\paragraph{Kernel MMD.}
For the $G = 8$ bbox descriptor every dimension is first standardized by the reference standard deviation, $\tilde\phi = \phi / (\sigma_R + 10^{-9})$, so every cell weighs the same. With the Gaussian kernel $k(x, y) = \exp(-\|x - y\|^2 / 2h^2)$ and the bandwidth $h$ set to the median pairwise Euclidean distance within a 4,000-row subsample of the standardized reference, the unbiased estimator \citep{DBLP:journals/jmlr/GrettonBRSS12} on subsamples $X = \{x_i\}_{i=1}^{n}$ and $R = \{r_j\}_{j=1}^{m}$ of 4,000 rows each is
\begin{equation}
\mathrm{MMD}^2_u(X, R) = \frac{1}{n(n-1)}\sum_{i\neq i'} k(x_i, x_{i'}) + \frac{1}{m(m-1)}\sum_{j\neq j'} k(r_j, r_{j'}) - \frac{2}{nm}\sum_{i, j} k(x_i, r_j).
\label{eq:mmd}
\end{equation}
This is MMD8b. The Gaussian kernel is characteristic, so the population MMD is zero only when the two distributions coincide and it is sensitive to every moment, which complements the two-moment Fr\'echet distance; in every table of the paper the two rank the models the same way. The estimate depends on the 4,000-row subsample, so the same run re-evaluated in another batch moves by a few percent: Trinity reads 2.5 in the batch of Table~\ref{tab:generator} and 2.3 in the seed batch of Appendix~\ref{sec:appendix-seeds}. Every table prints the batch its rows were evaluated in, and a comparison is always read within one batch.

\paragraph{Permutation test.}
Whether a model $A$ is closer to the reference than a control $B$ is decided by a one-sided label-permutation test on MMD8b. With 2,000-row subsamples of $A$, $B$ and $R$ the statistic is $T = \mathrm{MMD}^2_u(A, R) - \mathrm{MMD}^2_u(B, R)$; under the null that $A$ and $B$ are exchangeable the pooled rows are relabeled at random 200 times, giving $T_1, \dots, T_{200}$, and
\begin{equation}
p = \frac{1 + \#\{\ell : T_\ell \leq T\}}{1 + 200}.
\label{eq:perm}
\end{equation}
A small $p$ means the observed difference is farther below zero than a random relabeling produces, so $A$ is closer to the reference than $B$; $p = 1/201 \approx 0.005$ is the floor. The same model with another training seed tests at $p = 0.20$ to $1.00$ against the original (Appendix~\ref{sec:appendix-seeds}), the null band of the test.

\paragraph{Reading the distribution metrics.}
Lower is closer to the reference distribution, which is fidelity to the training layouts, not cost: the reference layouts leave 9.1\% of the boundary-coded blocks off their edge, so a model that places more coded blocks on the edge moves away from the reference in the cells along the outline, which is why the boundary term at high weight and Trinity's own FLD8f can rise while every per-layout metric improves. The bbox frame answers "is the arrangement right", the fixed frame "is the spread right", the position descriptor "is the center cloud right", and the four together are read as a set. The main text writes FLD for FLD8b and MMD for MMD8b.

\paragraph{Correlation measures.}
Per $(\nfe, \text{steps})$ cell of a line, the Spearman $\rho$ and Kendall $\tau$ of the per-draw soft cost against the per-draw hard cost over the 48,000 draws and over the 12,000 per-case means, Pearson $r$ and $R^2$ on raw and on log values, and the Spearman per soft column. Across a grid, the same measures on the cell means. Pooled, over every draw of a line and over the eight lines together.

\paragraph{A property of the data set.}
The reference layouts satisfy every cluster abutment. They leave 9.1\% of the boundary-coded blocks off their required edge and pay about 10\% in hard cost for it. The six quantities read the constraint annotations and never the reference layout, so a model trained with them can place a coded block on its edge where the reference does not, and Trinity's boundary distance is one quarter of the plain control's.

\section{Properties of the Construction}
\label{sec:appendix-theory}

\subsection{Setup}
\label{sec:appendix-theory-setup}

Notation follows Section~\ref{sec:method}. A clean latent is $x_0$, noise is
$x_1 \sim \mathcal{N}(0, I)$, the interpolant is $x_t = (1-t)x_0 + t\,x_1$, and
\begin{equation}
\bar x_0(x,t) = \mathbb{E}[x_0 \mid x_t = x],
\qquad
\Sigma_{0|t}(x) = \operatorname{Cov}(x_0 \mid x_t = x),
\label{eq:posterior}
\end{equation}
so $\Sigma_{0|0} = 0$ and $\Sigma_{0|1} = \operatorname{Cov}(x_0)$. Write
$U = \sum_k \omega_k \mathcal{E}_k$ for the energy of Section~\ref{sec:method-physics} on the
decoded boxes, split into the four constraint terms and the two objectives,
\begin{equation}
U = U_{\mathrm{C}} + U_{\mathrm{O}},
\qquad
\mathcal{F} = \{x : U_{\mathrm{C}}(x) = 0\},
\label{eq:split}
\end{equation}
and let
$\operatorname{prox}_{\gamma U}(y) = \arg\min_m \tfrac12\lVert m - y\rVert^2 + \gamma\,U(m)$ for
$\gamma > 0$. The decode of Eq.~\eqref{eq:decode} satisfies $w_ih_i = a_i$ for every
$\rho_i \in \mathbb{R}$, so it is a bijection onto $\{(w,h) \in \mathbb{R}_{>0}^2 : wh = a_i\}$
with inverse $\rho_i = \log(w_i/h_i)$, and linear interpolation of $\rho$ is the geometric path
$w_t = w_0^{1-t}w_1^{t}$. The area rule is therefore the coordinate system and not a constraint,
at every $t$ and for every noise draw. Two conventions of Eq.~\eqref{eq:cost} hold throughout:
a gap is measured against the reference value or $10^{-6}$, whichever is larger, and an instance
without soft-constraint slots ($N_{\mathrm{soft}} = 0$) has $\Vrel = \mathrm{sev}_{\mathrm{rel}} = 0$.

$\mathcal{F}$ is the zero set of the constraint energy, the feasible set of the constraints other
than the area rule as in Section~\ref{sec:method-theory}; it also asks for the three soft
constraints, so it lies inside the hard feasible set of Section~\ref{sec:intro} up to the
tolerance of the area rule, and every statement below is about $\mathcal{F}$.
Every numbered result of this appendix is checked in Lean~4 against Mathlib, with the code in the
supplementary material; conditional expectations are read at a fixed $x$ as integrals against
the conditional law, and three standard steps enter those checks as hypotheses:
differentiation under the integral in Lemma~\ref{lem:tilt}, the third-order Taylor expansion of
$U$ in Eq.~\eqref{eq:bias}, and Tweedie's formula with the channel covariance identity in
Proposition~\ref{prop:support}.

\begin{assumption}\label{as:standing}
(i) $U \in C^2(\mathbb{R}^d)$, $U \ge 0$, and $\nabla^2 U \succeq -\beta I$ for some
$\beta \ge 0$ (weak convexity). (ii) $U_{\mathrm{C}}(x_0) = 0$ for $p_{\mathrm{data}}$-almost
every $x_0$: every reference layout lies in $\mathcal{F}$. (iii) The head is an arbitrary measurable
function of $(x,t)$ and $x_0$ has finite second moments, so the training objective is minimized
pointwise, for $p_t$-almost every $x$ at every $t$.
\end{assumption}

\begin{lemma}\label{lem:prox}
Under Assumption~\ref{as:standing}(i), for every $\gamma > 0$ with $\gamma\beta < 1$ the
objective $m \mapsto \tfrac12\lVert m - y\rVert^2 + \gamma U(m)$ has Hessian
$I + \gamma\nabla^2 U \succeq (1 - \gamma\beta) I \succ 0$, so it is strongly convex and
$\operatorname{prox}_{\gamma U}(y)$ is a single point for every $y$.
\end{lemma}

By Assumption~\ref{as:standing}(ii) the reweighting $p_{\mathrm{data}}\,e^{-U_{\mathrm{C}}}$ is
proportional to $p_{\mathrm{data}}$ itself, since the density ratio is constant on a set of full
measure; what the constraint terms do to the trained head is the subject of
Theorem~\ref{thm:prox}, not a statement about the fitted distribution. The objective part
differs: $p_{\mathrm{data}}\,e^{-U_{\mathrm{O}}}$ prefers shorter wires and smaller outlines, and
the proximal step applies it with the prox weight $\gamma_t$. Assumption (ii) holds
approximately: the references' residual violations are small on average, a soft cost of 1.065
against 1 for a violation-free layout (Table~\ref{tab:theory_reference}).

\subsection{What the Latent Buys}
\label{sec:appendix-theory-latent}

\begin{proposition}\label{prop:averaging}
Let a block have target area $a > 0$ and let $(W,H)$ be a random box with $WH = a$ almost
surely, $W > 0$ and $\mathbb{E}\lvert\log W\rvert < \infty$. Then
\begin{equation}
\exp\bigl(\mathbb{E}[\log W]\bigr)\,\exp\bigl(\mathbb{E}[\log H]\bigr) = a,
\qquad
\mathbb{E}[W]\,\mathbb{E}[H] \ \ge\ a,
\label{eq:averaging}
\end{equation}
with equality on the right if and only if $W$ is almost surely constant.
\end{proposition}
\begin{proof}
$\log W + \log H = \log a$ almost surely gives the left identity. For the right,
$\mathbb{E}[W]\mathbb{E}[H] = a\,\mathbb{E}[W]\,\mathbb{E}[1/W]$, and Cauchy-Schwarz gives
$1 = \bigl(\mathbb{E}[\sqrt{W}\cdot W^{-1/2}]\bigr)^2 \le \mathbb{E}[W]\,\mathbb{E}[1/W]$, with
equality if and only if $W$ is almost surely constant.
\end{proof}

On a fixed instance with fixed target areas, a head trained under a squared loss with finite
second moments emits, at the population optimum, a conditional mean, so
Proposition~\ref{prop:averaging} decides what area that output has: in the latent the mean
decodes at area exactly $a$, and in width-height coordinates it decodes at an area above $a$
unless the shape is deterministic. The two regressors of Section~\ref{sec:exp-prior} are the
two sides of this.

\begin{proposition}\label{prop:nonconvex}
(i) There are instances on which $\mathcal{F}$ is not convex, and on them
$\bar x_0(x,t) \notin \mathcal{F}$ for some $(x,t)$. (ii) Under Assumption~\ref{as:standing}(ii),
for every $(x,t)$,
\begin{equation}
\operatorname{dist}\bigl(\bar x_0(x,t), \mathcal{F}\bigr)^{2} \;\le\; \operatorname{tr}\Sigma_{0|t}(x).
\label{eq:mean-bound}
\end{equation}
\end{proposition}
\begin{proof}
(i) Take an instance with two blocks of equal target area and identical annotations, so that
exchanging the two preserves every constraint. Let $x \in \mathcal{F}$ place them at distinct
centers $p \neq q$ without overlap and let $x'$ be $x$ with the two exchanged, so
$x' \in \mathcal{F}$. Their midpoint puts both at $\tfrac12(p+q)$, where they coincide and
overlap, so $\tfrac12(x+x') \notin \mathcal{F}$. If $p(x_0 \mid x_t = x)$ charges $x$ and $x'$
equally, $\bar x_0(x,t)$ is that midpoint. (ii) Since $x_0 \in \mathcal{F}$ almost surely,
$\operatorname{dist}(\bar x_0, \mathcal{F})^2 \le \lVert \bar x_0 - x_0 \rVert^2$ for almost every draw of
$x_0$ from $p(x_0 \mid x_t = x)$, and the conditional expectation of the right side is
$\operatorname{tr}\Sigma_{0|t}(x)$.
\end{proof}

Part (ii) bounds the defect by the posterior spread, which is zero at $t = 0$ and
$\operatorname{tr}\operatorname{Cov}(x_0)$ at $t = 1$. So a correction is needed even for the
exact denoiser, and only its place is open. The regressors are the extreme case, 41 to 59\% of
the block area claimed twice.

\subsection{Physics in Training}
\label{sec:appendix-theory-prox}

\begin{theorem}\label{thm:prox}
Under Assumption~\ref{as:standing}, for $\lambda \ge 0$ and every $t \in (0, 1]$ with
$\gamma_t \beta < 1$, the head minimizing Eq.~\eqref{eq:aux} is, for $p_t$-almost every $x$,
\begin{equation}
m^{\star}(x,t) \;=\; \operatorname{prox}_{\gamma_t U}\bigl(\bar x_0(x,t)\bigr),
\qquad
\gamma_t = \tfrac12 \lambda t^{2},
\label{eq:prox}
\end{equation}
the prox being single valued by Lemma~\ref{lem:prox}; with the flow loss averaged over the $c$ real latent coordinates of a case, as implemented, the same argument gives $\gamma_t = c\lambda t^{2}/2$; at $t = 0$ the right side extends
continuously to $\bar x_0$ itself, since $\gamma_0 = 0$. The velocity it induces is
\begin{equation}
v_{\mathrm{train}}(x,t) \;=\; v^{*}(x,t) + \tfrac12\lambda t\,\nabla U(m^{\star}),
\qquad
v^{*}(x,t) = \frac{x_t - \bar x_0}{t}.
\label{eq:field}
\end{equation}
\end{theorem}
\begin{proof}
For a head emitting the clean estimate $m$ under the velocity target, Eq.~\eqref{eq:aux} gives
$\hat v - v = -(m - x_0)/t$, hence
$\lVert \hat v - v \rVert^2 = t^{-2}\lVert m - x_0\rVert^2$ \citep{Li_2026_CVPR}. The
terms of Eq.~\eqref{eq:aux} read only $m$, so at $x_t = x$, with
$\mathbb{E}[\lVert m - x_0\rVert^2 \mid x_t = x] = \lVert m - \bar x_0\rVert^2 + \operatorname{tr}\Sigma_{0|t}$,
\begin{equation}
\mathbb{E}\bigl[\,\lVert \hat v - v\rVert^2 + \lambda U(m) \,\big|\, x_t = x\,\bigr]
= \frac{1}{t^{2}}\Bigl(\lVert m - \bar x_0 \rVert^{2} + \operatorname{tr}\Sigma_{0|t}\Bigr)
+ \lambda\,U(m).
\label{eq:pointwise}
\end{equation}
The trace is finite by Assumption~\ref{as:standing}(iii) and constant in $m$, and multiplying by
$t^2/2$ leaves $\tfrac12\lVert m - \bar x_0\rVert^2 + \gamma_t U(m)$, whose unique minimizer is
Eq.~\eqref{eq:prox} by Lemma~\ref{lem:prox}; the pointwise minimizers assemble into the global
one by Assumption~\ref{as:standing}(iii), on a set of full $p_t$-measure. For
Eq.~\eqref{eq:field}, Eq.~\eqref{eq:aux} gives $\hat v = (x_t - m)/t$ for any clean estimate $m$,
and prox optimality gives $\bar x_0 - m^{\star} = \gamma_t \nabla U(m^{\star})$.
\end{proof}

The step is implicit: on a convex quadratic model the operator $(I + \gamma_t H)^{-1}$ contracts
every mode for any $\gamma_t > 0$, whereas an explicit step contracts a mode only for
$\lvert 1 - \eta\lambda_i\rvert < 1$, and it anneals as $t^2$ with no hand-set schedule. The
$t$-scaled arms of Appendix~\ref{sec:appendix-single-terms} add a second factor to one already
present.

\begin{proposition}\label{prop:head}
Let a network $f_\theta$ emit $y = f_\theta(x_t, t)$, let its clean estimate be $m = y$ for the
$x_0$ head and $m = x_t - t\,y$ for the velocity head, and let the auxiliary term act on $m$.
(i) Both heads have the conditional loss of Eq.~\eqref{eq:pointwise} and the minimizer
$m^{\star}$ of Theorem~\ref{thm:prox}.
(ii) The gradient of the auxiliary term with respect to the weights is
$\lambda\,J_{m}^{\!\top}\nabla U(m)$ for the $x_0$ head and
$-\lambda t\,J_{y}^{\!\top}\nabla U(m)$ for the velocity head, with $J$ the head output's
Jacobian in $\theta$. The two coincide when the velocity network satisfies
$f_\theta = (x_t - m_\theta)/t$, that is, reproduces $x_t/t$ at its output on top of the clean
estimate.
(iii) The velocity head's target contains the noise explicitly,
$(x_t - m^{\star})/t = x_1 + \bigl((1-t)x_0 - m^{\star}\bigr)/t$, which equals $x_1 - m^{\star}$
at $t = 1$, whereas the target $m^{\star}$ of the $x_0$ head does not.
\end{proposition}
\begin{proof}
(i) For the velocity head $\hat v = y$ and $v = (x_t - x_0)/t$, so $\hat v - v = (x_0 - m)/t$;
for the $x_0$ head $\hat v = (x_t - m)/t$ gives the same difference, and both losses are the
function of $m$ in Eq.~\eqref{eq:pointwise}. (ii) is the chain rule through $m = y$ and
$m = x_t - t\,y$; the two parameter gradients are equal whenever the two networks implement
the same $m_\theta$. (iii) follows from $x_t = (1-t)x_0 + t x_1$.
\end{proof}

The head does not change what a perfect model learns. Which head a network of finite capacity
trains better is an empirical question, the clean-sample argument of \citet{Li_2026_CVPR}, and
Table~\ref{tab:generator}a measures it without the term: the $x_0$ head has the lowest soft cost
of the three at $\nfe{=}32$ and the only usable one-step estimate.

\begin{proposition}\label{prop:transfer}
Let a model with weights $\theta$ output $f_c(\theta)$ for an input $c$, $J_c = \partial f_c/\partial\theta$,
and let $\ell(\cdot\,; c) \ge 0$ be a differentiable loss that needs no reference sample. The step
$\theta_\eta = \theta - \eta\lambda J_c^{\!\top}\nabla\ell(f_c; c)$ on a training input $c$ changes the
loss on any input $c'$ at rate $-\lambda\,\nabla\ell(f_{c'}; c')^{\!\top} K\, \nabla\ell(f_c; c)$ with
$K = J_{c'}J_c^{\!\top}$ at $\eta = 0$, so $\ell(f_{c'}; c')$ strictly decreases for small $\eta > 0$
when this rate is negative; a loss that needs a reference has no term at an input without one.
\end{proposition}
\begin{proof}
$\nabla_\theta\, \ell(f_c; c) = J_c^{\!\top}\nabla\ell$ by the chain rule, and the derivative of
$\ell(f_{c'}; c')$ along $-\lambda J_c^{\!\top}\nabla\ell(f_c; c)$ is the stated rate.
\end{proof}
In Trinity $c$ is a chip and $\ell = U$; energies summed over block pairs and single blocks, with
weights shared across blocks, are where $K$ aligns the two gradients, a condition the proposition
states and does not establish.

\subsection{Physics in Sampling}
\label{sec:appendix-theory-inloop}

\begin{lemma}\label{lem:tilt}
Let $x_t = (1-t)x_0 + t\,x_1$ with $x_1 \sim \mathcal{N}(0, I)$ independent of $x_0$ and
$t \in (0, 1]$, and let $U$ be measurable with $\mathbb{E}[e^{-U(x_0)}] < \infty$ and
$\mathbb{E}[\lVert x_0\rVert\, e^{-U(x_0)}] < \infty$. With the reweighted posterior mean
\begin{equation}
\tilde x_0(x) = \frac{\mathbb{E}\bigl[x_0\, e^{-U(x_0)} \,\big|\, x_t = x\bigr]}{\mathbb{E}\bigl[e^{-U(x_0)} \,\big|\, x_t = x\bigr]},
\label{eq:tilted-mean}
\end{equation}
for every $x$,
\begin{equation}
\nabla_x \log \mathbb{E}\bigl[e^{-U(x_0)} \,\big|\, x_t = x\bigr]
= \frac{1-t}{t^{2}}\bigl(\tilde x_0(x) - \bar x_0(x)\bigr).
\label{eq:tilt-score}
\end{equation}
\end{lemma}
\begin{proof}
The channel is $p(x \mid x_0) = \mathcal{N}((1-t)x_0, t^2 I)$, so
$\nabla_x \log p(x \mid x_0) = ((1-t)x_0 - x)/t^2$. By Bayes' rule
$\mathbb{E}[e^{-U(x_0)} \mid x_t = x] = \int e^{-U(x_0)}\,p(x \mid x_0)\,p_{\mathrm{data}}(x_0)\,dx_0 \,/\, p_t(x)$
with $p_t(x) = \int p(x \mid x_0)\,p_{\mathrm{data}}(x_0)\,dx_0$. Both integrands are
differentiable in $x$, with
$\lVert \nabla_x p(x \mid x_0)\rVert \le t^{-2}\bigl((1-t)\lVert x_0\rVert + \lVert x\rVert\bigr)\,p(x \mid x_0)$
and $p(x \mid x_0) \le (2\pi t^2)^{-d/2}$, so on a bounded set of $x$ the two moment conditions
supply an integrable dominating function and the derivative passes under both integrals
(dominated convergence). Differentiating the logarithm of the quotient gives
$\nabla_x \log \mathbb{E}[e^{-U(x_0)} \mid x_t = x] = \mathbb{E}_U[\nabla_x \log p(x \mid x_0)] - \mathbb{E}[\nabla_x \log p(x \mid x_0) \mid x_t = x]$,
where $\mathbb{E}_U$ is the expectation under the posterior reweighted by $e^{-U(x_0)}$.
Inserting the score of the channel, the $-x/t^2$ terms cancel and the $x_0$ terms leave
$\tfrac{1-t}{t^2}(\tilde x_0 - \bar x_0)$.
\end{proof}

\paragraph{The substitution error does not vanish with $\Sigma_{0|t}$.}
Take a point mass $x_0 \equiv b$ with $b \neq 0$ and $U(x) = \lVert x\rVert^2$. Then
$\bar x_0 = \tilde x_0 = b$ at every $(x,t)$, so the right side of Eq.~\eqref{eq:tilt-score} is
zero and the exact guidance vanishes, while a loop step $-\nabla U(\bar x_0) = -2b$ does not. The
difference between the clean-space gradient and the exact guidance is therefore not controlled
by $\Sigma_{0|t}$ and does not vanish as $t \to 0$, which is Proposition~\ref{prop:body-loop}(i).

\begin{corollary}\label{thm:bias}
Let $q \propto p_{\mathrm{data}}\,e^{-U}$ with marginals $q_t$, $U$ as in Lemma~\ref{lem:tilt}.
Then
\begin{equation}
\nabla \log q_t(x) - \nabla \log p_t(x)
= \nabla \log \mathbb{E}\bigl[e^{-U(x_0)} \,\big|\, x_t = x\bigr]
= \frac{1-t}{t^{2}}\bigl(\tilde x_0(x) - \bar x_0(x)\bigr).
\label{eq:tilt}
\end{equation}
If moreover $U \in C^{3}$ with bounded third derivatives on the support of $p(x_0 \mid x_t)$ and
$\mathbb{E}[\lVert x_0 - \bar x_0\rVert^{3} \mid x_t] < \infty$, then with
$H = \nabla^2 U(\bar x_0)$ and $\Sigma = \Sigma_{0|t}$,
\begin{equation}
\log \mathbb{E}\bigl[e^{-U(x_0)} \mid x_t\bigr]
= -\Bigl[\,U(\bar x_0) + \tfrac12 \operatorname{tr}(H\Sigma)
- \tfrac12 \nabla U^{\!\top}\!\Sigma\,\nabla U \Bigr]
+ O\bigl(\mathbb{E}[\lVert x_0 - \bar x_0\rVert^3 \mid x_t]\bigr).
\label{eq:bias}
\end{equation}
Moreover $\mathbb{E}[e^{-U(x_0)} \mid x_t] = e^{-U(\bar x_0)}$ for every $U$ of the form
$U(z) = \lVert z - \mu\rVert^2$, $\mu \in \mathbb{R}^d$, if and only if $\Sigma_{0|t} = 0$.
\end{corollary}
\begin{proof}
For Eq.~\eqref{eq:tilt},
$q_t(x) = \int q(x_0)p_t(x \mid x_0)\,dx_0 = Z^{-1}p_t(x)\,\mathbb{E}[e^{-U(x_0)} \mid x_t = x]$;
take logarithms and gradients, and apply Lemma~\ref{lem:tilt}. For Eq.~\eqref{eq:bias}, put
$\Delta = x_0 - \bar x_0$, so $\mathbb{E}[\Delta \mid x_t] = 0$ and $\operatorname{Cov}(\Delta) = \Sigma$.
Taylor's theorem with the third-derivative bound gives
$U(x_0) - U(\bar x_0) = \nabla U^{\!\top}\Delta + \tfrac12 \Delta^{\!\top} H \Delta + r$ with
$\lvert r \rvert \le c\,\lVert\Delta\rVert^3$, the expectation of the exponential is
$1 - \tfrac12\operatorname{tr}(H\Sigma) + \tfrac12 \nabla U^{\!\top}\Sigma\nabla U + O(\mathbb{E}[\lVert\Delta\rVert^3 \mid x_t])$
by the third-moment condition, and $\log(1+u) = u + O(u^2)$. For the equivalence, sufficiency
is immediate. For necessity the single choice $\mu = \bar x_0$ suffices: the identity reads
$\mathbb{E}[e^{-\lVert\Delta\rVert^2} \mid x_t] = 1$, and $1 - e^{-\lVert\Delta\rVert^2} \ge 0$ has
zero expectation, so it vanishes almost surely and $\Delta = 0$ almost surely.
\end{proof}

\paragraph{After sampling and inside it.}\label{cor:refiner}
Descending $U$ after sampling, Use 2 of Section~\ref{sec:method-physics}, acts at $t = 0$, where
$\Sigma_{0|0} = 0$ and the covariance terms of Eq.~\eqref{eq:bias} are absent, a reading
consistent with Lemma~\ref{lem:tilt}; it does not by itself produce a sample of $q$ or a feasible
layout. Descending $U$ inside the sampler at $t > 0$ substitutes a clean-space gradient for the
exact guidance of Lemma~\ref{lem:tilt}.

\begin{proposition}\label{prop:family}
Both schemes produce a velocity $\bigl(x_t - A(\bar x_0)\bigr)/t$, with
\begin{equation}
A_{\mathrm{train}} = (I + \gamma_t \nabla U)^{-1},
\qquad
A_{\mathrm{loop}} = (I - \eta\,\nabla U)^{k},
\label{eq:operators}
\end{equation}
where $A_{\mathrm{train}}$, the resolvent of $\nabla U$, is single valued and equal to
$\operatorname{prox}_{\gamma_t U}$ under Assumption~\ref{as:standing}(i) with
$\gamma_t\beta < 1$ (Lemma~\ref{lem:prox}). The two velocity fields are equal pointwise if and
only if the loop solves $\operatorname{prox}_{\gamma_t U}$ exactly, and if $\nabla U$ is
$L$-Lipschitz they agree to first order in the total step when $k\eta = \gamma_t$ and
$\gamma_t L < 1$, the difference at $y = \bar x_0$ being at most
$L(k\eta)^{2}\lVert\nabla U(y)\rVert\bigl(\tfrac12 e^{k\eta L} + (1 - k\eta L)^{-1}\bigr)$.
\end{proposition}
\begin{proof}
Eq.~\eqref{eq:aux} gives $\hat v = (x_t - m)/t$ for any clean estimate $m$.
Theorem~\ref{thm:prox} identifies $A_{\mathrm{train}}$ as the resolvent, single valued by
Lemma~\ref{lem:prox}, and $A_{\mathrm{loop}}$ is $k$ explicit steps by definition. Equality of
the fields is equality of $A(\bar x_0)$. For the first order, $a = A_{\mathrm{train}}(y)$ satisfies
$a = y - \gamma_t\nabla U(a)$, so the Lipschitz bound gives
$\lVert a - y\rVert \le \gamma_t\lVert\nabla U(y)\rVert + \gamma_t L\lVert a - y\rVert$, hence
$\lVert a - y\rVert \le \gamma_t\lVert\nabla U(y)\rVert/(1 - \gamma_t L)$ and
$\lVert a - (y - \gamma_t\nabla U(y))\rVert \le \gamma_t L\lVert a - y\rVert \le L\gamma_t^{2}\lVert\nabla U(y)\rVert/(1 - \gamma_t L)$.
The loop ends within $\tfrac12 L(k\eta)^{2}\lVert\nabla U(y)\rVert e^{k\eta L}$ of
$y - k\eta\,\nabla U(y)$ by Theorem~\ref{thm:tradeoff}(i), and the triangle inequality adds the
two.
\end{proof}

\begin{lemma}\label{lem:foc}
If $U_{\mathrm{C}}$ is differentiable, $U_{\mathrm{C}} \ge 0$ and $U_{\mathrm{C}}(x^{*}) = 0$,
then $\nabla U_{\mathrm{C}}(x^{*}) = 0$.
\end{lemma}
\begin{proof}
$x^{*}$ is a global minimizer of the differentiable function $U_{\mathrm{C}}$, so its gradient
vanishes there.
\end{proof}

\begin{theorem}\label{thm:tradeoff}
Let $x^{*} \in \operatorname{supp}(p_{\mathrm{data}})$, let $\nabla U$ be $L$-Lipschitz on a ball
around $x^{*}$ that contains the iterates below, and for (ii) assume a local quadratic model
$U \approx \tfrac12 (x-c)^{\!\top} H (x-c)$ near a feasible point with $H \succeq 0$ of
eigenvalues $\lambda_i$, together with the assumption that the null space of $H$ is the tangent
space of $\mathcal{F}$ there, which holds for $U_{\mathrm{C}}$ alone and need not once
$U_{\mathrm{O}}$ is added. Then
\begin{enumerate}[leftmargin=2em, itemsep=1pt, topsep=2pt]
\item $\nabla U(x^{*}) = \nabla U_{\mathrm{O}}(x^{*})$, and $k$ gradient steps of size $\eta$
started at $x^{*}$ end at distance $k\eta\,\lVert \nabla U_{\mathrm{O}}(x^{*})\rVert + R$ from it
with $\lvert R \rvert \le \tfrac12 L\,\lVert \nabla U_{\mathrm{O}}(x^{*})\rVert\,(k\eta)^{2} e^{k\eta L}$,
a displacement that is nonzero to first order whenever the objectives do not vanish at $x^{*}$;
\item for $0 < \eta < 2/\max_i \lambda_i$, mode by mode the operators of Eq.~\eqref{eq:operators}
act as
\begin{equation}
\text{training:}\quad \frac{1}{1 + \gamma_t \lambda_i} \longrightarrow 1 \ \text{ as } t \to 0,
\qquad
\text{in-sampler:}\quad (1 - \eta\lambda_i)^{k}, \ \text{ independent of } t,
\label{eq:modes}
\end{equation}
the explicit factor contracting a mode only when $\lvert 1 - \eta\lambda_i \rvert < 1$;
\item $\lVert m^{\star} - \bar x_0 \rVert^{2} \le 2\gamma_t\,U(\bar x_0)$, which vanishes as
$t \to 0$ wherever $U(\bar x_0(x,t))$ stays bounded, while
$\lVert A_{\mathrm{loop}}(\bar x_0) - \bar x_0 \rVert$ does not.
\end{enumerate}
\end{theorem}
\begin{proof}
For (i), $U_{\mathrm{C}}(x^{*}) = 0$ by Assumption~\ref{as:standing}(ii), so
Lemma~\ref{lem:foc} gives $\nabla U(x^{*}) = \nabla U_{\mathrm{O}}(x^{*}) =: g$. With
$x_{j+1} = x_j - \eta\,\nabla U(x_j)$ and $x_0 = x^{*}$,
$\lVert x_j - x^{*}\rVert \le \eta \sum_{i<j} \bigl(\lVert g\rVert + L\lVert x_i - x^{*}\rVert\bigr)$,
so by induction $\lVert x_j - x^{*}\rVert \le j\eta\lVert g\rVert (1 + \eta L)^{j}$, and
$x_k - x^{*} = -k\eta\, g - \eta \sum_{j<k} \bigl(\nabla U(x_j) - g\bigr)$ with the sum bounded in
norm by $\eta L \sum_{j<k} j\eta \lVert g\rVert (1+\eta L)^{j} \le \tfrac12 L \lVert g\rVert (k\eta)^{2} e^{k\eta L}$.
For (ii), both operators are linear in the quadratic model
and diagonal in the eigenbasis of $H$; $\gamma_t \to 0$ as $t \to 0$ while $\eta$ and $k$ are
fixed. For (iii), prox optimality compares $m^{\star}$ with the candidate $\bar x_0$,
$\tfrac12\lVert m^{\star} - \bar x_0\rVert^2 + \gamma_t U(m^{\star}) \le \gamma_t U(\bar x_0)$,
and $U \ge 0$.
\end{proof}

In the model the stiff modes are the constrained ones, so a fixed loop removes the spread along
the constrained directions at every step including the last and never releases the over-travel of
(i); Section~\ref{sec:exp-refine} measures both, the ported loops leaving an area gap of 10 to 30\%
and every generator ending near one value under the best of them. The quadratic model holds on each
smooth piece of $U$: no term is smoothed, and $U$ has kinks at zero penetration, at contact, at the
bounding-box frontier, at the extreme pins and at the group median. As a hypothesis, the overlap
gradient is the penetration force of one pair while the resolution the data supports depends on the
netlist and the layout, so descent on the non-convex $U$ may reach a nearby local minimum rather
than the plausible one, whereas Theorem~\ref{thm:prox} minimizes the same term jointly with the flow
loss; Appendix~\ref{sec:appendix-theory-measured} is consistent with this reading but does not
establish it.

\begin{proposition}\label{prop:support}
Let the learned head be $\bar x_0 + \varepsilon_\theta$ with
$\mathbb{E}_{x \sim p_t}\lVert \varepsilon_\theta \rVert^2 \le \delta^2$, and let $p_t^{g}$, the
law of the states a guided sampler visits, be absolutely continuous with respect to $p_t$ with an
essentially bounded density ratio. Then
\begin{equation}
\mathbb{E}_{x \sim p_t^{g}}\lVert \varepsilon_\theta \rVert^{2} \le \chi_t\,\delta^{2},
\qquad
\chi_t = \operatorname*{ess\,sup}_x \frac{p_t^{g}(x)}{p_t(x)} < \infty,
\qquad\text{and}\qquad
\frac{\partial \bar x_0}{\partial x_t} = \frac{(1-t)\,\Sigma_{0|t}}{t^{2}} \xrightarrow[t \to 1]{} 0 .
\label{eq:shift}
\end{equation}
\end{proposition}
\begin{proof}
The bound is the importance-weight inequality. Tweedie's formula for the channel
$x_t \mid x_0 \sim \mathcal{N}((1-t)x_0, t^2 I)$ gives
$\bar x_0 = (x_t + t^2 \nabla \log p_t)/(1-t)$ and hence
$\partial \bar x_0/\partial x_t = (I + t^2 \nabla^2 \log p_t)/(1-t)$, while the same channel gives
$(1-t)^2 \Sigma_{0|t} = t^2 (I + t^2 \nabla^2 \log p_t)$. Eliminating the Hessian gives the stated
form, and $(1-t) \to 0$ with $\Sigma_{0|t}$ bounded.
\end{proof}

Training with the loop in place would remove the support mismatch, at the cost of running the
loop at every training step.

\subsection{What Is Predicted and What Is Measured}
\label{sec:appendix-theory-check}

Two limits remain. Training with the terms does not sample the reweighted law of
Corollary~\ref{thm:bias} either; the claim is asymmetric, the loop's clean-space gradient missing the
exact guidance of Lemma~\ref{lem:tilt} by an amount $\Sigma_{0|t}$ does not control, while
Theorem~\ref{thm:prox} characterizes the trained head exactly at no sampling cost. The results
concern where a correction is applied, not whether the six quantities are right, which
Section~\ref{sec:exp-generator} measures.

\section{The Closed-Form Refiner}
\label{sec:appendix-refiner}

\subsection{Derivatives}
\label{sec:appendix-refiner-derivatives}

The refiner works in normalized units ($s = 1$) on the latent $z = (u, v, \rho)$ with $c_x = u s$, $c_y = v s$, $w = \sqrt{a}\,e^{\rho/2}$, $h = \sqrt{a}\,e^{-\rho/2}$. Every term is written on the box edges $L = c_x - w/2$, $R = c_x + w/2$, $B = c_y - h/2$, $T = c_y + h/2$, its derivative is taken with respect to $(L, B, R, T)$ and folded,
\begin{equation}
\frac{\partial E}{\partial c_x} = \frac{\partial E}{\partial L} + \frac{\partial E}{\partial R},\quad
\frac{\partial E}{\partial w} = \tfrac12\Bigl(\frac{\partial E}{\partial R} - \frac{\partial E}{\partial L}\Bigr),\quad
\frac{\partial E}{\partial \rho} = \frac{w}{2}\frac{\partial E}{\partial w} - \frac{h}{2}\frac{\partial E}{\partial h},
\end{equation}
and likewise in $y$. Position-frozen blocks (preplaced) get $\partial E/\partial c_x = \partial E/\partial c_y = 0$, shape-frozen blocks (fixed, preplaced) get $\partial E/\partial w = \partial E/\partial h = 0$, padding gets zero, and $\partial E/\partial\rho = 0$ where $|\rho|$ sits on the clamp at 4. The six derivatives:
\begin{itemize}[leftmargin=*,nosep]
\item \textbf{overlap}, $E = \sum_{i<j} o_x o_y / s^2$ with $o_x = \max(0, \min(R_i, R_j) - \max(L_i, L_j))$: a pair with $o_x o_y > 0$ gives $\partial E/\partial R_i = o_y [R_i < R_j]/s^2$ and $\partial E/\partial L_i = -o_y[L_i > L_j]/s^2$, and the same in $y$ with $o_x$. The force on a block is the overlap length of the crossing edge, and a pair at exact contact feels none.
\item \textbf{cluster}, $E = \sum_p \sum_{(i,j)\in\mathrm{MST}_p}\mathrm{gap}_{ij}/s$ with $\mathrm{gap} = \max(g_x, g_y)$, $g_x = \max(0, \max(L_i, L_j) - \min(R_i, R_j))$: the tree is recomputed every step by batched Prim's algorithm, and every tree edge with a positive gap pulls its endpoints together along the active axis with unit force, the right block through $\partial E/\partial L = +1/s$, the left through $\partial E/\partial R = -1/s$, both widening toward each other through $\partial E/\partial w = -1/(2s)$.
\item \textbf{MIB}, $E = \sum_q \sum_{i\in q} |\rho_i - \tilde\rho_q|$ with the median detached: $\partial E/\partial \rho_i = \mathrm{sign}(\rho_i - \tilde\rho_q)$, zero for the median member. A pure shape force.
\item \textbf{boundary}, $E = \sum_{i\in B}\sum_{e\in\mathrm{code}(i)} (\text{side} - \text{edge})/s$: a left-coded block gets $\partial E/\partial L = +1/s$ and the block that owns $x_{\min}$ gets $\partial E/\partial L \mathrel{-}= 1/s$ per left-coded block, symmetric for the other sides. The coded block moves to the outline and the outline's owner moves toward the coded block.
\item \textbf{wirelength}, $E = H/(s W)$ with $H = \tfrac12\sum_{ij} A_{ij}|c_i - c_j|_1 + \sum_p w_p |c_{b_p} - q_p|_1$: $\partial E/\partial c_{x,i} = [\sum_j \bar A_{ij}\,\mathrm{sign}(c_{x,i} - c_{x,j}) + \sum_{p\to i} w_p\,\mathrm{sign}(c_{x,i} - q^x_p)]/(sW)$ with $\bar A = \tfrac12(A + A^{\top})$, no shape force.
\item \textbf{area}, $E = (x_{\max} - x_{\min})(y_{\max} - y_{\min})/s^2$: only the four outline owners move, $\partial E/\partial L_{\arg\min L} = -H_{\mathrm{bb}}/s^2$, $\partial E/\partial R_{\arg\max R} = +H_{\mathrm{bb}}/s^2$, and likewise in $y$ with $W_{\mathrm{bb}}$.
\end{itemize}
In latent units the cluster, boundary and MIB forces are exactly 1 per active coordinate, wirelength is $A_{ij}/W$, area is $H_{\mathrm{bb}}/s$ and overlap is $o_y/s$, of order $0.01$. The closed form equals autograd of the batched terms to $10^{-6}$ in float64, per term and summed, on the test suite of the supplementary code.

\subsection{The Step}
\label{sec:appendix-refiner-step}

Adam with $\beta_1 = 0$: $v \leftarrow \beta_2 v + (1 - \beta_2) g^2$, $z \leftarrow z - \mathrm{lr}_t\, g / (\sqrt{v/(1 - \beta_2^t)} + \epsilon)$. Without momentum a block that has just been pushed out of an overlap stops when the force stops instead of coasting into the next block, and the per-coordinate normalization lets each term move its block at the full step whatever its magnitude, so the weights set the priority where terms compete on one coordinate. The learning-rate vector is precomputed from the schedule, the overrides are re-applied after every step, and $K$ steps run inside one compiled function. \kb{At 400 steps the refiner takes about one millisecond per layout: 2.6~$\mu$s per step per layout with 1,000 layouts in a batch on one RTX 5090, wall-clock with the device synchronized, the median of five timed repeats after a warm-up (Table~\ref{tab:matched_time}), and 20~ms per layout at 32 per batch.}

\subsection{Choosing the Setting}
\label{sec:appendix-refiner-setting}

The setting was chosen on Trinity and on Trinity at $\lambda{=}0$, on the same 12k cases, by the mean soft cost over the ten refined cells of $\nfe\in\{1, 32\}\times$ steps $\in\{10, 25, 50, 100, 200\}$ (Table~\ref{tab:refiner_schedule}). A decaying schedule with peak $10^{-3}$ is the operating point on both models, linear and cosine are the same within 0.005 at every cell, a constant learning rate over-refines (overlap turns up after 50 steps at $10^{-3}$), peak $3\times10^{-3}$ pays a transient at 10 steps, and $10^{-4}$ is too slow at every budget. Optimizer cross-checks at the pick: Adam $(0, 0.9)$ 1.294 / 1.340, a 10-step warmup 1.291 / 1.293, plain gradient steps 1.564 / 1.578 at $10^{-3}$ and 1.964 / 1.967 at $3\times10^{-3}$. Plain steps do not refine, because the overlap force per coordinate is a hundredth of the unit forces of the constraint terms, and Adam's normalization is what lets the six terms act at one step size.

\begin{table}[htbp]\centering\scriptsize
\caption{Refiner schedule and peak learning rate, all weights 1, betas $(0, 0.99)$, on Trinity and on Trinity at $\lambda{=}0$: the rule (the mean soft cost over the ten refined cells of $\nfe\in\{1, 32\}\times$ steps $\in\{10, 25, 50, 100, 200\}$), and for the six schedules run further the soft cost / overlap ratio at $\nfe{=}32$ after 100 and 200 steps (raw: 1.5137 / 0.0225 and 2.0151 / 0.0222).}
\label{tab:refiner_schedule}
\label{tab:refiner_schedule_cells}
\setlength{\tabcolsep}{5pt}
\begin{tabular}{lrrrrrr}
\toprule
Schedule, peak rate & Rule, Trinity & Rule, $\lambda{=}0$ & Trinity, 100 steps & Trinity, 200 steps & $\lambda{=}0$, 100 steps & $\lambda{=}0$, 200 steps \\
\midrule
linear $10^{-3}$ & \textbf{1.2698} & \textbf{1.2688} & 1.1625 / 0.0105 & 1.1617 / 0.0118 & 1.1549 / 0.0096 & 1.1506 / 0.0107 \\
cosine $10^{-3}$ & 1.2747 & 1.2713 & 1.1578 / 0.0105 & 1.1581 / 0.0117 & 1.1501 / 0.0095 & 1.1471 / 0.0107 \\
linear $3\times10^{-4}$ & 1.3182 & 1.3750 & 1.1786 / 0.0110 & 1.1594 / 0.0108 & 1.1755 / 0.0097 & 1.1486 / 0.0094 \\
cosine $3\times10^{-4}$ & 1.3188 & 1.3746 & -- & -- & -- & -- \\
constant $3\times10^{-4}$ & 1.3370 & 1.3597 & 1.2347 / 0.0120 & 1.2356 / 0.0131 & 1.2230 / 0.0104 & 1.2193 / 0.0115 \\
linear $3\times10^{-3}$ & 1.3529 & 1.3349 & 1.1910 / 0.0112 & 1.1876 / 0.0131 & 1.1859 / 0.0106 & 1.1801 / 0.0125 \\
cosine $3\times10^{-3}$ & 1.3642 & 1.3450 & -- & -- & -- & -- \\
constant $10^{-4}$ & 1.3765 & 1.4699 & -- & -- & -- & -- \\
linear $10^{-4}$ & 1.4247 & 1.5819 & -- & -- & -- & -- \\
cosine $10^{-4}$ & 1.4246 & 1.5818 & -- & -- & -- & -- \\
constant $10^{-3}$ & 1.4953 & 1.4869 & 1.4338 / 0.0161 & 1.4583 / 0.0186 & 1.4166 / 0.0150 & 1.4406 / 0.0175 \\
constant $3\times10^{-3}$ & 2.1682 & 2.1433 & -- & -- & -- & -- \\
\bottomrule
\end{tabular}
\end{table}

\begin{table}[htbp]\centering\footnotesize
\caption{Weights and variants at the pick, by the rule of Table~\ref{tab:refiner_schedule}.}
\label{tab:refiner_weights}
\setlength{\tabcolsep}{12pt}
\begin{tabular}{lrrlrr}
\toprule
Arm & Trinity & $\lambda{=}0$ & Arm & Trinity & $\lambda{=}0$ \\
\midrule
overlap $\times10$ & \textbf{1.2215} & \textbf{1.2332} & all weights 1 & 1.2698 & 1.2688 \\
overlap $\times30$ & 1.2270 & 1.2471 & cluster $\times3$ & 1.2936 & 1.2938 \\
overlap $\times3$ & 1.2434 & 1.2476 & squared cluster gap & 1.2938 & 1.3000 \\
wirelength $\times0.3$ & 1.2604 & 1.2599 & boundary $\times0.3$ & 1.3003 & 1.2972 \\
cluster $\times0.3$ & 1.2616 & 1.2590 & area $\times3$ & 1.3033 & 1.2927 \\
area $\times0.3$ & 1.2624 & 1.2654 & wirelength $\times3$ & 1.3368 & 1.3346 \\
boundary $\times3$ & 1.2664 & 1.2661 & overlap $\times0.3$ & 1.3629 & 1.3592 \\
MIB $\times0.3$ or $\times3$ & 1.2698 & 1.2689 & contact margin $0.005\,s$ & 1.4212 & 1.4045 \\
\bottomrule
\end{tabular}
\end{table}

\begin{table}[htbp]\centering\footnotesize
\caption{The overlap-weight ladder at $\nfe{=}32$ on Trinity after 100 and 200 steps.}
\label{tab:refiner_overlap_ladder}
\setlength{\tabcolsep}{15pt}
\begin{tabular}{lrrrrr}
\toprule
Overlap weight & Steps & Soft cost & Overlap & HPWL gap & Area gap \\
\midrule
$\times1$ & 100 & 1.1625 & 0.0105 & $-$0.0058 & 0.0060 \\
$\times3$ & 100 & 1.1344 & 0.0079 & $-$0.0002 & 0.0087 \\
$\times10$ & 100 & 1.1114 & 0.0048 & 0.0040 & 0.0137 \\
$\times30$ & 100 & 1.1174 & 0.0032 & 0.0076 & 0.0207 \\
$\times1$ & 200 & 1.1617 & 0.0118 & $-$0.0094 & $-$0.0006 \\
$\times3$ & 200 & 1.1153 & 0.0077 & $-$0.0017 & 0.0048 \\
$\times10$ & 200 & \textbf{1.0833} & \textbf{0.0035} & 0.0038 & 0.0128 \\
$\times30$ & 200 & 1.0897 & 0.0017 & 0.0088 & 0.0228 \\
\bottomrule
\end{tabular}
\end{table}

\begin{figure}[htbp]\centering
\includegraphics[width=\linewidth]{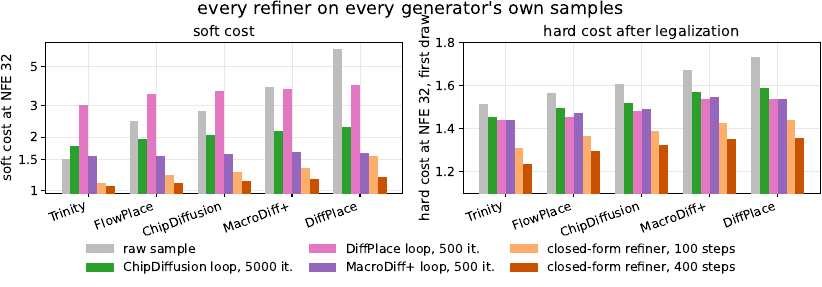}
\caption{Generator $\times$ refiner at $\nfe{=}32$ with every ported loop on every generator's own samples, the full version of Table~\ref{tab:genref}. Left: soft cost, where the DiffPlace loop makes every generator worse than its raw sample. Right: the hard cost of the first legalized draw (Table~\ref{tab:app_gen_ref_hard}), where the same loop is the best of the three.}
\label{fig:app_g4}
\end{figure}

After the soft slots close, the one conflict left is overlap against wirelength and area, and the overlap weight is the one change that moves the endpoint (Tables~\ref{tab:refiner_weights} and \ref{tab:refiner_overlap_ladder}): $\times10$ takes overlap from 0.0118 to 0.0035 and the soft cost from 1.162 to 1.083 at 200 steps, $\times30$ reaches 0.0017 at a larger wirelength and area price. The MIB weight has no effect because the MIB variance is $10^{-11}$ by 25 steps at any weight, a contact margin keeps every touching pair under force and inflates the layout, and a squared gap weakens the cluster pull near contact. The ranking is the same on both models, and combinations with overlap $\times10$ move nothing beyond the single change. The setting of every refined number in the paper is therefore: Adam $(0, 0.99)$, linear decay from $10^{-3}$, overlap 10, the other five terms 1.

\subsection{The Prior Arts' Loops}
\label{sec:appendix-refiner-ports}

Figure~\ref{fig:app_g4} draws every loop on every generator's own samples, and Table~\ref{tab:ports_soft} is the full soft vector of every loop on the $\nfe{=}32$ samples of Trinity (draw 0 of every case, the \method\ row of Table~\ref{tab:genref}); on the samples of Trinity at $\lambda{=}0$ every column moves the same way. Every port removes overlap harder than the closed-form refiner at 100 steps and pays in every other column: the die-boundary term lets the blocks fill the canvas, the area gap lands at $+0.09$ to $+0.16$, wirelength at $+2$ to $+6\%$, and the soft rules the loops do not model get worse than raw; the MIB variance and the immutable-block deviation are unchanged because the loops cannot move shapes. Without the 5\% canvas margin the ChipDiffusion loop stops removing overlap (0.0225 to 0.0148 in 5,000 iterations).

\begin{table}[htbp]\centering\footnotesize
\caption{The prior-art loops on the $\nfe{=}32$ samples of Trinity: the ChipDiffusion loop in its standard, scheduled and Adam-with-dual-ascent variants (5,000 iterations), the DiffPlace and MacroDiff+ loops (500), and the closed-form refiner at all weights 1. Units as in Section~\ref{sec:exp-setup}.}
\label{tab:ports_soft}
\setlength{\tabcolsep}{4.5pt}
\begin{tabular}{lrrrrrrr}
\toprule
Refiner & Overlap & Boundary & Group & HPWL gap & Area gap & $\mathrm{sev}_{\mathrm{rel}}$ & Soft cost \\
\midrule
raw & 0.0225 & 0.8 & 2.4 & $-$0.0008 & 0.0022 & 0.0851 & 1.5131 \\
closed-form, 100 steps & 0.0105 & 0.2 & 0.9 & $-$0.0058 & 0.0060 & 0.0152 & 1.1625 \\
ChipDiffusion, standard & 0.0016 & 2.8 & 6.6 & 0.0470 & 0.1052 & 0.2289 & 1.7472 \\
ChipDiffusion, scheduled & 0.0016 & 3.0 & 6.5 & 0.0449 & 0.1046 & 0.2420 & 1.7908 \\
ChipDiffusion, dual ascent & 0.0096 & 12.0 & 4.2 & 0.0171 & 0.0940 & 0.4010 & 2.6232 \\
DiffPlace & 0.0009 & 18.3 & 4.6 & 0.0379 & 0.1588 & 0.4945 & 3.0349 \\
MacroDiff+ & 0.0014 & 1.7 & 6.4 & 0.0491 & 0.1049 & 0.1765 & 1.5649 \\
\bottomrule
\end{tabular}
\end{table}

\begin{figure}[htbp]\centering
\includegraphics[width=\linewidth]{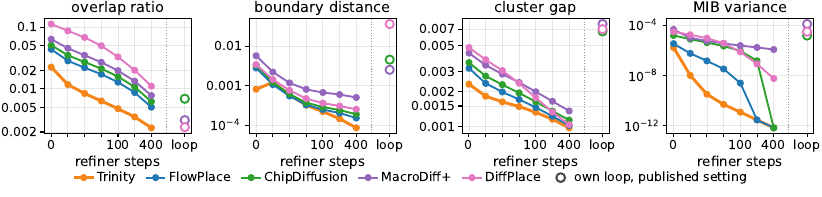}
\caption{The four constraint quantities at $\nfe{=}32$ along the closed-form refiner for the five generators, and each ported placer's own post-hoc loop at its published setting as a hollow marker (ChipDiffusion 5,000 iterations, MacroDiff+ and DiffPlace 500). The loops match the refiner on overlap and leave the other three quantities one to four orders of magnitude above it.}
\label{fig:app_constraints_loops}
\end{figure}

Figure~\ref{fig:app_g6} places every loop against the refiner budget in both costs. On the prior arts' own samples the loops end at soft cost / overlap: FlowPlace 1.961 / 0.0055 (ChipDiffusion loop), 3.497 / 0.0009 (DiffPlace loop), 1.577 / 0.0027 (MacroDiff+ loop); ChipDiffusion 2.048 / 0.0069 (own), 3.614 / 0.0008, 1.602 / 0.0032; DiffPlace 2.286 / 0.0124, 3.914 / 0.0024 (own), 1.633 / 0.0044; MacroDiff+ 2.181 / 0.0075, 3.755 / 0.0010, 1.661 / 0.0031 (own). On these samples the area gap after the loop is $+0.11$ to $+0.13$ under the ChipDiffusion and MacroDiff+ loops and $+0.22$ to $+0.30$ under the DiffPlace loop.

\section{The Legalizer}
\label{sec:appendix-legalizer}

Legalization reads a sequence-pair relation for every pair of blocks from the refined layout (the axis of least penetration, ordered by center) and solves one linear program that packs the blocks under those relations with preplaced blocks pinned, cluster members abutting along a spanning tree and boundary-coded edges on the outline, minimizing the outline plus the displacement from the refined positions \citep{DBLP:journals/tcad/MurataFNK96, DBLP:conf/ispd/MurataK98}. When the program is infeasible the soft terms become penalties and the relations of the conflicting preplaced pairs are flipped. A second program then lets soft blocks change aspect at fixed area under the same relations, kept when the checker scores it better, and the smallest uniform expansion that separates every pair is reported as the price of residual overlap. Table~\ref{tab:legalizer_steps} lists the eight steps.

\begin{table}[htbp]\centering\footnotesize
\caption{The legalizer, one stage of eight steps. Every step is a one-line rule or a standard constraint-graph method; the program of step E and the knobs are below.}
\label{tab:legalizer_steps}
\begin{tabular}{cp{1.5cm}p{8.4cm}p{2.5cm}}
\toprule
& Step & Rule & Kind \\
\midrule
A & anchors & pre-placed blocks to their target box, fixed blocks to their shape, soft blocks to their exact area, an all-soft MIB group to one shape & rule \\
B & expansion & every free block moved outward from the layout centre by the smallest $\lambda \geq 1$ (cap 10) at which no pair overlaps; $\lambda$ recorded per draw & rule \\
C & relations & per pair, the axis of larger separation (least penetration if overlapping) and the order along it, read from the unexpanded layout & sequence pair \citep{DBLP:journals/tcad/MurataFNK96}, TCG \citep{DBLP:conf/dac/LinC01} \\
C$'$ & rewrites & a block ordered between an abutting pair, or outside a coded block on its coded side, re-ordered on the other axis; pinned pairs never rewritten & rule \\
D & abutment tree & per cluster a spanning tree over member pairs ranked by gap, missing perpendicular overlap, blocks in between and pre-placed collisions & rule \\
E & packing LP & one linear program over positions at fixed shapes & LP packing with pre-placed modules \citep{DBLP:conf/ispd/MurataK98} \\
F & ladder & on infeasibility: abutment and boundary as penalties, then the violated pinned relations flipped (six rounds), then every free block above the pre-placed ones & rule \\
G & contact snap & tree pairs within $10^{-6}$ of contact set to exact contact, overlap-checked & rule \\
H & aspect pass & one more program with every soft non-MIB shape free within 1.3 per side at fixed area, kept when the checker scores it feasible and cheaper & sequential LP \\
\bottomrule
\end{tabular}
\end{table}

\paragraph{The program (step E).}
Variables: $X_i, Y_i$ for every block (fixed to the target for preplaced blocks), the outline $L, R, B, T$, and non-negative auxiliaries for absolute values and slacks. Constraints: outline $L \leq X_i$, $X_i + w_i \leq R$, $B \leq Y_i$, $Y_i + h_i \leq T$; for every relation $(\text{axis}, \text{lo}, \text{hi})$, $P_{\mathrm{lo}} + \text{size}_{\mathrm{lo}} \leq P_{\mathrm{hi}}$ on that axis; for every tree edge, hard $P_{\mathrm{lo}} + \text{size}_{\mathrm{lo}} = P_{\mathrm{hi}}$ with perpendicular overlap at least $\varepsilon$ on both sides, or soft with the absolute deviation penalized; for every coded block $X_i = L$, $X_i + w_i = R$, $Y_i = B$ or $Y_i + h_i = T$ per code bit, or the deviation penalized, with a preplaced block's code dropped when another preplaced block makes it impossible. Objective, minimized: $(R - L)\,H_0/A_0 + (T - B)\,W_0/A_0$ (the outline, linearized at the input bounding box $W_0\times H_0 = A_0$) $+\ \mu/(n s)\sum_i (|X_i - \hat x_i| + |Y_i - \hat y_i|)$ (displacement from the restored positions) $+\ (\text{penalty}/s)\times$ soft slacks $+\ (1{,}000/s)\times$ probe slacks. Size: $n^2/2$ relation rows plus $O(n)$ others, at most about 7,200 rows at $n = 120$; one solve takes a few to about 50 milliseconds, and a draw takes 240 to 300 milliseconds end to end including every ladder rung it needs, of which the solver is 87 to 99 milliseconds.

\paragraph{The ladder (step F).}
\kb{The program of step E runs first with the abutment and boundary constraints hard, then with them as penalties, and when that program is infeasible the flip loop runs at most six rounds of two programs each. The first is a probe: the relation of every free and pre-placed pair carries a probe slack, and the pairs whose slack the solution uses are the pinned relations the layout cannot hold. The second is the packing program with those relations moved to the other axis, ordered as the probe placed the two blocks, and every pinned relation hard again. A packing program that places without overlap ends the loop with its layout. When the six rounds are spent, or a probe is infeasible or uses no slack, the flipped relations are discarded and only the pre-placed blocks of the pairs the probes flagged are kept: the next rung reads the relations from the layout again, without the rewrites of step C$'$, and orders every free block above each of those blocks, and the last rung orders every free block above every pre-placed block, which is feasible by construction.}

\paragraph{The snap (step G).}
\kb{The snap is a gate on single-block moves. For every tree pair whose gap is within $10^{-6}$ of contact, the checker's tolerance, the later block along the pair's axis is moved onto the edge of the earlier one, and the move is kept only when the moved block overlaps no other block, over at most three sweeps of the pairs or until a sweep keeps no move. The layout enters feasible, each kept move displaces one block by at most that tolerance and is admitted only when the block overlaps none of the others, and no shape or other pair changes, so the layout is feasible after every kept move and at exit.}

\paragraph{Knobs.}
Relations read from the restored layout (reading them from the expanded layout costs $+0.23$ in outline area because the larger $\lambda$ flips the axis on many pairs), $\mu = 0.5$ and penalty $10$ (0.1 / 0.5 / 2 and 10 / 100 are identical at three decimals), $\varepsilon = 10^{-3}$, $\lambda$ cap 10, six flip rounds (three to six moved 6 of 960 draws off the last rung on the tuning subset), aspect pass at trust 1.3 (off: best-of-4 1.103 on the subset; 1.3: 1.066, kept on 57\% of draws; 1.6: 1.071).

\paragraph{Residual analysis.}
On the 240-case tuning subset at 400 refiner steps, before the aspect pass, the hard program is feasible on 24\% of draws at cost 1.02, the penalized program on 57\% at 1.2, pinned flips on 15\%, and the last rungs on 4\% at about 3. The residual violations of the penalized rung on 80 draws are 95 boundary counts (14 on preplaced blocks, 20 with a preplaced block already beyond the coded edge, 38 corner codes, 69 within 2\% of $s$ of the edge) and 142 cluster counts (10 with a preplaced member, 133 with the nearest member within 2\% of $s$), near-misses a position-only program under fixed relations cannot close, which is what the aspect pass addresses: with it, the counted violations per draw fall from $1.4 + 1.5$ to $1.0 + 1.2$, the area gap from $+0.117$ to $+0.081$, and the best-of-4 cost from 1.103 to 1.066.

\paragraph{Rejected designs.}
Relations read from the expanded layout ($+0.23$ to $+0.34$ in outline area); expansion that moves preplaced blocks too (relations conflict with the true pinned positions, 22\% of draws on the last rung); a re-read of all relations from the penalized solve instead of flipping the violated pairs (25\% on the last rung); rewrites of pinned-pair relations (infeasible on 5\% of draws); a tighter position box of $\pm 3s$ (the last rung infeasible on 2.5\%); and the competition-era portfolio whose overlap removal is a 200-round pairwise push that converges on 25\% of the refiner's draws and falls to a shelf pack on 37\%, at a single-draw mean of 2.21 and 90th percentile 7.7 on the full 12k against 1.23 and 1.55 for the stage above.

\makeatletter
\setlength{\@fptop}{0pt}\setlength{\@fpsep}{\floatsep}\setlength{\@fpbot}{0pt plus 1fil}
\makeatother
\section{Measuring the Propositions}
\label{sec:appendix-theory-measured}
\suppressfloats[t]

Every figure and table here is an inference on the trained checkpoints over the dev split: the raw $\hat x_0$ of eight checkpoints on the same $x_t$ (Proposition~\ref{prop:body-prox}), the reference layouts through every correction loop (Proposition~\ref{prop:body-loop}; no loop moves a pre-placed block), layout snapshots of three cases, GSRC with three checkpoints, and the loops' time per step (2,000 dev cases padded to 120 blocks, 100 steps, the median of five repeats on one idle RTX 5090). In the prediction $\gamma_t\|\nabla U\|$, $\gamma_t = c\lambda t^{2}/2$ with $c$ the mean number of real latent coordinates of a case, three per block: the flow loss is averaged over these coordinates and the energy is not, so the pointwise minimizer of Theorem~\ref{thm:prox} carries the factor $c$. The prediction is the population optimum. The measured excess of \method\ over the seed floor agrees with it in sign and rises with $t$, but its ratio to the prediction runs from about $10^{2}$ at $t{=}0.001$ to about $10^{-3}$ at $t{=}1$ (Table~\ref{tab:theory_prox}): a network of finite capacity does not reach the pointwise minimizer where $\gamma_t$ is large, so the check confirms the direction and the growth with $t$, not the magnitude. On GSRC n100 (Table~\ref{tab:theory_gsrc}) the share of draws that fit the outline differs across checkpoints of one recipe, 0.36, 0.80 and 0.98, while the best fitting wirelength stays within 1\%; the second seed of \method\ fits 98\% of its draws, so the gap is a checkpoint-to-checkpoint variance of the raw outline on netlists outside the training distribution rather than an effect of the physics term.

\begin{figure}[htbp]
\subsection{The Proximal Correction Against $t$}
\label{sec:appendix-prox}
\centering
\includegraphics[width=\linewidth]{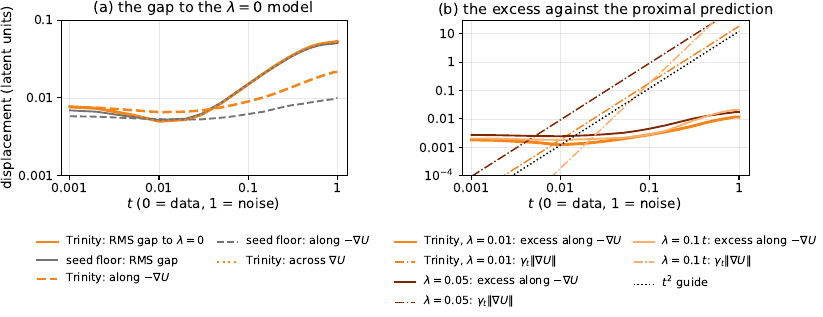}
\caption{The raw $\hat x_0$ of \method\ against its backbone at $\lambda{=}0$ on the same $x_t$, 2,000 cases. \textbf{(a)} The RMS gap, its component along $-\nabla U$ and across it, and the seed floor (two $\lambda{=}0$ seeds against the third). \textbf{(b)} The excess of the along-component over the floor for three physics weights, each with its proximal prediction $\gamma_t\|\nabla U\|$, $\gamma_t=\lambda t^2/2c$, and a $t^2$ guide.}
\label{fig:app_prox}
\end{figure}

\begin{table}[htbp]\centering\footnotesize
\caption{\method\ against its $\lambda{=}0$ backbone on the same $x_t$, 2,000 cases, latent units, one column per $t$: the RMS gap and its part along $-\nabla U$ for the seed floor (two $\lambda{=}0$ seeds against the third) and for \method, the excess over the floor and its part across $\nabla U$, the prediction with $\gamma_t=c\lambda t^2/2$, their ratio, and the excess at two other weights.}
\label{tab:theory_prox}
\setlength{\tabcolsep}{4pt}
\begin{tabular}{lrrrrrrr}
\toprule
$t$ & 0.001 & 0.005 & 0.01 & 0.03 & 0.1 & 0.3 & 1 \\
\midrule
RMS gap, seed floor & 0.0069 & 0.0057 & 0.0052 & 0.0062 & 0.0150 & 0.0337 & 0.0510 \\
RMS gap, \method & 0.0077 & 0.0060 & 0.0050 & 0.0061 & 0.0150 & 0.0344 & 0.0533 \\
along $-\nabla U$, seed floor & 0.0058 & 0.0055 & 0.0053 & 0.0053 & 0.0062 & 0.0080 & 0.0100 \\
along $-\nabla U$, \method & 0.0076 & 0.0070 & 0.0066 & 0.0069 & 0.0090 & 0.0136 & 0.0216 \\
excess of \method\ over the floor & 0.0018 & 0.0015 & 0.0012 & 0.0016 & 0.0028 & 0.0056 & 0.0116 \\
across $\nabla U$, \method & 0.0077 & 0.0060 & 0.0050 & 0.0061 & 0.0150 & 0.0343 & 0.0533 \\
prediction $\gamma_t\|\nabla U\|$ & 1.9e-05 & 0.00047 & 0.0019 & 0.017 & 0.19 & 1.7 & 19 \\
ratio, excess over prediction & 97 & 3.2 & 0.65 & 0.096 & 0.015 & 0.0033 & 0.00063 \\
excess at $\lambda{=}0.05$ & 0.0027 & 0.0025 & 0.0024 & 0.0028 & 0.0046 & 0.0101 & 0.0176 \\
excess at $\lambda{=}0.1\,t$ & 0.0021 & 0.0019 & 0.0018 & 0.0020 & 0.0028 & 0.0085 & 0.0210 \\
\bottomrule
\end{tabular}
\end{table}

\begin{figure}[htbp]
\subsection{Reference Layouts Through Every Loop}
\label{sec:appendix-reference-loops}
\centering
\includegraphics[width=\linewidth]{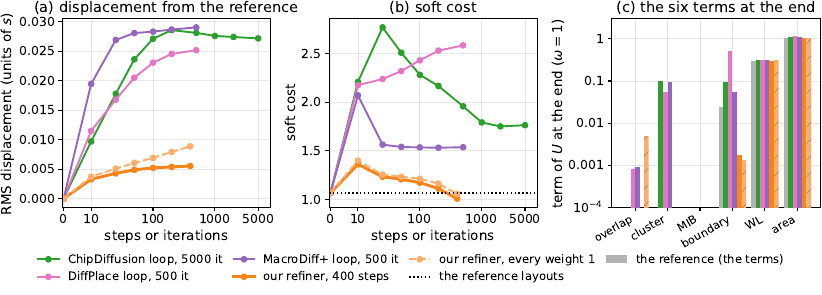}
\caption{The 12,000 reference layouts through every correction. \textbf{(a)} RMS displacement of the movable blocks against steps. \textbf{(b)} Soft cost against steps, the reference layouts dotted. \textbf{(c)} The six terms of $U$ at the end, every weight 1.}
\label{fig:app_reference_loops}
\end{figure}

\begin{table}[htbp]\centering
\caption{The 12,000 reference layouts through each correction, means at the end: RMS and largest displacement of movable blocks ($s$ units), soft cost and share of cases whose soft cost rose, overlap, boundary distance and cluster gap ($10^{-3}$), wirelength and area gaps in percent of the reference.}
\label{tab:theory_reference}
\scriptsize\setlength{\tabcolsep}{5.5pt}
\begin{tabular}{lrrrrrrrrrr}
\toprule
Correction & Steps & RMS & Max & Soft cost & Rose & Overlap & Boundary & Cluster & WL & Area \\
\midrule
reference layouts & -- & -- & -- & 1.065 & -- & 0.0000 & 0.75 & 0.00 & $-0.0$ & $+0.0$ \\
\midrule
ChipDiffusion loop & 5,000 & 0.0272 & 0.0497 & 1.762 & 100\% & 0.0001 & 2.92 & 5.15 & $+4.0$ & $+10.2$ \\
DiffPlace loop & 500 & 0.0252 & 0.0576 & 2.584 & 100\% & 0.0008 & 14.41 & 2.25 & $+2.8$ & $+12.5$ \\
MacroDiff+ loop & 500 & 0.0290 & 0.0560 & 1.536 & 100\% & 0.0009 & 1.73 & 4.57 & $+4.4$ & $+10.2$ \\
\rowcolor{ourscolor} \method's refiner & 400 & \textbf{0.0055} & 0.0246 & \textbf{1.007} & 15\% & 0.0000 & 0.05 & 0.00 & $-0.2$ & $+0.0$ \\
\method's refiner, $\omega{=}1$ & 400 & 0.0089 & 0.0315 & 1.055 & 36\% & 0.0048 & 0.04 & 0.00 & $-1.1$ & $-1.1$ \\
\bottomrule
\end{tabular}
\end{table}

\begin{figure}[htbp]
\subsection{Snapshots of the Corrections}
\label{sec:appendix-snapshots}
\centering
\includegraphics[width=\linewidth]{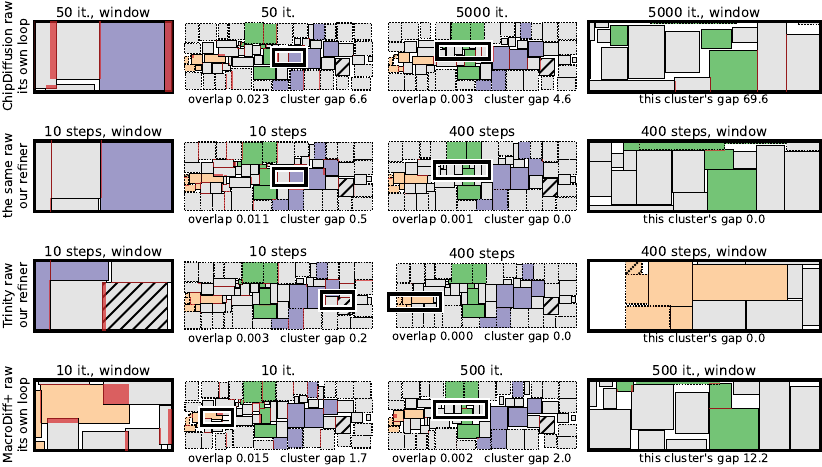}
\caption{The 67-block case of Figure~\ref{fig:motivation}, every correction from a raw draw (the row labels) at an early and a late budget, the box on each layout magnified in the window beside it: left, where the early state keeps the most overlap; right, the cluster the correction moves most, its gap under the window. Marks as in Figure~\ref{fig:motivation}.}
\label{fig:app_snap_n67}
\end{figure}

\begin{figure}[htbp]\centering
\includegraphics[width=\linewidth]{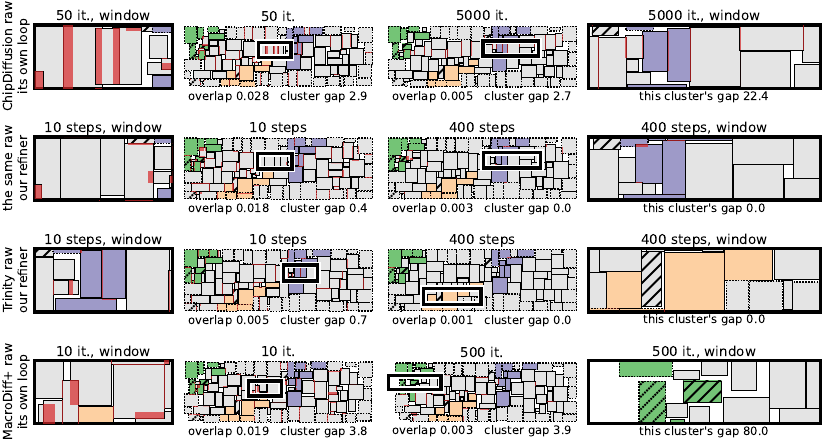}
\caption{A 106-block case, the rows of Figure~\ref{fig:app_snap_n67}.}
\label{fig:app_snap_n106}
\end{figure}

\begin{figure}[htbp]\centering
\includegraphics[width=\linewidth]{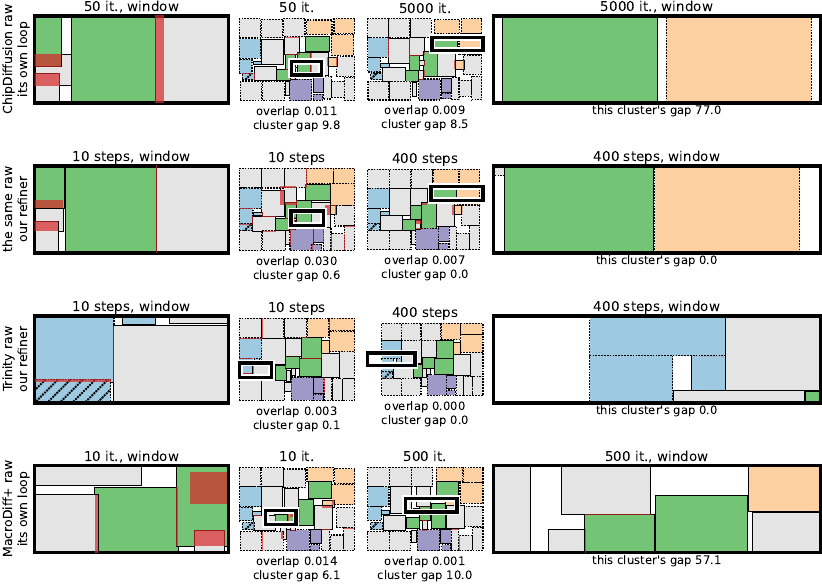}
\caption{A 31-block case, the rows of Figure~\ref{fig:app_snap_n67}.}
\label{fig:app_snap_n31}
\end{figure}

\begin{table}[htbp]
\subsection{Beyond FloorSet: Three Checkpoints on GSRC}
\label{sec:appendix-gsrc-checkpoints}
\centering
\caption{GSRC soft blocks in a 10\% square outline, three checkpoints of one recipe, one pipeline (4 seeds, 16 or 64 draws per seed, NFE 8, 400 steps, legalized): best fitting wirelength in mm over the seeds, in brackets the share of draws $\times$ seeds that fit. PARSAC's best of 16 workers: 52.0 / 155.7 / 183.4 / 303.5 / 576.0 / 705.2. $^{\dagger}$Oversized: no draw fits, the best legal layout's box against the outline at 16 draws \method{} on n300 1.29 W $\times$ 0.82 H; \method\ at $\lambda{=}0${} on n200 1.00 W $\times$ 1.13 H; \method\ at $\lambda{=}0${} on n300 1.18 W $\times$ 0.92 H; \method, seed 1{} on n300 1.02 W $\times$ 1.03 H.}
\label{tab:theory_gsrc}
\scriptsize\setlength{\tabcolsep}{7pt}
\begin{tabular}{lrrrrrrr}
\toprule
Checkpoint & Draws & n10 & n30 & n50 & n100 & n200 & n300 \\
\midrule
\method & 16 & 50.6 (0.92) & 150.6 (0.91) & 178.1 (0.98) & 294.7 (0.36) & 554.9 (1.00) & 681.3$^{\dagger}$ (0.00) \\
 & 64 & 50.3 (0.93) & 150.0 (0.89) & 177.9 (0.99) & 294.0 (0.38) & 553.1 (1.00) & 681.8$^{\dagger}$ (0.00) \\
\method\ at $\lambda{=}0$ & 16 & 51.4 (0.70) & 155.6 (0.81) & 179.2 (0.81) & 293.9 (0.80) & 560.1$^{\dagger}$ (0.00) & 679.8$^{\dagger}$ (0.00) \\
 & 64 & 51.3 (0.61) & 154.2 (0.80) & 179.1 (0.73) & 294.4 (0.65) & 549.2 (0.01) & 679.4$^{\dagger}$ (0.00) \\
\method, seed 1 & 16 & 49.4 (0.75) & 150.1 (0.81) & 178.2 (0.95) & 297.0 (0.98) & 548.1 (1.00) & 694.3$^{\dagger}$ (0.00) \\
 & 64 & 49.2 (0.85) & 150.0 (0.84) & 177.6 (0.97) & 294.7 (0.98) & 548.2 (1.00) & 692.9$^{\dagger}$ (0.00) \\
\bottomrule
\end{tabular}
\end{table}

\begin{table}[htbp]
\subsection{Time per Step of the Loops}
\label{sec:appendix-loop-time}
\centering
\caption{Each loop against our refiner: iterations to one soft cost on the same samples (loop / ours, Figure~\ref{fig:teaser} right), microseconds per step and layout and peak memory at 32 / 256 / 1,000 layouts per batch, milliseconds per layout and our time advantage at 32 / 1,000.}
\label{tab:theory_time}
\label{tab:theory_matched}
\scriptsize\setlength{\tabcolsep}{9pt}
\begin{tabular}{lrrrrr}
\toprule
Loop & Iterations & $\mu$s / step & MB & ms / layout & Less time \\
\midrule
\rowcolor{ourscolor} \method's refiner, 400 steps & -- & 50.8 / 6.1 / 2.6 & 14 / 112 / 436 & 20.32 / 1.03 & -- \\
DiffPlace loop & 500 / 9 & 60.7 / 7.5 / 3.1 & 19 / 153 / 597 & 30.34 / 1.54 & 67$\times$ / 67$\times$ \\
MacroDiff+ loop & 500 / 31 & 292.5 / 36.9 / 20.7 & 76 / 627 / 2,450 & 146.25 / 10.35 & 92$\times$ / 129$\times$ \\
ChipDiffusion loop & 5,000 / 8 & 69.6 / 8.8 / 5.3 & 33 / 272 / 1,062 & 348.01 / 26.43 & 904$\times$ / 1,359$\times$ \\
\bottomrule
\end{tabular}
\end{table}

\FloatBarrier

\section{Generator Results in Full}
\label{sec:appendix-generator}

Three reference models recur through the appendix. The \emph{base} is the plain backbone with the $x_0$ head, $\lambda{=}0$ and no wire dropout; sets 1 to 3 (objective, graph mechanisms, wire dropout) are on it. The \emph{plain control} is the plain backbone on the adopted recipe (the wire-dropout mixture, $\lambda{=}0$); the term ladders, the composed grid, the prior arts, the regressors and the size rungs are compared against it. \emph{Trinity at $\lambda{=}0$} is Trinity's backbone (spectral encoding and information mover) on the adopted recipe without the auxiliary loss, the model the refiner and seed studies pair with Trinity. MMD8b is a subsample estimate (Appendix~\ref{sec:appendix-dist}), so a run re-evaluated in another batch moves by a few percent, and every table prints the batch its rows were evaluated in. Units follow the convention of Section~\ref{sec:exp-setup}: boundary distance, group gap and fixed deviation in units of $10^{-3}$, MIB variance in units of $10^{-6}$, MMD8b in units of $10^{-3}$; "fixed" is the immutable-block deviation, "HPWL" and "area" the two signed gaps to the reference.

\subsection{Objective}
\label{sec:appendix-objective}

\begin{table}[htbp]\centering\scriptsize
\caption{The three heads on the plain base at $\nfe{=}32$ and 1: overlap ratio, cluster gap, the two gaps, the soft cost and the distribution metrics ($p$ against the base at the same $\nfe$). Boundary, MIB and fixed at $\nfe{=}32$: 3.2 / 10 / 7.3 ($x_0$), 4.0 / 15 / 4.6 ($v$), 3.0 / 6.3 / 2.6 ($\epsilon$).}
\label{tab:app_objective}
\setlength{\tabcolsep}{4.5pt}
\begin{tabular}{lrrrrrrrrrrr}
\toprule
Head & $\nfe$ & Overlap & Group & HPWL gap & Area gap & Soft cost & FLD8b & FLD8f & FPD & MMD8b & $p$ \\
\midrule
$x_0$ head (the base) & 32 & 0.0262 & 2.7 & $-$0.0024 & 0.0105 & 2.106 & 0.108 & 0.0045 & 0.0013 & 17.7 & -- \\
$v$ head & 32 & 0.0373 & 2.9 & $-$0.0047 & 0.0123 & 2.597 & 0.275 & 0.0097 & 0.0028 & 33.7 & 1.00 \\
$\epsilon$ head & 32 & 0.0365 & 3.2 & $-$0.0008 & 0.0149 & 2.387 & 0.416 & 0.0184 & 0.0025 & 33.0 & 1.00 \\
\midrule
$x_0$ head & 1 & 0.0383 & 2.2 & $-$0.015 & 0.0093 & 2.491 & 0.417 & 0.0184 & 0.0041 & 34.1 & -- \\
$v$ head & 1 & 0.2505 & 13 & 0.18 & 0.61 & 77.2 & 32.0 & 2.37 & 0.184 & 595 & 1.00 \\
$\epsilon$ head & 1 & 0.042 & 200 & 5.9 & 20.3 & 152 & 102.7 & 17.8 & 5.50 & 1,238 & 1.00 \\
\bottomrule
\end{tabular}
\end{table}

The $x_0$ head has the lowest overlap ($-30\%$ against both), the lowest area gap, the lowest distribution distances and the lowest soft cost; the two standard heads have a lower immutable-block deviation (0.0046 / 0.0026 against 0.0073) and the $\epsilon$ head a lower MIB variance. At $\nfe{=}1$ the $x_0$ head still yields a layout (soft cost 2.49 against 2.11, overlap 0.038 against 0.026), and neither standard head produces one.

\subsection{Architecture on the Plain Base}
\label{sec:appendix-arch-base}

One graph mechanism each, $\lambda{=}0$, no wire dropout: \emph{attention heads} lets 4 of the 12 heads attend by a learned affine of the adjacency alone, and the \emph{information mover} adds the gated one-hop residual of Section~\ref{sec:method-arch} in every layer (92.9M parameters against 85.8M), here with the symmetric normalization $\hat A = D^{-1/2}AD^{-1/2}$; the composed arms of Appendix~\ref{sec:appendix-composed} and Trinity use the row normalization $\hat A = D^{-1}A$.

\begin{table}[htbp]\centering\scriptsize
\caption{Graph mechanisms on the plain base at $\nfe{=}32$: overlap ratio, cluster gap, the two gaps, the soft cost and the distribution metrics, $p$ against the base at $\nfe{=}32$ and at 1. Boundary / MIB / fixed: 3.22 / 10.5 / 7.3 (the base), 3.18 / 8.4 / 5.5 (attention heads), 3.11 / 6.7 / 8.4 (information mover); normalized wirelength 0.3014 / 0.3013 / 0.3020 and compactness 1.038 / 1.038 / 1.042.}
\label{tab:app_arch_base}
\setlength{\tabcolsep}{3pt}
\begin{tabular}{lrrrrrrrrrrr}
\toprule
Arm & Overlap & Group & HPWL gap & Area gap & Soft cost & FLD8b & FLD8f & FPD & MMD8b & $p$ (32) & $p$ (1) \\
\midrule
the base & 0.0262 & 2.67 & $-$0.0024 & 0.0105 & 2.106 & 0.108 & 0.0045 & 0.0013 & 17.7 & -- & -- \\
attention heads & 0.0279 & 2.71 & $-$0.0027 & 0.0107 & 2.149 & 0.137 & 0.0040 & 0.0015 & 19.8 & 0.84 & 0.95 \\
information mover & 0.0239 & 2.67 & $-$0.0004 & 0.0148 & 2.050 & 0.113 & 0.0038 & 0.0017 & 20.1 & 1.00 & 0.99 \\
\bottomrule
\end{tabular}
\end{table}

Against the base at $\nfe{=}32$ ($\nfe{=}1$): the attention heads move overlap $+6.6\%$ ($+5.2\%$), MIB $-20\%$, fixed $-25\%$, soft cost $+2.0\%$; the information mover moves overlap $-8.5\%$ ($-5.6\%$), boundary $-3.6\%$, MIB $-36\%$, fixed $+15\%$, the HPWL gap toward 0, the area gap $+41\%$, soft cost $-2.7\%$. Overlap across $\nfe = 1 \ldots 32$: the base 0.0383 / 0.0368 / 0.0339 / 0.0313 / 0.0286 / 0.0262, the mover 0.0362 to 0.0239, the heads 0.0403 to 0.0279. Neither arm is closer to the reference than the base on the density descriptor. The mover places blocks with less overlap and lower MIB variance at the price of wirelength and compactness, which is the trade the terms then close.

\subsection{Recipe: Wire Dropout}
\label{sec:appendix-recipe}

Three mixtures were trained on the plain base: \emph{mild} (70\% of instances untouched, 30\% at a rate from Beta(1,3)), the adopted \emph{mixture} (50\% untouched, 40\% Beta(1,3), 10\% U(0.9,1)) and \emph{heavy} (30\% untouched, 50\% Beta(1,2), 20\% U(0.8,1)).

\begin{table}[htbp]\centering\scriptsize
\caption{Wire-dropout mixtures on the plain base: the delta against the base at $\nfe{=}32$ on overlap, cluster gap, HPWL gap and area gap (at $\nfe{=}1$ the overlap deltas are $+0.4$, $+0.9$ and $+2.1\%$; boundary / MIB / fixed move $+1.7$ / $-19$ / $-5\%$, $+0.2$ / $-16$ / $-20\%$ and $-1.0$ / $-22$ / $-16\%$), the absolute soft cost and the distribution metrics.}
\label{tab:app_recipe}
\setlength{\tabcolsep}{3pt}
\begin{tabular}{lrrrrrrrrrr}
\toprule
Arm & Overlap & Group & HPWL gap & Area gap & Soft cost & FLD8b & FLD8f & FPD & MMD8b & $p$ \\
\midrule
the base (absolute) & 0.0262 & 2.67 & $-$0.0024 & 0.0105 & 2.106 & 0.108 & 0.0045 & 0.0013 & 17.7 & -- \\
mild & $-$0.8\% & $+$5.0\% & $+$66\% & $+$43\% & 2.132 & 0.126 & 0.0045 & 0.0016 & 23.4 & 1.00 \\
mixture (adopted, the plain control) & $-$0.5\% & $-$0.7\% & $+$10\% & $-$1.5\% & 2.092 & 0.103 & 0.0033 & 0.0011 & 15.6 & 0.005 \\
heavy & 0.0\% & $+$2.7\% & $+$71\% & $+$32\% & 2.107 & 0.115 & 0.0050 & 0.0014 & 19.9 & 1.00 \\
\bottomrule
\end{tabular}
\end{table}

Absolute HPWL gaps: base $-0.0024$, mixture $-0.0021$, mild $-0.0008$, heavy $-0.0007$; area gaps 0.0105 / 0.0104 / 0.0150 / 0.0139. The adopted recipe's $\nfe{=}32$ values are the plain control's, the reference of every term ladder: overlap 0.02602, boundary 3.227, group 2.651, MIB 8.75, fixed 0.00582, HPWL gap $-0.00214$, area gap 0.01037, normalized wirelength 0.3014, compactness 1.037, soft cost 2.0915 (best-of-4 1.9074; $\nfe{=}1$ 2.4900 with overlap 0.03864), FLD8b 0.1031, FLD8f 0.0033, FPD 0.0011, MMD8b 15.6.

\subsection{One Term at a Time}
\label{sec:appendix-single-terms}

\begin{table}[htbp]\centering\footnotesize
\caption{One term at a time on the plain backbone, absolute at $\nfe{=}32$. Bold: the term's own target. Normalized wirelength and compactness stay within $\pm0.6\%$ of the plain control (0.3014 and 1.037) for every constraint arm; the wirelength arms move it to 0.3015 / 0.3008 / 0.3000 / 0.3013 and the area arms move compactness to 1.033 / 1.026 / 1.017 / 1.031 for $0.01 / 0.05 / 0.1 / 0.05\,t$.}
\label{tab:app_single}
\setlength{\tabcolsep}{4pt}
\begin{tabular}{lrrrrrrrr}
\toprule
Arm & Overlap & Boundary & Group & MIB & Fixed & HPWL gap & Area gap & Soft cost \\
\midrule
the plain control & 0.02602 & 3.227 & 2.651 & 8.75 & 5.82 & $-$0.00214 & 0.01037 & 2.0915 \\
overlap 0.01 & \textbf{0.02423} & 3.179 & 2.631 & 10.6 & 5.56 & $-$0.00042 & 0.01180 & 2.0479 \\
overlap 0.05 & \textbf{0.01948} & 2.929 & 2.529 & 12.6 & 5.15 & $-$0.00107 & 0.01235 & 1.9027 \\
overlap 0.1 & \textbf{0.01796} & 3.034 & 2.570 & 11.5 & 5.37 & $-$0.00063 & 0.01348 & 1.8950 \\
overlap $0.05\,t$ & \textbf{0.02108} & 3.124 & 2.636 & 12.1 & 6.59 & $-$0.00088 & 0.01461 & 1.9864 \\
cluster 0.01 & 0.02764 & 3.318 & \textbf{2.294} & 9.32 & 8.69 & $-$0.00280 & 0.00761 & 2.1048 \\
cluster 0.05 & 0.02507 & 3.262 & \textbf{2.158} & 9.06 & 6.99 & $-$0.00125 & 0.01057 & 2.0345 \\
cluster 0.1 & 0.02552 & 3.453 & \textbf{2.035} & 9.60 & 7.34 & $-$0.00075 & 0.01031 & 2.0643 \\
cluster $0.05\,t$ & 0.02556 & 3.261 & \textbf{2.472} & 8.46 & 6.32 & $-$0.00119 & 0.01259 & 2.0754 \\
MIB 0.01 & 0.02704 & 3.178 & 2.673 & \textbf{3.51} & 7.05 & $-$0.00148 & 0.01051 & 2.0960 \\
MIB 0.05 & 0.02599 & 3.183 & 2.767 & \textbf{0.687} & 6.87 & $-$0.00066 & 0.01375 & 2.0812 \\
MIB 0.1 & 0.02795 & 3.289 & 2.756 & \textbf{1.80} & 7.12 & $-$0.00190 & 0.01109 & 2.1405 \\
MIB $0.05\,t$ & 0.02714 & 3.183 & 2.710 & \textbf{2.86} & 6.31 & $-$0.00139 & 0.01242 & 2.1140 \\
boundary 0.01 & 0.02468 & \textbf{0.988} & 2.768 & 9.05 & 5.46 & $-$0.00106 & 0.00377 & 1.6273 \\
boundary 0.05 & 0.03094 & \textbf{0.739} & 3.213 & 12.5 & 11.73 & $+$0.00122 & 0.00738 & 1.7282 \\
boundary 0.1 & 0.03845 & \textbf{0.884} & 3.470 & 14.1 & 15.51 & $+$0.00087 & 0.00431 & 1.9016 \\
boundary $0.05\,t$ & 0.02952 & \textbf{1.080} & 2.759 & 10.4 & 6.14 & $-$0.00253 & 0.00009 & 1.7237 \\
\midrule
wirelength 0.01 & 0.02687 & 3.217 & 2.722 & 8.52 & 5.30 & $-$0.00204 & 0.01195 & 2.1213 \\
wirelength 0.05 & 0.02726 & 3.214 & 2.758 & 8.76 & 6.83 & $-$0.00412 & 0.01138 & 2.1297 \\
wirelength 0.1 & 0.02816 & 3.217 & 2.732 & 11.3 & 6.71 & $-$0.00688 & 0.00985 & 2.1499 \\
wirelength $0.05\,t$ & 0.02647 & 3.151 & 2.724 & 8.68 & 6.15 & $-$0.00256 & 0.01195 & 2.1072 \\
area 0.01 & 0.02652 & 2.432 & 2.658 & 8.73 & 5.62 & $-$0.00202 & 0.00626 & 1.9380 \\
area 0.05 & 0.02948 & 1.546 & 2.697 & 8.55 & 6.07 & $-$0.00188 & $-$0.00127 & 1.8141 \\
area 0.1 & 0.03578 & 1.376 & 2.660 & 10.1 & 6.88 & $-$0.00442 & $-$0.00925 & 1.8833 \\
area $0.05\,t$ & 0.02803 & 1.966 & 2.667 & 9.03 & 6.77 & $-$0.00211 & 0.00359 & 1.8772 \\
\bottomrule
\end{tabular}
\end{table}

The target's move against the plain control across $0.01 / 0.05 / 0.1 / 0.05\,t$: overlap $-6.9 / -25.1 / -31.0 / -19.0\%$, cluster $-13.5 / -18.6 / -23.2 / -6.7\%$, MIB $-59.9 / -92.1 / -79.5 / -67.3\%$, boundary $-69.4 / -77.1 / -72.6 / -66.5\%$. At $\nfe{=}1$ the same order with larger moves (overlap $-20 / -42 / -48 / -39\%$, cluster $-22 / -28 / -32 / -17\%$, MIB $-70 / -93 / -82 / -75\%$, boundary $-70 / -77 / -73 / -67\%$). Each term moves its own target and only its own target by more than a few percent. Side effects at 0.01: the cluster term raises overlap $+6\%$ and fixed $+49\%$ while cutting the area gap $-27\%$, the overlap term raises MIB $+21\%$ and the area gap $+14\%$ and moves the HPWL gap from $-0.0021$ to $-0.0004$, the boundary term cuts the area gap $-64\%$, the MIB term raises fixed $+21\%$. With the weight, the overlap target keeps improving at a rising area gap ($+14 / +19 / +30\%$), the cluster term's fixed cost shrinks from $+49\%$ to $+26\%$, MIB is strongest at 0.05, and boundary saturates near 0.0009 while its side effects grow (overlap $+48\%$, group $+31\%$, fixed $+167\%$ at 0.1). The $t$-weighting changes the boundary term's character: target $-66.5\%$ with the side effects collapsed (fixed $+5.6\%$, group $+4.1\%$) and the area gap at about zero. The wirelength term moves the HPWL gap from 0.2\% below the reference to 0.7\% below at weight 0.1 ($\nfe{=}1$: $-0.015$ to $-0.032$), paid in overlap ($+3 / +5 / +8\%$) and fixed ($+15$ to $+18\%$ at $\geq 0.05$). The area term is the one single term whose main side effect is a gain: compacting pushes blocks onto the outline, so the boundary distance falls 25 to 57\%.

\begin{table}[htbp]\centering\footnotesize
\caption{Distribution metrics of the single-term arms, the five columns at $\nfe{=}32$ and then at $\nfe{=}1$; $p$ is the permutation test against the plain control at the same $\nfe$.}
\label{tab:app_single_dist}
\setlength{\tabcolsep}{3pt}
\begin{tabular}{lrrrrr@{\hspace{10pt}}rrrrr}
\toprule
Arm & FLD8b & FLD8f & FPD & MMD8b & $p$ & FLD8b & FLD8f & FPD & MMD8b & $p$ \\
\midrule
the plain control & 0.1031 & 0.0033 & 0.0011 & 15.6 & -- & 0.4325 & 0.0179 & 0.0036 & 32.6 & -- \\
overlap 0.01 & 0.0993 & 0.0078 & 0.0010 & 18.9 & 1.000 & 0.2238 & 0.0121 & 0.0020 & 28.4 & 0.015 \\
overlap 0.05 & 0.0735 & 0.0026 & 0.0007 & 14.1 & 0.010 & 0.1229 & 0.0044 & 0.0008 & 19.3 & 0.005 \\
overlap 0.1 & 0.0707 & 0.0032 & 0.0006 & 16.2 & 0.975 & 0.1221 & 0.0058 & 0.0007 & 21.9 & 0.005 \\
overlap $0.05\,t$ & 0.0957 & 0.0034 & 0.0010 & 19.3 & 1.000 & 0.1548 & 0.0059 & 0.0015 & 26.6 & 0.005 \\
cluster 0.01 & 0.1022 & 0.0038 & 0.0012 & 14.8 & 0.751 & 0.4322 & 0.0196 & 0.0039 & 32.0 & 0.935 \\
cluster 0.05 & 0.0986 & 0.0044 & 0.0010 & 17.7 & 1.000 & 0.4413 & 0.0216 & 0.0037 & 34.5 & 1.000 \\
cluster 0.1 & 0.0948 & 0.0032 & 0.0009 & 19.7 & 1.000 & 0.4621 & 0.0221 & 0.0037 & 35.9 & 1.000 \\
cluster $0.05\,t$ & 0.1115 & 0.0046 & 0.0013 & 20.9 & 1.000 & 0.5117 & 0.0240 & 0.0045 & 35.6 & 1.000 \\
MIB 0.01 & 0.1102 & 0.0051 & 0.0011 & 17.1 & 1.000 & 0.4399 & 0.0193 & 0.0035 & 35.2 & 0.940 \\
MIB 0.05 & 0.1102 & 0.0041 & 0.0013 & 19.0 & 1.000 & 0.4315 & 0.0182 & 0.0038 & 31.9 & 0.990 \\
MIB 0.1 & 0.1104 & 0.0034 & 0.0014 & 17.8 & 1.000 & 0.4397 & 0.0172 & 0.0042 & 32.6 & 0.891 \\
MIB $0.05\,t$ & 0.1167 & 0.0040 & 0.0015 & 19.1 & 1.000 & 0.4472 & 0.0182 & 0.0040 & 35.0 & 0.980 \\
boundary 0.01 & 0.0785 & 0.0043 & 0.0014 & 5.15 & 0.005 & 0.4227 & 0.0188 & 0.0040 & 22.3 & 0.005 \\
boundary 0.05 & 0.1044 & 0.0112 & 0.0022 & 11.5 & 0.005 & 0.4435 & 0.0257 & 0.0052 & 30.3 & 0.831 \\
boundary 0.1 & 0.1342 & 0.0107 & 0.0024 & 15.7 & 0.617 & 0.4443 & 0.0244 & 0.0050 & 39.8 & 1.000 \\
boundary $0.05\,t$ & 0.0797 & 0.0058 & 0.0016 & 5.27 & 0.005 & 0.4132 & 0.0207 & 0.0040 & 23.5 & 0.005 \\
area 0.01 & 0.0882 & 0.0040 & 0.0013 & 10.2 & 0.005 & 0.4097 & 0.0212 & 0.0038 & 25.6 & 0.005 \\
area 0.05 & 0.0847 & 0.0062 & 0.0014 & 7.4 & 0.005 & 0.7222 & 0.0510 & 0.0061 & 41.7 & 1.000 \\
area $0.05\,t$ & 0.0872 & 0.0054 & 0.0015 & 6.9 & 0.005 & 0.6399 & 0.0417 & 0.0063 & 39.7 & 0.995 \\
area 0.1 & 0.1210 & 0.0100 & 0.0018 & 16.1 & 0.821 & 1.3757 & 0.1042 & 0.0111 & 69.1 & 1.000 \\
\bottomrule
\end{tabular}
\end{table}

The arms significantly closer to the reference than the plain control at $\nfe{=}32$ are overlap 0.05, boundary 0.01, 0.05 and $0.05\,t$, and the area arms up to 0.05 ($0.05\,t$ included). Overlap keeps its $\nfe{=}1$ gain at every weight (FLD8b 0.12 to 0.22 against 0.43), the cluster and MIB arms are within noise or farther at every weight, and boundary's distribution columns worsen with the weight. Across the ladder the distribution metrics put the closest weight one step below where the raw target metric is best (overlap 0.05, boundary 0.01). Every wirelength arm is farther from the reference than the plain control at both $\nfe$ ($p = 1.000$, FLD8b 0.117 to 0.128 at $\nfe{=}32$).

\subsection{All Six Terms}
\label{sec:appendix-all-six}

\begin{table}[htbp]\centering\footnotesize
\caption{All six terms at one common weight on the plain backbone: absolute at $\nfe{=}32$ and the delta against the plain control. Normalized wirelength stays at 0.3011 to 0.3019 and compactness falls from 1.037 to 1.023 to 1.030.}
\label{tab:app_all_six}
\setlength{\tabcolsep}{3.5pt}
\begin{tabular}{lrrrrrrrr}
\toprule
Arm & Overlap & Boundary & Group & MIB & Fixed & HPWL gap & Area gap & Soft cost \\
\midrule
the plain control & 0.02602 & 3.227 & 2.651 & 8.75 & 5.82 & $-$0.00214 & 0.01037 & 2.0915 \\
all six 0.01 & 0.02341 & 0.948 & 2.424 & 2.86 & 5.52 & $-$0.00151 & 0.00199 & 1.5565 \\
all six 0.05 & 0.02627 & 0.759 & 2.348 & 0.667 & 6.72 & $-$0.00320 & $-$0.00014 & 1.5615 \\
all six 0.1 & 0.02959 & 0.769 & 2.554 & 0.515 & 16.13 & $-$0.00063 & 0.00272 & 1.6312 \\
all six $0.05\,t$ & 0.02588 & 1.082 & 2.471 & 2.53 & 6.30 & $-$0.00253 & 0.00070 & 1.6224 \\
all six $0.1\,t$ & 0.02728 & 0.880 & 2.571 & 1.18 & 6.67 & $-$0.00180 & 0.00032 & 1.6175 \\
all six $0.2\,t$ & 0.03227 & 0.650 & 2.842 & 1.44 & 14.26 & $-$0.00043 & $-$0.00410 & 1.6698 \\
\midrule
delta 0.01 & $-$10.0\% & $-$70.6\% & $-$8.6\% & $-$67.3\% & $-$5.2\% & toward 0 & $-$80.8\% & $-$25.6\% \\
delta 0.05 & $+$0.9\% & $-$76.5\% & $-$11.4\% & $-$92.4\% & $+$15.5\% & $-$49\% & crosses 0 & $-$25.3\% \\
delta 0.1 & $+$13.7\% & $-$76.2\% & $-$3.7\% & $-$94.1\% & $+$177\% & $+$71\% & $-$73.8\% & $-$22.0\% \\
delta $0.05\,t$ & $-$0.5\% & $-$66.5\% & $-$6.8\% & $-$71.1\% & $+$8.2\% & $-$18\% & $-$93.2\% & $-$22.4\% \\
delta $0.1\,t$ & $+$4.9\% & $-$72.7\% & $-$3.0\% & $-$86.5\% & $+$14.6\% & $+$16\% & $-$96.9\% & $-$22.7\% \\
delta $0.2\,t$ & $+$24.0\% & $-$79.9\% & $+$7.2\% & $-$83.6\% & $+$145\% & $+$80\% & crosses 0 & $-$20.2\% \\
\bottomrule
\end{tabular}
\end{table}

\begin{table}[htbp]\centering\footnotesize
\caption{Distribution metrics of the all-six arms, the five columns at $\nfe{=}32$ and then at $\nfe{=}1$.}
\label{tab:app_all_six_dist}
\setlength{\tabcolsep}{3pt}
\begin{tabular}{lrrrrr@{\hspace{10pt}}rrrrr}
\toprule
Arm & FLD8b & FLD8f & FPD & MMD8b & $p$ & FLD8b & FLD8f & FPD & MMD8b & $p$ \\
\midrule
the plain control & 0.1031 & 0.0033 & 0.0011 & 15.6 & -- & 0.4325 & 0.0179 & 0.0036 & 32.6 & -- \\
all six 0.01 & 0.0578 & 0.0037 & 0.0010 & 3.05 & 0.005 & 0.1947 & 0.0103 & 0.0022 & 9.98 & 0.005 \\
all six 0.05 & 0.0564 & 0.0029 & 0.0012 & 3.08 & 0.005 & 0.1701 & 0.0097 & 0.0024 & 10.4 & 0.005 \\
all six 0.1 & 0.0581 & 0.0052 & 0.0013 & 5.23 & 0.005 & 0.1530 & 0.0117 & 0.0023 & 11.4 & 0.005 \\
all six $0.05\,t$ & 0.0693 & 0.0042 & 0.0012 & 4.03 & 0.005 & 0.2144 & 0.0112 & 0.0025 & 12.6 & 0.005 \\
all six $0.1\,t$ & 0.0564 & 0.0044 & 0.0011 & 4.05 & 0.005 & 0.1764 & 0.0097 & 0.0022 & 10.7 & 0.005 \\
all six $0.2\,t$ & 0.0621 & 0.0068 & 0.0012 & 5.49 & 0.005 & 0.1744 & 0.0134 & 0.0019 & 14.8 & 0.005 \\
\bottomrule
\end{tabular}
\end{table}

$\nfe{=}1$ (control 0.03864 / 3.205 / 2.158 / 11.4 / 0.00619 / $-0.0153$ / 0.00932, soft cost 2.490): all six 0.01 gives 0.03194 / 0.932 / 1.809 / 2.65 / 0.00592 / $-0.0128$ / $-0.00024$ and 1.720; 0.05 gives 0.03400 / 0.741 / 1.734 / 0.662 / 0.00737 / $-0.0145$ / $-0.00346$ and 1.700; 0.1 gives 0.03630 / 0.774 / 1.872 / 0.483 / 0.01832 / $-0.0127$ / $-0.00132$ and 1.755; $0.05\,t$ 1.801, $0.1\,t$ 1.760, $0.2\,t$ 1.824. At the common weight 0.01 every constraint metric improves at once with nothing worse except the HPWL gap moving toward 0, and the side effects of the single terms cancel. At 0.05 the constraint targets are stronger but overlap is back at the plain control and fixed is $+16\%$; at 0.1 fixed is $+177\%$ and overlap $+14\%$; the $t$-weightings sit between 0.01 and 0.05 ($0.1\,t$) or repeat the 0.1 failure ($0.2\,t$). The all-six arms are the closest to the reference distribution of the plain-backbone family (FLD8b 0.056 to 0.069 against 0.103, MMD8b 3 to 5 against 15.6), every arm significantly closer at both $\nfe$, with the spread field FLD8f at the plain control's level for 0.01 and 0.05.

\subsection{The Composed Grid}
\label{sec:appendix-composed}

Three backbones on the adopted recipe, the plain backbone, the plain backbone with the spectral encoding, and the plain backbone with the spectral encoding and the information mover, each without the auxiliary loss, with all six at 0.01 and with all six at 0.05, plus the last backbone at $0.1\,t$ and $0.2\,t$. The plain backbone at $\lambda{=}0$ is the plain control, the last backbone at $\lambda{=}0$ is Trinity at $\lambda{=}0$, and Trinity is the last backbone with all six at 0.01.

\begin{table}[htbp]\centering\scriptsize
\caption{The composed grid at $\nfe{=}32$: the constraint columns, the two gaps, the soft cost, the best-of-4 soft cost and the $\nfe{=}1$ soft cost. Normalized wirelength is 0.3011 to 0.3028 on every row; compactness is 1.037 / 1.029 / 1.027 / 1.046 / 1.030 / 1.028 / 1.045 / 1.029 / 1.023 / 1.032 / 1.024 down the rows.}
\label{tab:app_composed}
\setlength{\tabcolsep}{2.5pt}
\begin{tabular}{lrrrrrrrrrr}
\toprule
Arm & Overlap & Boundary & Group & MIB & Fixed & HPWL gap & Area gap & Soft cost & Best of 4 & $\nfe{=}1$ \\
\midrule
plain, $\lambda{=}0$ (the plain control) & 0.02602 & 3.227 & 2.651 & 8.75 & 5.82 & $-$0.00214 & 0.01037 & 2.0915 & 1.9074 & 2.4900 \\
plain + all six 0.01 & 0.02341 & 0.948 & 2.424 & 2.86 & 5.52 & $-$0.00151 & 0.00199 & 1.5565 & 1.4479 & 1.7198 \\
plain + all six 0.05 & 0.02627 & 0.759 & 2.348 & 0.667 & 6.72 & $-$0.00320 & $-$0.00014 & 1.5615 & 1.4586 & 1.7001 \\
spectral PE, $\lambda{=}0$ & 0.02433 & 3.310 & 2.942 & 10.5 & 7.21 & $+$0.00116 & 0.01906 & 2.1305 & 1.9404 & 2.5024 \\
spectral PE + all six 0.01 & 0.02343 & 1.029 & 2.554 & 3.52 & 7.18 & $-$0.00062 & 0.00307 & 1.5912 & 1.4763 & 1.7489 \\
spectral PE + all six 0.05 & 0.02556 & 0.652 & 2.405 & 0.712 & 11.19 & $-$0.00177 & 0.00128 & 1.5373 & 1.4378 & 1.6649 \\
PE + mover, $\lambda{=}0$ (Trinity at $\lambda{=}0$) & 0.02220 & 2.997 & 2.827 & 6.66 & 7.47 & $+$0.00219 & 0.01810 & 2.0151 & 1.8304 & 2.4271 \\
PE + mover + all six 0.01 (Trinity) & 0.02253 & 0.817 & 2.325 & 1.72 & 8.34 & $-$0.00069 & 0.00224 & 1.5137 & 1.4092 & 1.6785 \\
PE + mover + all six 0.05 & 0.02570 & 0.632 & 2.333 & 0.346 & 7.95 & $-$0.00256 & $-$0.00368 & 1.5192 & 1.4320 & 1.6341 \\
PE + mover + all six $0.1\,t$ & 0.02456 & 0.842 & 2.585 & 1.39 & 10.27 & $+$0.00002 & 0.00487 & 1.5784 & 1.4797 & 1.7064 \\
PE + mover + all six $0.2\,t$ & 0.03156 & 0.847 & 2.667 & 0.599 & 11.02 & $-$0.00175 & $-$0.00231 & 1.6867 & 1.5807 & 1.8199 \\
\bottomrule
\end{tabular}
\end{table}

\begin{table}[htbp]\centering\scriptsize
\caption{Distribution metrics of the composed grid, the five columns at $\nfe{=}32$ and then at $\nfe{=}1$.}
\label{tab:app_composed_dist}
\setlength{\tabcolsep}{3pt}
\begin{tabular}{lrrrrr@{\hspace{10pt}}rrrrr}
\toprule
Arm & FLD8b & FLD8f & FPD & MMD8b & $p$ & FLD8b & FLD8f & FPD & MMD8b & $p$ \\
\midrule
plain, $\lambda{=}0$ (the plain control) & 0.1031 & 0.0033 & 0.0011 & 15.6 & -- & 0.4325 & 0.0179 & 0.0036 & 32.6 & -- \\
plain + all six 0.01 & 0.0578 & 0.0037 & 0.0010 & 3.05 & 0.005 & 0.1947 & 0.0103 & 0.0022 & 9.98 & 0.005 \\
plain + all six 0.05 & 0.0564 & 0.0029 & 0.0012 & 3.08 & 0.005 & 0.1701 & 0.0097 & 0.0024 & 10.4 & 0.005 \\
spectral PE, $\lambda{=}0$ & 0.1358 & 0.0059 & 0.0015 & 28.4 & 1.000 & 0.4317 & 0.0186 & 0.0041 & 44.9 & 1.000 \\
spectral PE + all six 0.01 & 0.0578 & 0.0036 & 0.0011 & 3.80 & 0.005 & 0.1854 & 0.0108 & 0.0025 & 12.7 & 0.005 \\
spectral PE + all six 0.05 & 0.0581 & 0.0036 & 0.0013 & 4.56 & 0.005 & 0.1668 & 0.0099 & 0.0023 & 11.8 & 0.005 \\
PE + mover, $\lambda{=}0$ (Trinity at $\lambda{=}0$) & 0.1123 & 0.0048 & 0.0015 & 21.7 & 1.000 & 0.4264 & 0.0159 & 0.0042 & 35.1 & 0.980 \\
PE + mover + all six 0.01 (Trinity) & 0.0423 & 0.0075 & 0.0012 & 2.48 & 0.005 & 0.1624 & 0.0173 & 0.0025 & 9.63 & 0.005 \\
PE + mover + all six 0.05 & 0.0366 & 0.0026 & 0.0007 & 3.44 & 0.005 & 0.1109 & 0.0085 & 0.0012 & 10.0 & 0.005 \\
PE + mover + all six $0.1\,t$ & 0.0416 & 0.0025 & 0.0011 & 3.98 & 0.005 & 0.1218 & 0.0065 & 0.0019 & 9.41 & 0.005 \\
PE + mover + all six $0.2\,t$ & 0.0487 & 0.0074 & 0.0011 & 6.27 & 0.005 & 0.1395 & 0.0125 & 0.0018 & 12.8 & 0.005 \\
\bottomrule
\end{tabular}
\end{table}

Overlap across $\nfe = 1 \ldots 32$: Trinity 0.0316 / 0.0306 / 0.0288 / 0.0271 / 0.0248 / 0.0225, the plain backbone with all six 0.01 0.0319 / 0.0312 / 0.0297 / 0.0279 / 0.0256 / 0.0234, Trinity at $\lambda{=}0$ 0.0357 / 0.0339 / 0.0308 / 0.0279 / 0.0250 / 0.0222, the plain control 0.0386 / 0.0368 / 0.0339 / 0.0313 / 0.0286 / 0.0260. The terms do the same work on every backbone (boundary $-70$ to $-80\%$, MIB $-70$ to $-90\%$, area gap to about 0, soft cost $-25\%$), and the backbones differ in where they start: the mover starts 15\% lower on overlap and keeps that lead through the ladder. With terms on, the spectral encoding alone is a near-null change on the soft metrics, and the mover is the component that moves the constraint metrics: at both weights the composed backbone leads the grid on overlap (at 0.01), boundary, group and MIB. On the distribution metrics the mover cuts FLD8b by a quarter to a third at both weights (0.042 / 0.037) and at 0.05 is best on every Fr\'echet column at both $\nfe$; all nine grid arms with terms are significantly closer to the reference than the plain control. Trinity's FLD8f of 0.0075 at 0.01 is a seed outlier (its other two seeds give 0.0036 and 0.0034, Appendix~\ref{sec:appendix-seeds}). The composed backbone at $0.1\,t$ sits between constant 0.01 and 0.05 on the constraint columns and $0.2\,t$ repeats the constant-0.1 failure (overlap $+21\%$, fixed $+89\%$).

\subsection{The Prior Arts and the Regressors}
\label{sec:appendix-prior-full}

\begin{table}[htbp]\centering\scriptsize
\caption{Every generator at $\nfe{=}32$ and 1, the annealer and the regressors on their single layout: overlap ratio, cluster gap, the two gaps, the soft cost (the regressors' as the per-case median) and the distribution metrics, $p$ against the plain control at the same $\nfe$. Normalized wirelength is 0.297 to 0.303 for every diffusion row at $\nfe{=}32$; the boundary, MIB and fixed columns are in Table~\ref{tab:app_ranking}'s sources.}
\label{tab:app_prior_full}
\label{tab:app_prior_dist}
\setlength{\tabcolsep}{2.5pt}
\begin{tabular}{lrrrrrrrrrrr}
\toprule
Arm & $\nfe$ & Overlap & Group & HPWL gap & Area gap & Soft cost & FLD8b & FLD8f & FPD & MMD8b & $p$ \\
\midrule
the plain control & 32 & 0.02602 & 2.651 & $-$0.00214 & 0.01037 & 2.0915 & 0.1031 & 0.0033 & 0.0011 & 15.6 & -- \\
Trinity & 32 & 0.02253 & 2.325 & $-$0.00069 & 0.00224 & 1.5137 & 0.0423 & 0.0075 & 0.0012 & 2.5 & 0.005 \\
FlowPlace & 32 & 0.04383 & 3.196 & $-$0.00231 & 0.00871 & 2.4564 & 0.2274 & 0.0088 & 0.0013 & 24.4 & 1.00 \\
ChipDiffusion & 32 & 0.05113 & 3.581 & $-$0.00076 & 0.01185 & 2.8020 & 0.2733 & 0.0098 & 0.0011 & 31.2 & 1.00 \\
MacroDiff+ & 32 & 0.0639 & 4.33 & $+$0.00236 & 0.02417 & 3.8458 & 0.624 & 0.0276 & 0.280 & 81.5 & 1.00 \\
DiffPlace & 32 & 0.1126 & 4.833 & $-$0.00291 & 0.01229 & 6.2293 & 2.855 & 0.191 & 0.0146 & 247 & 1.00 \\
\midrule
the plain control & 1 & 0.03864 & 2.158 & $-$0.0153 & 0.0093 & 2.4900 & 0.4325 & 0.0179 & 0.0036 & 32.6 & -- \\
Trinity & 1 & 0.03161 & 1.773 & $-$0.01237 & $-$0.00125 & 1.6785 & 0.1624 & 0.0173 & 0.0025 & 9.6 & 0.005 \\
FlowPlace & 1 & 0.1082 & 3.803 & $-$0.0148 & 0.0643 & 11.465 & 3.771 & 0.182 & 0.0192 & 350 & 1.00 \\
ChipDiffusion & 1 & 0.0421 & 198 & 5.93 & 19.8 & 151.6 & 102.4 & 17.7 & 5.50 & 1,223 & 1.00 \\
MacroDiff+ & 1 & 0.0444 & 194 & 5.66 & 19.5 & 150.2 & 102.3 & 17.5 & 5.72 & 1,216 & 1.00 \\
DiffPlace & 1 & 0.0420 & 198 & 5.84 & 19.9 & 149.8 & 102.5 & 17.8 & 5.36 & 1,230 & 1.00 \\
\midrule
sequence-pair SA & -- & 0.0000 & 59 & 1.54 & 1.11 & 12.341 & -- & -- & -- & -- & -- \\
Reg-$z$ & -- & 0.4116 & 6.303 & $-$0.0164 & 0.0134 & median 21.5 & 42.06 & 2.78 & 0.614 & 457 & 1.00 \\
Reg-$xywh$ & -- & 0.5884 & 5.028 & $-$0.0507 & $-$0.145 & median 99.5 & 66.56 & 5.55 & 1.418 & 587 & 1.00 \\
\bottomrule
\end{tabular}
\end{table}

Overlap across $\nfe = 1 \ldots 32$: FlowPlace 0.108 / 0.081 / 0.072 / 0.063 / 0.052 / 0.044, ChipDiffusion 0.042 / 0.196 / 0.114 / 0.059 / 0.054 / 0.051, MacroDiff+ 0.044 / 0.106 / 0.216 / 0.102 / 0.071 / 0.064, DiffPlace 0.042 / 0.106 / 0.251 / 0.146 / 0.116 / 0.113. FlowPlace at $\nfe{=}32$ has 1.7$\times$ the plain control's overlap and 1.2$\times$ the group gap with a lower boundary distance ($-13\%$), MIB variance ($-60\%$) and fixed deviation; ChipDiffusion 2.0$\times$ the overlap and 1.35$\times$ the group gap; MacroDiff+ is the weakest learned generator (overlap 2.5$\times$, boundary 1.8$\times$, wirelength above the reference, FLD8b 0.62); DiffPlace has the highest overlap of all learned models (4.3$\times$ the plain control) with boundary, wirelength and outline near the plain control's. None is closer to the reference than the plain control ($p = 1.00$ everywhere), and Trinity is closer than all of them by 5 to 70$\times$ on FLD8b. The regressors collapse blocks onto each other: Reg-$z$ keeps the outline at the reference size with 41\% of the block area claimed twice, Reg-$xywh$ is worse on every column (overlap 0.59, the outline 14.5\% below the reference, fixed 24$\times$), so the area-preserving latent alone is worth 30\% of the regressor's overlap.

\subsection{Every Trained Model by Soft Cost}
\label{sec:appendix-ranking}

\begin{table}[p]\centering\footnotesize
\caption{Every model of the study on our 12,000 val set, sorted by the $\nfe{=}32$ soft cost. Columns: soft cost / overlap at $\nfe{=}32$, the best-of-4 soft cost, soft cost / overlap at $\nfe{=}1$ (or on the single layout, the regressors' soft cost as the per-case median), and the wirelength / area gap at $\nfe{=}32$ in percent. Unless stated, a model is on the adopted recipe (wire-dropout mixture); "plain base" means no wire dropout.}
\label{tab:app_ranking}
\setlength{\tabcolsep}{5pt}
\begin{tabular}{lrrrr}
\toprule
Setup & $\nfe{=}32$ & Best of 4 & $\nfe{=}1$ & WL / Area (\%) \\
\midrule
\textbf{Trinity}: PE + mover + all six 0.01 & \textbf{1.5137} / 0.0225 & \textbf{1.4092} & 1.6785 / 0.0316 & $-0.07$/$+0.22$ \\
PE + mover + all six 0.05 & 1.5192 / 0.0257 & 1.4320 & 1.6341 / 0.0329 & $-0.26$/$-0.37$ \\
spectral PE + all six 0.05 & 1.5373 / 0.0256 & 1.4378 & 1.6649 / 0.0327 & $-0.18$/$+0.13$ \\
plain + all six 0.01 & 1.5565 / 0.0234 & 1.4479 & 1.7198 / 0.0319 & $-0.15$/$+0.20$ \\
plain + all six 0.05 & 1.5615 / 0.0263 & 1.4586 & 1.7001 / 0.0340 & $-0.32$/$-0.01$ \\
PE + mover + all six $0.1\,t$ & 1.5784 / 0.0246 & 1.4797 & 1.7064 / 0.0324 & $+0.00$/$+0.49$ \\
spectral PE + all six 0.01 & 1.5912 / 0.0234 & 1.4763 & 1.7489 / 0.0318 & $-0.06$/$+0.31$ \\
plain + all six $0.1\,t$ & 1.6175 / 0.0273 & 1.5094 & 1.7595 / 0.0350 & $-0.18$/$+0.03$ \\
plain + all six $0.05\,t$ & 1.6224 / 0.0259 & 1.5031 & 1.8007 / 0.0349 & $-0.25$/$+0.07$ \\
plain + boundary 0.01 & 1.6273 / 0.0247 & 1.5023 & 1.9471 / 0.0379 & $-0.11$/$+0.38$ \\
plain + all six 0.1 & 1.6312 / 0.0296 & 1.5305 & 1.7546 / 0.0363 & $-0.06$/$+0.27$ \\
plain + all six $0.2\,t$ & 1.6698 / 0.0323 & 1.5618 & 1.8238 / 0.0407 & $-0.04$/$-0.41$ \\
PE + mover + all six $0.2\,t$ & 1.6867 / 0.0316 & 1.5807 & 1.8199 / 0.0392 & $-0.17$/$-0.23$ \\
plain + boundary $0.05\,t$ & 1.7237 / 0.0295 & 1.5916 & 2.0439 / 0.0420 & $-0.25$/$+0.01$ \\
plain + boundary 0.05 & 1.7282 / 0.0309 & 1.6006 & 2.0425 / 0.0432 & $+0.12$/$+0.74$ \\
plain + area 0.05 & 1.8141 / 0.0295 & 1.6700 & 2.5639 / 0.0530 & $-0.19$/$-0.13$ \\
plain + area $0.05\,t$ & 1.8772 / 0.0280 & 1.7198 & 2.5955 / 0.0500 & $-0.21$/$+0.36$ \\
plain + area 0.1 & 1.8833 / 0.0358 & 1.7320 & 3.3476 / 0.0684 & $-0.44$/$-0.92$ \\
plain + overlap 0.1 & 1.8950 / 0.0180 & 1.7349 & 1.9991 / 0.0203 & $-0.06$/$+1.35$ \\
plain + boundary 0.1 & 1.9016 / 0.0385 & 1.7684 & 2.1868 / 0.0492 & $+0.09$/$+0.43$ \\
plain + overlap 0.05 & 1.9027 / 0.0195 & 1.7429 & 2.0049 / 0.0225 & $-0.11$/$+1.23$ \\
plain + area 0.01 & 1.9380 / 0.0265 & 1.7719 & 2.3380 / 0.0406 & $-0.20$/$+0.63$ \\
plain + overlap $0.05\,t$ & 1.9864 / 0.0211 & 1.8163 & 2.0836 / 0.0238 & $-0.09$/$+1.46$ \\
Trinity at $\lambda{=}0$: PE + mover & 2.0151 / 0.0222 & 1.8304 & 2.4271 / 0.0357 & $+0.22$/$+1.81$ \\
plain + cluster 0.05 & 2.0345 / 0.0251 & 1.8642 & 2.4845 / 0.0388 & $-0.13$/$+1.06$ \\
plain + overlap 0.01 & 2.0479 / 0.0242 & 1.8758 & 2.2334 / 0.0309 & $-0.04$/$+1.18$ \\
plain base + mover, $\lambda{=}0$ & 2.0498 / 0.0239 & 1.8560 & 2.4313 / 0.0362 & $-0.04$/$+1.48$ \\
plain + cluster 0.1 & 2.0643 / 0.0255 & 1.8923 & 2.5674 / 0.0402 & $-0.08$/$+1.03$ \\
plain + cluster $0.05\,t$ & 2.0754 / 0.0256 & 1.8912 & 2.5664 / 0.0405 & $-0.12$/$+1.26$ \\
plain + MIB 0.05 & 2.0812 / 0.0260 & 1.8994 & 2.4724 / 0.0385 & $-0.07$/$+1.37$ \\
the plain control: plain, $\lambda{=}0$ & 2.0915 / 0.0260 & 1.9074 & 2.4900 / 0.0386 & $-0.21$/$+1.04$ \\
plain + MIB 0.01 & 2.0960 / 0.0270 & 1.9125 & 2.4701 / 0.0392 & $-0.15$/$+1.05$ \\
plain + cluster 0.01 & 2.1048 / 0.0276 & 1.9237 & 2.5151 / 0.0405 & $-0.28$/$+0.76$ \\
the base: plain base, $\lambda{=}0$ & 2.1064 / 0.0262 & 1.9197 & 2.4906 / 0.0383 & $-0.24$/$+1.05$ \\
plain, heavy wire-dropout mixture & 2.1072 / 0.0262 & 1.9214 & 2.5229 / 0.0391 & $-0.07$/$+1.39$ \\
plain + wirelength $0.05\,t$ & 2.1072 / 0.0265 & 1.9147 & 2.7539 / 0.0449 & $-0.26$/$+1.20$ \\
plain + MIB $0.05\,t$ & 2.1140 / 0.0271 & 1.9221 & 2.5044 / 0.0396 & $-0.14$/$+1.24$ \\
plain + wirelength 0.01 & 2.1213 / 0.0269 & 1.9217 & 2.6143 / 0.0415 & $-0.20$/$+1.20$ \\
plain + wirelength 0.05 & 2.1297 / 0.0273 & 1.9440 & 2.8166 / 0.0466 & $-0.41$/$+1.14$ \\
spectral PE, $\lambda{=}0$ & 2.1305 / 0.0243 & 1.9404 & 2.5024 / 0.0363 & $+0.12$/$+1.91$ \\
plain, mild wire-dropout mixture & 2.1321 / 0.0259 & 1.9375 & 2.5308 / 0.0384 & $-0.08$/$+1.50$ \\
plain + MIB 0.1 & 2.1405 / 0.0279 & 1.9479 & 2.5551 / 0.0409 & $-0.19$/$+1.11$ \\
plain base + 4 adjacency heads & 2.1493 / 0.0279 & 1.9430 & 2.5585 / 0.0403 & $-0.27$/$+1.07$ \\
plain + wirelength 0.1 & 2.1499 / 0.0282 & 1.9586 & 3.1962 / 0.0532 & $-0.69$/$+0.98$ \\
plain base, $\epsilon$ head & 2.3872 / 0.0365 & 2.0053 & 152.43 / 0.0417 & $-0.08$/$+1.49$ \\
FlowPlace & 2.4564 / 0.0438 & 2.2110 & 11.465 / 0.1082 & $-0.23$/$+0.87$ \\
plain base, velocity head & 2.5970 / 0.0373 & 2.2538 & 77.23 / 0.2505 & $-0.47$/$+1.23$ \\
ChipDiffusion & 2.8020 / 0.0511 & 2.4824 & 151.57 / 0.0421 & $-0.08$/$+1.19$ \\
MacroDiff+ & 3.8458 / 0.0639 & 3.2835 & 150.20 / 0.0444 & $+0.24$/$+2.42$ \\
DiffPlace & 6.2293 / 0.1126 & 5.1245 & 149.79 / 0.0420 & $-0.29$/$+1.23$ \\
Reg-$xywh$ & -- & -- & 99.5 / 0.5884 & $-5.1$/$-14.5$ \\
Reg-$z$ & -- & -- & 21.5 / 0.4116 & $-1.6$/$+1.3$ \\
sequence-pair SA & -- & -- & 12.341 / 0.0000 & $+153.8$/$+111.0$ \\
\bottomrule
\end{tabular}
\end{table}

Trinity is the lowest raw model at $\nfe{=}32$ and the $3\times2$ backbone $\times$ term grid keeps its order (mover 0.01 $<$ mover 0.05 $<$ spectral 0.05 $<$ plain 0.01 $<$ plain 0.05). The overlap factor and the severity slots separate the term-trained models (1.51 to 1.63) from the recipes without the auxiliary loss (2.02 to 2.15) mainly through the soft slots, since overlap itself is 0.022 to 0.026 for all of them; the single-term overlap arms at 0.05 and 0.1 have the lowest overlap of the study (0.018 to 0.020) and land at 1.90.

\section{Refinement Results in Full}
\label{sec:appendix-refine}

\subsection{Trinity With and Without the Auxiliary Loss}
\label{sec:appendix-refine-grid}

\begin{table}[htbp]\centering\scriptsize
\caption{The full $\nfe \times$ steps grid, Trinity and Trinity at $\lambda{=}0$ (the same backbone without the auxiliary loss), two lines per cell: soft cost / overlap ratio, then wirelength gap / area gap to the reference in percent (positive above the reference).}
\label{tab:app_grid_full}
\label{tab:app_grid_overlap}
\label{tab:app_grid_hpwl}
\label{tab:app_grid_area}
\setlength{\tabcolsep}{2pt}
\begin{tabular}{llrrrrrrr}
\toprule
Model & $\nfe$ & 0 steps & 10 & 25 & 50 & 100 & 200 & 400 \\
\midrule
Trinity & 1 & 1.679\,/\,0.0316 & 1.517\,/\,0.0201 & 1.309\,/\,0.0152 & 1.208\,/\,0.0115 & 1.149\,/\,0.0083 & 1.104\,/\,0.0056 & 1.066\,/\,0.0034 \\
 &  & $-1.2$\,/\,$-0.1$ & $-1.0$\,/\,$+1.0$ & $-0.8$\,/\,$+1.1$ & $-0.6$\,/\,$+1.1$ & $-0.4$\,/\,$+1.2$ & $-0.3$\,/\,$+1.2$ & $-0.3$\,/\,$+1.2$ \\
\addlinespace[1.5pt]
 & 2 & 1.660\,/\,0.0305 & 1.480\,/\,0.0188 & 1.289\,/\,0.0141 & 1.194\,/\,0.0106 & 1.137\,/\,0.0076 & 1.096\,/\,0.0052 & 1.062\,/\,0.0031 \\
 &  & $-1.2$\,/\,$-0.0$ & $-0.9$\,/\,$+1.0$ & $-0.7$\,/\,$+1.1$ & $-0.6$\,/\,$+1.1$ & $-0.4$\,/\,$+1.2$ & $-0.4$\,/\,$+1.2$ & $-0.3$\,/\,$+1.1$ \\
\addlinespace[1.5pt]
 & 4 & 1.631\,/\,0.0288 & 1.460\,/\,0.0173 & 1.274\,/\,0.0128 & 1.182\,/\,0.0096 & 1.129\,/\,0.0068 & 1.090\,/\,0.0047 & 1.059\,/\,0.0029 \\
 &  & $-1.0$\,/\,$+0.0$ & $-0.7$\,/\,$+1.1$ & $-0.5$\,/\,$+1.2$ & $-0.4$\,/\,$+1.2$ & $-0.3$\,/\,$+1.2$ & $-0.2$\,/\,$+1.2$ & $-0.2$\,/\,$+1.0$ \\
\addlinespace[1.5pt]
 & 8 & 1.598\,/\,0.0271 & 1.435\,/\,0.0157 & 1.257\,/\,0.0116 & 1.171\,/\,0.0086 & 1.123\,/\,0.0062 & 1.088\,/\,0.0043 & 1.059\,/\,0.0028 \\
 &  & $-0.7$\,/\,$+0.1$ & $-0.4$\,/\,$+1.2$ & $-0.2$\,/\,$+1.2$ & $-0.1$\,/\,$+1.2$ & $-0.0$\,/\,$+1.2$ & $-0.0$\,/\,$+1.2$ & $-0.0$\,/\,$+1.0$ \\
\addlinespace[1.5pt]
 & 16 & 1.554\,/\,0.0248 & 1.401\,/\,0.0137 & 1.236\,/\,0.0099 & 1.159\,/\,0.0074 & 1.117\,/\,0.0054 & 1.086\,/\,0.0039 & 1.059\,/\,0.0026 \\
 &  & $-0.3$\,/\,$+0.2$ & $-0.1$\,/\,$+1.2$ & $+0.1$\,/\,$+1.3$ & $+0.1$\,/\,$+1.3$ & $+0.2$\,/\,$+1.3$ & $+0.2$\,/\,$+1.2$ & $+0.2$\,/\,$+1.0$ \\
\addlinespace[1.5pt]
 & 32 & 1.514\,/\,0.0225 & 1.369\,/\,0.0117 & 1.216\,/\,0.0084 & 1.147\,/\,0.0064 & 1.111\,/\,0.0048 & 1.083\,/\,0.0035 & 1.058\,/\,0.0023 \\
 &  & $-0.1$\,/\,$+0.2$ & $+0.1$\,/\,$+1.3$ & $+0.3$\,/\,$+1.4$ & $+0.4$\,/\,$+1.4$ & $+0.4$\,/\,$+1.4$ & $+0.4$\,/\,$+1.3$ & $+0.3$\,/\,$+1.1$ \\
\midrule
Trinity, $\lambda{=}0$ & 1 & 2.427\,/\,0.0357 & 1.578\,/\,0.0229 & 1.355\,/\,0.0173 & 1.234\,/\,0.0127 & 1.155\,/\,0.0087 & 1.103\,/\,0.0057 & 1.065\,/\,0.0034 \\
 &  & $-1.1$\,/\,$+1.6$ & $-0.9$\,/\,$+1.6$ & $-0.7$\,/\,$+1.6$ & $-0.5$\,/\,$+1.6$ & $-0.4$\,/\,$+1.6$ & $-0.3$\,/\,$+1.5$ & $-0.3$\,/\,$+1.4$ \\
\addlinespace[1.5pt]
 & 2 & 2.364\,/\,0.0339 & 1.540\,/\,0.0212 & 1.330\,/\,0.0159 & 1.218\,/\,0.0116 & 1.146\,/\,0.0080 & 1.098\,/\,0.0053 & 1.063\,/\,0.0032 \\
 &  & $-1.1$\,/\,$+1.6$ & $-0.8$\,/\,$+1.6$ & $-0.7$\,/\,$+1.6$ & $-0.5$\,/\,$+1.6$ & $-0.3$\,/\,$+1.6$ & $-0.3$\,/\,$+1.5$ & $-0.3$\,/\,$+1.3$ \\
\addlinespace[1.5pt]
 & 4 & 2.253\,/\,0.0308 & 1.483\,/\,0.0184 & 1.293\,/\,0.0136 & 1.195\,/\,0.0100 & 1.132\,/\,0.0069 & 1.091\,/\,0.0047 & 1.060\,/\,0.0029 \\
 &  & $-0.8$\,/\,$+1.6$ & $-0.6$\,/\,$+1.6$ & $-0.4$\,/\,$+1.6$ & $-0.3$\,/\,$+1.6$ & $-0.2$\,/\,$+1.5$ & $-0.1$\,/\,$+1.4$ & $-0.2$\,/\,$+1.1$ \\
\addlinespace[1.5pt]
 & 8 & 2.164\,/\,0.0279 & 1.436\,/\,0.0159 & 1.263\,/\,0.0117 & 1.177\,/\,0.0085 & 1.123\,/\,0.0060 & 1.087\,/\,0.0041 & 1.059\,/\,0.0027 \\
 &  & $-0.4$\,/\,$+1.7$ & $-0.2$\,/\,$+1.7$ & $-0.1$\,/\,$+1.7$ & $+0.0$\,/\,$+1.6$ & $+0.1$\,/\,$+1.6$ & $+0.1$\,/\,$+1.4$ & $+0.0$\,/\,$+1.1$ \\
\addlinespace[1.5pt]
 & 16 & 2.083\,/\,0.0250 & 1.392\,/\,0.0134 & 1.235\,/\,0.0097 & 1.161\,/\,0.0071 & 1.116\,/\,0.0051 & 1.085\,/\,0.0037 & 1.059\,/\,0.0025 \\
 &  & $-0.1$\,/\,$+1.7$ & $+0.1$\,/\,$+1.7$ & $+0.2$\,/\,$+1.7$ & $+0.3$\,/\,$+1.7$ & $+0.3$\,/\,$+1.6$ & $+0.3$\,/\,$+1.4$ & $+0.2$\,/\,$+1.1$ \\
\addlinespace[1.5pt]
 & 32 & 2.015\,/\,0.0222 & 1.356\,/\,0.0110 & 1.212\,/\,0.0079 & 1.148\,/\,0.0059 & 1.109\,/\,0.0044 & 1.082\,/\,0.0032 & 1.058\,/\,0.0022 \\
 &  & $+0.2$\,/\,$+1.8$ & $+0.4$\,/\,$+1.8$ & $+0.5$\,/\,$+1.8$ & $+0.5$\,/\,$+1.8$ & $+0.5$\,/\,$+1.7$ & $+0.5$\,/\,$+1.5$ & $+0.4$\,/\,$+1.2$ \\
\bottomrule
\end{tabular}
\end{table}

Steps to threshold, the smallest step count with soft cost at most $\tau$, per $\nfe = 1/2/4/8/16/32$: $\tau = 1.3$: Trinity 50 / 25 / 25 / 25 / 25 / 25, Trinity at $\lambda{=}0$ 50 / 50 / 25 / 25 / 25 / 25; $\tau = 1.2$: 100 / 50 / 50 / 50 / 50 / 50 against 100 / 100 / 50 / 50 / 50 / 50; $\tau = 1.15$: 100 / 100 / 100 / 100 / 100 / 50 against 200 / 100 / 100 / 100 / 100 / 50; $\tau = 1.1$: 400 / 200 / 200 / 200 / 200 / 200 on both. Every one of the 42 cells is Pareto-optimal on both models. At fixed steps the gain from $\nfe{=}1$ to 32 is 0.16 raw, 0.09 at 25 steps, 0.04 at 100, 0.02 at 200 and 0.008 at 400; at fixed $\nfe$ each doubling of steps gains 0.04 to 0.10. Trinity's advantage over Trinity at $\lambda{=}0$ is the raw sample (1.51 against 2.02 at $\nfe{=}32$, 1.68 against 2.43 at $\nfe{=}1$) and the low-$\nfe$ cells through 50 to 100 steps; the raw samples differ in the soft slots and the outline ($\mathrm{sev}_{\mathrm{rel}}$ 0.085 and area gap 0.002 against 0.221 and 0.018) at the same overlap, and ten refiner steps bring Trinity at $\lambda{=}0$ to Trinity's raw $\mathrm{sev}_{\mathrm{rel}}$.

\begin{table}[htbp]\centering\footnotesize
\caption{Every soft column along the descent on Trinity at the setting of Appendix~\ref{sec:appendix-refiner} (Trinity at $\lambda{=}0$ is within 0.005 on every column from 25 steps on). Units as in Section~\ref{sec:exp-setup}.}
\label{tab:app_trajectory}
\setlength{\tabcolsep}{4pt}
\begin{tabular}{rrrrrrrrrr}
\toprule
$\nfe$ & Steps & Overlap & Boundary & Group & MIB & HPWL gap & Area gap & $\mathrm{sev}_{\mathrm{rel}}$ & Soft cost \\
\midrule
1 & 0 & 0.0316 & 0.8 & 1.8 & 1.7 & $-$0.0124 & $-$0.0013 & 0.0819 & 1.6785 \\
1 & 10 & 0.0201 & 1.4 & 1.3 & 0.0080 & $-$0.0096 & 0.0100 & 0.0911 & 1.5170 \\
1 & 25 & 0.0152 & 0.7 & 1.1 & $5.7\times10^{-4}$ & $-$0.0075 & 0.0106 & 0.0464 & 1.3090 \\
1 & 50 & 0.0115 & 0.4 & 0.9 & $5.1\times10^{-5}$ & $-$0.0058 & 0.0113 & 0.0279 & 1.2084 \\
1 & 100 & 0.0083 & 0.3 & 0.8 & $1.2\times10^{-5}$ & $-$0.0043 & 0.0122 & 0.0206 & 1.1490 \\
1 & 400 & 0.0034 & 0.1 & 0.5 & $7.0\times10^{-7}$ & $-$0.0029 & 0.0116 & 0.0100 & 1.0661 \\
32 & 0 & 0.0225 & 0.8 & 2.3 & 1.7 & $-$0.0007 & 0.0022 & 0.0849 & 1.5137 \\
32 & 10 & 0.0117 & 1.2 & 1.8 & 0.010 & 0.0014 & 0.0129 & 0.0874 & 1.3695 \\
32 & 25 & 0.0084 & 0.6 & 1.6 & $3.3\times10^{-4}$ & 0.0028 & 0.0135 & 0.0464 & 1.2159 \\
32 & 50 & 0.0064 & 0.3 & 1.5 & $5.1\times10^{-5}$ & 0.0036 & 0.0137 & 0.0289 & 1.1473 \\
32 & 100 & 0.0048 & 0.2 & 1.3 & $1.2\times10^{-5}$ & 0.0040 & 0.0137 & 0.0215 & 1.1114 \\
32 & 200 & 0.0035 & 0.1 & 1.2 & $3.0\times10^{-6}$ & 0.0038 & 0.0128 & 0.0158 & 1.0833 \\
32 & 400 & 0.0023 & 0.1 & 1.0 & $7.0\times10^{-7}$ & 0.0032 & 0.0106 & 0.0110 & 1.0583 \\
\bottomrule
\end{tabular}
\end{table}

Overlap, the cluster gap, the MIB variance and the fixed deviation fall monotonically through 400 steps on both models and overlap does not turn up; on Trinity the boundary distance and $\mathrm{sev}_{\mathrm{rel}}$ rise at 10 steps (0.8 to 1.2 and 0.085 to 0.087 at $\nfe{=}32$) and fall from there, and the remaining cost at 400 steps is wirelength and area at $+0.3\%$ and $+1.1\%$ of the reference and a $\mathrm{sev}_{\mathrm{rel}}$ of 0.011. Along the $\nfe$ axis at 400 steps the columns trade: overlap and the boundary distance fall with $\nfe$ while the cluster gap grows (0.5 to 1.0).

\subsection{Every Model on the Reduced Grid}
\label{sec:appendix-refine-all}

\begin{table}[htbp]\centering\scriptsize
\caption{The reduced grid for every model, one column per $\nfe$ / refiner steps, two lines per model: soft cost / overlap ratio with the rule (the mean soft cost over the nine refined cells; the regressors' raw soft cost is the per-case median, their means overflow), then wirelength gap / area gap to the reference in percent (positive above the reference).}
\label{tab:app_reduced}
\label{tab:app_reduced_overlap}
\label{tab:app_reduced_hpwl}
\label{tab:app_reduced_area}
\setlength{\tabcolsep}{2pt}
\begin{tabular}{lrrrrrrr}
\toprule
Model & 32 / 0 & 32 / 25 & 32 / 100 & 32 / 400 & 1 / 0 & 1 / 400 & Rule \\
\midrule
Trinity & 1.514\,/\,0.0225 & 1.216\,/\,0.0084 & 1.111\,/\,0.0048 & 1.058\,/\,0.0023 & 1.679\,/\,0.0316 & 1.066\,/\,0.0034 & -- \\
 & $-0.1$\,/\,$+0.2$ & $+0.3$\,/\,$+1.4$ & $+0.4$\,/\,$+1.4$ & $+0.3$\,/\,$+1.1$ & $-1.2$\,/\,$-0.1$ & $-0.3$\,/\,$+1.2$ &  \\
\addlinespace[1.5pt]
PE + mover, $\lambda{=}0$ & 2.015\,/\,0.0222 & 1.212\,/\,0.0079 & 1.109\,/\,0.0044 & 1.058\,/\,0.0022 & 2.427\,/\,0.0357 & 1.065\,/\,0.0034 & -- \\
 & $+0.2$\,/\,$+1.8$ & $+0.5$\,/\,$+1.8$ & $+0.5$\,/\,$+1.7$ & $+0.4$\,/\,$+1.2$ & $-1.1$\,/\,$+1.6$ & $-0.3$\,/\,$+1.4$ &  \\
\addlinespace[1.5pt]
PE + mover + all six 0.05 & 1.519\,/\,0.0257 & 1.232\,/\,0.0101 & 1.125\,/\,0.0060 & 1.066\,/\,0.0028 & 1.634\,/\,0.0329 & 1.077\,/\,0.0038 & 1.1587 \\
 & $-0.3$\,/\,$-0.4$ & $+0.1$\,/\,$+1.0$ & $+0.3$\,/\,$+1.3$ & $+0.4$\,/\,$+1.3$ & $-1.5$\,/\,$-0.9$ & $-0.4$\,/\,$+1.2$ &  \\
\addlinespace[1.5pt]
PE + mover + all six $0.1\,t$ & 1.578\,/\,0.0246 & 1.244\,/\,0.0098 & 1.128\,/\,0.0058 & 1.069\,/\,0.0027 & 1.706\,/\,0.0324 & 1.082\,/\,0.0039 & 1.1669 \\
 & $+0.0$\,/\,$+0.5$ & $+0.4$\,/\,$+1.7$ & $+0.6$\,/\,$+1.8$ & $+0.5$\,/\,$+1.6$ & $-1.4$\,/\,$-0.2$ & $-0.3$\,/\,$+1.4$ &  \\
\addlinespace[1.5pt]
PE + mover + all six $0.2\,t$ & 1.687\,/\,0.0316 & 1.307\,/\,0.0136 & 1.158\,/\,0.0079 & 1.083\,/\,0.0034 & 1.820\,/\,0.0392 & 1.097\,/\,0.0047 & 1.2015 \\
 & $-0.2$\,/\,$-0.2$ & $+0.3$\,/\,$+1.3$ & $+0.5$\,/\,$+1.7$ & $+0.6$\,/\,$+1.8$ & $-1.7$\,/\,$-0.9$ & $-0.4$\,/\,$+1.6$ &  \\
\addlinespace[1.5pt]
plain + all six 0.01 & 1.556\,/\,0.0234 & 1.217\,/\,0.0091 & 1.115\,/\,0.0053 & 1.062\,/\,0.0026 & 1.720\,/\,0.0319 & 1.066\,/\,0.0034 & \textbf{1.1504} \\
 & $-0.1$\,/\,$+0.2$ & $+0.2$\,/\,$+1.2$ & $+0.3$\,/\,$+1.3$ & $+0.3$\,/\,$+1.1$ & $-1.3$\,/\,$-0.0$ & $-0.3$\,/\,$+1.2$ &  \\
\addlinespace[1.5pt]
plain + all six 0.05 & 1.562\,/\,0.0263 & 1.248\,/\,0.0110 & 1.132\,/\,0.0064 & 1.069\,/\,0.0030 & 1.700\,/\,0.0340 & 1.079\,/\,0.0040 & 1.1703 \\
 & $-0.3$\,/\,$-0.0$ & $+0.1$\,/\,$+1.2$ & $+0.3$\,/\,$+1.4$ & $+0.3$\,/\,$+1.3$ & $-1.5$\,/\,$-0.4$ & $-0.4$\,/\,$+1.2$ &  \\
\addlinespace[1.5pt]
spectral PE + all six 0.05 & 1.537\,/\,0.0256 & 1.264\,/\,0.0109 & 1.137\,/\,0.0064 & 1.072\,/\,0.0029 & 1.665\,/\,0.0327 & 1.085\,/\,0.0040 & 1.1783 \\
 & $-0.2$\,/\,$+0.1$ & $+0.2$\,/\,$+1.4$ & $+0.4$\,/\,$+1.6$ & $+0.4$\,/\,$+1.4$ & $-1.3$\,/\,$-0.2$ & $-0.3$\,/\,$+1.4$ &  \\
\addlinespace[1.5pt]
the plain control & 2.091\,/\,0.0260 & 1.214\,/\,0.0092 & 1.114\,/\,0.0053 & 1.061\,/\,0.0026 & 2.490\,/\,0.0386 & 1.068\,/\,0.0037 & 1.1580 \\
 & $-0.2$\,/\,$+1.0$ & $+0.1$\,/\,$+1.1$ & $+0.3$\,/\,$+1.2$ & $+0.3$\,/\,$+1.1$ & $-1.5$\,/\,$+0.9$ & $-0.4$\,/\,$+1.2$ &  \\
\addlinespace[1.5pt]
the base (no wire dropout) & 2.106\,/\,0.0262 & 1.219\,/\,0.0094 & 1.115\,/\,0.0053 & 1.061\,/\,0.0026 & 2.491\,/\,0.0383 & 1.068\,/\,0.0037 & 1.1590 \\
 & $-0.2$\,/\,$+1.1$ & $+0.1$\,/\,$+1.1$ & $+0.3$\,/\,$+1.2$ & $+0.2$\,/\,$+1.0$ & $-1.5$\,/\,$+0.9$ & $-0.4$\,/\,$+1.2$ &  \\
\addlinespace[1.5pt]
FlowPlace & 2.456\,/\,0.0438 & 1.415\,/\,0.0221 & 1.222\,/\,0.0128 & 1.107\,/\,0.0050 & 11.5\,/\,0.1082 & 1.134\,/\,0.0066 & 1.5003 \\
 & $-0.2$\,/\,$+0.9$ & $+0.2$\,/\,$+1.3$ & $+0.6$\,/\,$+1.9$ & $+0.9$\,/\,$+2.5$ & $-1.5$\,/\,$+6.4$ & $+0.4$\,/\,$+3.3$ &  \\
\addlinespace[1.5pt]
ChipDiffusion & 2.802\,/\,0.0511 & 1.516\,/\,0.0272 & 1.271\,/\,0.0158 & 1.131\,/\,0.0062 & 151.6\,/\,0.0421 & 72.6\,/\,0.0039 & 32.15 \\
 & $-0.1$\,/\,$+1.2$ & $+0.4$\,/\,$+1.5$ & $+0.8$\,/\,$+2.2$ & $+1.2$\,/\,$+3.0$ & $+593.1$\,/\,$+1982.8$ & $+499.2$\,/\,$+1614.3$ &  \\
\addlinespace[1.5pt]
DiffPlace & 6.229\,/\,0.1126 & 2.606\,/\,0.0683 & 1.570\,/\,0.0330 & 1.202\,/\,0.0111 & 149.8\,/\,0.0420 & 71.4\,/\,0.0039 & 32.31 \\
 & $-0.3$\,/\,$+1.2$ & $-0.0$\,/\,$+1.6$ & $+0.6$\,/\,$+2.5$ & $+1.4$\,/\,$+3.7$ & $+583.8$\,/\,$+1986.3$ & $+487.4$\,/\,$+1617.2$ &  \\
\addlinespace[1.5pt]
MacroDiff+ & 3.846\,/\,0.0639 & 1.714\,/\,0.0351 & 1.348\,/\,0.0197 & 1.167\,/\,0.0078 & 150.2\,/\,0.0444 & 69.4\,/\,0.0041 & 31.56 \\
 & $+0.2$\,/\,$+2.4$ & $+0.7$\,/\,$+2.1$ & $+1.3$\,/\,$+2.9$ & $+1.8$\,/\,$+3.7$ & $+565.9$\,/\,$+1946.8$ & $+472.0$\,/\,$+1578.1$ &  \\
\addlinespace[1.5pt]
SP-SA (one shard) & -- & -- & -- & -- & 12.3\,/\,0.0000 & 10.9\,/\,0.0392 & 12.61 \\
 & -- & -- & -- & -- & $+153.8$\,/\,$+111.0$ & $+126.9$\,/\,$+64.2$ &  \\
\addlinespace[1.5pt]
Reg-$z$ (one layout) & -- & -- & -- & -- & 21.5\,/\,0.4116 & 42.5\,/\,0.0500 & -- \\
 & -- & -- & -- & -- & $-1.6$\,/\,$+1.3$ & $+4.5$\,/\,$+1.9$ &  \\
\addlinespace[1.5pt]
Reg-$xywh$ (one layout) & -- & -- & -- & -- & 99.5\,/\,0.5884 & 35,214\,/\,0.0800 & -- \\
 & -- & -- & -- & -- & $-5.1$\,/\,$-14.5$ & $+5.0$\,/\,$-2.2$ &  \\
\bottomrule
\end{tabular}
\end{table}

\begin{table}[htbp]\centering\footnotesize
\caption{Per-case median / 90th percentile of the soft cost along the descent, for the lines whose means are dominated by a few cases.}
\label{tab:app_medians}
\setlength{\tabcolsep}{7pt}
\begin{tabular}{lrrrrr}
\toprule
Model & $\nfe$ & 0 steps & 25 & 100 & 400 \\
\midrule
Trinity & 1 & 1.492 / 2.329 & 1.224 / 1.647 & 1.102 / 1.331 & 1.038 / 1.162 \\
Trinity & 32 & 1.423 / 1.876 & 1.174 / 1.421 & 1.079 / 1.244 & 1.033 / 1.144 \\
Trinity, $\lambda{=}0$ & 1 & 2.014 / 3.704 & 1.203 / 1.817 & 1.091 / 1.365 & 1.034 / 1.163 \\
Trinity, $\lambda{=}0$ & 32 & 1.888 / 2.611 & 1.161 / 1.426 & 1.074 / 1.245 & 1.032 / 1.146 \\
FlowPlace & 1 & 7.340 / 16.53 & 2.122 / 4.766 & 1.235 / 1.907 & 1.083 / 1.297 \\
FlowPlace & 32 & 2.246 / 3.352 & 1.319 / 1.809 & 1.163 / 1.463 & 1.072 / 1.238 \\
ChipDiffusion & 1 & 152.7 / 182.5 & 117.0 / 143.3 & 93.8 / 114.0 & 73.6 / 90.8 \\
ChipDiffusion & 32 & 2.578 / 3.864 & 1.418 / 1.961 & 1.214 / 1.526 & 1.096 / 1.270 \\
DiffPlace & 1 & 151.3 / 180.3 & 116.5 / 142.1 & 93.0 / 112.8 & 72.5 / 89.7 \\
DiffPlace & 32 & 5.038 / 11.31 & 2.203 / 4.340 & 1.437 / 2.165 & 1.162 / 1.402 \\
MacroDiff+ & 1 & 151.4 / 180.7 & 115.2 / 140.4 & 90.9 / 110.3 & 70.1 / 86.6 \\
MacroDiff+ & 32 & 3.487 / 5.465 & 1.576 / 2.297 & 1.278 / 1.644 & 1.127 / 1.326 \\
sequence-pair SA & -- & 12.55 / 14.91 & 13.93 / 17.24 & 13.27 / 16.14 & 10.97 / 13.53 \\
Reg-$z$ & -- & 21.53 / 24,599 & 5.618 / 1,754 & 2.253 / 55.1 & 1.367 / 3.966 \\
\bottomrule
\end{tabular}
\end{table}

Across the models of this work the refined endpoints are within 0.01 of each other (1.058 to 1.072 at $\nfe{=}32$ and 400 steps, 1.066 to 1.097 at $\nfe{=}1$), and the plain backbone with all six terms at 0.01 has the lowest reduced-grid rule. Through the legalizer at 400 steps with the overlap weight at 30, Trinity at $\lambda{=}0$ reaches a first-draw / best-of-4 hard cost of 1.192 / 1.107 at $\nfe{=}1$ and 1.206 / 1.076 at $\nfe{=}32$, against 1.193 / 1.111 and 1.214 / 1.080 for Trinity on the same 12,000 cases; every paired difference is at most 0.008 and 2.5 paired standard errors, below the seed std of the raw soft cost (0.015 to 0.037), so at this budget the training term leaves no end-to-end difference beyond the seed spread. This reproduces the effect the stage-by-stage evaluation is built to expose: the full pipeline absorbs the weaker raw samples of an early component, so the endpoint alone cannot tell the two generators apart, while the raw and short-budget stages do (Tables~\ref{tab:nfe} and~\ref{tab:refine}). The published-method generators refine much more slowly on the same refiner (1.107, 1.131, 1.167 and 1.202 at $\nfe{=}32$ and 400 steps against 1.058), start far worse, and DiffPlace keeps an overlap of 0.011 after 400 steps: the generator's quality survives 400 refiner steps. The annealer's layouts, legal and overlap-free at the start, move under the wirelength and area pull (overlap 0 to 0.039) while their soft cost falls from 12.3 to 10.9. The regressors are out of the refiner's reach (overlap 0.41 to 0.59 raw and 0.05 to 0.08 after 400 steps).

\section{The Ported Placers With Their Own In-Sampler Guidance}
\label{sec:appendix-guided}

Three of the four ported placers also correct the sample inside their samplers, and the generator comparison of Section~\ref{sec:exp-generator} leaves those corrections out so that every network meets the same Euler sampler. ChipDiffusion, DiffPlace and MacroDiff+ are sampled with the guidance their released code applies at every step, ported onto our sampler as Table~\ref{tab:guided_ports} describes, on the 12,000 val set at $\nfe \in \{1, 2, 4, 8, 16, 32\}$ with four draws from the noise of their free shards, so that every guided draw pairs with a free one. FlowPlace's in-sampler step is a projection onto the hard constraints that its own evaluation configuration disables, and it is not reproduced. The guided samples are scored on the fifteen soft columns, refined by the closed-form refiner on the reduced grid of Appendix~\ref{sec:appendix-refine-all}, passed through the three ported loops at $\nfe{=}32$, and measured against the reference distribution as in Appendix~\ref{sec:appendix-scoring}. The legalizer is not run on them.

DiffPlace's released update raises overlap. Its overlap loss is the positive pairwise overlap area and its sampler subtracts the loss gradient from the predicted noise, which moves the implied clean sample up that gradient. The port runs the released sign as the reproduction and a reference arm with the sign flipped, whose per-block normalization still pushes every overlapping block by the same $0.03\,\lambda_t$ whatever its overlap.

\begin{table}[htbp]\centering\scriptsize
\caption{What each released sampler does to the sample at every step, and how it is ported onto our Euler sampler. Every port moves positions only, keeps the pre-placed blocks fixed, works on the post-hoc ports' canvas (the predicted layout's bounding box plus 5\%), and guides the final prediction as well.}
\label{tab:guided_ports}
\begin{tabular}{l>{\raggedright\arraybackslash}p{6.4cm}>{\raggedright\arraybackslash}p{4.6cm}}
\toprule
Placer & Inside its released sampler & The port \\
\midrule
ChipDiffusion & Every step, ten Adam iterations at learning rate $8\times10^{-3}$ on the predicted clean sample of $\alpha\cdot\text{legality} + 10^{-4}\cdot\text{HPWL}$, with $\alpha$ adapted per sample from 0 while $t < 0.5$, and the diffusion step continues from the guided prediction. & The same loop on the post-hoc port's legality and wirelength functions, the next state the Euler step plus the scaled correction. \\
DiffPlace & From $1 - t \ge 0.85$ on, the predicted noise is moved by $\lambda_t \nabla\,\text{overlap}(x_t)$ with $\lambda_t = 200\,(1-t)^2$, the gradient rescaled per block to 0.03 and clipped at 0.08. The released update moves the sample up the overlap gradient. & The same on the port's overlap term, the step continuing from the guided noise. A reference arm flips the sign of the update. \\
MacroDiff+ & Every DDIM step, an Adam loop on the predicted clean sample with HPWL and overlap phases, clipping and a die clamp, 300 to 700 iterations per design in its table and 500 by default. & 500 iterations at learning rate 0.01 each step, the post-hoc port's settings, the step continuing from the guided prediction. \\
FlowPlace & A projection onto the hard constraints inside the sampler, disabled in its evaluation configuration. & Not reproduced. \\
\bottomrule
\end{tabular}
\end{table}

\subsection{Raw Samples}
\label{sec:appendix-guided-raw}

Figure~\ref{fig:guided_nfe} and Table~\ref{tab:guided_raw} give the raw samples. Each guidance does what its code targets and pays in the terms it does not. At $\nfe{=}32$ ChipDiffusion's guidance cuts overlap from 0.051 to 0.018 and MacroDiff+'s from 0.064 to 0.0008, while DiffPlace's released update raises it from 0.113 to 0.328. On the same samples the area gap grows nine to thirteen times, to 0.13 to 0.22 of the reference area, the wirelength gap goes from within 0.3\% of the reference to 3 to 13\% above it, the cluster gap grows 1.8 to 2.5 times on all three, and the boundary distance grows five to fourteen times on ChipDiffusion and DiffPlace, because the corrections separate blocks without a term that holds clusters together or the outline closed. Net on the soft cost, the guidance helps where the raw draw is poor and hurts where it is good. ChipDiffusion's soft cost halves at $\nfe = 1$ to 4 (151.6 to 88.3 at $\nfe{=}1$, 11.9 to 5.96 at 4) and rises from $\nfe{=}8$ on (3.83 to 3.97, and 2.80 to 3.78 at 32), so its guided $\nfe{=}32$ sample stands about where its free $\nfe{=}8$ sample already does. MacroDiff+ gains at every $\nfe$, 3.85 to 2.04 at 32, its 500 Adam iterations per step leaving an overlap of 0.0008. DiffPlace's released update loses everywhere, 6.23 to 132 at $\nfe{=}32$, and the flipped sign is below free sampling at $\nfe = 1$ to 4 but far above it from $\nfe{=}8$ on (40.1 against 6.23 at 32), because the normalized push scatters a nearly finished layout. No guided sample reaches Trinity's free sample at 1.514, nor the 192 / 12 rung at 1.827.

\begin{figure}[htbp]\centering
\includegraphics[width=\linewidth]{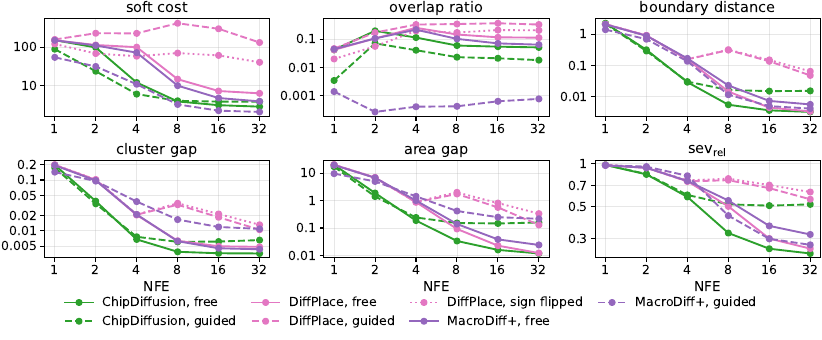}
\caption{The raw samples against $\nfe$, free sampling solid, the placer's own guidance dashed, DiffPlace's update with the sign flipped dotted: the soft cost, the overlap ratio, the boundary distance, the cluster gap, the area gap and $\mathrm{sev}_{\mathrm{rel}}$, log axes.}
\label{fig:guided_nfe}
\end{figure}

\begin{table}[htbp]\centering\scriptsize
\caption{The raw samples of the three placers at every $\nfe$, free (our Euler sampler) and guided (their own in-sampler correction), and DiffPlace's update with its sign flipped: the soft columns and the soft cost, means over four draws and the 12,000 cases. Boundary, cluster and fixed in $10^{-3}$, MIB in $10^{-6}$, the gaps as fractions of the reference.}
\label{tab:guided_raw}
\setlength{\tabcolsep}{3pt}
\begin{tabular}{llrrrrrrrrrr}
\toprule
Placer & Sampling & $\nfe$ & Overlap & Boundary & Cluster & MIB & Fixed & HPWL gap & Area gap & $\mathrm{sev}_{\mathrm{rel}}$ & Soft cost \\
\midrule
ChipDiffusion & free & 1 & 0.0421 & 2,218 & 198 & 347,608 & 1,304 & $+5.931$ & $+19.8$ & 0.976 & 151.6 \\
 &  & 2 & 0.1961 & 316 & 38.3 & 13,364 & 216 & $+0.841$ & $+1.870$ & 0.840 & 97.81 \\
 &  & 4 & 0.1138 & 29.4 & 6.77 & 275 & 26.3 & $+0.053$ & $+0.186$ & 0.585 & 11.91 \\
 &  & 8 & 0.0595 & 5.59 & 3.86 & 38.1 & 4.40 & $+0.003$ & $+0.034$ & 0.327 & 3.829 \\
 &  & 16 & 0.0543 & 3.70 & 3.61 & 19.6 & 2.50 & $-0.002$ & $+0.016$ & 0.254 & 3.057 \\
 &  & 32 & 0.0511 & 3.30 & 3.58 & 16.0 & 2.20 & $-0.001$ & $+0.012$ & 0.236 & 2.802 \\
\addlinespace[1.5pt]
 & guided & 1 & 0.0035 & 2,135 & 174 & 347,506 & 1,304 & $+4.905$ & $+17.0$ & 0.979 & 88.29 \\
 &  & 2 & 0.0711 & 289 & 34.0 & 13,320 & 216 & $+0.655$ & $+1.398$ & 0.844 & 23.49 \\
 &  & 4 & 0.0407 & 30.5 & 7.62 & 275 & 26.3 & $+0.053$ & $+0.241$ & 0.604 & 5.956 \\
 &  & 8 & 0.0231 & 16.6 & 6.07 & 39.1 & 4.63 & $+0.025$ & $+0.152$ & 0.519 & 3.974 \\
 &  & 16 & 0.0211 & 15.2 & 6.15 & 20.6 & 2.77 & $+0.027$ & $+0.146$ & 0.508 & 3.794 \\
 &  & 32 & 0.0182 & 15.5 & 6.56 & 17.6 & 2.66 & $+0.033$ & $+0.158$ & 0.519 & 3.782 \\
\midrule
DiffPlace & free & 1 & 0.0420 & 2,067 & 198 & 358,943 & 1,243 & $+5.838$ & $+19.9$ & 0.973 & 149.8 \\
 &  & 2 & 0.1062 & 900 & 100 & 95,322 & 588 & $+2.659$ & $+6.806$ & 0.935 & 113.5 \\
 &  & 4 & 0.2510 & 157 & 20.5 & 4,838 & 131 & $+0.367$ & $+0.881$ & 0.752 & 99.49 \\
 &  & 8 & 0.1457 & 14.7 & 6.37 & 175 & 16.3 & $+0.024$ & $+0.092$ & 0.502 & 14.56 \\
 &  & 16 & 0.1162 & 4.52 & 4.96 & 45.9 & 3.75 & $-0.001$ & $+0.023$ & 0.299 & 7.159 \\
 &  & 32 & 0.1126 & 3.43 & 4.83 & 31.3 & 2.34 & $-0.003$ & $+0.012$ & 0.255 & 6.229 \\
\addlinespace[1.5pt]
 & guided & 1 & 0.0463 & 2,067 & 198 & 358,942 & 1,243 & $+5.837$ & $+19.9$ & 0.974 & 158.0 \\
 &  & 2 & 0.1704 & 900 & 101 & 95,322 & 588 & $+2.655$ & $+6.795$ & 0.934 & 229.7 \\
 &  & 4 & 0.3275 & 157 & 20.8 & 4,838 & 131 & $+0.363$ & $+0.865$ & 0.750 & 228.1 \\
 &  & 8 & 0.3419 & 310 & 32.6 & 258 & 23.2 & $+0.848$ & $+1.732$ & 0.765 & 418.3 \\
 &  & 16 & 0.3637 & 136 & 18.5 & 67.1 & 6.89 & $+0.382$ & $+0.561$ & 0.673 & 300.3 \\
 &  & 32 & 0.3279 & 48.2 & 10.5 & 39.6 & 3.79 & $+0.125$ & $+0.129$ & 0.563 & 131.8 \\
\addlinespace[1.5pt]
 & sign flipped & 1 & 0.0200 & 2,068 & 199 & 358,943 & 1,243 & $+5.851$ & $+19.9$ & 0.976 & 120.6 \\
 &  & 2 & 0.0557 & 902 & 102 & 95,322 & 588 & $+2.678$ & $+6.833$ & 0.943 & 67.95 \\
 &  & 4 & 0.1961 & 159 & 20.9 & 4,838 & 131 & $+0.380$ & $+0.906$ & 0.762 & 57.90 \\
 &  & 8 & 0.1713 & 317 & 34.9 & 266 & 23.9 & $+0.909$ & $+1.983$ & 0.785 & 69.73 \\
 &  & 16 & 0.2093 & 150 & 21.5 & 70.6 & 7.47 & $+0.452$ & $+0.806$ & 0.707 & 60.83 \\
 &  & 32 & 0.2046 & 64.9 & 13.1 & 42.4 & 4.36 & $+0.192$ & $+0.333$ & 0.632 & 40.14 \\
\midrule
MacroDiff+ & free & 1 & 0.0444 & 2,040 & 194 & 359,381 & 1,234 & $+5.659$ & $+19.5$ & 0.972 & 150.2 \\
 &  & 2 & 0.1058 & 885 & 97.8 & 96,614 & 589 & $+2.561$ & $+6.699$ & 0.933 & 110.2 \\
 &  & 4 & 0.2162 & 168 & 20.9 & 5,624 & 141 & $+0.368$ & $+0.997$ & 0.756 & 71.66 \\
 &  & 8 & 0.1021 & 23.1 & 6.24 & 321 & 24.6 & $+0.029$ & $+0.140$ & 0.552 & 9.887 \\
 &  & 16 & 0.0708 & 7.40 & 4.53 & 79.1 & 7.92 & $+0.003$ & $+0.039$ & 0.367 & 4.697 \\
 &  & 32 & 0.0639 & 5.73 & 4.33 & 47.5 & 5.73 & $+0.002$ & $+0.024$ & 0.320 & 3.846 \\
\addlinespace[1.5pt]
 & guided & 1 & 0.0014 & 1,388 & 143 & 359,385 & 1,234 & $+3.711$ & $+9.499$ & 0.972 & 54.05 \\
 &  & 2 & 0.0003 & 689 & 95.4 & 96,701 & 589 & $+2.312$ & $+4.954$ & 0.951 & 31.21 \\
 &  & 4 & 0.0004 & 136 & 37.6 & 5,679 & 142 & $+0.694$ & $+1.417$ & 0.821 & 10.76 \\
 &  & 8 & 0.0004 & 11.7 & 16.6 & 326 & 25.4 & $+0.199$ & $+0.415$ & 0.433 & 3.196 \\
 &  & 16 & 0.0006 & 5.05 & 11.9 & 74.9 & 8.54 & $+0.123$ & $+0.247$ & 0.298 & 2.201 \\
 &  & 32 & 0.0008 & 4.23 & 10.9 & 42.2 & 6.28 & $+0.108$ & $+0.216$ & 0.271 & 2.042 \\
\bottomrule
\end{tabular}
\end{table}

\subsection{Through the Refiners}
\label{sec:appendix-guided-refine}

Tables~\ref{tab:guided_refine} and \ref{tab:guided_loops} and Figure~\ref{fig:guided_refiner} continue both sample sets through the closed-form refiner and through the three ported loops. The closed-form refiner is nearly indifferent to the start. At $\nfe{=}32$ with 400 steps it brings ChipDiffusion's guided samples to 1.152 against 1.131 from the free ones and MacroDiff+'s to 1.218 against 1.167, and only DiffPlace's damaged guided samples end higher, 1.869 against 1.202, with 5.62 against 1.23 at $\nfe{=}8$. At $\nfe{=}1$ the guided start is the better one on ChipDiffusion and MacroDiff+ (59.9 and 32.3 against 72.6 and 69.4), the low-$\nfe$ gain of the raw samples carried through. The ported loops do worse from the guided samples than from the free ones on all three placers. The complete published methods, guided sampling through the own loop, end at 3.58, 5.37 and 2.69 for ChipDiffusion, DiffPlace and MacroDiff+ against 2.05, 3.91 and 1.66 from free sampling, and MacroDiff+ loses its raw-sample gain here because its loop keeps the spread its guidance introduced, an area gap of 0.34 against 0.13 from the free start. The same guided samples through our refiner end at 1.15, 1.87 and 1.22.

\begin{figure}[htbp]\centering
\includegraphics[width=\linewidth]{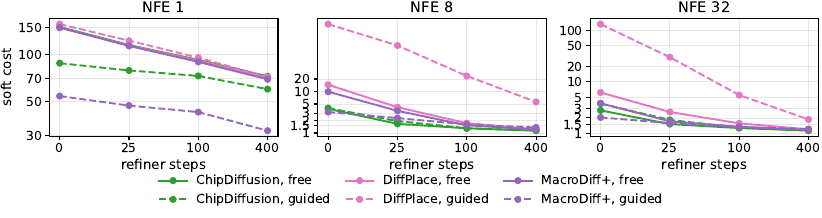}
\caption{The soft cost along the closed-form refiner from the free (solid) and the guided (dashed) samples of each placer at $\nfe = 1$, 8 and 32.}
\label{fig:guided_refiner}
\end{figure}

\begin{table}[htbp]\centering\scriptsize
\caption{The closed-form refiner on the free and the guided samples of each placer, the reduced grid of Appendix~\ref{sec:appendix-refine-all}: soft cost / overlap ratio per cell of $\nfe$ and refiner steps.}
\label{tab:guided_refine}
\setlength{\tabcolsep}{5pt}
\begin{tabular}{llrrrrr}
\toprule
Placer & Sampling & $\nfe$ & 0 steps & 25 & 100 & 400 \\
\midrule
ChipDiffusion & free & 1 & 151.6\,/\,0.0421 & 115.9\,/\,0.0287 & 92.74\,/\,0.0111 & 72.63\,/\,0.0039 \\
 &  & 8 & 3.829\,/\,0.0595 & 1.662\,/\,0.0307 & 1.299\,/\,0.0162 & 1.137\,/\,0.0061 \\
 &  & 32 & 2.802\,/\,0.0511 & 1.516\,/\,0.0272 & 1.271\,/\,0.0158 & 1.131\,/\,0.0062 \\
\addlinespace[1.5pt]
 & guided & 1 & 88.29\,/\,0.0035 & 79.26\,/\,0.0069 & 72.99\,/\,0.0041 & 59.94\,/\,0.0038 \\
 &  & 8 & 3.974\,/\,0.0231 & 1.940\,/\,0.0138 & 1.305\,/\,0.0071 & 1.151\,/\,0.0028 \\
 &  & 32 & 3.782\,/\,0.0182 & 1.828\,/\,0.0107 & 1.291\,/\,0.0057 & 1.152\,/\,0.0025 \\
\midrule
DiffPlace & free & 1 & 149.8\,/\,0.0420 & 115.1\,/\,0.0292 & 91.73\,/\,0.0113 & 71.39\,/\,0.0039 \\
 &  & 8 & 14.56\,/\,0.1457 & 4.181\,/\,0.0868 & 1.728\,/\,0.0347 & 1.231\,/\,0.0101 \\
 &  & 32 & 6.229\,/\,0.1126 & 2.606\,/\,0.0683 & 1.570\,/\,0.0330 & 1.202\,/\,0.0111 \\
\addlinespace[1.5pt]
 & guided & 1 & 158.0\,/\,0.0463 & 123.6\,/\,0.0360 & 95.91\,/\,0.0161 & 71.85\,/\,0.0048 \\
 &  & 8 & 418.3\,/\,0.3419 & 128.3\,/\,0.2471 & 23.68\,/\,0.1075 & 5.619\,/\,0.0269 \\
 &  & 32 & 131.8\,/\,0.3279 & 30.09\,/\,0.2290 & 5.552\,/\,0.1129 & 1.869\,/\,0.0373 \\
\midrule
MacroDiff+ & free & 1 & 150.2\,/\,0.0444 & 114.2\,/\,0.0308 & 90.00\,/\,0.0119 & 69.44\,/\,0.0041 \\
 &  & 8 & 9.887\,/\,0.1021 & 3.427\,/\,0.0537 & 1.537\,/\,0.0211 & 1.211\,/\,0.0071 \\
 &  & 32 & 3.846\,/\,0.0639 & 1.714\,/\,0.0351 & 1.348\,/\,0.0197 & 1.167\,/\,0.0078 \\
\addlinespace[1.5pt]
 & guided & 1 & 54.05\,/\,0.0014 & 46.93\,/\,0.0064 & 42.59\,/\,0.0035 & 32.26\,/\,0.0040 \\
 &  & 8 & 3.196\,/\,0.0004 & 2.290\,/\,0.0025 & 1.605\,/\,0.0014 & 1.360\,/\,0.0013 \\
 &  & 32 & 2.042\,/\,0.0008 & 1.589\,/\,0.0007 & 1.322\,/\,0.0010 & 1.218\,/\,0.0012 \\
\bottomrule
\end{tabular}
\end{table}

\begin{table}[tbp]\centering\scriptsize
\caption{Each placer's $\nfe{=}32$ samples, free and guided, after the closed-form refiner (400 steps) and after the three ported loops at their released iteration counts (ChipDiffusion 5,000 in its scheduled variant, DiffPlace and MacroDiff+ 500); the raw rows are the $\nfe{=}32$ rows of Table~\ref{tab:guided_raw}, whose units apply. A dagger marks guided sampling through the method's own loop, its complete published method.}
\label{tab:guided_loops}
\setlength{\tabcolsep}{2pt}
\begin{tabular}{lllrrrrrrrrr}
\toprule
Placer & Sampling & Refiner & Overlap & Boundary & Cluster & MIB & Fixed & HPWL gap & Area gap & $\mathrm{sev}_{\mathrm{rel}}$ & Soft cost \\
\midrule
ChipDiffusion & free & ours, 400 steps & 0.0062 & 0.192 & 1.13 & 0.0000 & 0.0001 & $+0.012$ & $+0.030$ & 0.017 & 1.131 \\
 &  & ChipDiffusion loop & 0.0069 & 4.56 & 6.68 & 15.7 & 2.19 & $+0.053$ & $+0.117$ & 0.273 & 2.048 \\
 &  & DiffPlace loop & 0.0008 & 29.1 & 6.10 & 15.7 & 2.19 & $+0.058$ & $+0.246$ & 0.563 & 3.614 \\
 &  & MacroDiff+ loop & 0.0032 & 2.21 & 6.80 & 15.7 & 2.19 & $+0.061$ & $+0.116$ & 0.174 & 1.602 \\
\addlinespace[1.5pt]
 & guided & ours, 400 steps & 0.0025 & 0.699 & 1.11 & 0.0000 & 0.0001 & $+0.024$ & $+0.089$ & 0.029 & 1.152 \\
 &  & ChipDiffusion loop$^{\dagger}$ & 0.0022 & 15.5 & 11.5 & 17.2 & 2.66 & $+0.099$ & $+0.273$ & 0.533 & 3.582 \\
 &  & DiffPlace loop & 0.0007 & 28.0 & 7.57 & 17.2 & 2.66 & $+0.069$ & $+0.272$ & 0.580 & 3.803 \\
 &  & MacroDiff+ loop & 0.0006 & 7.21 & 12.2 & 17.2 & 2.66 & $+0.126$ & $+0.276$ & 0.319 & 2.334 \\
\midrule
DiffPlace & free & ours, 400 steps & 0.0111 & 0.256 & 1.03 & 0.0059 & 0.0001 & $+0.014$ & $+0.037$ & 0.019 & 1.202 \\
 &  & ChipDiffusion loop & 0.0124 & 5.48 & 7.10 & 31.6 & 2.35 & $+0.058$ & $+0.119$ & 0.294 & 2.286 \\
 &  & DiffPlace loop & 0.0024 & 36.1 & 6.99 & 31.6 & 2.35 & $+0.067$ & $+0.299$ & 0.580 & 3.914 \\
 &  & MacroDiff+ loop & 0.0044 & 2.57 & 7.14 & 31.6 & 2.35 & $+0.069$ & $+0.116$ & 0.175 & 1.633 \\
\addlinespace[1.5pt]
 & guided & ours, 400 steps & 0.0373 & 4.58 & 1.99 & 0.014 & 0.0001 & $+0.131$ & $+0.066$ & 0.066 & 1.869 \\
 &  & ChipDiffusion loop & 0.0155 & 33.1 & 13.3 & 39.7 & 3.77 & $+0.236$ & $+0.245$ & 0.496 & 4.048 \\
 &  & DiffPlace loop$^{\dagger}$ & 0.0053 & 73.7 & 13.1 & 39.7 & 3.77 & $+0.239$ & $+0.479$ & 0.654 & 5.369 \\
 &  & MacroDiff+ loop & 0.0035 & 23.7 & 14.7 & 39.7 & 3.77 & $+0.277$ & $+0.245$ & 0.350 & 2.702 \\
\midrule
MacroDiff+ & free & ours, 400 steps & 0.0078 & 0.508 & 1.36 & 1.21 & 0.0001 & $+0.018$ & $+0.037$ & 0.021 & 1.167 \\
 &  & ChipDiffusion loop & 0.0075 & 5.61 & 7.61 & 129 & 5.71 & $+0.063$ & $+0.132$ & 0.293 & 2.180 \\
 &  & DiffPlace loop & 0.0010 & 31.5 & 6.99 & 129 & 5.71 & $+0.068$ & $+0.269$ & 0.574 & 3.755 \\
 &  & MacroDiff+ loop & 0.0031 & 2.53 & 7.75 & 129 & 5.71 & $+0.072$ & $+0.127$ & 0.186 & 1.661 \\
\addlinespace[1.5pt]
 & guided & ours, 400 steps & 0.0012 & 1.44 & 1.49 & 0.036 & 0.0001 & $+0.060$ & $+0.162$ & 0.039 & 1.218 \\
 &  & ChipDiffusion loop & 0.0001 & 12.7 & 15.5 & 104 & 6.25 & $+0.147$ & $+0.334$ & 0.552 & 3.801 \\
 &  & DiffPlace loop & 0.0001 & 8.05 & 11.2 & 104 & 6.25 & $+0.110$ & $+0.239$ & 0.381 & 2.592 \\
 &  & MacroDiff+ loop$^{\dagger}$ & 0.0004 & 8.67 & 15.0 & 104 & 6.25 & $+0.163$ & $+0.340$ & 0.369 & 2.694 \\
\bottomrule
\end{tabular}
\end{table}

\begin{table}[tbp]\centering\scriptsize
\caption{The distribution metrics of Appendix~\ref{sec:appendix-scoring} for the free and the guided samples at $\nfe{=}1$ and 32: the Fr\'echet distances on the density descriptors and on positions, the MMD in $10^{-3}$, and the permutation $p$ that the run is closer to the reference than the plain control.}
\label{tab:guided_dist}
\setlength{\tabcolsep}{4pt}
\begin{tabular}{llrrrrrrrr}
\toprule
Placer & Sampling & $\nfe$ & FLD8b & FLD8f & FLD12b & FLD12f & FPD & MMD8b & $p$ \\
\midrule
ChipDiffusion & free & 1 & 102 & 17.67 & 234 & 48.12 & 5.501 & 1,223 & 1.00 \\
 &  & 32 & 0.2733 & 0.009784 & 1.361 & 0.06643 & 0.001091 & 31.2 & 1.00 \\
 & guided & 1 & 101 & 15.71 & 230 & 43.13 & 7.418 & 1,171 & 1.00 \\
 &  & 32 & 3.1 & 0.2972 & 8.633 & 1.071 & 0.00531 & 512 & 1.00 \\
\midrule
DiffPlace & free & 1 & 102 & 17.78 & 234 & 48.36 & 5.361 & 1,230 & 1.00 \\
 &  & 32 & 2.855 & 0.1911 & 10.73 & 1.139 & 0.01458 & 247 & 1.00 \\
 & guided & 1 & 102 & 17.87 & 234 & 48.65 & 5.364 & 1,231 & 1.00 \\
 &  & 32 & 19.72 & 1.613 & 64.18 & 7.93 & 0.1404 & 594 & 1.00 \\
\midrule
MacroDiff+ & free & 1 & 102 & 17.51 & 234 & 47.81 & 5.724 & 1,216 & 1.00 \\
 &  & 32 & 0.624 & 0.02763 & 2.679 & 0.1678 & 0.2799 & 81.5 & 1.00 \\
 & guided & 1 & 92.14 & 15.52 & 210 & 43.44 & 3.102 & 1,077 & 1.00 \\
 &  & 32 & 4.361 & 1.157 & 11.3 & 3.63 & 0.02319 & 615 & 1.00 \\
\bottomrule
\end{tabular}
\end{table}

\subsection{Distance to the Reference}
\label{sec:appendix-guided-dist}

Table~\ref{tab:guided_dist} measures the guided samples against the reference distribution. At $\nfe{=}32$ every guidance moves its placer away from it, FLD8b by seven to eleven times (ChipDiffusion 0.273 to 3.10, DiffPlace 2.86 to 19.7, MacroDiff+ 0.624 to 4.36) and MMD8b by two to sixteen times, and no guided run is closer to the reference than the plain control ($p = 1.00$). At $\nfe{=}1$ the guided and the free samples sit at the same distance (FLD8b 92 to 102), where the one-step draw dominates whatever the guidance adds.

The three corrections move the sample along the terms they model, overlap and wirelength, and away from the reference layouts on every term they do not, which is the pattern the ported loops show on the reference layouts in Appendix~\ref{sec:appendix-theory-measured}. Measured at the stage where the guidance acts, the trade is visible in the columns of Table~\ref{tab:guided_raw}. After the legalizer, which removes the overlap on every arm, the same trade would appear as a few percent of area and wirelength with no column to name it.

\section{Hard-Cost Results in Full}
\label{sec:appendix-hard}

\subsection{Trinity's Grid}
\label{sec:appendix-hard-grid}

\begin{table}[htbp]\centering\scriptsize
\caption{Trinity through the legalizer on the $\nfe \times$ steps grid (selected cells; every cell is in the supplementary data): the soft cost of the refined layout, the hard cost of the first legalized draw of every case as mean / median / 90th percentile, its hard-cost total, its wirelength / area gap in percent, the best of four legalized draws as mean / median / 90th percentile and total, the mean of the soft pick of four, and the legalizer's time per draw. Every cell is feasible on all draws.}
\label{tab:app_hard_grid}
\label{tab:app_hard_grid_best}
\setlength{\tabcolsep}{3.5pt}
\begin{tabular}{rrrrrrrrrr}
\toprule
$\nfe$ & Steps & Soft cost & First draw & Total & WL / area & Best of 4 & Total & Soft pick & ms / draw \\
\midrule
1 & 0 & 1.6785 & 1.4903\,/\,1.2875\,/\,2.362 & 1.7111 & $+15.5$\,/\,$+17.8$ & 1.3763\,/\,1.2394\,/\,1.787 & 1.5513 & 1.4841 & 358 \\
1 & 10 & 1.5170 & 1.3375\,/\,1.1177\,/\,2.148 & 1.5836 & $+15.1$\,/\,$+17.3$ & 1.2296\,/\,1.0747\,/\,1.516 & 1.4183 & 1.3309 & 351 \\
1 & 25 & 1.3090 & 1.3239\,/\,1.1123\,/\,2.071 & 1.5612 & $+14.2$\,/\,$+16.6$ & 1.2171\,/\,1.0698\,/\,1.485 & 1.3976 & 1.3166 & 348 \\
1 & 50 & 1.2084 & 1.3088\,/\,1.1046\,/\,1.975 & 1.5470 & $+13.6$\,/\,$+15.8$ & 1.2071\,/\,1.0651\,/\,1.450 & 1.3968 & 1.3038 & 342 \\
1 & 100 & 1.1490 & 1.2932\,/\,1.0965\,/\,1.865 & 1.5226 & $+12.6$\,/\,$+14.7$ & 1.1903\,/\,1.0595\,/\,1.413 & 1.3557 & 1.2847 & 330 \\
1 & 200 & 1.1042 & 1.2624\,/\,1.0838\,/\,1.632 & 1.4648 & $+10.9$\,/\,$+12.9$ & 1.1623\,/\,1.0498\,/\,1.335 & 1.3024 & 1.2527 & 314 \\
1 & 400 & 1.0661 & 1.2222\,/\,1.0682\,/\,1.486 & 1.3827 & $+8.7$\,/\,$+10.5$ & 1.1298\,/\,1.0394\,/\,1.267 & 1.2337 & 1.2081 & 290 \\
2 & 0 & 1.6603 & 1.4928\,/\,1.2869\,/\,2.403 & 1.7236 & $+15.2$\,/\,$+17.6$ & 1.3294\,/\,1.2184\,/\,1.652 & 1.4820 & 1.4748 & 344 \\
2 & 400 & 1.0617 & \textbf{1.2095}\,/\,\textbf{1.0645}\,/\,1.453 & \textbf{1.3535} & $+8.1$\,/\,$+9.9$ & 1.0978\,/\,1.0281\,/\,1.221 & 1.1775 & 1.1918 & 275 \\
4 & 0 & 1.6306 & 1.4945\,/\,1.2888\,/\,2.384 & 1.7198 & $+15.4$\,/\,$+17.7$ & 1.3104\,/\,1.2133\,/\,1.620 & 1.4510 & 1.4723 & 342 \\
4 & 400 & 1.0593 & 1.2153\,/\,1.0652\,/\,1.477 & 1.3701 & $+8.3$\,/\,$+9.9$ & 1.0879\,/\,1.0177\,/\,1.204 & 1.1597 & 1.1832 & 272 \\
8 & 0 & 1.5976 & 1.5097\,/\,1.2968\,/\,2.433 & 1.7359 & $+15.9$\,/\,$+18.0$ & 1.3080\,/\,1.2133\,/\,1.605 & 1.4523 & 1.4785 & 340 \\
8 & 100 & 1.1232 & 1.2961\,/\,1.0981\,/\,1.865 & 1.4925 & $+12.4$\,/\,$+14.2$ & 1.1313\,/\,1.0413\,/\,1.301 & 1.2479 & 1.2575 & 308 \\
8 & 400 & 1.0593 & 1.2249\,/\,1.0685\,/\,1.514 & 1.3658 & $+8.6$\,/\,$+10.1$ & \textbf{1.0861}\,/\,\textbf{1.0173}\,/\,1.205 & \textbf{1.1473} & 1.1826 & 267 \\
16 & 400 & 1.0593 & 1.2286\,/\,1.0714\,/\,1.522 & 1.3669 & $+9.0$\,/\,$+10.4$ & 1.0878\,/\,1.0181\,/\,1.207 & 1.1534 & 1.1847 & 269 \\
32 & 0 & 1.5137 & 1.5125\,/\,1.3019\,/\,2.451 & 1.7319 & $+16.4$\,/\,$+18.0$ & 1.3057\,/\,1.2135\,/\,1.601 & 1.4489 & 1.4836 & 341 \\
32 & 10 & 1.3695 & 1.3561\,/\,1.1318\,/\,2.219 & 1.5948 & $+15.8$\,/\,$+17.5$ & 1.1697\,/\,1.0544\,/\,1.383 & 1.3309 & 1.3191 & 337 \\
32 & 25 & 1.2159 & 1.3453\,/\,1.1230\,/\,2.163 & 1.5667 & $+15.2$\,/\,$+16.9$ & 1.1596\,/\,1.0517\,/\,1.364 & 1.3126 & 1.3090 & 331 \\
32 & 50 & 1.1473 & 1.3321\,/\,1.1162\,/\,2.091 & 1.5476 & $+14.3$\,/\,$+16.0$ & 1.1475\,/\,1.0492\,/\,1.339 & 1.2860 & 1.2913 & 323 \\
32 & 100 & 1.1114 & 1.3093\,/\,1.1046\,/\,1.951 & 1.5055 & $+13.2$\,/\,$+14.8$ & 1.1353\,/\,1.0441\,/\,1.316 & 1.2550 & 1.2672 & 311 \\
32 & 200 & 1.0833 & 1.2785\,/\,1.0914\,/\,1.766 & 1.4558 & $+11.5$\,/\,$+13.0$ & 1.1136\,/\,1.0370\,/\,1.269 & 1.2040 & 1.2327 & 294 \\
32 & 400 & 1.0583 & 1.2341\,/\,1.0734\,/\,1.551 & 1.3773 & $+9.2$\,/\,$+10.5$ & 1.0887\,/\,1.0197\,/\,1.207 & 1.1580 & 1.1853 & 270 \\
\bottomrule
\end{tabular}
\end{table}

Hard-cost Pareto sets over $(\nfe, \text{steps})$, a cell kept when no cell with fewer sampler steps and fewer refiner steps is cheaper: for the first-draw mean the $\nfe{=}1$ column, $\nfe{=}2$ at 10 to 400 steps (1.210 at 400) and $(4, 200)$; for the best-of-4 mean the whole $\nfe = 1, 2, 4, 8$ columns, $\nfe{=}16$ at 0 / 10 / 25 / 100 and $(32, 0)$; for the best-of-4 total the $\nfe = 1, 2, 4$ columns, $\nfe{=}8$ at 25 / 50 / 100 / 400 (1.147 at 400, the grid's minimum) and $\nfe{=}16$ at 0 to 200. The refiner pays after legalization in every column at every $\nfe$ (first draw 1.49 to 1.51 $\to$ 1.21 to 1.23, best of 4 1.31 to 1.38 $\to$ 1.09 to 1.13, $\Vrel$ 0.10 $\to$ 0.04, area gap $+0.18 \to +0.10$ from raw to 400 steps), the legalizer gets faster on refined input (340 to 358 $\to$ 267 to 290 ms per draw), and the largest single move is raw $\to$ 10 steps. Along the $\nfe$ axis at 400 steps the soft cost is monotone (1.066 $\to$ 1.058) while the first-draw cost has its minimum at $\nfe{=}2$ and the best of 4 at $\nfe{=}8$. The soft pick of four lands at 1.18 to 1.21 at 400 steps, between the first draw and the best of 4: choosing by the soft cost recovers a third of the best-of-4 gain without legalizing the other three draws.

\begin{table}[htbp]\centering\scriptsize
\caption{Per cell: the share of draws that ended on each rung of the legalizer's ladder, the expansion factor $\lambda$, the rank correlation of $\lambda$ with the area gap and the cost, and the gaps after legalization.}
\label{tab:app_rungs}
\setlength{\tabcolsep}{2.5pt}
\begin{tabular}{rrrrrrrrrrrrr}
\toprule
$\nfe$ & Steps & Hard LP & Penalties & Flips & Last rungs & $\lambda$ p50 & $\lambda$ p90 & $\rho(\lambda, \text{area gap})$ & $\rho(\lambda, \text{cost})$ & Area gap & HPWL gap & $\Vrel$ \\
\midrule
1 & 0 & 0.114 & 0.443 & 0.325 & 0.117 & 2.54 & 8.21 & 0.33 & 0.19 & $+$0.178 & $+$0.155 & 0.104 \\
1 & 400 & 0.248 & 0.491 & 0.189 & 0.073 & 1.39 & 3.65 & 0.46 & 0.42 & $+$0.105 & $+$0.088 & 0.041 \\
8 & 0 & 0.113 & 0.467 & 0.302 & 0.118 & 2.40 & 8.24 & 0.37 & 0.22 & $+$0.180 & $+$0.159 & 0.107 \\
8 & 400 & 0.252 & 0.514 & 0.168 & 0.066 & 1.35 & 3.48 & 0.50 & 0.46 & $+$0.101 & $+$0.086 & 0.043 \\
32 & 0 & 0.108 & 0.470 & 0.303 & 0.120 & 2.17 & 6.54 & 0.33 & 0.20 & $+$0.180 & $+$0.164 & 0.109 \\
32 & 400 & 0.245 & 0.518 & 0.170 & 0.068 & 1.32 & 2.80 & 0.49 & 0.48 & $+$0.105 & $+$0.092 & 0.045 \\
\bottomrule
\end{tabular}
\end{table}

\subsection{Every Line on the Same Pipeline}
\label{sec:appendix-hard-lines}

\begin{table}[htbp]\centering\footnotesize
\caption{The prior arts and Trinity on the same pipeline at 400 refiner steps, per $\nfe$ (the $\epsilon$-head models produce no layout at $\nfe = 1$ and 2, where their legalized costs are 9 to 33). Columns: the soft cost, the first draw's mean and median cost, the best-of-4 mean and hard-cost total, $\Vrel$ and the area gap after legalization, the median expansion factor $\lambda$, the share of draws the legalizer's hard packing program solved directly (LP), and the per-draw soft--hard Spearman $\rho$.}
\label{tab:app_lines}
\scriptsize
\setlength{\tabcolsep}{6pt}
\begin{tabular}{lrrrrrrrrrrrr}
\toprule
Line & $\nfe$ & Soft & First & Median & Best 4 & Total & $\Vrel$ & Area & $\lambda$ & LP & $\rho$ & HPWL \\
\midrule
FlowPlace & 4 & 1.116 & 1.270 & 1.091 & 1.119 & 1.214 & 0.047 & $+$0.13 & 1.53 & 0.202 & 0.41 & $+0.12$ \\
FlowPlace & 8 & 1.114 & 1.275 & 1.096 & 1.118 & 1.215 & 0.048 & $+$0.14 & 1.53 & 0.194 & 0.42 & $+0.12$ \\
FlowPlace & 16 & 1.109 & 1.287 & 1.101 & 1.121 & 1.215 & 0.050 & $+$0.14 & 1.51 & 0.190 & 0.45 & $+0.12$ \\
FlowPlace & 32 & 1.107 & 1.296 & 1.106 & 1.124 & 1.228 & 0.051 & $+$0.15 & 1.48 & 0.186 & 0.46 & $+0.13$ \\
ChipDiffusion & 4 & 1.226 & 1.325 & 1.154 & 1.143 & 1.232 & 0.054 & $+$0.17 & 1.66 & 0.162 & 0.62 & $+0.15$ \\
ChipDiffusion & 8 & 1.137 & 1.309 & 1.118 & 1.130 & 1.234 & 0.053 & $+$0.16 & 1.54 & 0.174 & 0.48 & $+0.14$ \\
ChipDiffusion & 16 & 1.132 & 1.316 & 1.119 & 1.134 & 1.238 & 0.053 & $+$0.16 & 1.55 & 0.169 & 0.46 & $+0.14$ \\
ChipDiffusion & 32 & 1.131 & 1.322 & 1.122 & 1.137 & 1.253 & 0.054 & $+$0.16 & 1.55 & 0.166 & 0.47 & $+0.14$ \\
DiffPlace & 4 & 2.68 & 2.044 & 1.874 & 1.624 & 1.843 & 0.150 & $+$0.48 & 3.30 & 0.013 & 0.51 & $+0.45$ \\
DiffPlace & 8 & 1.231 & 1.343 & 1.145 & 1.152 & 1.262 & 0.056 & $+$0.18 & 1.90 & 0.156 & 0.55 & $+0.15$ \\
DiffPlace & 16 & 1.199 & 1.351 & 1.134 & 1.155 & 1.280 & 0.056 & $+$0.18 & 1.81 & 0.157 & 0.48 & $+0.16$ \\
DiffPlace & 32 & 1.202 & 1.355 & 1.142 & 1.158 & 1.281 & 0.057 & $+$0.18 & 1.83 & 0.151 & 0.47 & $+0.16$ \\
MacroDiff+ & 4 & 2.62 & 1.976 & 1.813 & 1.570 & 1.737 & 0.143 & $+$0.45 & 2.97 & 0.015 & 0.49 & $+0.42$ \\
MacroDiff+ & 8 & 1.211 & 1.342 & 1.156 & 1.150 & 1.248 & 0.058 & $+$0.18 & 1.68 & 0.153 & 0.58 & $+0.15$ \\
MacroDiff+ & 16 & 1.172 & 1.349 & 1.143 & 1.154 & 1.274 & 0.058 & $+$0.18 & 1.67 & 0.149 & 0.50 & $+0.16$ \\
MacroDiff+ & 32 & 1.167 & 1.352 & 1.147 & 1.159 & 1.286 & 0.059 & $+$0.18 & 1.65 & 0.144 & 0.50 & $+0.16$ \\
Trinity & 4 & 1.059 & 1.215 & 1.065 & 1.088 & 1.160 & 0.041 & $+$0.10 & 1.35 & 0.255 & 0.55 & $+0.08$ \\
Trinity & 8 & 1.059 & 1.225 & 1.069 & 1.086 & 1.147 & 0.043 & $+$0.10 & 1.35 & 0.252 & 0.58 & $+0.09$ \\
Trinity & 16 & 1.059 & 1.229 & 1.071 & 1.088 & 1.153 & 0.044 & $+$0.10 & 1.33 & 0.247 & 0.59 & $+0.09$ \\
Trinity & 32 & 1.058 & 1.234 & 1.073 & 1.089 & 1.158 & 0.045 & $+$0.10 & 1.32 & 0.245 & 0.60 & $+0.09$ \\
\bottomrule
\end{tabular}
\end{table}

FlowPlace at $\nfe = 1 / 2$ and 400 steps: first draw 1.261 / 1.266, best of 4 1.134 / 1.126 (Trinity 1.222 / 1.210 and 1.130 / 1.098). The order on the same pipeline at $\nfe{=}8$ and 400 steps, first draw / best of 4: Trinity 1.225 / 1.086, FlowPlace 1.275 / 1.118, ChipDiffusion 1.309 / 1.130, MacroDiff+ 1.342 / 1.150, DiffPlace 1.343 / 1.152, Reg-$z$ 1.55, Reg-$xywh$ 1.74, the annealer 13.8 / 10.4. The two linear-$\beta$ lines are within 0.02 of each other cell by cell at $\nfe = 8$ to 32, and on them $\nfe = 16$ and 32 add nothing over 8. FlowPlace shows Trinity's $\nfe$ inversion, larger: at 400 steps the first-draw cost rises with $\nfe$ (1.261 $\to$ 1.296) and its $\nfe{=}1$ column dominates every refined cell of $\nfe = 8$ to 32. After legalization the prior arts keep more residual soft violations, a larger area gap, need more expansion and reach the hard program on 14 to 21\% of draws against 25\%.

\begin{table}[htbp]\centering\scriptsize
\caption{The single-shard lines (the sequence-pair annealer and the two regressors) across refiner steps: the soft cost (median / p90 where the mean overflows), the first draw's mean, median and 90th-percentile hard cost (the single layout for the regressors), the best-of-4 mean, the share of draws solved by the hard LP and by the last rungs, the median expansion factor $\lambda$, and the area gap, $\Vrel$ and HPWL gap after legalization.}
\label{tab:app_single_shard}
\setlength{\tabcolsep}{3pt}
\begin{tabular}{lrrrrrrrrrrrr}
\toprule
Line & Steps & Soft & First & Median & p90 & Best 4 & LP & Last & $\lambda$ & Area & $\Vrel$ & HPWL \\
\midrule
SP-SA & 0 & 12.341 & 19.019 & 18.266 & 28.18 & 14.382 & 0.000 & 0.170 & 1.51 & $+$3.28 & 0.755 & $+2.88$ \\
SP-SA & 100 & 13.094 & 17.195 & 16.510 & 25.43 & 12.984 & 0.000 & 0.164 & 1.94 & $+$3.10 & 0.721 & $+2.76$ \\
SP-SA & 400 & 10.891 & 13.809 & 13.294 & 20.13 & 10.427 & 0.000 & 0.150 & 2.76 & $+$2.61 & 0.660 & $+2.50$ \\
\addlinespace[1.5pt]
Reg-$z$ & 0 & 21.53 / 24,599 & 2.306 & 1.910 & 3.932 & -- & 0.034 & 0.137 & 10.0 (cap) & $+$0.61 & 0.162 & $+0.52$ \\
Reg-$z$ & 100 & 2.25 / 55.1 & 1.766 & 1.442 & 3.018 & -- & 0.065 & 0.137 & 8.50 & $+$0.40 & 0.095 & $+0.37$ \\
Reg-$z$ & 400 & 1.37 / 3.97 & 1.552 & 1.278 & 2.587 & -- & 0.096 & 0.125 & 4.56 & $+$0.27 & 0.078 & $+0.28$ \\
\addlinespace[1.5pt]
Reg-$xywh$ & 0 & 99.5 / 239,235 & 2.623 & 2.067 & 4.355 & -- & 0.023 & 0.145 & 10.0 (cap) & $+$0.71 & 0.185 & $+0.60$ \\
Reg-$xywh$ & 100 & 4.57 / 174 & 2.013 & 1.560 & 3.351 & -- & 0.040 & 0.155 & 10.0 & $+$0.49 & 0.115 & $+0.46$ \\
Reg-$xywh$ & 400 & 1.76 / 6.40 & 1.736 & 1.371 & 2.929 & -- & 0.065 & 0.143 & 5.74 & $+$0.34 & 0.095 & $+0.36$ \\
\bottomrule
\end{tabular}
\end{table}

The legalizer takes 379 to 393 ms per draw on the annealer's layouts and 329 to 375 ms on the regressors'. The annealer's layouts are overlap-free but far from the cases' reference (wirelength and area gaps 1.5 and 1.1) and do not hold the preplaced blocks; the anchor step restores those blocks, the expansion prices the overlaps this creates, the ladder never reaches the hard program, and the refiner moves the soft cost up before down (12.34 $\to$ 14.04 at 10 steps, 10.89 at 400) while the hard cost falls monotonically (19.0 $\to$ 13.8). The regressors sit at the expansion cap on every raw draw; the refiner moves both costs monotonically and the soft cost by orders of magnitude at the tail (Reg-$z$ cases above 100: 37\% $\to$ 0.1\%), and Reg-$xywh$ is worse than Reg-$z$ at every step count on every column.

\subsection{The Soft--Hard Correlation}
\label{sec:appendix-correlation}

\begin{table}[htbp]\centering\footnotesize
\caption{Trinity, one column per $\nfe$ / refiner steps: the rank correlation of the soft cost with the legalized cost over the 48,000 draws ($\rho$, $\tau$) and over the 12,000 per-case means, then the Spearman of every soft column with the legalized cost.}
\label{tab:app_corr_flagship}
\setlength{\tabcolsep}{3pt}
\begin{tabular}{lrrrrrrrrr}
\toprule
Quantity & 1 / 0 & 1 / 10 & 1 / 400 & 8 / 0 & 8 / 400 & 32 / 0 & 32 / 10 & 32 / 100 & 32 / 400 \\
\midrule
$\rho$ over draws & 0.198 & 0.491 & 0.495 & 0.261 & 0.580 & 0.267 & 0.551 & 0.572 & \textbf{0.603} \\
$\tau$ over draws & 0.134 & 0.338 & 0.345 & 0.177 & 0.412 & 0.180 & 0.382 & 0.401 & 0.430 \\
$\rho$ over case means & 0.203 & 0.505 & 0.502 & 0.261 & 0.583 & 0.271 & 0.583 & 0.581 & \textbf{0.610} \\
\midrule
overlap ratio & 0.223 & 0.365 & 0.418 & 0.278 & 0.506 & 0.272 & 0.450 & 0.460 & 0.531 \\
$\mathrm{sev}_{\mathrm{rel}}$ & 0.058 & 0.482 & 0.507 & 0.102 & 0.556 & 0.112 & 0.496 & 0.558 & 0.572 \\
boundary distance & 0.034 & 0.381 & 0.410 & 0.043 & 0.396 & 0.044 & 0.339 & 0.421 & 0.388 \\
$\mathrm{sev}_{\mathrm{bnd}}$ & 0.138 & 0.391 & 0.427 & 0.136 & 0.402 & 0.137 & 0.356 & 0.448 & 0.394 \\
area gap & 0.074 & 0.297 & 0.257 & 0.096 & 0.328 & 0.100 & 0.303 & 0.260 & 0.317 \\
compactness & 0.173 & 0.338 & 0.289 & 0.189 & 0.347 & 0.197 & 0.348 & 0.297 & 0.339 \\
$\mathrm{sev}_{\mathrm{grp}}$ & 0.158 & 0.237 & 0.092 & 0.209 & 0.178 & 0.220 & 0.308 & 0.242 & 0.220 \\
group gap & 0.113 & 0.171 & 0.012 & 0.161 & 0.101 & 0.175 & 0.243 & 0.182 & 0.145 \\
HPWL gap & $-$0.116 & $-$0.112 & 0.005 & $-$0.105 & 0.077 & $-$0.019 & 0.021 & 0.089 & 0.157 \\
normalized wirelength & $-$0.128 & $-$0.217 & $-$0.190 & $-$0.121 & $-$0.173 & $-$0.099 & $-$0.178 & $-$0.163 & $-$0.151 \\
fixed deviation & 0.228 & 0.059 & 0.055 & 0.222 & 0.051 & 0.226 & 0.055 & 0.055 & 0.046 \\
MIB variance & 0.099 & $-$0.029 & $-$0.026 & 0.098 & $-$0.023 & 0.090 & $-$0.035 & $-$0.035 & $-$0.024 \\
\bottomrule
\end{tabular}
\end{table}

$\rho$ over draws at 400 steps by $\nfe = 1 / 2 / 4 / 8 / 16 / 32$ is 0.495 / 0.519 / 0.553 / 0.580 / 0.595 / 0.603, and at $\nfe{=}32$ by steps $0 / 10 / 25 / 50 / 100 / 200 / 400$ is 0.267 / 0.551 / 0.547 / 0.558 / 0.572 / 0.584 / 0.603. Across the 42 cell means: Spearman 0.938, Kendall 0.754; pooled over the 2,016,000 draws: Spearman 0.472, Kendall 0.326. With the overlap ratio alone in place of the soft cost: across cells 0.904, pooled 0.437, per cell 0.22 to 0.28 raw and 0.33 to 0.53 refined.

\begin{table}[htbp]\centering\scriptsize
\caption{Every line: Spearman $\rho$, Kendall $\tau$, Pearson $r$ and $R^2$, log-log $r$ and $R^2$ of soft against hard cost, across the cell means (first six columns: 42 cells per diffusion line, 7 per single-shard line, 231 for all eight lines, 175 for the five lines with finite means) and pooled over every draw (last six: 2,016,000 per diffusion line, 336,000 for the annealer, 84,000 per regressor, 10,584,000 and 8,400,000 for the two unions); n/a where the regressors' overflowing means make a linear fit meaningless.}
\label{tab:app_corr_lines}
\setlength{\tabcolsep}{3pt}
\begin{tabular}{lrrrrrr@{\hspace{9pt}}rrrrrr}
\toprule
Line & $\rho$ & $\tau$ & $r$ & $R^2$ & $r_{\log}$ & $R^2_{\log}$ & $\rho$ & $\tau$ & $r$ & $R^2$ & $r_{\log}$ & $R^2_{\log}$ \\
\midrule
Trinity & 0.938 & 0.754 & 0.899 & 0.808 & 0.899 & 0.809 & 0.472 & 0.326 & 0.238 & 0.057 & 0.311 & 0.097 \\
FlowPlace & 0.923 & 0.738 & 0.747 & 0.558 & 0.846 & 0.716 & 0.369 & 0.252 & 0.065 & 0.004 & 0.237 & 0.056 \\
ChipDiffusion & 0.989 & 0.919 & 0.925 & 0.855 & 0.935 & 0.875 & 0.805 & 0.619 & 0.843 & 0.711 & 0.901 & 0.811 \\
DiffPlace & 0.975 & 0.898 & 0.892 & 0.796 & 0.906 & 0.822 & 0.858 & 0.665 & 0.796 & 0.633 & 0.873 & 0.763 \\
MacroDiff+ & 0.985 & 0.919 & 0.923 & 0.853 & 0.925 & 0.856 & 0.861 & 0.671 & 0.830 & 0.688 & 0.890 & 0.793 \\
SP-SA & 0.643 & 0.619 & 0.795 & 0.633 & 0.822 & 0.676 & 0.484 & 0.335 & 0.459 & 0.210 & 0.538 & 0.290 \\
Reg-$z$ & 1.000 & 1.000 & n/a & n/a & n/a & n/a & 0.608 & 0.442 & 0.031 & 0.001 & 0.554 & 0.307 \\
Reg-$xywh$ & 1.000 & 1.000 & n/a & n/a & n/a & n/a & 0.579 & 0.417 & 0.083 & 0.007 & 0.589 & 0.347 \\
\midrule
all eight lines & \textbf{0.947} & \textbf{0.807} & n/a & n/a & n/a & n/a & \textbf{0.768} & \textbf{0.578} & n/a & n/a & \textbf{0.827} & \textbf{0.684} \\
five finite lines & 0.962 & 0.836 & 0.856 & 0.732 & 0.901 & 0.812 & 0.760 & 0.570 & 0.774 & 0.598 & 0.865 & 0.748 \\
\bottomrule
\end{tabular}
\end{table}

Within a cell the relation is a ranking: Trinity's per-draw Spearman reaches 0.60 at 400 steps while the linear $R^2$ is 0.05 to 0.10 (log-log 0.08 to 0.15), because the legalized cost has a heavy tail (p90 1.5 against a median 1.07) that no linear fit on a compact soft cost captures. Within-cell ranges from the raw to the 400-step cells: Trinity Spearman 0.20--0.27 $\to$ 0.50--0.60, FlowPlace 0.21--0.27 $\to$ 0.38--0.46, ChipDiffusion 0.24--0.34 $\to$ 0.21--0.62, DiffPlace 0.27--0.40 $\to$ 0.39--0.55, MacroDiff+ 0.25--0.35 $\to$ 0.36--0.58, the annealer 0.46 $\to$ 0.56. The soft column that carries the correlation is $\mathrm{sev}_{\mathrm{rel}}$ (0.48 to 0.57 in the refined cells), then the overlap ratio (0.33 to 0.53, rising with $\nfe$ and steps), the boundary distance (0.34 to 0.48) and the area columns; group gap, HPWL gap, fixed and MIB are at or near zero after refinement, and normalized wirelength is negative ($-0.15$ to $-0.28$: longer wires, cheaper legalized layout). In the raw cells the fixed deviation (0.22) is among the top columns: the raw layouts' hard cost is set by their preplaced-block and overlap defects, the refined layouts' by their residual soft slots. The soft cost hardly predicts the legalizer's milliseconds per draw ($\rho$ $-0.08$ to 0.35) but does predict the ladder rung reached (0.28 to 0.40 at 400 steps), and the milliseconds correlate with the final cost (0.33 raw, 0.39 to 0.51 refined).

\subsection{The Prior Arts' Loops in Hard Cost}
\label{sec:appendix-ports-hard}

\begin{table}[htbp]\centering\scriptsize
\caption{The prior-art loops at their published iteration counts (MacroDiff+ 500, ChipDiffusion 5,000 in its scheduled variant, DiffPlace 500) and the closed-form refiner (ours) on Trinity's $\nfe{=}32$ samples, then the legalizer (the raw row reproduces the grid's $(32, 0)$ cell): the soft cost, the first draw's mean, median and p90 hard cost, the best-of-4 mean and median, the share of draws solved by the hard LP and by the last rungs, the median expansion factor $\lambda$, the area gap and $\Vrel$ after legalization, the per-draw soft--hard Spearman $\rho$ and the milliseconds per draw, and the HPWL gap.}
\label{tab:app_ports_hard}
\setlength{\tabcolsep}{2pt}
\begin{tabular}{lrrrrrrrrrrrrr}
\toprule
Refiner & Soft & First & Median & p90 & Best 4 & Med.\,4 & LP & Last & $\lambda$ & Area & $\Vrel$ & $\rho$ / ms & HPWL \\
\midrule
none (raw) & 1.514 & 1.5125 & 1.302 & 2.451 & 1.3057 & 1.214 & 0.108 & 0.120 & 2.17 & $+$0.180 & 0.109 & 0.267 / 341 & $+0.164$ \\
MacroDiff+ loop & 1.564 & 1.4424 & 1.284 & 2.007 & 1.2588 & 1.200 & 0.107 & 0.078 & 1.48 & $+$0.139 & 0.106 & 0.164 / 271 & $+0.121$ \\
ChipDiffusion loop & 1.793 & 1.4563 & 1.284 & 2.075 & 1.2699 & 1.205 & 0.118 & 0.085 & 1.46 & $+$0.150 & 0.106 & 0.171 / 278 & $+0.130$ \\
DiffPlace loop & 3.035 & 1.4415 & 1.286 & 1.957 & 1.2581 & 1.200 & 0.090 & 0.073 & 1.53 & $+$0.137 & 0.108 & 0.351 / 258 & $+0.114$ \\
ours, 10 steps & 1.370 & 1.3561 & 1.132 & 2.219 & 1.1697 & 1.054 & 0.136 & 0.118 & 1.93 & $+$0.175 & 0.058 & 0.551 / 337 & $+0.158$ \\
ours, 100 steps & 1.111 & 1.3093 & 1.105 & 1.951 & 1.1353 & 1.044 & 0.170 & 0.098 & 1.54 & $+$0.148 & 0.054 & 0.572 / 311 & $+0.132$ \\
ours, 400 steps & 1.058 & \textbf{1.2341} & \textbf{1.073} & 1.551 & \textbf{1.0887} & \textbf{1.020} & 0.245 & 0.068 & 1.32 & $+$0.105 & 0.045 & 0.603 / 270 & $+0.092$ \\
\bottomrule
\end{tabular}
\end{table}

The three ported loops remove overlap ($\lambda$ p50 2.17 $\to$ 1.46 to 1.53) and lower the legalized cost by a similar amount while leaving the soft slots untouched ($\Vrel$ 0.109 $\to$ 0.106 to 0.108) and raising the soft cost (1.51 $\to$ 1.56, 1.79 and 3.03). The closed-form refiner at 10 steps is already below all three on every hard column and at 400 steps by 0.21 to 0.22 on the first draw. After the MacroDiff+ and ChipDiffusion loops the per-draw soft-hard correlation drops to 0.16 to 0.17, and after the DiffPlace loop it is 0.35, carried by the area columns that loop inflates.

\begin{table}[htbp]\centering\scriptsize
\caption{Generator $\times$ refiner at $\nfe{=}32$, then the legalizer, two lines per generator: the hard cost of the first draw / best of four, then the first draw's wirelength gap / area gap to the reference in percent. Every prior-art loop runs at its published iteration count on every generator's own samples, the closed-form refiner for 400 steps; every draw of the 13 loop rows is feasible, the raw cell of every row reproduces its grid's $(32, 0)$ cell, and a dagger marks a method's published generator and loop pair. Every loop moves the outline and the wires it did not model; the closed-form refiner leaves the smallest gaps on every generator.}
\label{tab:app_gen_ref_hard}
\label{tab:app_gen_ref_gaps}
\setlength{\tabcolsep}{7pt}
\begin{tabular}{lrrrrr}
\toprule
Generator & Raw sample & ChipDiffusion loop & DiffPlace loop & MacroDiff+ loop & Closed-form refiner \\
\midrule
Trinity & 1.513 / 1.306 & 1.456 / 1.270 & 1.442 / 1.258 & 1.442 / 1.259 & \textbf{1.234 / 1.089} \\
 & $+16.2$\,/\,$+18.0$ & $+13.0$\,/\,$+14.9$ & $+11.4$\,/\,$+13.7$ & $+12.1$\,/\,$+13.9$ & $+9.3$\,/\,$+10.5$ \\
\addlinespace[1.5pt]
FlowPlace & 1.567 / 1.343 & 1.495 / 1.284 & 1.454 / 1.265 & 1.471 / 1.267 & 1.296 / 1.124 \\
 & $+18.5$\,/\,$+21.6$ & $+14.5$\,/\,$+16.9$ & $+11.2$\,/\,$+14.7$ & $+13.1$\,/\,$+15.2$ & $+12.8$\,/\,$+14.6$ \\
\addlinespace[1.5pt]
ChipDiffusion & 1.609 / 1.364 & 1.521 / 1.300$^{\dagger}$ & 1.483 / 1.280 & 1.492 / 1.285 & 1.322 / 1.137 \\
 & $+20.1$\,/\,$+23.7$ & $+15.5$\,/\,$+18.2$ & $+12.2$\,/\,$+16.1$ & $+13.7$\,/\,$+16.3$ & $+14.3$\,/\,$+16.2$ \\
\addlinespace[1.5pt]
DiffPlace & 1.731 / 1.447 & 1.587 / 1.334 & 1.537 / 1.313$^{\dagger}$ & 1.539 / 1.304 & 1.355 / 1.158 \\
 & $+24.7$\,/\,$+30.9$ & $+18.7$\,/\,$+21.0$ & $+13.9$\,/\,$+19.1$ & $+15.7$\,/\,$+18.1$ & $+16.1$\,/\,$+18.3$ \\
\addlinespace[1.5pt]
MacroDiff+ & 1.672 / 1.406 & 1.571 / 1.326 & 1.536 / 1.310 & 1.546 / 1.308$^{\dagger}$ & 1.352 / 1.159 \\
 & $+22.4$\,/\,$+26.5$ & $+17.5$\,/\,$+20.1$ & $+14.1$\,/\,$+18.4$ & $+15.9$\,/\,$+18.3$ & $+15.6$\,/\,$+17.7$ \\
\bottomrule
\end{tabular}
\end{table}

Table~\ref{tab:app_gen_ref_hard} completes the grid with every loop on every prior art's own samples, and Figure~\ref{fig:app_g6} draws the same cells against the refiner budget. The refiner axis dominates: exchanging a ported loop for the closed-form refiner moves the first-draw cost by 0.16 to 0.23 on every generator, and exchanging a prior-art generator for Trinity moves it by 0.06 to 0.12 under the closed-form refiner and by 0.01 to 0.13 under the ported loops, so the loops erase most of the generator difference after legalization as they do on the soft side (Appendix~\ref{sec:appendix-refiner-ports}). Trinity is the best row of every column and the closed-form refiner the best column of every row. A prior-art generator with the closed-form refiner (1.296 to 1.355 for the first draw, 1.124 to 1.158 for the best of four) beats the same generator with its best ported loop (1.454 to 1.537 and 1.265 to 1.308) by 0.16 to 0.18 and 0.14 to 0.15, and the closed-form refiner is past that point by 25 steps on every generator (1.395 to 1.513 against 1.454 to 1.537). After legalization the three loops rank the reverse of their soft order. The DiffPlace loop, the worst in soft cost (3.5 to 3.9), separates blocks hardest ($\lambda$ p50 1.32 to 1.47 on the prior arts) and the packing program recovers the area it spends, so it is the best of the three on every generator's first draw, and the ChipDiffusion loop, which leaves the most overlap, is the worst. A method's own generator and loop pair is not the best cell of its row, except DiffPlace by 0.002. The loops act on overlap only: after legalization $\Vrel$ stays at 0.111 to 0.122 on every prior-art and loop cell (raw 0.115 to 0.130) against 0.051 to 0.059 after the closed-form refiner, the hard-LP share at 0.05 to 0.09 against 0.14 to 0.19, and the per-draw soft-hard correlation at 0.22 to 0.30 against 0.46 to 0.50.

\subsection{Best of $N$}
\label{sec:appendix-bestofn}

\begin{table}[htbp]\centering\footnotesize
\caption{Best of $N$ on the 12,000 cases from Trinity's 16-draw caches (fresh noise), every draw refined 400 steps and legalized: the cheapest of the first $N$ legalized draws per case (mean, median, 90th percentile, hard-cost total), and the soft pick (the legalized cost of the lowest-soft-cost draw among $N$, mean / median). Every one of the 768,000 legalized draws is feasible.}
\label{tab:app_bestofn}
\scriptsize
\setlength{\tabcolsep}{7pt}
\begin{tabular}{rrrrrrrrr}
\toprule
$\nfe$ & Steps & $N$ & Mean & Median & p90 & Total & Soft pick (mean / median) & HPWL / Area (best of $N$) \\
\midrule
32 & 0 & 1 & 1.5079 & 1.3003 & 2.427 & 1.7300 & -- & -- \\
32 & 0 & 4 & 1.3014 & 1.2098 & 1.595 & 1.4423 & 1.4758 / 1.2775 & -- \\
32 & 0 & 16 & 1.2140 & 1.1696 & 1.459 & 1.2857 & 1.4548 / 1.2629 & -- \\
32 & 400 & 1 & 1.2331 & 1.0736 & 1.552 & 1.3857 & -- & $+9.3$ / $+10.5$ \\
32 & 400 & 2 & 1.1412 & 1.0435 & 1.313 & 1.2452 & 1.2079 / 1.0597 & $+4.9$ / $+6.1$ \\
32 & 400 & 4 & 1.0885 & 1.0203 & 1.208 & 1.1622 & 1.1870 / 1.0500 & $+2.6$ / $+3.6$ \\
32 & 400 & 8 & 1.0585 & 1.0066 & 1.151 & 1.1070 & 1.1742 / 1.0437 & $+1.4$ / $+2.2$ \\
32 & 400 & 16 & \textbf{1.0406} & \textbf{1.0024} & 1.116 & \textbf{1.0765} & 1.1600 / 1.0406 & $+0.6$ / $+1.3$ \\
8 & 400 & 1 & 1.2245 & 1.0689 & 1.523 & 1.3743 & -- & $+8.8$ / $+10.3$ \\
8 & 400 & 2 & 1.1393 & 1.0408 & 1.298 & 1.2486 & 1.2036 / 1.0563 & $+4.8$ / $+6.2$ \\
8 & 400 & 4 & 1.0875 & 1.0164 & 1.206 & 1.1612 & 1.1806 / 1.0474 & $+2.4$ / $+3.7$ \\
8 & 400 & 8 & 1.0596 & 1.0054 & 1.151 & 1.1107 & 1.1639 / 1.0427 & $+1.3$ / $+2.3$ \\
8 & 400 & 16 & 1.0424 & 1.0021 & 1.118 & 1.0775 & 1.1546 / 1.0397 & $+0.6$ / $+1.5$ \\
1 & 400 & 1 & 1.2223 & 1.0678 & 1.491 & 1.3848 & -- & $+8.7$ / $+10.5$ \\
1 & 400 & 2 & 1.1663 & 1.0488 & 1.346 & 1.2908 & 1.2144 / 1.0641 & $+6.1$ / $+7.8$ \\
1 & 400 & 4 & 1.1285 & 1.0392 & 1.269 & 1.2303 & 1.2108 / 1.0621 & $+4.4$ / $+5.9$ \\
1 & 400 & 8 & 1.1011 & 1.0290 & 1.225 & 1.1794 & 1.2072 / 1.0586 & $+3.1$ / $+4.6$ \\
1 & 400 & 16 & 1.0839 & 1.0153 & 1.195 & 1.1499 & 1.2064 / 1.0571 & $+2.3$ / $+3.8$ \\
\bottomrule
\end{tabular}
\end{table}

Each doubling of $N$ removes 31 to 39\% of the excess over 1 at $\nfe = 8$ and 32, the median reaches 1.007 at $N = 8$ and 1.002 at $N = 16$, the p90 falls from 1.55 to 1.12, and the mean is under 1.1 from $N = 4$. Raw samples also gain from $N$ (1.508 $\to$ 1.214 at $N = 16$) but the refiner is the larger lever: $N = 2$ refined (1.141) beats $N = 16$ raw. The $\nfe{=}8$ ladder is the $\nfe{=}32$ ladder within 0.002 at every $N$. At $\nfe{=}1$ the ladder starts at the same $N = 1$ but gains less from $N$ (22 to 25\% of the excess per doubling) and ends at 1.084: the one-step draws of a case are more alike, and $N = 16$ at $\nfe{=}1$ equals $N = 4$ at $\nfe{=}8$. The $N = 1$ and $N = 4$ points reproduce the four-draw grid's cells within 0.002 on a fresh-noise cache, and the per-draw soft--hard Spearman at 400 steps is 0.603 / 0.578 / 0.492 for $\nfe = 32 / 8 / 1$.

\subsection{The FloorSet Val Set in Full}
\label{sec:appendix-official}
Best of $N$ keeps, per case, the feasible draw with the lowest hard cost of the official evaluator, which reads the reference wirelength and area of the instance; a reference-free selection rule is not evaluated here. Table~\ref{tab:app_official_full} lists all twelve settings run on the val set, and Table~\ref{tab:official} shows five of them, picked after the runs to span cost and time; the sampler, refiner and legalizer settings were fixed on our 12,000-case split before any val-set run (Appendices~\ref{sec:appendix-refiner} and~\ref{sec:appendix-legalizer}).

\begin{table}[htbp]\centering\footnotesize
\caption{The FloorSet val set per sample on an idle machine: the $e^{n/12}$-weighted hard-cost total, the mean and median cost, and the wall time per case (mean, median, max).}
\label{tab:app_official_full}
\setlength{\tabcolsep}{5.5pt}
\scriptsize
\begin{tabular}{rrrrrrrrrrr}
\toprule
$\nfe$ & Steps & $N$ & Total & Mean & Median & Mean s / case & Median s / case & Max s / case & HPWL gap & Area gap \\
\midrule
8 & 400 & 48 & \textbf{1.0213} & \textbf{1.0144} & 1.0000 & 1.63 & 1.34 & 5.10 & $-0.004$ & $+0.000$ \\
32 & 400 & 48 & 1.0223 & 1.0147 & 1.0000 & 1.86 & 1.59 & 5.61 & $-0.003$ & $-0.001$ \\
8 & 400 & 16 & 1.0290 & 1.0193 & 1.0009 & 1.14 & 0.98 & 2.27 & $-0.002$ & $+0.004$ \\
8 & 400 & 8 & 1.0469 & 1.0284 & 1.0035 & 1.04 & 0.88 & 2.21 & $-0.002$ & $+0.008$ \\
8 & 400 & 4 & 1.1316 & 1.0896 & 1.0081 & 1.01 & 0.87 & 2.22 & $+0.032$ & $+0.032$ \\
4 & 400 & 16 & 1.0351 & 1.0243 & 1.0007 & 1.08 & 0.99 & 2.34 & $-0.002$ & $+0.007$ \\
1 & 400 & 16 & 1.0564 & 1.0447 & 1.0081 & 1.02 & 0.93 & 2.22 & $+0.001$ & $+0.014$ \\
8 & 100 & 16 & 1.0670 & 1.0417 & 1.0054 & 0.70 & 0.52 & 1.93 & $+0.002$ & $+0.015$ \\
4 & 100 & 16 & 1.0574 & 1.0425 & 1.0059 & 0.63 & 0.51 & 1.88 & $+0.003$ & $+0.017$ \\
1 & 100 & 16 & 1.1528 & 1.1063 & 1.0194 & 0.55 & 0.41 & 1.78 & $+0.030$ & $+0.046$ \\
8 & 50 & 16 & 1.0882 & 1.0512 & 1.0082 & 0.59 & 0.44 & 1.85 & $+0.004$ & $+0.022$ \\
1 & 50 & 16 & 1.1793 & 1.1146 & 1.0380 & 0.48 & 0.36 & 1.72 & $+0.027$ & $+0.055$ \\
\bottomrule
\end{tabular}
\end{table}

\begin{table}[htbp]\centering\footnotesize
\caption{The same settings: mean seconds per case by stage (legalization is the wall time over the $N$ parallel draws on 16 workers), the legalizer's milliseconds per draw and the solver's share of it.}
\label{tab:app_official_stages}
\setlength{\tabcolsep}{6.5pt}
\begin{tabular}{rrrrrrrr}
\toprule
$\nfe$ & Steps & $N$ & Sample s & Refine s & Legalize s & Legalizer ms / draw & Solver ms / draw \\
\midrule
8 & 400 & 48 & 0.10 & 0.61 & 0.91 & 240 & 87 \\
32 & 400 & 48 & 0.32 & 0.61 & 0.93 & 244 & 88 \\
8 & 400 & 16 & 0.07 & 0.62 & 0.45 & 258 & 92 \\
8 & 400 & 8 & 0.07 & 0.61 & 0.35 & 240 & 87 \\
8 & 400 & 4 & 0.06 & 0.62 & 0.33 & 249 & 89 \\
4 & 400 & 16 & 0.04 & 0.62 & 0.42 & 264 & 94 \\
1 & 400 & 16 & 0.02 & 0.62 & 0.38 & 252 & 89 \\
8 & 100 & 16 & 0.06 & 0.16 & 0.48 & 281 & 97 \\
4 & 100 & 16 & 0.04 & 0.16 & 0.43 & 275 & 95 \\
1 & 100 & 16 & 0.02 & 0.16 & 0.37 & 284 & 97 \\
8 & 50 & 16 & 0.06 & 0.08 & 0.45 & 289 & 98 \\
1 & 50 & 16 & 0.02 & 0.08 & 0.38 & 291 & 99 \\
\bottomrule
\end{tabular}
\end{table}

At $N = 48$, $\nfe{=}8$ and 400 steps the gaps are HPWL $-0.0045$ and area $+0.0003$ with $\Vrel = 0.0048$ (one cluster count and 23 boundary counts on 17 cases); at $\nfe{=}32$: $-0.0027$ / $-0.0009$ / 0.0053. The compile warm-up of the refiner (about 4.5 minutes) and the pool warm-up run once and are outside the times. $N$ is the strongest lever at fixed $\nfe$ and steps (1.132 $\to$ 1.047 $\to$ 1.029 $\to$ 1.021 on the total for $N = 4 / 8 / 16 / 48$ at $\nfe{=}8$ and 400 steps, for 1.0 to 1.6 seconds per case), $\nfe$ trades against steps ($\nfe{=}4 \times 100$ at 1.057 equals $\nfe{=}1 \times 400$ at 1.056 on the total, $\nfe{=}8 \times 50$ at 1.088 between), and the rows with $N \geq 8$ reproduce an earlier run of the same settings to four decimals while the $N = 4$ row moves by 0.06 on the total, a single case's outcome on 100 cases.

\makeatletter
\setlength{\@fptop}{0pt}\setlength{\@fpsep}{\floatsep}\setlength{\@fpbot}{0pt plus 1fil}
\makeatother
\section{Classical Floorplanners}
\label{sec:appendix-classical}
\suppressfloats[t]

The full pipeline of \method\ meets the classical floorplanners on two protocols: the FloorSet val set under the contest cost, and the soft-block benchmarks on which the prior work reports its numbers. The pipeline is always the same: \method\ samples at $\nfe{=}8$, the closed-form refiner runs 400 steps, the legalizer closes every draw on 16 CPU workers and the cheapest feasible draw is kept. Every number is the legalized layout's, and every timed run had the machine to itself.

\paragraph{The classical lines.}
PARSAC \citep{mostafa2024parsacfasthumanqualityfloorplanning} is the constraints-aware B*-tree annealer of the group that built FloorSet. We load its C++ engine as released and run its published schedule of 100 stages from a random tree, with the boundary codes, clusters, fixed and preplaced blocks and MIB shapes of each case mapped one to one, for $10^7$ steps per worker, its budget in the paper. A worker is one independent run on one CPU core, and the best of $k$ workers is the per-case minimum over $k$ runs at $k$ times the CPU time. Every run is snapshotted after each stage and every snapshot is legalized and scored, so the curves below are one worker's best legalized cost so far against its own elapsed seconds. The sequence-pair annealer is the one of BBOPlace-Bench \citep{xue2026bboplacebenchbenchmarkingblackboxoptimization} with $10^5$ cost evaluations, and Astra is the public contest solver, whose published layouts we re-score with our checker and which reproduce its own numbers.

\subsection{The FloorSet Val Set}
\label{sec:appendix-classical-official}

Table~\ref{tab:classical_official} runs every line to its final layout. With 16 draws the pipeline reaches a mean hard cost of 1.019 in 1.14 seconds per case, and PARSAC reaches 1.222 with the best of 16 workers after 1,121 seconds per worker, about five CPU hours per case, 1.348 with the best of 4 and 1.685 with one worker. The gap is in the objectives: PARSAC's layouts are 26 to 47\% above the reference in wirelength, while the pipeline's are within 0.4\%. Astra's published layouts lie between the two at 1.282, and the sequence-pair annealer stays above 9 at its whole budget.

\begin{table}[t]\centering
\caption{The FloorSet val set (100 cases), every line to its final legalized layout: seconds per case of wall time (for an annealer, one worker's run), the CPU core-seconds per case (for \method, the legalizer's, plus the GPU seconds of sampling and refinement), the mean and median hard cost, and the legalized layout's wirelength and area gaps to the reference in percent. PARSAC workers are independent runs from random trees, one CPU core each. $^\dagger$Astra's reported runtime on its own CPU.}
\label{tab:classical_official}
\scriptsize\setlength{\tabcolsep}{7pt}
\begin{tabular}{llrrrrrr}
\toprule
Line & Draws or workers & s / case & Core-s / case & Mean & Median & WL & Area \\
\midrule
\rowcolor{ourscolor} \method, $\nfe{=}8$, 400 steps & 48 draws & 1.62 & 12 + 0.72 GPU & \textbf{1.014} & \textbf{1.000} & $-0.4$ & $+0.0$ \\
\method, $\nfe{=}8$, 400 steps & 16 draws & 1.14 & 4 + 0.69 GPU & 1.019 & 1.001 & $-0.2$ & $+0.4$ \\
\method, $\nfe{=}4$, 100 steps & 16 draws & 0.63 & 4 + 0.20 GPU & 1.042 & 1.006 & $+0.2$ & $+1.8$ \\
\method, $\nfe{=}1$, 50 steps & 16 draws & 0.48 & 5 + 0.10 GPU & 1.115 & 1.038 & $+3.6$ & $+5.0$ \\
\midrule
PARSAC, $10^7$ steps & best of 16 workers & 1,121 & 17,937 & 1.222 & 1.179 & $+26.0$ & $+4.7$ \\
PARSAC, $10^7$ steps & best of 4 workers & 1,121 & 4,484 & 1.348 & 1.241 & $+32.9$ & $+8.8$ \\
PARSAC, $10^7$ steps & one worker & 1,121 & 1,121 & 1.685 & 1.426 & $+47.2$ & $+22.2$ \\
SP-SA, $10^5$ evaluations & best of 4 workers & 104 & 416 & 9.24 & 9.34 & $+146.4$ & $+155.7$ \\
SP-SA, $10^5$ evaluations & one worker & 104 & 104 & 12.80 & 12.03 & $+182.8$ & $+223.8$ \\
Astra, published layouts & -- & 0.45$^\dagger$ & -- & 1.282 & 1.221 & $+15.5$ & $+14.0$ \\
\bottomrule
\end{tabular}
\end{table}

\begin{table}[t]
\begin{minipage}[b]{0.6\linewidth}\centering
\caption{PARSAC on the FloorSet val set after the given seconds per worker: the mean (median) over cases of the best legalized snapshot so far. The best of 16 workers is flat from 2,400 seconds on; \method\ reaches 1.019 (1.001) in 1.14 seconds per case with 16 draws.}
\label{tab:classical_anytime}
\scriptsize\setlength{\tabcolsep}{4pt}
\begin{tabular}{lrrrrrr}
\toprule
PARSAC & 60 s & 300 s & 600 s & 1,200 s & 2,400 s & 3,600 s \\
\midrule
one worker, mean & 2.34 & 1.62 & 1.51 & 1.45 & 1.42 & 1.41 \\
best of 4 workers, mean & 1.83 & 1.39 & 1.32 & 1.29 & 1.27 & 1.27 \\
best of 16 workers, mean & 1.60 & 1.28 & 1.23 & 1.20 & 1.19 & 1.19 \\
best of 16 workers, median & 1.34 & 1.20 & 1.18 & 1.17 & 1.16 & 1.16 \\
\bottomrule
\end{tabular}
\end{minipage}\hfill
\begin{minipage}[b]{0.38\linewidth}\centering
\caption{Per-unit CPU costs on the FloorSet val set by block count, one call at a time on an idle machine: the legalizer on one draw of \method, one PARSAC annealing step and one SP-SA cost evaluation, in milliseconds.}
\label{tab:classical_units}
\scriptsize\setlength{\tabcolsep}{3pt}
\begin{tabular}{lrrr}
\toprule
Blocks & Legalizer & PARSAC & SP-SA \\
\midrule
21 to 40 & 41 & 0.025 & 0.48 \\
41 to 60 & 98 & 0.057 & 0.75 \\
61 to 80 & 184 & 0.096 & 1.04 \\
81 to 100 & 401 & 0.151 & 1.34 \\
101 to 120 & 476 & 0.222 & 1.66 \\
\bottomrule
\end{tabular}
\end{minipage}
\end{table}

\begin{figure}[t]\centering
\includegraphics[width=\linewidth]{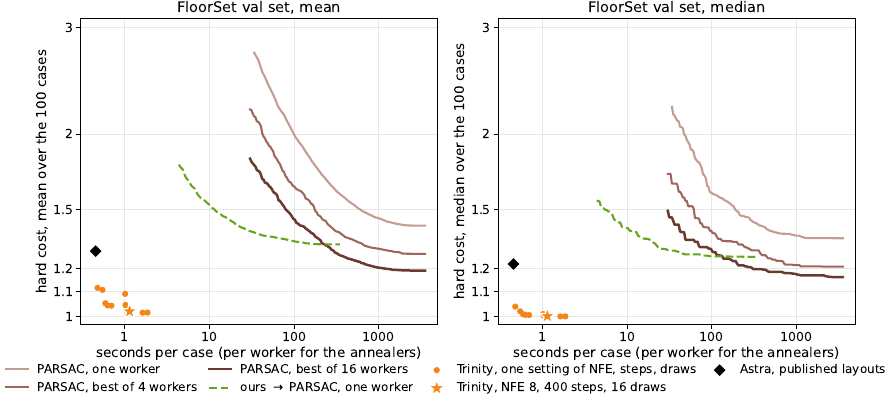}
\caption{The FloorSet val set: the mean (left) and median (right) hard cost against seconds per case. PARSAC as the best legalized snapshot so far of one worker and of the best of 4 and 16 workers, against seconds per worker on one CPU core; \method\ at every setting of $\nfe$, refiner steps and draws, on one GPU with 16 legalizer workers, the star at $\nfe{=}8$, 400 steps and 16 draws; ours $\to$ PARSAC is PARSAC started from the tree of the pipeline's legalized layout, drawn from the pipeline's time on. The sequence-pair annealer stays above 10 and is off the axis.}
\label{fig:classical_floorset}
\end{figure}

Table~\ref{tab:classical_anytime} and Figure~\ref{fig:classical_floorset} read PARSAC at a budget. No PARSAC line reaches the pipeline's point at any time: one worker needs 300 seconds to reach 1.62 and an hour to reach 1.41, and the best of 16 workers is flat at 1.19 from 2,400 seconds per worker on, so a longer run would not close the gap and only more workers narrow it. Table~\ref{tab:classical_units} gives the per-unit CPU costs behind these times: one PARSAC step costs 0.03 to 0.22 milliseconds from 21 to 120 blocks, so its $10^7$ steps take 4 to 37 minutes, while one legalizer call costs 41 to 476 milliseconds and runs once per draw.

\paragraph{Continuing from the pipeline.}
PARSAC started from the tree of the pipeline's legalized layout does not improve on it (the dashed line of Figure~\ref{fig:classical_floorset}). The first snapshot after $10^4$ steps is already at a mean of 2.05 against the pipeline's 1.019, and $10^6$ steps bring it back only to 1.31. A step-by-step probe of the hand-over on eight cases places the loss before any annealing: the compaction into a B*-tree moves 15 to 26 blocks, the tree does not reproduce the compacted layout on six of the eight cases, and the first greedy moves in PARSAC's own cost, which weighs white space and normalized wirelength rather than the contest cost, raise the contest cost further. A continuation that helps would need the legalized layout's exact representation and the contest cost as its acceptance rule.

\subsection{The Soft-Block Benchmarks of the Prior Work}
\label{sec:appendix-classical-bookshelf}

The prior floorplanning work reports on the GSRC and MCNC benchmarks with every block soft at its area, the aspect ratio within $[1/3, 3]$ (GSRC) or the file's range (MCNC), a square outline of 1.1 times the total block area, terminals at their benchmark positions, and the wirelength of block centres and terminals. We run the pipeline unchanged on this protocol with four sampling seeds and 16 or 64 draws each; a layout fits when it is overlap-free, keeps every area and aspect bound and lies inside the outline. \method\ was trained on FloorSet cases with 21 to 120 blocks and never saw these benchmarks.

\begin{table}[t]\centering
\caption{\method\ on the soft-block benchmarks of the prior work: every block soft at its area, a square outline of 1.1 times the block area, terminals at their benchmark positions. Per case, how many of four sampling seeds give a fitting legalized layout with 16 draws, the best fitting wirelength over the seeds with 16 and 64 draws (HPWL of block centres and terminals, benchmark units / 1,000), the dead space of that layout in percent, and the seconds per case of one seed with 16 draws. $^{\dagger}$Oversized: no seed fits, the best legal layout.}
\label{tab:classical_bookshelf}
\footnotesize\setlength{\tabcolsep}{10pt}
\begin{tabular}{lrrrrrr}
\toprule
Case & Blocks & Seeds fitting & HPWL, 16 & HPWL, 64 & Dead space & s / case \\
\midrule
\multicolumn{7}{l}{\textit{GSRC}} \\
n10 & 10 & 4 / 4 & 50.6 & 50.3 & 3.4 & 0.84 \\
n30 & 30 & 4 / 4 & 150.6 & 150.0 & 1.1 & 4.53 \\
n50 & 50 & 4 / 4 & 178.1 & 177.9 & 1.6 & 1.10 \\
n100 & 100 & 4 / 4 & 294.7 & 294.0 & 2.1 & 2.20 \\
n200 & 200 & 4 / 4 & 554.9 & 553.1 & 0.7 & 8.60 \\
n300 & 300 & 0 / 4 & 681.3$^{\dagger}$ & 681.8$^{\dagger}$ & 13.8 & 7.01 \\
\midrule
\multicolumn{7}{l}{\textit{MCNC}} \\
apte & 9 & 4 / 4 & 566.1 & 565.4 & 0.6 & 0.78 \\
xerox & 10 & 3 / 4 & 574.5 & 548.8 & 3.4 & 1.68 \\
hp & 11 & 4 / 4 & 173.6 & 173.1 & 2.3 & 0.71 \\
ami33 & 33 & 4 / 4 & 69.3 & 67.8 & 3.5 & 0.99 \\
ami49 & 49 & 4 / 4 & 932.5 & 883.3 & 4.6 & 1.20 \\
\bottomrule
\end{tabular}
\end{table}

Table~\ref{tab:classical_bookshelf} gives the pipeline on every case. It fits the outline on every MCNC case and on GSRC n10 to n200 on every seed, at 1 to 5\% dead space and in one to nine seconds per case. The one failure is GSRC n300, two and a half times the largest block count of training, where no seed fits and the best layout leaves 14\% dead space.

\begin{table}[t]\centering
\caption{GSRC soft blocks in a 10\% square outline, per column the best fitting wirelength in bold and the second underlined: the wirelength in mm (benchmark units / 1,000) of \method\ with 16 and 64 draws per seed over 4 seeds, of PARSAC run here from random trees with its published schedule for $10^7$ steps (the best final layout of 16 and of 4 workers, one CPU core each) and of the published rows. Seconds per case for \method\ with one seed and 16 draws (n100 / n200 / n300); minutes per worker for PARSAC. $^\dagger$From Table 2 of PARSAC \citep{mostafa2024parsacfasthumanqualityfloorplanning}, 112 workers of $10^7$ steps; best and worst are over its reported runs. $^{\ddagger}$Oversized: no draw fits, the best legal layout.}
\label{tab:classical_gsrc}
\scriptsize\setlength{\tabcolsep}{4pt}
\begin{tabular}{lrrrl}
\toprule
Line & n100 & n200 & n300 & Time \\
\midrule
\rowcolor{ourscolor} \method, 4 seeds, 16 / 64 draws per seed & 294.7 / 294.0 & 554.9 / 553.1 & 681.3$^{\ddagger}$ / 681.8$^{\ddagger}$ & 2.2 / 8.6 / 7.0 s \\
PARSAC here, best of 16 / 4 workers & 303.5 / 310.1 & 576.0 / 586.6 & 709.3 / 720.2 & 22 / 68 / 138 min \\
\midrule
PARSAC (soft), 112 workers, best$^\dagger$ & \textbf{288} & \textbf{526} & \textbf{627} & -- \\
PARSAC (soft), 112 workers, worst$^\dagger$ & 301 & 553 & 663 & -- \\
PARSAC (hard), 112 workers, best$^\dagger$ & \underline{291} & \underline{530} & \underline{628} & -- \\
GoodFloorplan$^\dagger$ & 319 & 575 & 705 & -- \\
SP-FOPP$^\dagger$ & 323 & 599 & 757 & -- \\
B*-FOPP$^\dagger$ & 344 & 654 & 799 & -- \\
\bottomrule
\end{tabular}
\end{table}

\begin{figure}[t]\centering
\includegraphics[width=\linewidth]{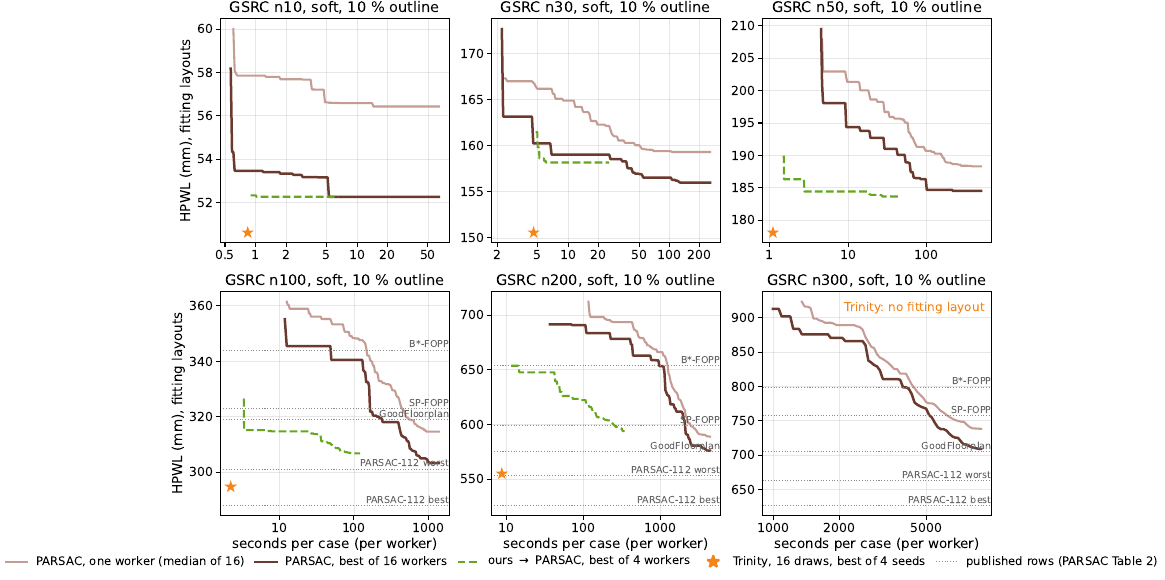}
\caption{GSRC soft blocks in a 10\% square outline: the wirelength of fitting layouts in mm against seconds per case. PARSAC as the median over 16 workers of one worker's best fitting snapshot so far and as the best of 16 workers; ours $\to$ PARSAC from the pipeline's layout, drawn from the pipeline's time on; \method\ with 16 draws, the best of four seeds; the published rows of Table~\ref{tab:classical_gsrc} as dotted levels.}
\label{fig:classical_gsrc}
\end{figure}

Table~\ref{tab:classical_gsrc} and Figure~\ref{fig:classical_gsrc} set the pipeline against PARSAC run here and the published rows. On n100 and n200 the pipeline's wirelength of 294.7 and 554.9 mm is 3 to 4\% below PARSAC's best of 16 workers, which takes 22 and 68 minutes per worker, inside the range of PARSAC's published best of 112 workers, and below GoodFloorplan \citep{Xu_2022} and both FOPP rows, in seconds per case. PARSAC's best of 4 workers here sits 3 to 9\% above the top of its published range, the size of the gap between 4 and 112 workers, so our run of PARSAC agrees with its paper. On n300 PARSAC fits after about 84 minutes per worker and the pipeline does not fit. As on FloorSet, PARSAC started from the pipeline's layout ends 3 to 7\% longer than where it started.

\section{Seeds}
\label{sec:appendix-seeds}

Four models were trained three times (the original seed plus two more, the data split unchanged) and evaluated as generators: the standard 12k $\times$ 6 $\nfe$ $\times$ 4-draw generation and scoring, soft and distribution metrics. Per metric the three values, their mean and the sample standard deviation. The four models are Trinity (spectral PE, information mover, all six terms at 0.01, the adopted recipe), the same backbone without the auxiliary loss (Trinity at $\lambda{=}0$), the plain backbone without the auxiliary loss (the plain control) and FlowPlace.

\begin{table}[htbp]\centering\footnotesize
\caption{Std / mean per model and column at $\nfe = 32$, 8 and 1, and the absolute std of the soft cost. Trinity's area gap at $\nfe{=}1$ has a mean near zero ($-0.0013$ to $+0.0043$), so no ratio is given. MIB variance and fixed deviation are the wide columns everywhere.}
\label{tab:app_seed_ratio}
\setlength{\tabcolsep}{3pt}
\begin{tabular}{lrrrrrrrrrr}
\toprule
Model & $\nfe$ & Overlap & Boundary & Group & MIB & Fixed & Area gap & $\mathrm{sev}_{\mathrm{rel}}$ & Soft cost & abs. std \\
\midrule
Trinity & 32 & 2.8\% & 9.3\% & 3.2\% & 31\% & 40\% & 68\% & 8.2\% & 1.1\% & 0.017 \\
Trinity & 8 & 2.4\% & 9.8\% & 3.7\% & 31\% & 40\% & 91\% & 8.6\% & 1.1\% & 0.018 \\
Trinity & 1 & 2.9\% & 10.6\% & 4.4\% & 32\% & 36\% & -- & 9.6\% & 1.0\% & 0.017 \\
Trinity, $\lambda{=}0$ & 32 & 1.0\% & 1.0\% & 3.1\% & 16\% & 17\% & 12.5\% & 1.9\% & 0.7\% & 0.015 \\
Trinity, $\lambda{=}0$ & 8 & 0.4\% & 1.1\% & 3.3\% & 12\% & 16\% & 14\% & 1.9\% & 1.0\% & 0.021 \\
Trinity, $\lambda{=}0$ & 1 & 0.9\% & 1.3\% & 3.3\% & 12\% & 11\% & 14.5\% & 2.0\% & 1.6\% & 0.037 \\
the plain control & 32 & 1.0\% & 3.0\% & 1.2\% & 4.3\% & 8.8\% & 8.4\% & 1.3\% & 0.4\% & 0.007 \\
the plain control & 8 & 1.5\% & 2.9\% & 1.3\% & 3.6\% & 9.8\% & 10\% & 1.4\% & 0.5\% & 0.011 \\
the plain control & 1 & 1.1\% & 2.5\% & 1.9\% & 4.2\% & 12\% & 12\% & 1.2\% & 0.3\% & 0.009 \\
FlowPlace & 32 & 2.9\% & 0.8\% & 1.4\% & 62\% & 1.4\% & 8.2\% & 1.5\% & 2.0\% & 0.050 \\
FlowPlace & 8 & 3.1\% & 3.4\% & 1.7\% & 23\% & 3.1\% & 7.4\% & 2.5\% & 3.5\% & 0.132 \\
FlowPlace & 1 & 3.0\% & 1.6\% & 1.0\% & 20\% & 4.2\% & 3.7\% & 1.2\% & 3.4\% & 0.38 \\
\bottomrule
\end{tabular}
\end{table}

The soft cost is reproducible to 0.3 to 1.1\% of its mean on the three models of this work and to 2 to 3.5\% on FlowPlace, overlap to 1 to 3\%, group to 1 to 4\%, boundary to 1 to 3\% on the two $\lambda{=}0$ models and FlowPlace and 9 to 11\% on Trinity, whose boundary slot is 3 to 4$\times$ smaller. The wide columns everywhere are the fixed deviation, the area gap and the MIB variance at $10^{-6}$ levels; the HPWL gap at $\nfe{=}32$ sits within $\pm0.001$ of zero on every model. A sweep-set delta below 2 to 3\% on a constraint column or 1\% on the soft cost is within one seed standard deviation of its control.

\begin{table}[htbp]\centering\scriptsize
\caption{Trinity per column, three training seeds, the five columns at $\nfe{=}32$ and then at $\nfe{=}1$ (std / mean in Table~\ref{tab:app_seed_ratio}). Units as in Section~\ref{sec:exp-setup}.}
\label{tab:app_seed_flagship}
\setlength{\tabcolsep}{2pt}
\begin{tabular}{lrrrrr@{\hspace{10pt}}rrrrr}
\toprule
Column & Original & Seed 1 & Seed 2 & Mean & Std & Original & Seed 1 & Seed 2 & Mean & Std \\
\midrule
overlap & 0.02253 & 0.02217 & 0.02132 & 0.02201 & 0.00062 & 0.03161 & 0.03054 & 0.02983 & 0.03066 & 0.00089 \\
boundary & 0.817 & 0.898 & 0.985 & 0.900 & 0.084 & 0.799 & 0.889 & 0.988 & 0.892 & 0.095 \\
group gap & 2.325 & 2.370 & 2.474 & 2.390 & 0.077 & 1.773 & 1.813 & 1.927 & 1.838 & 0.080 \\
MIB variance & 1.72 & 2.95 & 3.28 & 2.65 & 0.82 & 1.72 & 2.93 & 3.42 & 2.69 & 0.87 \\
fixed & 8.337 & 5.055 & 11.802 & 8.398 & 3.37 & 8.837 & 5.712 & 12.203 & 8.917 & 3.25 \\
HPWL gap & $-$0.00069 & $-$0.00054 & $+$0.00014 & $-$0.00036 & 0.00044 & $-$0.01237 & $-$0.01216 & $-$0.01183 & $-$0.01212 & 0.00027 \\
area gap & 0.00224 & 0.00259 & 0.00705 & 0.00396 & 0.00268 & $-$0.00125 & $+$0.00027 & $+$0.00429 & $+$0.00111 & 0.00286 \\
normalized wirelength & 0.30190 & 0.30193 & 0.30215 & 0.30199 & 0.00014 & -- & -- & -- & -- & -- \\
compactness & 1.0291 & 1.0295 & 1.0340 & 1.0309 & 0.0028 & -- & -- & -- & -- & -- \\
$\mathrm{sev}_{\mathrm{rel}}$ & 0.0849 & 0.0906 & 0.0998 & 0.0918 & 0.0075 & 0.0819 & 0.0888 & 0.0991 & 0.0899 & 0.0086 \\
soft cost & 1.5137 & 1.5277 & 1.5470 & 1.5295 & 0.0167 & 1.6785 & 1.6855 & 1.7108 & 1.6916 & 0.0170 \\
\bottomrule
\end{tabular}
\end{table}

\begin{table}[htbp]\centering\scriptsize
\caption{Trinity at $\lambda{=}0$ per column, the five columns at $\nfe{=}32$ and then at $\nfe{=}1$. Units as in Section~\ref{sec:exp-setup}.}
\label{tab:app_seed_control}
\setlength{\tabcolsep}{2pt}
\begin{tabular}{lrrrrr@{\hspace{10pt}}rrrrr}
\toprule
Column & Original & Seed 1 & Seed 2 & Mean & Std & Original & Seed 1 & Seed 2 & Mean & Std \\
\midrule
overlap & 0.02220 & 0.02261 & 0.02255 & 0.02245 & 0.00022 & 0.03573 & 0.03545 & 0.03508 & 0.03542 & 0.00032 \\
boundary & 2.997 & 2.946 & 2.947 & 2.963 & 0.029 & 3.011 & 2.943 & 2.941 & 2.965 & 0.040 \\
group gap & 2.827 & 2.675 & 2.682 & 2.728 & 0.086 & 2.320 & 2.177 & 2.219 & 2.239 & 0.073 \\
MIB variance & 6.66 & 6.30 & 8.38 & 7.11 & 1.1 & 7.19 & 7.32 & 8.85 & 7.79 & 0.92 \\
fixed & 7.474 & 6.905 & 9.472 & 7.951 & 1.35 & 7.956 & 7.221 & 9.027 & 8.068 & 0.91 \\
HPWL gap & $+$0.00219 & $-$0.00001 & $+$0.00024 & $+$0.00081 & 0.0012 & $-$0.01129 & $-$0.01308 & $-$0.01269 & $-$0.01235 & 0.00094 \\
area gap & 0.01810 & 0.01441 & 0.01498 & 0.01583 & 0.0020 & 0.01642 & 0.01262 & 0.01323 & 0.01409 & 0.0020 \\
$\mathrm{sev}_{\mathrm{rel}}$ & 0.2214 & 0.2137 & 0.2154 & 0.2168 & 0.0040 & 0.2216 & 0.2134 & 0.2151 & 0.2167 & 0.0044 \\
soft cost & 2.0151 & 1.9877 & 1.9919 & 1.9982 & 0.0148 & 2.4271 & 2.3692 & 2.3571 & 2.3845 & 0.0374 \\
\bottomrule
\end{tabular}
\end{table}

\begin{table}[htbp]\centering\scriptsize
\caption{The plain control (left five columns) and FlowPlace (right five) per column at $\nfe{=}32$. Units as in Section~\ref{sec:exp-setup}.}
\label{tab:app_seed_base}
\setlength{\tabcolsep}{2pt}
\begin{tabular}{lrrrrr@{\hspace{10pt}}rrrrr}
\toprule
Column & Original & Seed 1 & Seed 2 & Mean & Std & Original & Seed 1 & Seed 2 & Mean & Std \\
\midrule
overlap & 0.02602 & 0.02638 & 0.02656 & 0.02632 & 0.00028 & 0.04383 & 0.04602 & 0.04370 & 0.04452 & 0.0013 \\
boundary & 3.227 & 3.156 & 3.041 & 3.141 & 0.094 & 2.811 & 2.854 & 2.817 & 2.827 & 0.024 \\
group gap & 2.651 & 2.702 & 2.712 & 2.688 & 0.033 & 3.196 & 3.211 & 3.126 & 3.178 & 0.046 \\
MIB variance & 8.75 & 9.28 & 9.53 & 9.19 & 0.40 & 3.51 & 8.54 & 2.88 & 4.98 & 3.1 \\
fixed & 5.817 & 5.909 & 5.016 & 5.581 & 0.49 & 2.266 & 2.246 & 2.206 & 2.239 & 0.031 \\
HPWL gap & $-$0.00214 & $-$0.00089 & $-$0.00123 & $-$0.00142 & 0.00065 & $-$0.00231 & $-$0.00229 & $-$0.00279 & $-$0.00246 & 0.00028 \\
area gap & 0.01037 & 0.01224 & 0.01108 & 0.01123 & 0.00095 & 0.00870 & 0.00932 & 0.00791 & 0.00865 & 0.00071 \\
$\mathrm{sev}_{\mathrm{rel}}$ & 0.2238 & 0.2235 & 0.2185 & 0.2219 & 0.0030 & 0.2099 & 0.2140 & 0.2077 & 0.2105 & 0.0032 \\
soft cost & 2.0915 & 2.0993 & 2.0845 & 2.0918 & 0.0074 & 2.4564 & 2.5331 & 2.4406 & 2.4767 & 0.0495 \\
\bottomrule
\end{tabular}
\end{table}

At $\nfe{=}1$ the soft cost is 2.4900 / 2.5053 / 2.4905 ($2.4953 \pm 0.0087$) for the plain control and 11.465 / 11.351 / 10.755 ($11.19 \pm 0.38$) for FlowPlace.

\begin{figure}[htbp]\centering
\includegraphics[width=\linewidth]{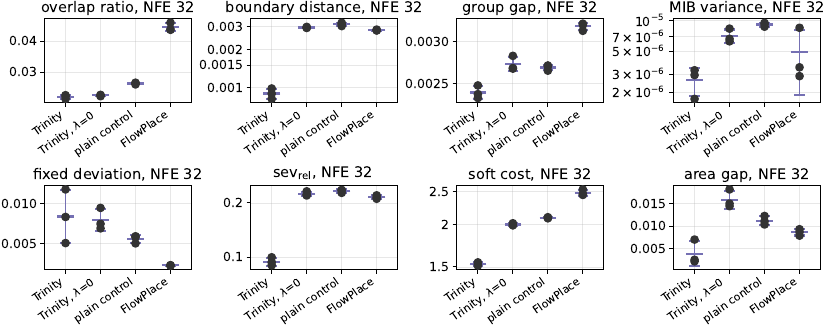}
\caption{Three training seeds per model at $\nfe{=}32$: the values, the mean and the standard deviation per soft column, for Trinity, Trinity at $\lambda{=}0$, the plain control and FlowPlace. The term effect on boundary, $\mathrm{sev}_{\mathrm{rel}}$ and the soft cost is many times the spread.}
\label{fig:app_h1}
\end{figure}

\begin{figure}[htbp]
\begin{minipage}[b]{0.52\linewidth}\centering
\includegraphics[width=\linewidth]{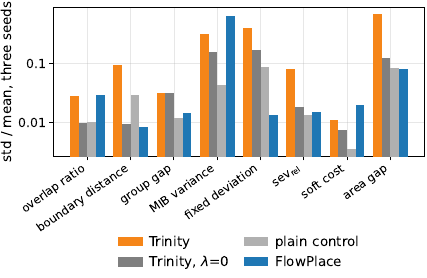}
\figcaption{The relative seed spread per column and model. The wide columns are the fixed deviation, the area gap and the MIB variance.}
\label{fig:app_h2}
\end{minipage}\hfill
\begin{minipage}[b]{0.46\linewidth}\centering
\includegraphics[width=\linewidth]{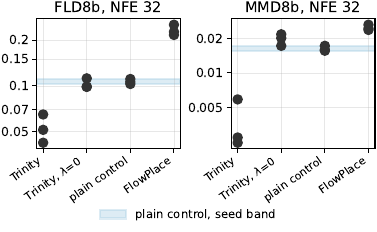}
\figcaption{FLD and MMD of every seed at $\nfe{=}32$ with the plain control's seed band. Trinity's worst seed is below the best seed of the plain control and of Trinity at $\lambda{=}0$.}
\label{fig:app_h3}
\end{minipage}
\end{figure}

\paragraph{The term effect against the seed spread.}
Trinity minus Trinity at $\lambda{=}0$ at $\nfe{=}32$, in units of the seed std of Trinity at $\lambda{=}0$ and of Trinity: boundary 0.00082 against 0.00296 is 70 std of Trinity at $\lambda{=}0$ and 25 of Trinity, $\mathrm{sev}_{\mathrm{rel}}$ 0.085 against 0.217 is 33 and 17, soft cost 1.514 against 1.998 is 33 and 28, MIB variance $1.7\times10^{-6}$ against $7.1\times10^{-6}$ is 5 and 4, group gap 0.00232 against 0.00273 is 5 and 5, overlap 0.0225 against 0.0225 is no effect at $\nfe{=}32$ (12 std of Trinity at $\lambda{=}0$ at $\nfe{=}1$, 0.0316 against 0.0354), and the fixed deviation 0.0083 against $0.0080 \pm 0.0014$ is no effect. Trinity's soft cost against FlowPlace's ($1.53 \pm 0.017$ against $2.48 \pm 0.05$) is 19 FlowPlace-std and 57 Trinity-std.

\begin{table}[htbp]\centering\scriptsize
\caption{Distribution metrics of the twelve caches at $\nfe{=}1$ and 32; the reference is the plain control's original run, $p$ the permutation test against it. Units as in Section~\ref{sec:exp-setup}.}
\label{tab:app_seed_dist1}
\label{tab:app_seed_dist32}
\setlength{\tabcolsep}{4pt}
\begin{tabular}{lrrrrrr}
\toprule
Run & $\nfe$ & FLD8b & FLD8f & FPD & MMD8b & $p$ \\
\midrule
the plain control, original & 1 & 0.4325 & 0.0179 & 0.0036 & 32.6 & -- \\
 & 32 & 0.1031 & 0.0033 & 0.0011 & 15.6 & -- \\
the plain control, seed 1 & 1 & 0.4474 & 0.0181 & 0.0038 & 34.4 & 0.199 \\
 & 32 & 0.1112 & 0.0041 & 0.0013 & 17.2 & 1.000 \\
the plain control, seed 2 & 1 & 0.4274 & 0.0161 & 0.0040 & 32.1 & 0.333 \\
 & 32 & 0.1066 & 0.0031 & 0.0013 & 15.9 & 0.955 \\
the plain control, 3 seeds & 1 & 0.436 $\pm$ 0.010 & 0.0174 $\pm$ 0.0011 & 0.0038 $\pm$ 0.0002 & 33.0 $\pm$ 1.2 &  \\
 & 32 & 0.1070 $\pm$ 0.0041 & 0.0035 $\pm$ 0.0005 & 0.0012 $\pm$ 0.0001 & 16.2 $\pm$ 0.9 &  \\
\midrule
Trinity, $\lambda{=}0$, original & 1 & 0.4264 & 0.0159 & 0.0042 & 37.6 & 0.821 \\
 & 32 & 0.1123 & 0.0048 & 0.0015 & 21.9 & 1.000 \\
Trinity, $\lambda{=}0$, seed 1 & 1 & 0.3885 & 0.0142 & 0.0040 & 30.2 & 0.020 \\
 & 32 & 0.0990 & 0.0031 & 0.0014 & 17.2 & 1.000 \\
Trinity, $\lambda{=}0$, seed 2 & 1 & 0.3708 & 0.0148 & 0.0039 & 30.9 & 0.378 \\
 & 32 & 0.0986 & 0.0045 & 0.0014 & 20.1 & 1.000 \\
Trinity, $\lambda{=}0$, 3 seeds & 1 & 0.395 $\pm$ 0.028 & 0.0150 $\pm$ 0.0009 & 0.0040 $\pm$ 0.0002 & 32.9 $\pm$ 4.1 &  \\
 & 32 & 0.1033 $\pm$ 0.0078 & 0.0041 $\pm$ 0.0009 & 0.0014 $\pm$ 0.0001 & 19.7 $\pm$ 2.4 &  \\
\midrule
Trinity, original & 1 & 0.1624 & 0.0173 & 0.0025 & 9.23 & 0.005 \\
 & 32 & 0.0423 & 0.0075 & 0.0012 & 2.34 & 0.005 \\
Trinity, seed 1 & 1 & 0.1698 & 0.0089 & 0.0022 & 8.14 & 0.005 \\
 & 32 & 0.0514 & 0.0036 & 0.0011 & 2.75 & 0.005 \\
Trinity, seed 2 & 1 & 0.1863 & 0.0072 & 0.0026 & 12.2 & 0.005 \\
 & 32 & 0.0650 & 0.0034 & 0.0013 & 5.89 & 0.005 \\
Trinity, 3 seeds & 1 & 0.173 $\pm$ 0.012 & 0.0111 $\pm$ 0.0054 & 0.0024 $\pm$ 0.0002 & 9.9 $\pm$ 2.1 &  \\
 & 32 & 0.0529 $\pm$ 0.0114 & 0.0048 $\pm$ 0.0023 & 0.0012 $\pm$ 0.0001 & 3.7 $\pm$ 1.9 &  \\
\midrule
FlowPlace, original & 1 & 3.7712 & 0.1823 & 0.0192 & 344 & 1.000 \\
 & 32 & 0.2274 & 0.0088 & 0.0013 & 24.9 & 1.000 \\
FlowPlace, seed 1 & 1 & 3.9361 & 0.1760 & 0.0195 & 382 & 1.000 \\
 & 32 & 0.2530 & 0.0103 & 0.0015 & 26.2 & 1.000 \\
FlowPlace, seed 2 & 1 & 3.6919 & 0.1777 & 0.0174 & 360 & 1.000 \\
 & 32 & 0.2167 & 0.0084 & 0.0012 & 23.7 & 1.000 \\
FlowPlace, 3 seeds & 1 & 3.80 $\pm$ 0.12 & 0.1787 $\pm$ 0.0033 & 0.0187 $\pm$ 0.0011 & 362 $\pm$ 19 &  \\
 & 32 & 0.232 $\pm$ 0.019 & 0.0092 $\pm$ 0.0010 & 0.0013 $\pm$ 0.0002 & 24.9 $\pm$ 1.3 &  \\
\bottomrule
\end{tabular}
\end{table}

Trinity is closer to the reference distribution than the plain control on every seed at both $\nfe$ ($p = 0.005$ for all six tests; MMD8b 2.3 to 5.9 at $\nfe{=}32$ against the plain control's 15.6 to 17.2 and Trinity at $\lambda{=}0$'s 17.2 to 21.9, FLD8b 0.042 to 0.065 against 0.099 to 0.112): its worst seed is 2.7$\times$ below the best control seed on MMD8b and 1.5$\times$ below on FLD8b. Its own spread on these two is the widest of the four models in relative terms (MMD8b 53\%, FLD8b 22\%), which a 2 to 5$\times$ advantage absorbs, and the original run's FLD8f of 0.0075 is the outlier of the three. Trinity at $\lambda{=}0$'s three seeds on FLD8b lie inside the plain control's at $\nfe{=}32$, so the backbone without the auxiliary loss is not separated from the plain control on the density descriptors within seed spread. The plain control's own replicates test at $p = 0.20$ to 0.33 at $\nfe{=}1$ and 0.96 to 1.00 at $\nfe{=}32$ against the original, the null band for the same model with another seed, and $p = 0.005$ is far outside it. FlowPlace's seed std on the distribution metrics is 3 to 8\% of the mean, and it tests at $p = 1.000$ on every seed.

\section{Figures: Every View of the Data}
\label{sec:appendix-figures}

Every figure below is rendered from the result files of the supplementary material with one panel per setup ($\nfe$, step count, term or column), the models as the lines of every panel and one shared legend. The main text uses the simplest view of each experiment; the views here show the same data from every side, and one result may appear in several of them.

\begin{figure}[htbp]
\subsection{The Soft Frontier}
\label{sec:appendix-fig-soft}
\centering
\includegraphics[width=\linewidth]{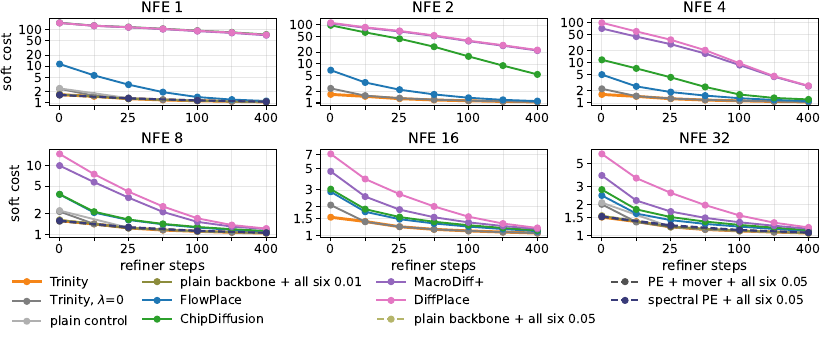}
\caption{Soft cost against refiner steps, one panel per $\nfe$, on our 12,000 val set with four draws. Trinity starts lowest in every panel, every line of this work converges to the same endpoint, and the prior arts stay above at 400 steps; DiffPlace and MacroDiff+ at $\nfe \leq 2$ start two decades higher.}
\label{fig:app_a1}
\end{figure}

\begin{figure}[htbp]\centering
\includegraphics[width=\linewidth]{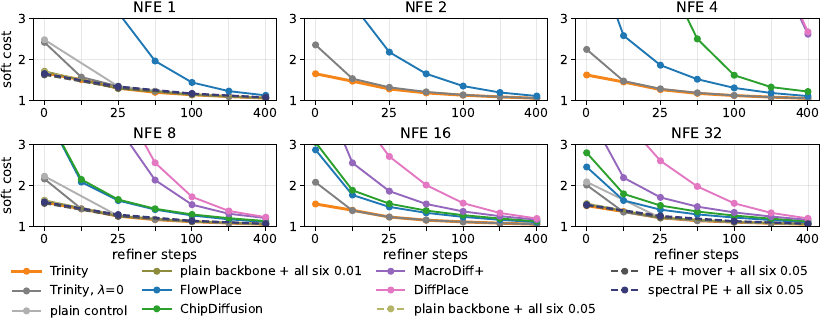}
\caption{The same data as Figure~\ref{fig:app_a1} on a linear axis from 1 to 3, where the lines of this work separate: the terms lift the raw sample and the first 25 steps, and the lead is largest at $\nfe{=}1$.}
\label{fig:app_a1_zoom}
\end{figure}

\begin{figure}[htbp]\centering
\includegraphics[width=\linewidth]{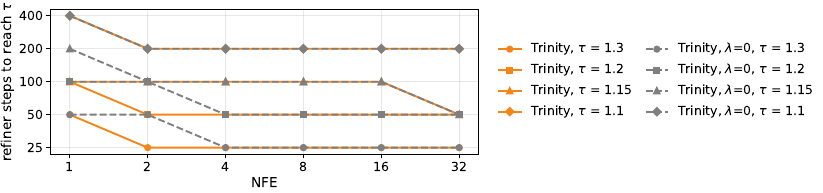}
\caption{Refiner steps needed to reach a soft cost of $\tau$ against $\nfe$, Trinity and the same backbone without the auxiliary loss. Trinity reaches each threshold one rung earlier at the low $\nfe$ values and the two agree from $\nfe{=}4$.}
\label{fig:app_a4}
\end{figure}

\begin{figure}[htbp]
\centering
\includegraphics[width=\linewidth]{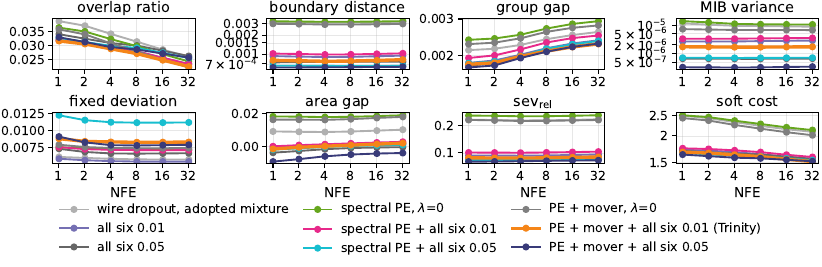}
\caption{Every soft column of the raw sample against $\nfe$ for the nine generators of the composed grid (three backbones $\times$ three term settings). Boundary and MIB variance fall with the terms at every $\nfe$, overlap falls with $\nfe$ on every model, and the fixed deviation is where the plain backbone leads.}
\label{fig:app_a5}
\end{figure}

\begin{figure}[htbp]
\centering
\includegraphics[width=\linewidth]{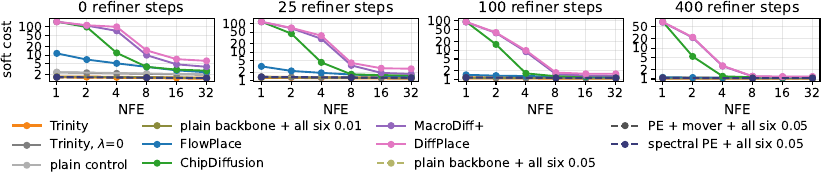}
\caption{Soft cost against $\nfe$, one panel per step count. The $\nfe$ gain shrinks with refinement: 0.16 on the raw Trinity sample, 0.008 at 400 steps.}
\label{fig:app_a2}
\end{figure}

\begin{figure}[htbp]
\centering
\includegraphics[width=\linewidth]{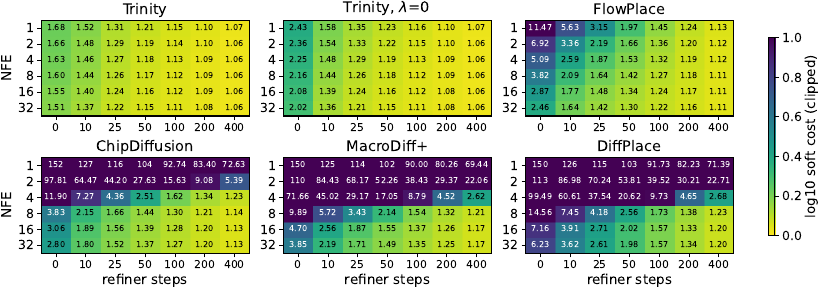}
\caption{The soft grid as heatmaps, rows $\nfe$, columns refiner steps, colour = $\log_{10}$ soft cost clipped at 10, one panel per line. Every grid is monotone in both axes.}
\label{fig:app_a3}
\end{figure}

\begin{figure}[htbp]
\subsection{The Hard Frontier}
\label{sec:appendix-fig-hard}
\centering
\includegraphics[width=\linewidth]{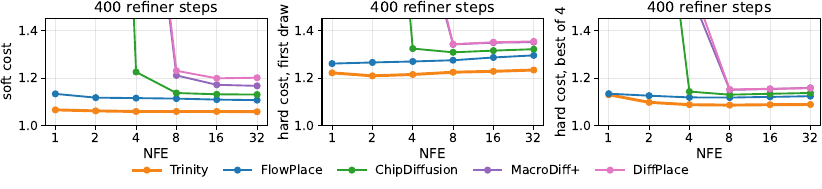}
\caption{Soft cost, first-draw hard cost and best-of-four hard cost against $\nfe$ at 400 refiner steps. The soft cost is monotone in $\nfe$ while the first draw has its minimum at $\nfe{=}2$ and the best of four at $\nfe{=}8$.}
\label{fig:app_b2}
\end{figure}

\begin{figure}[htbp]
\centering
\includegraphics[width=\linewidth]{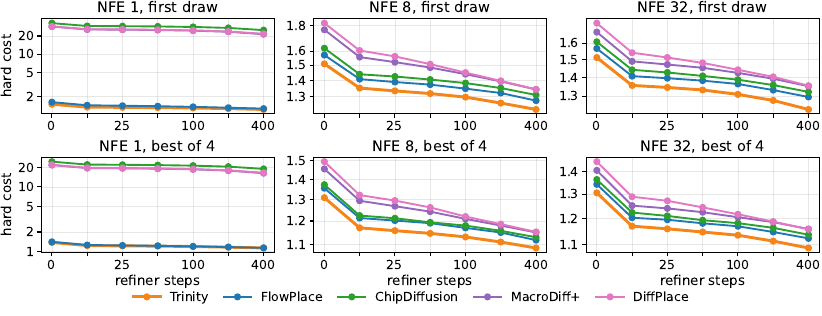}
\caption{Hard cost after legalization against refiner steps, one panel per $\nfe$ and selection (first draw, best of four), Trinity and the four prior arts. The refiner pays after legalization on every line, the largest move is from raw to 10 steps, and the order of the lines is the order of the raw soft cost.}
\label{fig:app_b1}
\end{figure}

\begin{figure}[htbp]\centering
\includegraphics[width=\linewidth]{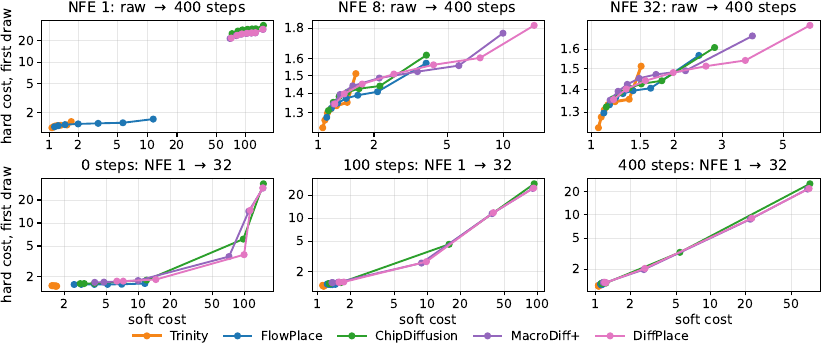}
\caption{Soft cost against first-draw hard cost, cell means, one connected path per line. Top: along the step axis (raw to 400 steps), one panel per $\nfe$; the paths are monotone along steps, and the prior arts' paths end above Trinity's. Bottom: along the $\nfe$ axis, one panel per step count; at 400 steps Trinity's path runs upward, the soft cost keeps falling with $\nfe$ while the first-draw cost rises slightly.}
\label{fig:app_b4}
\end{figure}

\begin{figure}[htbp]\centering
\includegraphics[width=\linewidth]{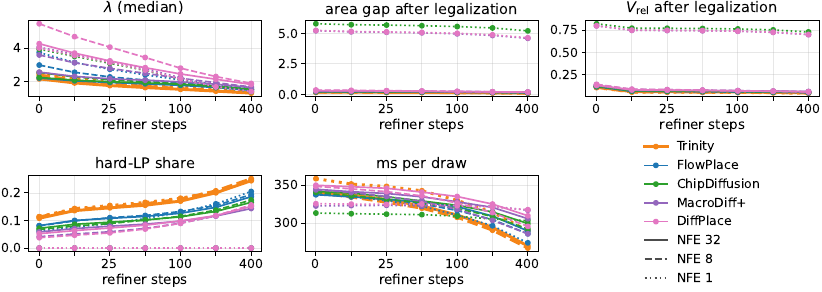}
\caption{The legalizer's ladder against refiner steps at $\nfe \in \{1, 8, 32\}$: the median expansion factor $\lambda$, the area gap and $\Vrel$ after legalization, the share of draws solved by the hard LP, and the milliseconds per draw. $\lambda$, the price of residual overlap, falls with refinement and Trinity needs the least; the two $\epsilon$-head prior arts at $\nfe{=}1$ have no layout to legalize.}
\label{fig:app_b6}
\end{figure}

\begin{figure}[htbp]
\centering
\includegraphics[width=\linewidth]{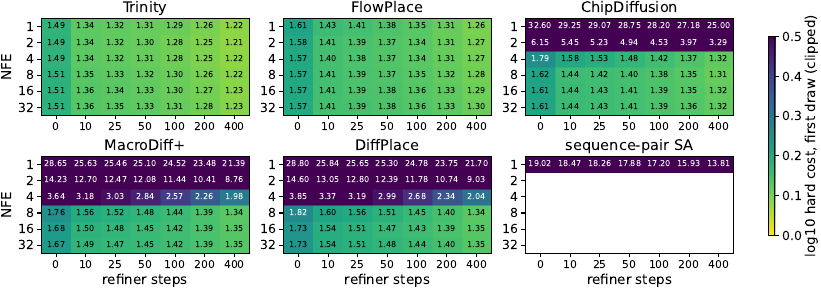}
\caption{The first-draw hard cost on the $\nfe \times$ steps grid, colour = $\log_{10}$ cost clipped at 3.2, one panel per line; the annealer has no sampler and fills one row.}
\label{fig:app_b3_n1}
\end{figure}

\begin{figure}[htbp]
\centering
\includegraphics[width=\linewidth]{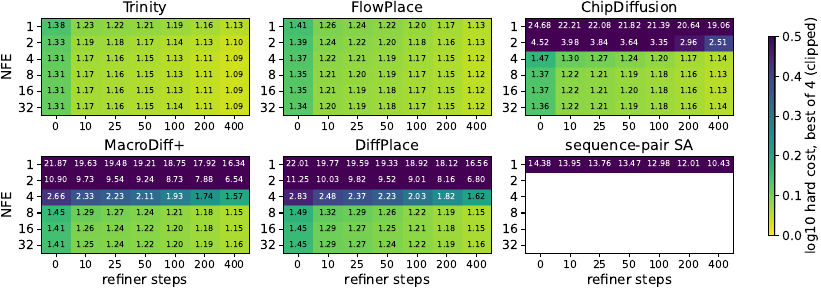}
\caption{The best-of-four hard cost on the $\nfe \times$ steps grid, same layout as Figure~\ref{fig:app_b3_n1}.}
\label{fig:app_b3_best4}
\end{figure}

\begin{figure}[htbp]\centering
\includegraphics[width=\linewidth]{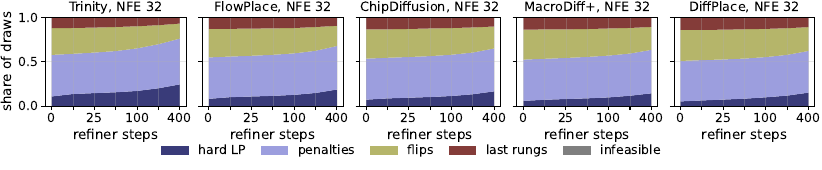}
\caption{Which rung of the legalizer solves each draw, as stacked shares against refiner steps at $\nfe{=}32$, one panel per line. The hard-LP share doubles from the raw sample to 400 steps and no draw is infeasible.}
\label{fig:app_b7}
\end{figure}

\begin{figure}[htbp]
\subsection{Soft Against Hard per Draw}
\label{sec:appendix-fig-corr}
\centering
\includegraphics[width=\linewidth]{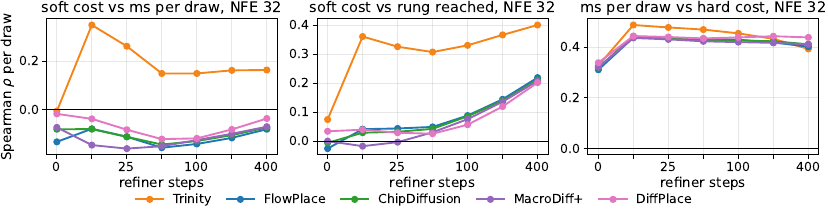}
\caption{Per-draw correlations with the legalizer's behaviour against refiner steps at $\nfe{=}32$: the soft cost predicts the rung a draw reaches and the time it takes, and the time predicts the hard cost on every line.}
\label{fig:app_c6}
\end{figure}

\begin{figure}[htbp]
\centering
\includegraphics[width=\linewidth]{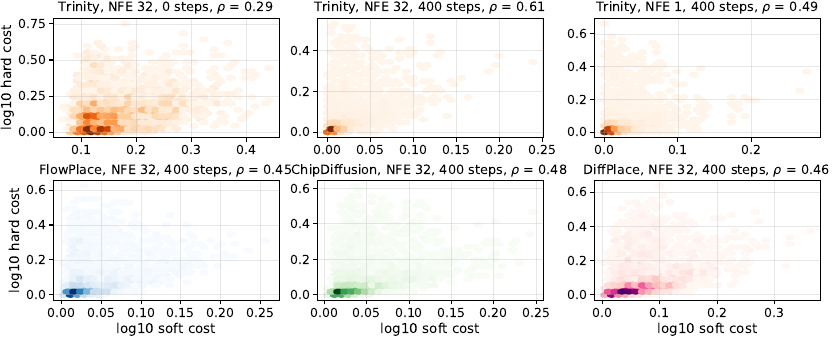}
\caption{The same six cells as Figure~\ref{fig:app_c1} on the per-case means over four draws.}
\label{fig:app_c2}
\end{figure}

\begin{figure}[htbp]
\centering
\includegraphics[width=\linewidth]{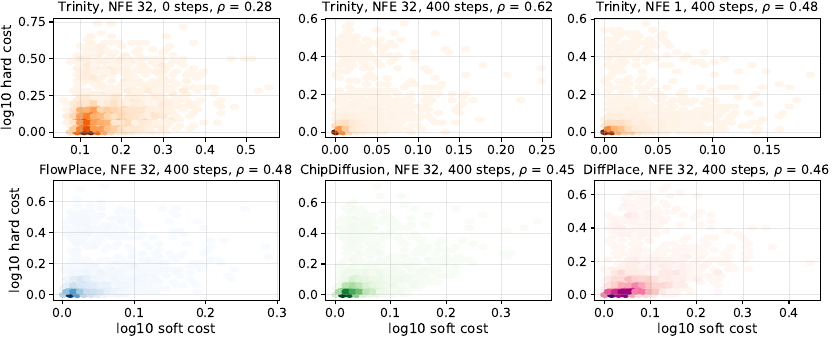}
\caption{Soft cost against legalized hard cost per draw, a 2,000-draw subsample of six cells. The relation is a ranking, not a line: the Spearman correlation reaches 0.62 at Trinity's $(\nfe{=}32, 400)$ cell, and the hard cost keeps a heavy tail.}
\label{fig:app_c1}
\end{figure}

\begin{figure}[htbp]\centering
\includegraphics[width=\linewidth]{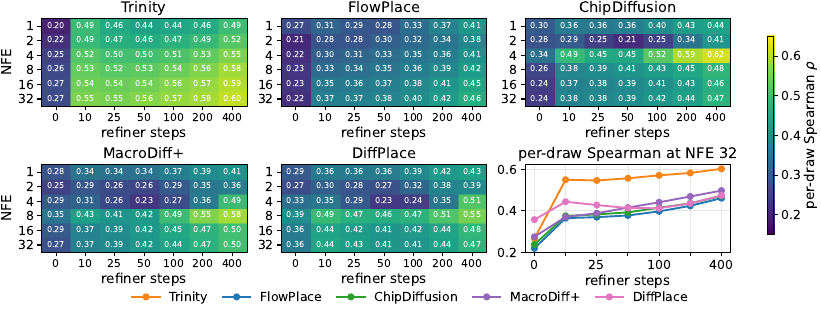}
\caption{The per-draw Spearman correlation between soft cost and hard cost per cell, one heatmap per line, and its course against refiner steps at $\nfe{=}32$. The correlation rises with both axes on Trinity, from 0.20 to 0.60, and less on the prior arts.}
\label{fig:app_c3}
\end{figure}

\begin{figure}[htbp]\centering
\includegraphics[width=\linewidth]{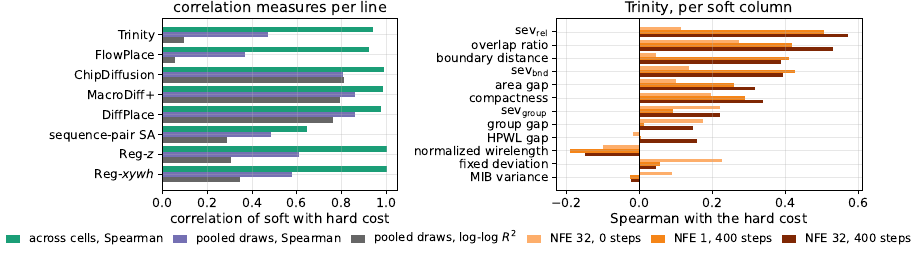}
\caption{Left: correlation measures per line, across cell means and pooled over draws. Right: the Spearman correlation of every soft column with the hard cost on Trinity at three cells; $\mathrm{sev}_{\mathrm{rel}}$ carries it, then overlap and boundary, and wirelength is negative because a shorter wirelength goes with a larger die.}
\label{fig:app_c4c5}
\end{figure}

\begin{figure}[htbp]\centering
\includegraphics[width=\linewidth]{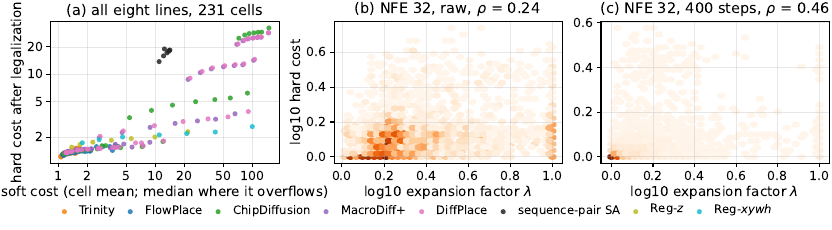}
\caption{\textbf{(a)} Figure~\ref{fig:results}b with all eight lines: the regressors' cells sit at a soft cost of 20 to 100, so their per-case median is drawn where the mean overflows. \textbf{(b, c)} The expansion factor $\lambda$ against the hard cost per draw on Trinity at $\nfe{=}32$, raw and after 400 refiner steps. Refinement moves the mass to $\lambda{=}1$, the draws that need no expansion.}
\label{fig:app_c_all_lines}
\label{fig:app_c7}
\end{figure}

\begin{figure}[htbp]
\subsection{The Term Ladders}
\label{sec:appendix-fig-terms}
\centering
\includegraphics[width=\linewidth]{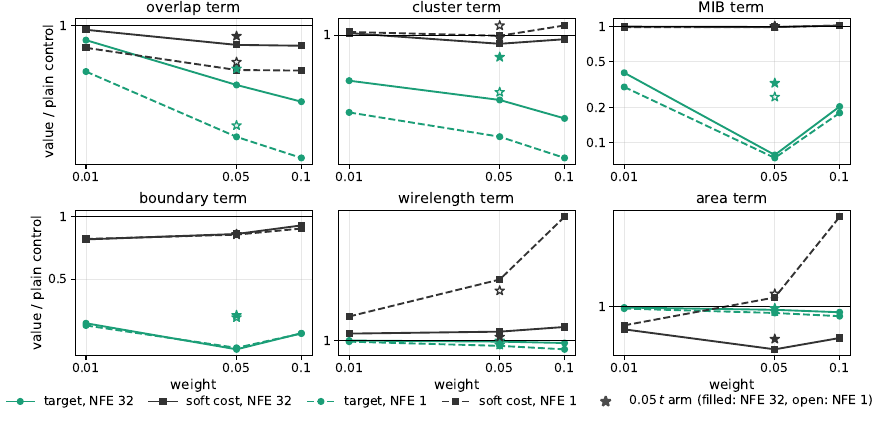}
\caption{Single-term ladders on the plain backbone: the term's own target metric and the soft cost relative to the plain control against the weight, at $\nfe{=}32$ (solid) and $\nfe{=}1$ (dashed), with the $0.05\,t$ arm as a star. Each term moves its own target; boundary and area move the soft cost, the others hardly.}
\label{fig:app_d1}
\end{figure}

\begin{figure}[htbp]\centering
\includegraphics[width=\linewidth]{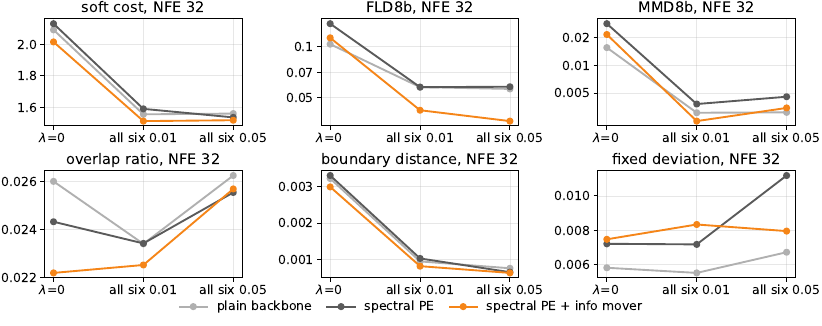}
\caption{The term ladder on the three backbones, one panel per metric. The terms do the same work on every backbone, and the backbones differ in where they start: the information mover lowers overlap and the distribution distances, the spectral encoding alone raises the fixed deviation.}
\label{fig:app_d5}
\end{figure}

\begin{figure}[htbp]
\centering
\includegraphics[width=\linewidth]{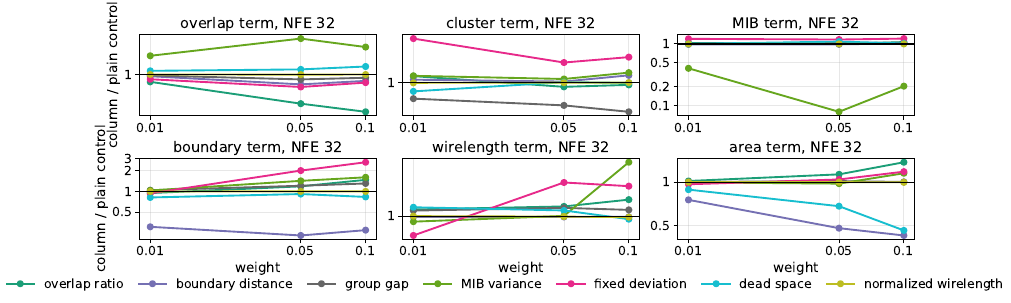}
\caption{The side effects: every soft column relative to the plain control against the weight, one panel per term. The boundary term raises the fixed deviation, the area term lowers the boundary distance, and the MIB term raises nothing.}
\label{fig:app_d2}
\end{figure}

\begin{figure}[htbp]\centering
\includegraphics[width=\linewidth]{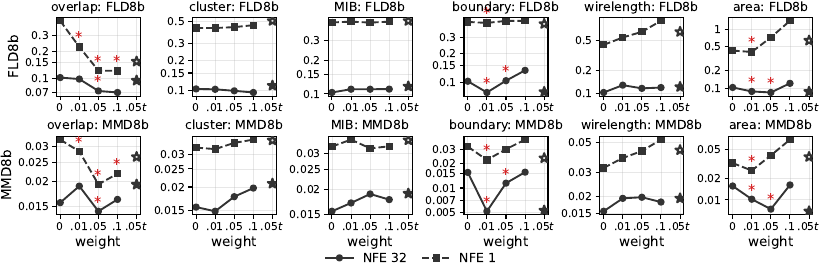}
\caption{FLD (top) and MMD (bottom) against the weight per term at $\nfe{=}32$ and 1, with a star where the arm is closer to the reference than the plain control at $p < 0.05$ under the permutation test of Appendix~\ref{sec:appendix-dist}. The distribution-closest weight is one step below the best raw target (overlap 0.05, boundary 0.01).}
\label{fig:app_d3_fld}
\label{fig:app_d3_mmd}
\end{figure}

\begin{figure}[htbp]\centering
\includegraphics[width=\linewidth]{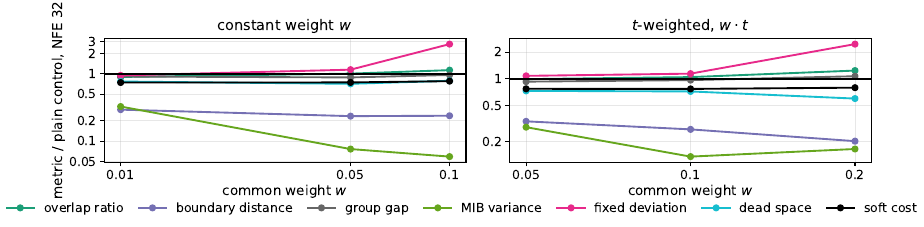}
\caption{All six terms at one common weight, constant (left) and $t$-weighted (right): every constraint metric and the soft cost relative to the plain control. The weight 0.01 improves every constraint at once, 0.1 breaks the fixed deviation, and the $t$-weighted arms trade the same way at twice the weight.}
\label{fig:app_d4}
\end{figure}

\begin{figure}[htbp]
\subsection{Metric Against Metric}
\label{sec:appendix-fig-metrics}
\centering
\includegraphics[width=\linewidth]{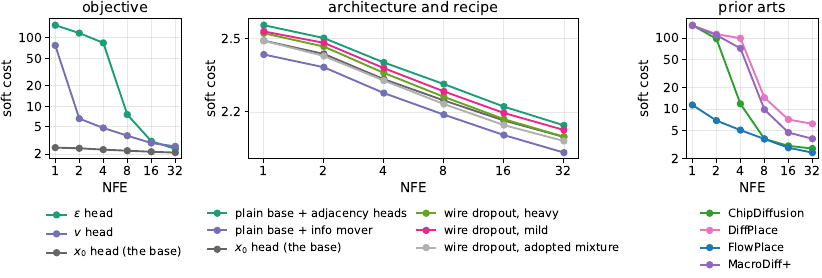}
\caption{Raw soft cost against $\nfe$ for the objective sweep, the architecture and recipe sweep, and the prior arts. The clean-sample head is flat in $\nfe$; the $\epsilon$ and $v$ heads, ours and the prior arts', are not.}
\label{fig:app_e4a}
\end{figure}

\begin{figure}[htbp]
\centering
\includegraphics[width=\linewidth]{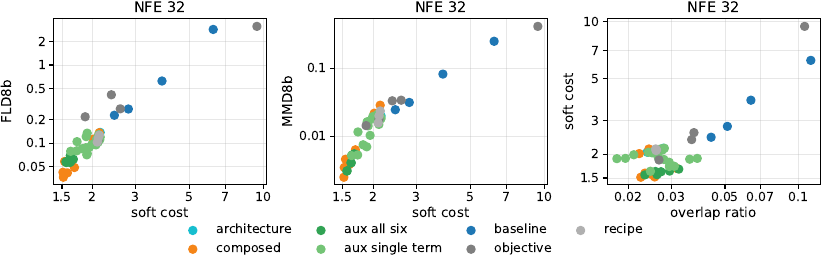}
\caption{Every model at $\nfe{=}32$: soft cost against FLD and against MMD, and overlap against soft cost. The two distribution metrics agree on the order, and the single-term overlap arms have the lowest overlap at the plain control's soft cost.}
\label{fig:app_e2e3}
\end{figure}

\begin{figure}[htbp]
\centering
\includegraphics[width=\linewidth]{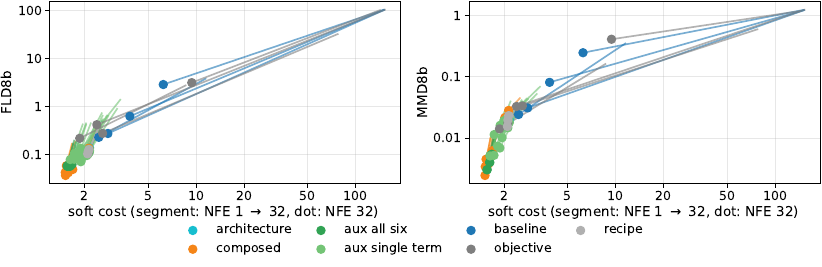}
\caption{Soft cost against FLD and MMD, one segment per model from $\nfe{=}1$ to 32, coloured by sweep set. The term-trained arms occupy the lower left, and the segments of the $\epsilon$-head and $v$-head models are long because they have no layout at $\nfe{=}1$.}
\label{fig:app_e1}
\end{figure}

\begin{figure}[htbp]
\centering
\includegraphics[width=\linewidth]{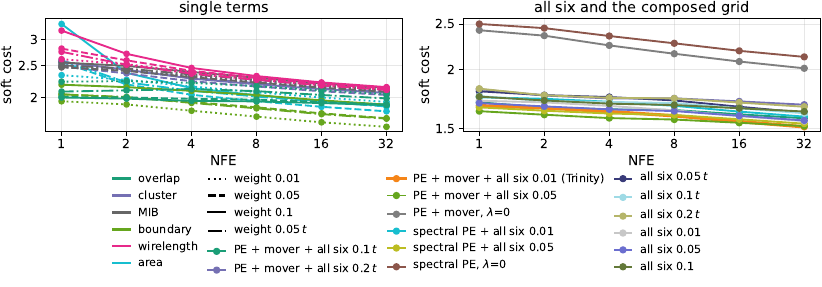}
\caption{Raw soft cost against $\nfe$ for every single-term arm and for the all-six and composed arms.}
\label{fig:app_e4b}
\end{figure}

\begin{figure}[htbp]
\subsection{Best of $N$ and the FloorSet Val Set}
\label{sec:appendix-fig-official}
\centering
\includegraphics[width=\linewidth]{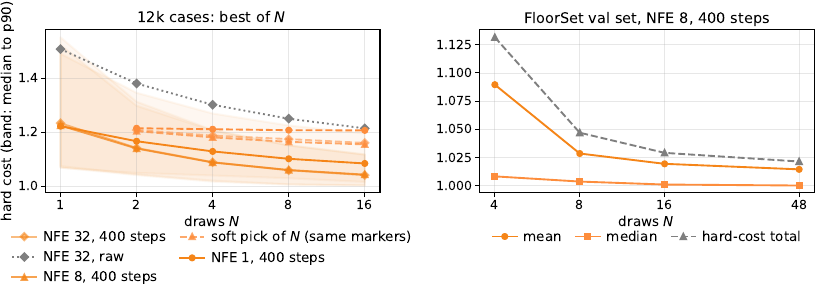}
\caption{Left: Trinity's best-of-$N$ mean hard cost on the 12,000 cases with the median-to-p90 band, at $\nfe \in \{32, 8, 1\}$ with 400 steps, at $\nfe{=}32$ raw, and the soft pick (the legalized cost of the draw with the lowest soft cost among $N$). $\nfe{=}8$ equals 32, $\nfe{=}1$ gains less from $N$, and the soft pick does not replace legalizing every draw. Right: the $N$ ladder on the FloorSet val set at $\nfe{=}8$ and 400 steps.}
\label{fig:app_f1f6}
\end{figure}

\begin{figure}[htbp]\centering
\includegraphics[width=\linewidth]{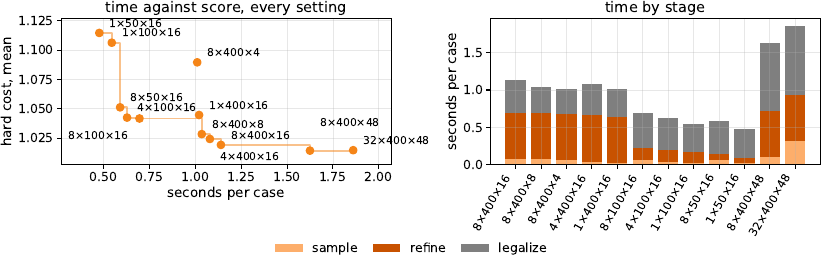}
\caption{The FloorSet val set solved one case at a time. Left: mean hard cost against seconds per case for every $\nfe \times$ steps $\times$ draws setting with the envelope; $N$ is the strongest lever and 16 draws at $\nfe{=}4$ with 100 steps reach 1.043 in 0.63 seconds. Right: the seconds per case by stage; the refiner is a fixed 0.6 seconds at 400 steps and the legalizer grows with $N$.}
\label{fig:app_f3f4}
\end{figure}

\begin{figure}[htbp]\centering
\includegraphics[width=\linewidth]{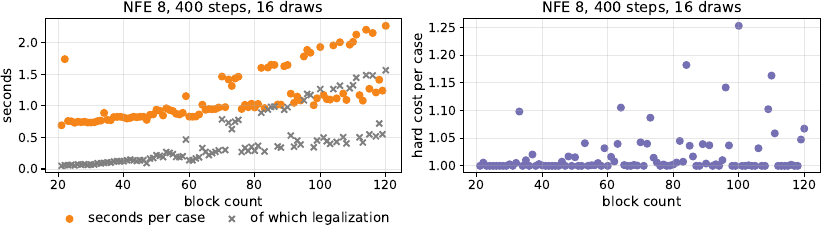}
\caption{Per FloorSet val set case at $\nfe{=}8$, 400 steps and 16 draws: seconds and hard cost against block count. The legalizer's time grows with the block count and the cost does not.}
\label{fig:app_f5}
\end{figure}

\begin{figure}[htbp]\centering
\includegraphics[width=\linewidth]{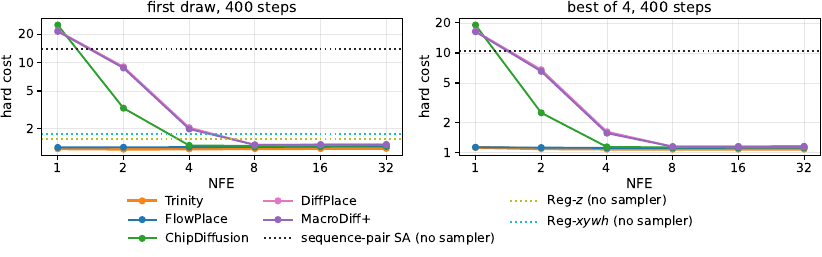}
\caption{First-draw and best-of-four hard cost at 400 steps against $\nfe$ for every line; the annealer and the regressors have no sampler and are drawn as levels.}
\label{fig:app_f2}
\end{figure}

\begin{figure}[htbp]
\subsection{The Refiner Sweeps}
\label{sec:appendix-fig-refiner}
\centering
\includegraphics[width=\linewidth]{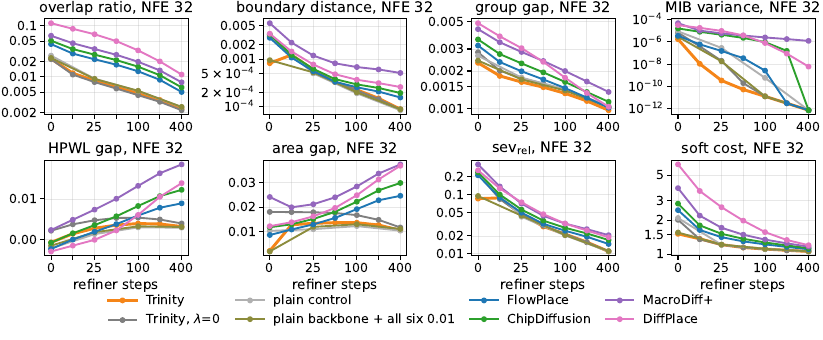}
\caption{Every soft column against refiner steps at $\nfe{=}32$ for the models of this work and the prior arts. Overlap, the cluster gap, the MIB variance and the fixed deviation fall monotonically, the boundary distance after a rise at 10 steps; the HPWL and area gaps are the residual the refiner trades for them.}
\label{fig:app_g3}
\end{figure}

\begin{figure}[htbp]\centering
\includegraphics[width=\linewidth]{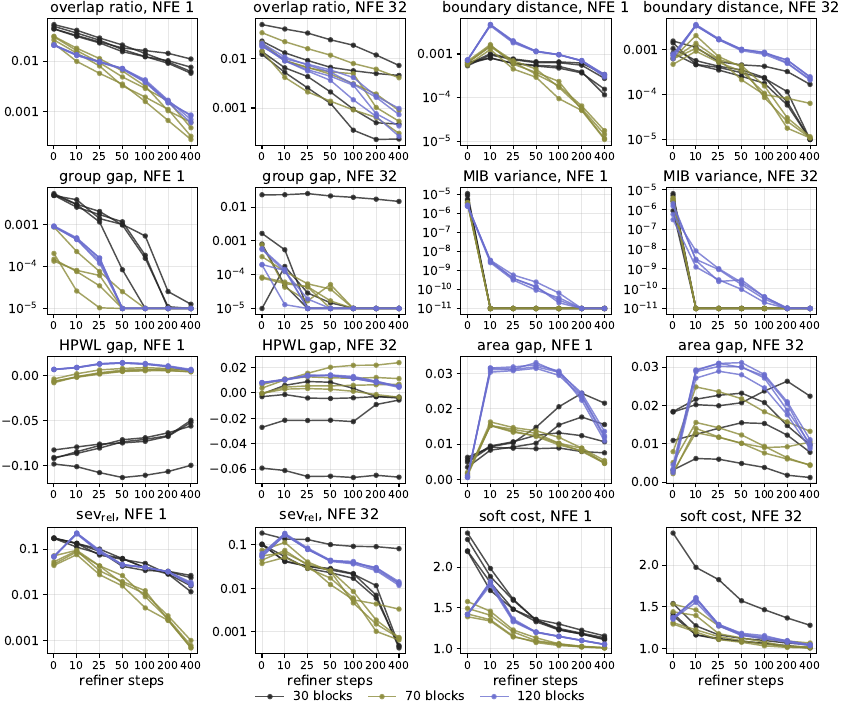}
\caption{Every draw of three cases of our 12,000 val set (30, 70 and 120 blocks, four draws each) along the refiner at $\nfe = 1$ and 32. Overlap and $\mathrm{sev}_{\mathrm{rel}}$ fall on every draw, the MIB variance and the cluster gap reach zero within 10 to 200 steps, the boundary distance rises at 10 steps on most draws before it falls, and the HPWL and area gaps are the residual each draw trades for them. One draw of the 30-block case keeps a cluster gap of 0.02 at $\nfe{=}32$ and ends at a soft cost of 1.28.}
\label{fig:app_g5}
\end{figure}

\begin{figure}[htbp]
\centering
\includegraphics[width=\linewidth]{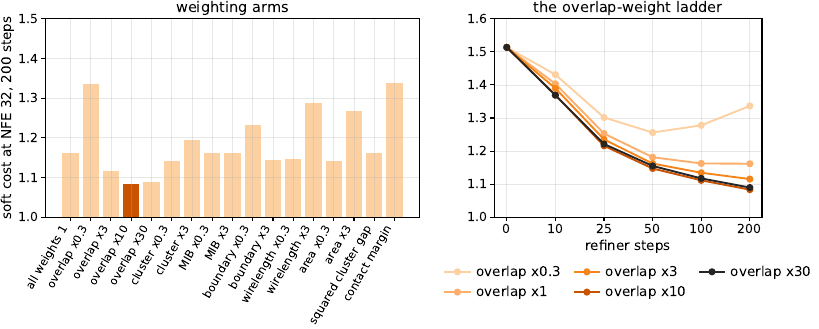}
\caption{Left: the seventeen weighting arms at 200 steps; the overlap weight of 10 is the one change that moves the endpoint. Right: the overlap-weight ladder against steps.}
\label{fig:app_g2}
\end{figure}

\begin{figure}[htbp]
\centering
\includegraphics[width=0.49\linewidth]{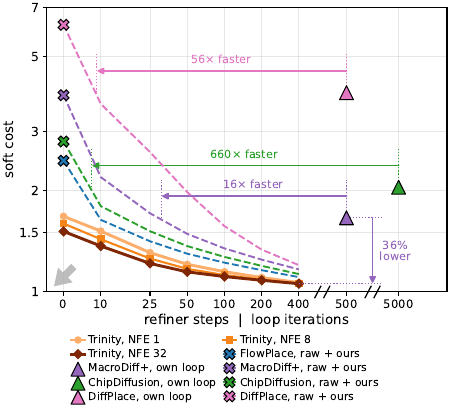}\hfill
\includegraphics[width=0.49\linewidth]{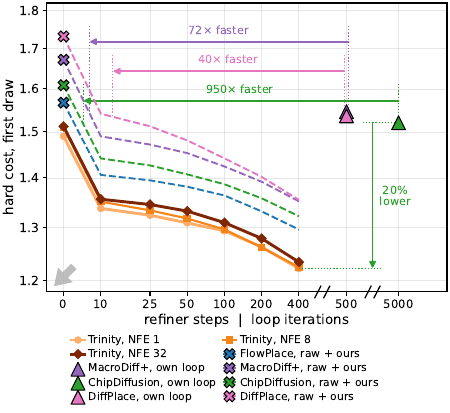}
\caption{The budget panel of Figure~\ref{fig:teaser} in soft cost (left) and in the hard cost of the first legalized draw (right): Trinity at $\nfe = 1$, 8 and 32, every prior-art generator at $\nfe{=}32$ raw (cross) and under the closed-form refiner (dashed), and its own loop at the published iteration count (triangle). A horizontal arrow ends at the refiner budget that matches the loop, by linear interpolation between the measured budgets. The vertical arrow runs from the cheapest loop to our cost at 400 steps.}
\label{fig:app_g6}
\end{figure}

\begin{figure}[htbp]
\centering
\includegraphics[width=\linewidth]{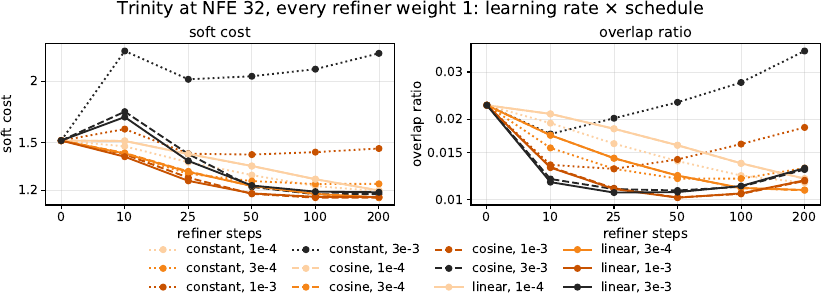}
\caption{Soft cost and overlap against refiner steps for the twelve schedule $\times$ peak-rate arms on Trinity at $\nfe{=}32$ with every weight at 1. Linear and cosine decay at $10^{-3}$ coincide, and a constant rate over-refines: overlap turns up after 50 steps.}
\label{fig:app_g1}
\end{figure}

\FloatBarrier
\subsection{Qualitative Examples}
\label{sec:appendix-qualitative}

Figure~\ref{fig:qualitative} follows three FloorSet val set cases through the pipeline of Section~\ref{sec:exp-hard} at the setting of Table~\ref{tab:official}: Trinity samples 48 draws at $\nfe{=}8$, the closed-form refiner runs 400 steps, the legalizer closes every draw, and the draw with the cheapest legalized layout is shown at every stage. The raw sample already has the outline, the boundary-coded blocks on their edges and the clusters together, with thin overlaps between neighbours; the refiner removes the overlaps and the legalizer packs the result. The gaps under the panels are signed, so a refined layout smaller than the reference, as on the 21-block case with an area gap of $-0.021$, is visible where the hard cost would clamp it to zero. The best of 48 lands at a hard cost of 1.000, 1.000 and 1.074 against 1.091, 1.000 and 1.127 for the reference layouts of the 21-, 60- and 120-block cases, and every one of the 144 legalized draws is feasible.

\begin{figure}[tp]\centering
\includegraphics[width=\linewidth]{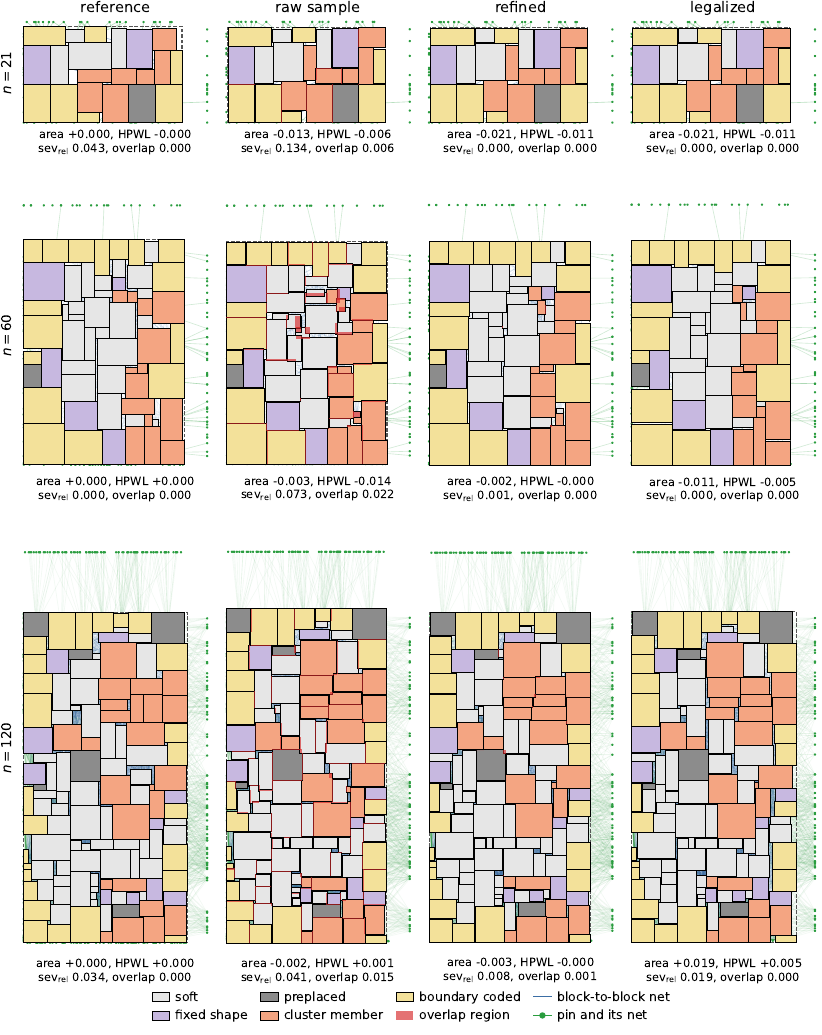}
\caption{Three FloorSet val set cases with 21, 60 and 120 blocks through the pipeline, 48 draws at $\nfe{=}8$ with 400 refiner steps and the shared legalizer, the draw with the cheapest legalized layout shown at every stage: the reference layout, the raw sample, the refined layout and the legalized layout. Blocks are coloured by constraint kind, every pairwise overlap region is red, block-to-block nets are faint lines between block centres, pins are dots on the outline with faint lines to their blocks, and the dashed box is the bounding outline. Under each panel: the signed area gap and HPWL gap to the reference, the severity $\mathrm{sev}_{\mathrm{rel}}$ and the overlap ratio of the soft vector of Appendix~\ref{sec:appendix-scoring}, which unlike the hard cost keep a layout smaller or shorter than the reference visible.}
\label{fig:qualitative}
\end{figure}

\begin{figure}[p]\centering
\includegraphics[width=0.78\linewidth]{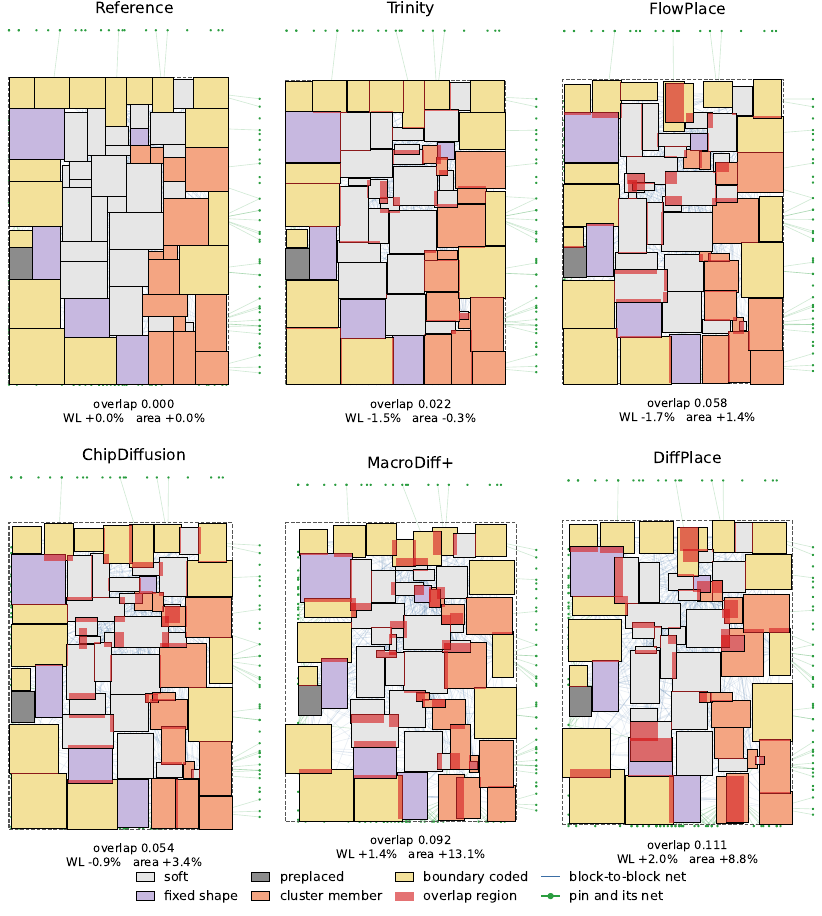}
\caption{Every generator on one 60-block FloorSet val set case from the same noise: the reference layout and the first draw of each line at $\nfe{=}8$, before any refinement. Colours as in Figure~\ref{fig:qualitative}. Under each panel, the overlap ratio and the soft cost of that layout.}
\label{fig:app_generators}
\label{fig:qual_row}
\end{figure}

\begin{figure}[p]\centering
\includegraphics[width=0.7\linewidth]{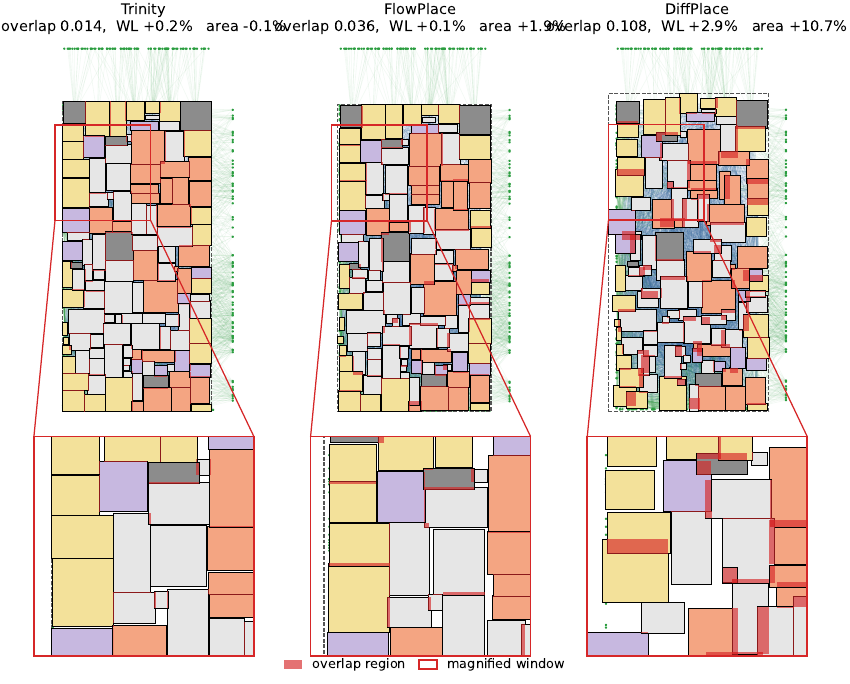}
\caption{A 120-block case under the protocol of Figure~\ref{fig:app_generators}, with one region magnified. The window is placed on the largest overlap of the worst line and is the same in all three columns.}
\label{fig:app_generator_zoom}
\end{figure}

Figure~\ref{fig:app_generators} puts the generators side by side on one of those cases. Because the sampler draws its noise from the case alone, the first draw of every line starts from one noise tensor, so the panels are paired rather than merely comparable. The ordering of Table~\ref{tab:prior} is visible block by block. On the case shown Trinity overlaps 0.022 of the block area at a soft cost of 1.44, FlowPlace and ChipDiffusion 0.058 and 0.054 at 3.24 and 3.55, and MacroDiff+ and DiffPlace 0.092 and 0.111 at 7.57 and 8.43; over the 21-, 60- and 120-block cases the same draw gives 0.014 to 0.022 for Trinity, 0.036 to 0.067 for the two velocity and cosine heads and 0.068 to 0.111 for the two remaining lines. The kind of failure also differs. Every line finds the outline and puts the boundary-coded blocks on their edges, so the gap is not in the global arrangement but in the local packing, which Figure~\ref{fig:app_generator_zoom} magnifies on a larger case: Trinity leaves thin contacts where a neighbour is a few thousandths of the layout scale too wide, while DiffPlace leaves whole blocks lying across one another. This is the defect the closed-form refiner of Section~\ref{sec:exp-refine} removes in ten steps on Trinity and cannot remove on the prior arts in four hundred.

\FloatBarrier

\section{Model Size}
\label{sec:appendix-size}

\begin{wraptable}{r}{0.36\linewidth}\centering\footnotesize
\vspace{-\baselineskip}
\caption{The size rungs of Trinity recipe, defined by the hidden width $d$ and the layer count $L$ (head dimension 64 throughout). Everything not listed is Trinity's.}
\label{tab:size_recipe}
\scriptsize
\begin{tabular}{lrrr}
\toprule
Rung & $d$ & $L$ & Parameters \\
\midrule
\multicolumn{4}{l}{\textit{Width rungs}} \\
192 / 12 & 192 & 12 & 5.9M \\
256 / 12 & 256 & 12 & 10.6M \\
384 / 12 & 384 & 12 & 23.3M \\
512 / 12 & 512 & 12 & 42.1M \\
\midrule
\multicolumn{4}{l}{\textit{Depth rungs}} \\
768 / 4 & 768 & 4 & 31.5M \\
768 / 6 & 768 & 6 & 46.8M \\
768 / 8 & 768 & 8 & 62.2M \\
\midrule
\multicolumn{4}{l}{\textit{Off both axes}} \\
512 / 16 & 512 & 16 & 56.0M \\
\midrule
768 / 12, Trinity & 768 & 12 & 92.9M \\
\bottomrule
\end{tabular}
\vspace{-\baselineskip}
\end{wraptable}
\paragraph{Setup.}
Two variables define a rung of Trinity recipe: the hidden width $d$ and the layer count $L$ (Table~\ref{tab:size_recipe}). The width rungs fix $L = 12$ and vary $d$, the depth rungs fix $d = 768$ and vary $L$, and one rung sits off both axes at width 512 with 16 layers; the head dimension is 64 throughout, so the head count follows $d$. The spectral positional encoding, the information mover, all six terms at $\lambda{=}0.01$, the wire-dropout mixture and the frozen training recipe of Appendix~\ref{sec:appendix-train} (AdamW at a peak learning rate of $2\times10^{-4}$ with cosine decay, effective batch 256, 200,000 steps, one GPU) are identical on every rung, so the runs differ in $d$ and $L$ alone. The width rungs sit at the sizes of the ported prior arts, 5.9M to 42.1M parameters against 5.0M to 23.9M for their networks; the depth rungs span 31.5M to 62.2M, so the two axes overlap in parameter count and separate the effect of the shape from that of the size. Every rung of the first six goes through the standard evaluation: our 12,000 val set at $\nfe \in \{1, 2, 4, 8, 16, 32\}$ with four draws and the fifteen soft columns, the distribution metrics at $\nfe{=}1$ and 32 against the plain control with 200 permutations, the full $\nfe \times$ steps refiner grid of Appendix~\ref{sec:appendix-refine}, and the same grid through the legalizer with the first draw and the best of four, as in Appendix~\ref{sec:appendix-hard}. The two rungs trained last, 192 / 12 at 5.9M parameters and 512 / 16 at 56.0M, have the raw scores of the val set only and appear in Tables~\ref{tab:size_recipe} and \ref{tab:size_raw}.

\paragraph{Width rungs, raw samples.}
Table~\ref{tab:size_raw} and Figure~\ref{fig:scaling} give the widths at $L = 12$ on our 12,000 val set, the figure the four with full grids. At width 512, with 45\% of Trinity's parameters, the recipe reproduces Trinity within its seed spread on the raw sample (soft cost 1.534 against 1.514, 1.2 seed standard deviations; overlap within 0.5) and is equal to it at $\nfe{=}1$; on the distribution metrics it is at or below Trinity at both $\nfe$ (FLD 0.0416 against 0.0423, MMD 0.0020 against 0.0025). Every rung is closer to the reference than the plain control at both $\nfe$ with $p = 0.005$, every rung keeps the constraint level the terms train (boundary distance 0.86 to $0.97 \times 10^{-3}$ against $3.0 \times 10^{-3}$ for Trinity at $\lambda{=}0$, MIB variance 2 to $4 \times 10^{-6}$ against $6.7 \times 10^{-6}$), and every rung improves the deviation of fixed and preplaced blocks over Trinity. Widths 384 and 256 give back overlap (23\% and 32\% more at $\nfe{=}32$) and soft cost (1.612 and 1.703), width 256 also on the cluster gap and the boundary distance, and the rung at 10.6M parameters still stands at a soft cost of 1.70 against 2.46 for FlowPlace at 6.3M and 6.23 for DiffPlace at 23.9M under the same recipe and steps. Width 192 at 5.9M parameters, the size of the smallest ported network, continues the ladder to a soft cost of 1.827 at $\nfe{=}32$ and 1.903 at $\nfe{=}1$ with an overlap of 0.0347, the boundary distance and the MIB variance at $1.0 \times 10^{-3}$ and $13.8 \times 10^{-6}$, and it stays below every ported placer's soft cost at the same $\nfe$. The shape of the $\nfe$ curve is the same at every width.

\paragraph{Depth rungs, raw samples.}
The depth rungs keep the width at 768 and remove layers (Table~\ref{tab:size_raw}, Figure~\ref{fig:scaling_depth}). Every depth rung is behind Trinity on every constraint column at $\nfe{=}32$: at $L = 8$ the overlap is 6\% higher, the boundary distance 17\%, the cluster gap 11\% and the MIB variance 2.2 times, for a soft cost of 1.593 against 1.514; at $L = 6$ the overlap is 37\% higher and the soft cost 1.733; at $L = 4$ the overlap doubles to 0.048, FlowPlace's level, the MIB variance is 3.9 times Trinity's and the soft cost is 2.128. Depth costs more than width per parameter, and depth beyond twelve layers buys little: width 512 with 16 layers (56.0M) sits between 512 / 12 and Trinity at every $\nfe$, 1.530 against 1.534 and 1.514 at $\nfe{=}32$ and 1.679 against 1.682 and 1.679 at $\nfe{=}1$, with the overlap of 512 / 12 (0.0226) and the boundary distance and MIB variance of the width rungs, so four more layers at width 512 are worth 0.004 in soft cost for 14M parameters. By soft cost at $\nfe{=}32$ the family orders as 768 / 12 (1.514), 512 / 16 (1.530), 512 / 12 (1.534), 768 / 8 (1.593), 384 / 12 (1.612), 256 / 12 (1.703), 768 / 6 (1.733), 192 / 12 (1.827) and 768 / 4 (2.128), so the 8-layer model at 62.2M parameters is within 0.02 of the width-384 model at 23.3M (and behind it after the refiner, 1.067 against 1.065, and on both distribution distances), the 6- and 4-layer models at 46.8M and 31.5M below the width-256 model at 10.6M, the 4-layer model below width 192 at 5.9M, and the loss per removed layer grows, $+0.08$ from 12 to 8 layers, $+0.14$ from 8 to 6 and $+0.39$ from 6 to 4. The deviation of fixed blocks is the one column that does not follow the depth. On the distribution metrics the 8- and 6-layer rungs are closer to the reference than the plain control at both $\nfe$ with $p = 0.005$ and inside Trinity's seed band on FLD, while the 4-layer rung is not ($p = 1.000$ at $\nfe{=}32$, FLD 0.151 against 0.103 for the plain control): it is the one rung of the study whose layouts are farther from the reference distribution than the plain control's while its soft cost is lower. The reading is that the graph-convolution residual needs depth. The information mover moves information along the netlist one hop per layer, so a 4-layer model cannot carry the netlist far enough to place its blocks, while a narrow model at 12 layers still can.

\begin{table}[htbp]\centering\footnotesize
\caption{The size rungs ($d$ / $L$) on raw samples, $\nfe{=}32$ unless stated: the constraint columns, the soft cost at $\nfe{=}32$ / 1, and the two distribution distances against the reference, which the two rungs trained last do not have. Units as in Section~\ref{sec:exp-setup}; \colorbox{ourscolor}{orange} marks Trinity.}
\label{tab:size_raw}
\setlength{\tabcolsep}{4pt}
\scriptsize
\begin{tabular}{lrrrrrrrrrr}
\toprule
$d$ / $L$ & Params & Overlap & Boundary & MIB & Fixed & Area gap & Soft cost 32 / 1 & FLD & MMD & HPWL gap \\
\midrule
\multicolumn{10}{l}{\textit{Width rungs}} \\
192 / 12 & 5.9M & 0.0347 & 1.00 & 13.8 & 7.2 & $+$0.003 & 1.827 / 1.903 & -- & -- & $+0.001$ \\
256 / 12 & 10.6M & 0.0297 & 0.97 & 3.7 & 6.6 & $+$0.003 & 1.703 / 1.812 & 0.055 & 5.0 & $+0.001$ \\
384 / 12 & 23.3M & 0.0277 & 0.86 & 2.0 & 6.2 & \textbf{$-$0.004} & 1.612 / 1.759 & 0.044 & 3.7 & $-0.002$ \\
512 / 12 & 42.1M & 0.0229 & 0.89 & 2.4 & 5.8 & $+$0.000 & 1.534 / 1.682 & \textbf{0.042} & \textbf{2.0} & $-0.001$ \\
\midrule
\multicolumn{10}{l}{\textit{Depth rungs}} \\
768 / 4 & 31.5M & 0.0483 & 1.17 & 6.6 & 12.3 & $-$0.009 & 2.128 / 2.258 & 0.151 & 17.2 & $-0.005$ \\
768 / 6 & 46.8M & 0.0309 & 1.06 & 4.2 & 6.0 & $+$0.000 & 1.733 / 1.865 & 0.051 & 3.6 & $-0.002$ \\
768 / 8 & 62.2M & 0.0239 & 0.96 & 3.8 & 11.2 & $+$0.005 & 1.593 / 1.740 & 0.048 & 4.0 & $-0.001$ \\
\midrule
\multicolumn{10}{l}{\textit{Off both axes}} \\
512 / 16 & 56.0M & 0.0226 & 0.93 & 3.0 & \textbf{5.1} & $-$0.002 & 1.530 / \textbf{1.679} & -- & -- & $-0.002$ \\
\midrule
\rowcolor{ourscolor} \textbf{768 / 12} & 92.9M & \textbf{0.0225} & \textbf{0.82} & \textbf{1.7} & 8.3 & $+$0.002 & \textbf{1.514} / \textbf{1.679} & \textbf{0.042} & 2.5 & $-0.001$ \\
\bottomrule
\end{tabular}
\end{table}

\begin{figure}[htbp]\centering
\includegraphics[width=\linewidth]{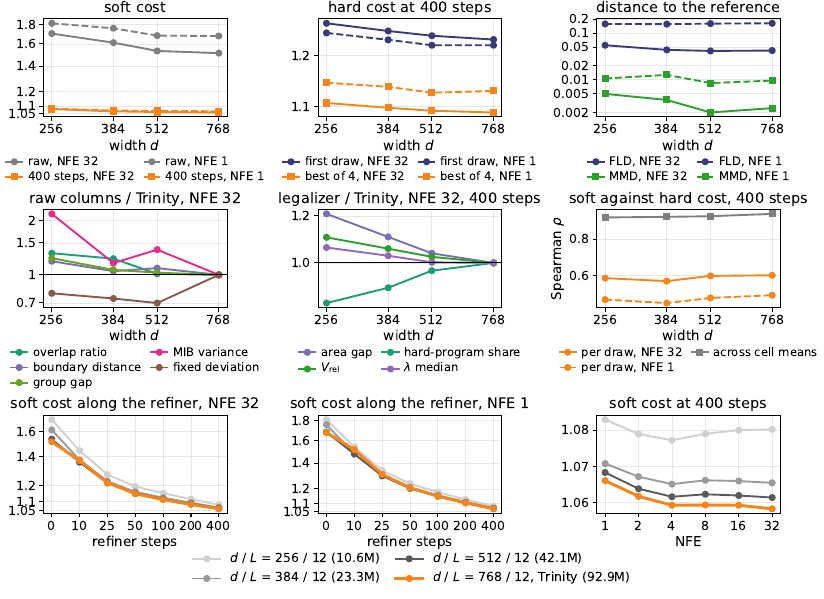}
\caption{The width rungs ($d$ / $L$ with $L = 12$) against the width $d$ and along the refiner. Top row: the soft cost raw and after 400 refiner steps at $\nfe{=}32$ and 1, the hard cost of the first draw and of the best of four at 400 steps, and the two distribution distances. Middle row: the raw constraint columns at $\nfe{=}32$ relative to Trinity, the legalizer's statistics at $\nfe{=}32$ with 400 steps relative to Trinity (area gap, $\Vrel$, share of draws closed by the hard program, median expansion $\lambda$), and the Spearman correlation between soft and hard cost per draw and across the cell means. Bottom row: the soft cost along the refiner at $\nfe{=}32$ and 1, and the 400-step endpoint against $\nfe$, one line per width.}
\label{fig:scaling}
\end{figure}

\begin{figure}[htbp]\centering
\includegraphics[width=\linewidth]{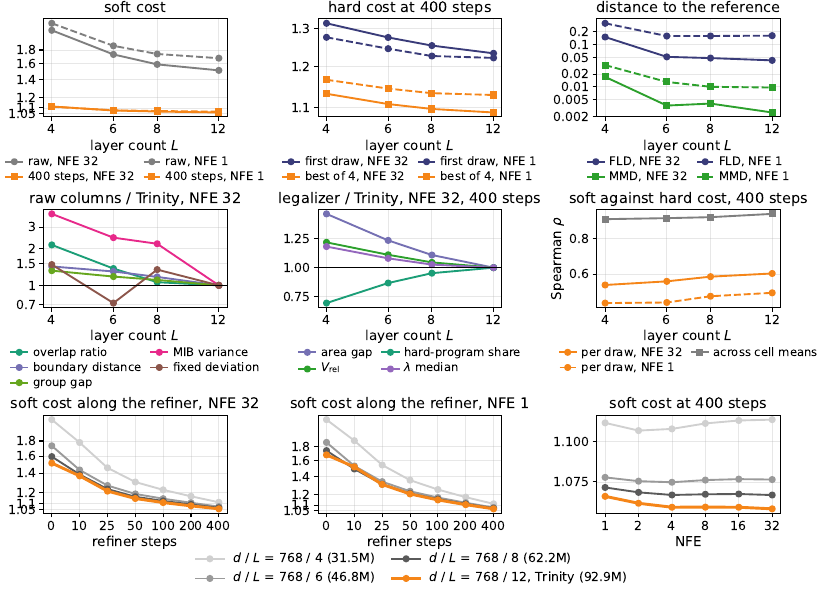}
\caption{The depth rungs ($d$ / $L$ with $d = 768$) against the layer count $L$ and along the refiner, the panels of Figure~\ref{fig:scaling}: the soft cost raw and after 400 refiner steps, the hard cost of the first draw and of the best of four at 400 steps, the two distribution distances, the raw constraint columns and the legalizer's statistics relative to Trinity, the soft--hard Spearman correlation, and the soft cost along the refiner and against $\nfe$, one line per depth.}
\label{fig:scaling_depth}
\end{figure}

\paragraph{Size rungs, after the refiner.}
The closed-form refiner narrows the size gaps by an order of magnitude (Table~\ref{tab:size_refined}). The raw width gaps to Trinity at $\nfe{=}32$ of $+0.020$, $+0.099$ and $+0.190$ become $+0.003$, $+0.007$ and $+0.022$ at 400 steps, the raw depth gaps of $+0.08$, $+0.22$ and $+0.61$ become $+0.009$, $+0.018$ and $+0.056$, and every rung's $\nfe{=}1$ endpoint (1.068 to 1.112) is at or below Trinity's $\nfe{=}32$ cell at 200 steps (1.083) except the 4-layer rung. The refined order across both axes at $\nfe{=}32$ with 400 steps is 768 / 12 (1.058), 512 / 12 (1.061), 384 / 12 (1.065), 768 / 8 (1.067), 768 / 6 (1.077), 256 / 12 (1.080) and 768 / 4 (1.114); the steps to a soft cost of 1.1 at $\nfe{=}32$ are 200 for widths 768, 512 and 384 and for 8 layers, 400 for width 256 and 6 layers, and the 4-layer rung never reaches 1.1 inside the grid, with an endpoint row that is not monotone in $\nfe$ (its minimum is at $\nfe{=}2$). The configs of the six rungs and of their grids are in the supplementary material.

\begin{table}[htbp]\centering\scriptsize
\caption{The size rungs ($d$ / $L$) along the closed-form refiner, one column per $\nfe$ / refiner steps, two lines per rung: soft cost / overlap ratio, then wirelength gap / area gap to the reference in percent (positive above the reference); \colorbox{ourscolor}{orange} marks Trinity.}
\label{tab:size_refined}
\setlength{\tabcolsep}{3pt}
\begin{tabular}{lrrrrrrr}
\toprule
$d$ / $L$ & 32 / 0 & 32 / 25 & 32 / 100 & 32 / 400 & 1 / 0 & 1 / 100 & 1 / 400 \\
\midrule
256 / 12 & 1.703\,/\,0.0297 & 1.270\,/\,0.0123 & 1.150\,/\,0.0072 & 1.080\,/\,0.0033 & 1.812\,/\,0.0365 & 1.177\,/\,0.0101 & 1.083\,/\,0.0040 \\
 & $+0.1$\,/\,$+0.3$ & $+0.4$\,/\,$+1.5$ & $+0.6$\,/\,$+1.8$ & $+0.7$\,/\,$+1.8$ & $-1.2$\,/\,$+0.0$ & $-0.4$\,/\,$+1.5$ & $-0.1$\,/\,$+1.7$ \\
\addlinespace[1.5pt]
384 / 12 & 1.612\,/\,0.0277 & 1.226\,/\,0.0103 & 1.123\,/\,0.0060 & 1.065\,/\,0.0028 & 1.759\,/\,0.0356 & 1.155\,/\,0.0094 & 1.071\,/\,0.0038 \\
 & $-0.2$\,/\,$-0.4$ & $+0.2$\,/\,$+0.9$ & $+0.4$\,/\,$+1.2$ & $+0.4$\,/\,$+1.2$ & $-1.5$\,/\,$-0.8$ & $-0.6$\,/\,$+0.9$ & $-0.3$\,/\,$+1.2$ \\
\addlinespace[1.5pt]
512 / 12 & 1.534\,/\,0.0228 & 1.215\,/\,0.0088 & 1.115\,/\,0.0051 & 1.061\,/\,0.0024 & 1.682\,/\,0.0312 & 1.149\,/\,0.0084 & 1.068\,/\,0.0034 \\
 & $-0.1$\,/\,$+0.0$ & $+0.3$\,/\,$+1.2$ & $+0.4$\,/\,$+1.3$ & $+0.4$\,/\,$+1.1$ & $-1.3$\,/\,$-0.3$ & $-0.5$\,/\,$+1.1$ & $-0.3$\,/\,$+1.2$ \\
\addlinespace[1.5pt]
768 / 4 & 2.128\,/\,0.0483 & 1.459\,/\,0.0237 & 1.228\,/\,0.0135 & 1.114\,/\,0.0055 & 2.258\,/\,0.0542 & 1.254\,/\,0.0162 & 1.112\,/\,0.0060 \\
 & $-0.5$\,/\,$-0.9$ & $+0.1$\,/\,$+0.8$ & $+0.6$\,/\,$+1.5$ & $+0.9$\,/\,$+2.2$ & $-1.9$\,/\,$-1.3$ & $-0.6$\,/\,$+1.1$ & $-0.1$\,/\,$+1.9$ \\
\addlinespace[1.5pt]
768 / 6 & 1.733\,/\,0.0309 & 1.270\,/\,0.0129 & 1.146\,/\,0.0076 & 1.077\,/\,0.0034 & 1.865\,/\,0.0380 & 1.172\,/\,0.0106 & 1.078\,/\,0.0042 \\
 & $-0.2$\,/\,$+0.0$ & $+0.2$\,/\,$+1.2$ & $+0.4$\,/\,$+1.5$ & $+0.5$\,/\,$+1.6$ & $-1.5$\,/\,$-0.3$ & $-0.6$\,/\,$+1.2$ & $-0.3$\,/\,$+1.5$ \\
\addlinespace[1.5pt]
768 / 8 & 1.593\,/\,0.0239 & 1.235\,/\,0.0098 & 1.126\,/\,0.0058 & 1.067\,/\,0.0027 & 1.740\,/\,0.0315 & 1.157\,/\,0.0089 & 1.072\,/\,0.0036 \\
 & $-0.1$\,/\,$+0.5$ & $+0.3$\,/\,$+1.5$ & $+0.5$\,/\,$+1.6$ & $+0.4$\,/\,$+1.4$ & $-1.3$\,/\,$+0.2$ & $-0.5$\,/\,$+1.4$ & $-0.3$\,/\,$+1.4$ \\
\addlinespace[1.5pt]
\rowcolor{ourscolor} \textbf{768 / 12} & 1.514\,/\,0.0225 & 1.216\,/\,0.0084 & 1.111\,/\,0.0048 & 1.058\,/\,0.0023 & 1.679\,/\,0.0316 & 1.149\,/\,0.0083 & 1.066\,/\,0.0034 \\
\rowcolor{ourscolor}  & $-0.1$\,/\,$+0.2$ & $+0.3$\,/\,$+1.4$ & $+0.4$\,/\,$+1.4$ & $+0.3$\,/\,$+1.1$ & $-1.2$\,/\,$-0.1$ & $-0.4$\,/\,$+1.2$ & $-0.3$\,/\,$+1.2$ \\
\bottomrule
\end{tabular}
\end{table}

\paragraph{Size rungs, through the legalizer.}
Table~\ref{tab:size_hard}, the top and middle rows of Figures~\ref{fig:scaling} and \ref{fig:scaling_depth}, and Figures~\ref{fig:scaling_hard} and \ref{fig:scaling_hard_depth} pass the six rungs through the pipeline of Appendix~\ref{sec:appendix-hard}: their own four-draw caches, the closed-form refiner and the shared legalizer on the full $\nfe \times$ steps grid, 2,016,000 legalized draws per rung, every one of the 12,096,000 feasible. The hard-cost order across both axes is the soft-cost order. At $\nfe{=}32$ with 400 steps the first draw runs 1.234 (768 / 12), 1.242 (512 / 12), 1.252 (384 / 12), 1.255 (768 / 8), 1.269 (256 / 12), 1.276 (768 / 6) and 1.315 (768 / 4), and the best of four 1.089, 1.092, 1.097, 1.097, 1.107, 1.108 and 1.133; the depth rungs pair with width rungs of far fewer parameters, 768 / 8 at 62.2M with 384 / 12 at 23.3M and 768 / 6 at 46.8M with 256 / 12 at 10.6M, within 0.007 at every key cell. After the refiner and the legalizer the gaps are of the order of the refined soft gaps, not of the raw ones: the family spans 0.08 on the first draw and 0.04 on the best of four against 0.61 on the raw soft cost, every rung except 768 / 4 is within 0.035 and 0.020 of Trinity, and 768 / 4 at 1.315 and 1.133 lands among the prior arts (FlowPlace 1.296 and 1.124 at the same cell). At the cheap corner, $\nfe{=}1$ with 400 steps, width 512 equals Trinity (first draw 1.222 against 1.222, best of four 1.126 against 1.130); the best-of-four minimum is at $\nfe{=}8$ for 768 / 12, 512 / 12, 384 / 12, 768 / 8 and 768 / 4 and at $\nfe{=}4$ for 256 / 12 and 768 / 6. The raw hard cost barely sees the size, 1.49 to 1.59 on every rung at every $\nfe$, and Trinity's $\nfe$ inversion holds on every rung: the first draw at 400 steps rises with $\nfe$, so the low-$\nfe$ columns are the hard-cost Pareto set. The legalizer statistics move with the soft cost, the hard-program share from 0.245 to 0.170, the median $\lambda$ from 1.32 to 1.55 and the area gap from $+0.105$ to $+0.153$ from Trinity to 768 / 4, and the soft--hard correlation keeps Trinity's shape and carrying columns ($\mathrm{sev}_{\mathrm{rel}}$, overlap and boundary) on every rung with a size that falls slowly with the soft cost, per draw at $\nfe{=}32$ with 400 steps from 0.603 to 0.539 and across the 42 cells from 0.938 to 0.908.

\begin{figure}[htbp]\centering
\includegraphics[width=\linewidth]{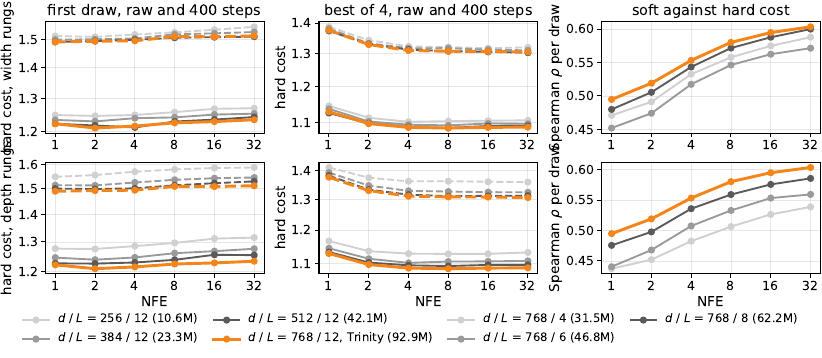}
\caption{The width rungs (top) and the depth rungs (bottom) through the legalizer against $\nfe$: the hard cost of the first draw and of the best of four, raw (dashed) and after 400 refiner steps (solid), and the per-draw Spearman correlation between soft and hard cost at 400 steps.}
\label{fig:scaling_hard}
\label{fig:scaling_hard_depth}
\end{figure}

\begin{table}[htbp]\centering\footnotesize
\caption{The size rungs ($d$ / $L$) through the legalizer at 400 refiner steps and $\nfe = 1$, 8 and 32: the soft cost of the refined layout, the hard cost of the first legalized draw as mean / median and its 90th percentile, the mean and the hard-cost total of the best of four, the legalized layout's area gap and $\Vrel$, and the per-draw Spearman correlation $\rho$ between soft and hard cost. Every draw of every rung is feasible; \colorbox{ourscolor}{orange} marks Trinity.}
\label{tab:size_hard}
\setlength{\tabcolsep}{5pt}
\scriptsize
\begin{tabular}{lrrrrrrrrrr}
\toprule
$d$ / $L$ & $\nfe$ & Soft cost & First draw & p90 & Best of 4 & Total & Area gap & $\Vrel$ & $\rho$ & HPWL gap \\
\midrule
\rowcolor{ourscolor} \textbf{768 / 12} & 1 & 1.0661 & 1.2222 / 1.0682 & 1.486 & 1.1298 & 1.2337 & $+$0.105 & 0.041 & 0.495 & $+0.087$ \\
512 / 12 & 1 & 1.0683 & 1.2221 / 1.0701 & 1.485 & 1.1263 & 1.2223 & $+$0.104 & 0.042 & 0.480 & $+0.085$ \\
384 / 12 & 1 & 1.0708 & 1.2336 / 1.0738 & 1.514 & 1.1378 & 1.2494 & $+$0.110 & 0.043 & 0.452 & $+0.092$ \\
256 / 12 & 1 & 1.0829 & 1.2483 / 1.0807 & 1.567 & 1.1459 & 1.2632 & $+$0.119 & 0.045 & 0.471 & $+0.099$ \\
768 / 8 & 1 & 1.0716 & 1.2271 / 1.0709 & 1.505 & 1.1341 & 1.2352 & $+$0.109 & 0.042 & 0.476 & $+0.090$ \\
768 / 6 & 1 & 1.0778 & 1.2463 / 1.0780 & 1.572 & 1.1452 & 1.2639 & $+$0.119 & 0.044 & 0.441 & $+0.099$ \\
768 / 4 & 1 & 1.1118 & 1.2767 / 1.0949 & 1.716 & 1.1668 & 1.3018 & $+$0.139 & 0.048 & 0.437 & $+0.117$ \\
\midrule
\rowcolor{ourscolor} \textbf{768 / 12} & 8 & 1.0593 & 1.2249 / 1.0685 & 1.514 & 1.0861 & 1.1473 & $+$0.101 & 0.042 & 0.580 & $+0.086$ \\
512 / 12 & 8 & 1.0623 & 1.2288 / 1.0733 & 1.516 & 1.0863 & 1.1490 & $+$0.104 & 0.044 & 0.572 & $+0.088$ \\
384 / 12 & 8 & 1.0661 & 1.2412 / 1.0782 & 1.564 & 1.0932 & 1.1718 & $+$0.112 & 0.045 & 0.546 & $+0.096$ \\
256 / 12 & 8 & 1.0790 & 1.2572 / 1.0878 & 1.615 & 1.1044 & 1.1856 & $+$0.123 & 0.047 & 0.558 & $+0.104$ \\
768 / 8 & 8 & 1.0673 & 1.2388 / 1.0758 & 1.562 & 1.0937 & 1.1655 & $+$0.111 & 0.044 & 0.559 & $+0.095$ \\
768 / 6 & 8 & 1.0762 & 1.2611 / 1.0846 & 1.672 & 1.1064 & 1.1991 & $+$0.125 & 0.047 & 0.533 & $+0.107$ \\
768 / 4 & 8 & 1.1114 & 1.2965 / 1.1066 & 1.802 & 1.1285 & 1.2334 & $+$0.148 & 0.052 & 0.506 & $+0.126$ \\
\midrule
\rowcolor{ourscolor} \textbf{768 / 12} & 32 & 1.0583 & 1.2341 / 1.0734 & 1.551 & 1.0887 & 1.1580 & $+$0.105 & 0.045 & 0.603 & $+0.092$ \\
512 / 12 & 32 & 1.0614 & 1.2422 / 1.0777 & 1.586 & 1.0919 & 1.1624 & $+$0.109 & 0.046 & 0.600 & $+0.096$ \\
384 / 12 & 32 & 1.0654 & 1.2524 / 1.0840 & 1.612 & 1.0975 & 1.1778 & $+$0.116 & 0.047 & 0.572 & $+0.103$ \\
256 / 12 & 32 & 1.0802 & 1.2691 / 1.0952 & 1.670 & 1.1067 & 1.1904 & $+$0.127 & 0.049 & 0.588 & $+0.112$ \\
768 / 8 & 32 & 1.0669 & 1.2545 / 1.0819 & 1.624 & 1.0968 & 1.1796 & $+$0.116 & 0.047 & 0.586 & $+0.103$ \\
768 / 6 & 32 & 1.0765 & 1.2763 / 1.0947 & 1.722 & 1.1083 & 1.1968 & $+$0.129 & 0.049 & 0.559 & $+0.114$ \\
768 / 4 & 32 & 1.1138 & 1.3152 / 1.1160 & 1.914 & 1.1327 & 1.2391 & $+$0.153 & 0.054 & 0.539 & $+0.134$ \\
\bottomrule
\end{tabular}
\end{table}

\section{Limitations and Future Work}
\label{sec:appendix-limitations}

\paragraph{Scope.}
Every result is on FloorSet at 21 to 120 blocks with its synthetic constraint statistics, in the MIB-consistent variant of the data. The reference layouts leave 9.1\% of the boundary-coded blocks off their edge, which bounds what a reference-matching metric can say about that constraint; the soft metrics read the constraint annotations instead, and the hard cost counts the violation, so the two costs of Section~\ref{sec:method-problem} are the right instruments for it.

\paragraph{Prior arts.}
The four diffusion placers are their published networks and objectives ported into the FloorSet vocabulary and trained under our recipe; their post-hoc loops are ported at their published settings, and their own legalizers are outside the study. This is the comparison that isolates the generator and the refiner, and it measures what a position-only design leaves on the table on a floorplanning task. Reproducing their published numbers on their own benchmarks was not part of this work.

\paragraph{Model size.}
Trinity has 92.9M parameters against 5 to 24M for the ported networks. Sampling is 0.02 to 0.32 seconds of the 0.5 to 1.9 seconds per case, so the size does not appear in the time budget, and FlowPlace's own objective on our 85.8M backbone lands at a soft cost of 2.60 against 2.46 for its 6.3M network, which places the gap in the objective and the physics rather than in the parameter count. The size study of Appendix~\ref{sec:appendix-size} trains Trinity recipe at widths $d = 512$, 384, 256 and 192 with $L = 12$ layers (42.1M, 23.3M, 10.6M and 5.9M parameters, the sizes of the ported networks), at $L = 8$, 6 and 4 layers with $d = 768$ (62.2M, 46.8M and 31.5M) and at width 512 with 16 layers (56.0M) under one learning rate: width 512 reproduces Trinity within its seed spread and four more layers on it add 0.004, the two smaller widths give back part of the overlap and soft cost on the raw sample while staying far ahead of the prior arts, and the refiner closes most of the remaining gap. Depth matters more than width: the 8-layer model at 62.2M parameters is within 0.02 of the 23.3M model of width 384 on the raw soft cost and behind it after the refiner and on both distribution distances, and the 4-layer model is the one rung whose layouts are farther from the reference distribution than the plain control's. The information mover, a one-hop graph-convolution residual, moves information along the netlist one hop per layer, so a shallow model cannot carry the netlist far enough to place its blocks.

\paragraph{Seeds.}
Three seeds cover Trinity, Trinity at $\lambda{=}0$, the plain control and FlowPlace; every other arm is a single run. A sweep-set delta under 2 to 3\% on a constraint column or 1\% on the soft cost is within one seed std of its control, and the claims of the paper are 17 to 70 std wide.

\paragraph{The legalizer.}
The legalizer is a fixed, minimal stage so that every comparison stays about the generator and the refiner. Its expansion step prices residual overlap in outline area, and after legalization Trinity still carries an area gap of 10\% and a $\Vrel$ of 0.04 that best-of-$N$ recovers; a legalizer designed for numerically generated layouts, with wirelength in its objective, is the natural next stage of the pipeline and is future work.

\paragraph{Future work.}
Beyond the legalizer: the soft cost as a training signal for the sampler itself, distillation of the one-step model to remove the refiner from the time budget, and interactive use, where a designer edits one constraint and expects a local change; the overrides make that possible and it has not been evaluated.

\section{Broader Impact}
\label{sec:appendix-impact}

\method\ is a physical-design study. Its positive impact is methodological: a way to write the constraints of a layout problem once and reuse them across training, refinement and evaluation, a generator that reaches a usable layout in one sampler step, and a three-stage evaluation that turns a discrete final score into a continuous development signal. Each of these transfers to other layout problems with hard rules and continuous objectives. The pipeline introduces no data set of proprietary designs: it trains on FloorSet's synthetic layouts and evaluates on the public FloorSet validation set. A learned placer is used behind a verifying legalizer and a design-rule check, as it is here, and we see no misuse risk specific to this work beyond those shared by any design-automation tool.

\section{Data, Licenses, and Supplementary Material}
\label{sec:appendix-data}

\paragraph{Data.}
FloorSet \citep{DBLP:conf/iccad/MallappaM0PM24} is released by Intel Labs and distributed through HuggingFace under the Creative Commons Attribution 4.0 license (CC BY 4.0).\footnote{\url{https://huggingface.co/datasets/IntelLabs/FloorSet}} We use the Lite training set and the 100-case validation set, transcoded into a columnar store without modification of content. In the raw Lite training set 99.25\% of MIB groups have members with different $(w, h)$, which makes every MIB loss and metric ill-defined; the variant used throughout sets each such group to one shape, and every model of the paper, the prior arts included, trains and is scored on it. The 100 validation cases carry consistent MIB annotations as released,\footnote{\url{https://github.com/IntelLabs/FloorSet/issues/12}} so they are used unchanged. No proprietary design data is used at any stage.

\paragraph{Code and models.}
The prior arts are ported from their public repositories into one codebase; their weights are not reused. The annealer follows BBOPlace-Bench at its published constants \citep{xue2026bboplacebenchbenchmarkingblackboxoptimization}.

\paragraph{Supplementary material.}
The supplementary material is the code needed to reproduce training, sampling, refinement, legalization and scoring: the solver-agnostic domain library (data, scoring, legalization), the generator library (framings, the six terms, the backbone, the sampler, the refiner), the ported prior arts, the configuration file of every trained arm, refiner sweep and legalizer run, and the scripts that print every table of this paper from what a run writes. It carries no weights, no data and no result files, and a run reproduces the numbers from FloorSet alone.

\end{document}